\documentclass{article}

\RequirePackage{etex}

\usepackage[nonatbib, final]{neurips_2020}
\usepackage{caption}
\usepackage{subcaption}
\usepackage[utf8]{inputenc} 
\usepackage[T1]{fontenc}    
\usepackage{hyperref}       
\usepackage{url}            
\usepackage{booktabs}       
\usepackage{amsfonts}       
\usepackage{nicefrac}       
\usepackage{microtype}      
\usepackage{amsmath}
\usepackage{graphicx}
\usepackage[ruled,vlined,linesnumbered]{algorithm2e}
\usepackage{mathtools}
\usepackage{mleftright}
\usepackage{amssymb}
\usepackage{footnote}
\makesavenoteenv{tabular}
\usepackage{tabularx,lipsum,environ,amsmath,amssymb}
\usepackage{multirow}
\usepackage{wrapfig}
\usepackage{thmtools,thm-restate}
\usepackage{amsthm}
\usepackage{todonotes}
\usepackage[normalem]{ulem}
\usepackage{xcolor}  
\allowdisplaybreaks  

\SetKwRepeat{Do}{do}{while}
\SetKwProg{Fn}{Function}{is}{end}

\newcommand{\upcap}{C_{\textrm{UP}}}
\newcommand{\dncap}{C_{\textrm{DN}}}

\theoremstyle{plain}
\newtheorem{thm}{Theorem}[section]
\newtheorem{lem}[thm]{Lemma}
\newtheorem{prop}[thm]{Proposition}

\newtheorem{observation}[thm]{\textbf{Observation}}
\makeatletter
\newtheorem*{rep@prop}{\rep@title}
\newcommand{\newrepprop}[2]{%
\newenvironment{rep#1}[1]{%
 \def\rep@title{#2 \ref{##1}}%
 \begin{rep@prop}}%
 {\end{rep@prop}}}
\makeatother

\newrepprop{prop}{Proposition}

\newcommand{\TSP}{\textsc{Tsp}}

\newcommand{\MCT}{\textsc{Mct}}
\newcommand{\MCTU}{\textsc{Mct-U}}

\newcommand{\MCTD}{\textsc{Mct-Decision}}

\newcommand{\HC}{\textsc{Hc}}

\newcommand{\DCST}{\textsc{Dcst}}
\newcommand{\MBST}{\textsc{Mbst}}

\theoremstyle{definition}

\theoremstyle{remark}

\title{Throughput-Optimal Topology Design \\for Cross-Silo Federated Learning}

\author{
    Othmane Marfoq\\
    Inria, Universit\'e C\^ote d'Azur,\\
    Accenture Labs,\\
    Sophia Antipolis, France\\
    othmane.marfoq@inria.fr\\
    \And
    Chuan Xu\\
    Inria, Universit\'e C\^ote d'Azur, \\
    Sophia Antipolis, France\\
    chuan.xu@inria.fr\\
    \AND
    Giovanni Neglia\\
    Inria, Universit\'e C\^ote d'Azur, \\
    Sophia Antipolis, France\\
    giovanni.neglia@inria.fr\\
    \And
    Richard Vidal\\
    Accenture Labs, \\ 
    Sophia Antipolis, France\\
    richard.vidal@accenture.com\\
}

\makeatletter
\newcommand{\problemtitle}[1]{\gdef\@problemtitle{#1}}
\newcommand{\probleminput}[1]{\gdef\@probleminput{#1}}
\newcommand{\problemquestion}[1]{\gdef\@problemquestion{#1}}
\NewEnviron{problem}{
  \problemtitle{}\probleminput{}\problemquestion{}
  \BODY
  \par\addvspace{.5\baselineskip}
  \noindent
  \begin{tabularx}{\textwidth}{@{\hspace{\parindent}} l X c}
    \multicolumn{2}{@{\hspace{\parindent}}l}{\@problemtitle} \\
    \textbf{Input:} & \@probleminput \\
    \textbf{Output:} & \@problemquestion
  \end{tabularx}
  \par\addvspace{.5\baselineskip}
}

\usepackage{biblatex}

\usepackage{amsmath,amsfonts,bm}

\def\1{\bm{1}}

\def\vw{{\bm{w}}}

\def\mA{{\bm{A}}}

\DeclareMathAlphabet{\mathsfit}{\encodingdefault}{\sfdefault}{m}{sl}
\SetMathAlphabet{\mathsfit}{bold}{\encodingdefault}{\sfdefault}{bx}{n}

\DeclareMathOperator*{\argmax}{arg\,max}

\begin{document}

\maketitle

\begin{abstract}

Federated learning usually employs a {server-client} architecture where an orchestrator iteratively aggregates model updates from remote clients and pushes them back a refined model. This approach may be inefficient in cross-silo settings, as close-by data silos with high-speed access links may exchange information faster than with the orchestrator, and the orchestrator may become a communication bottleneck. In this paper we define the problem of topology design for cross-silo federated learning using the theory of max-plus linear systems to compute the system throughput---number of communication rounds per time unit. We also propose practical algorithms that, under the knowledge of measurable network characteristics, find a topology with the largest throughput or with provable throughput guarantees. In realistic Internet networks with 10~Gbps access links at silos, our algorithms speed up training by a factor 9 and 1.5 in comparison to the {server-client}
architecture and to state-of-the-art MATCHA, respectively. Speedups are even larger with slower access links.

\end{abstract}

\section{Introduction}
\label{sec:introduction}
Federated learning (FL) ``\textit{involves training statistical models over remote devices or siloed data centers, such as mobile phones or hospitals, while keeping data localized}''~\cite{li2019federated} because of privacy concerns or limited communication resources. The definition implicitly distinguishes two different settings~\cite{kairouz2019advances}: the \emph{cross-device} scenario including a large number (millions or even more) of unreliable mobile/edge devices with limited computing capabilities and slow Internet connections, and the \emph{cross-silo} scenario with at most a few hundreds of reliable data silos with powerful computing resources and high-speed access links. While the first FL papers~\cite{mcmahan2016communicationefficient,konen15} emphasized the cross-device setting, the cross-silo scenario has become popular for distributed training among banks~\cite{webank}, hospitals~\cite{courtiol2019deep, silva2019federated, nvidia}, pharmaceutical labs~\cite{eucordis}, and manufacturers~\cite{mustketeer}. 

In federated learning, clients (e.g., mobile devices or whole organizations) usually train the model through an iterative procedure under the supervision of a central orchestrator, which, for example, decides to launch the training process and coordinates training advances. Often---e.g., in FedAvg~\cite{mcmahan2016communicationefficient}, SCAFFOLD~\cite{karimireddy2019scaffold}, and FedProx \cite{li2018federated}---the orchestrator directly participates to the training, by aggregating clients' updates, generating a new model, and pushing it back to the clients. Hence, clients only communicate with a potentially far-away (e.g., in another continent) orchestrator and do not exploit communication opportunities with close-by clients. This choice is justified in the cross-device setting, where inter-device communication is unreliable (devices may drop-out from training at any time) and slow (a message needs to traverse two slow access links). But in the cross-silo setting, data silos (e.g., data centers) are almost always available, enjoy high-speed connectivity comparable to the orchestrator's one, and may exchange information faster with some other silos than with the orchestrator. An orchestrator-centered communication topology is then potentially inefficient, because it ignores fast inter-silo communication opportunities and makes the orchestrator a candidate for congestion. A current trend~\cite{wang2019matcha, 10.5555/3045390.3045537, Vanhaesebrouck2017DecentralizedCL, pmlr-v80-tang18a, Bellet2018PersonalizedAP, pmlr-v97-koloskova19a, DBLP:journals/corr/abs-1901-11173} is then to replace communication with the orchestrator by peer-to-peer communications between individual silos, which perform local partial aggregations of model updates. We also consider this scenario and study how to design the communication topology.

The communication topology has two contrasting effects on training duration. First, a more connected topology leads to faster convergence in terms of iterations or communication rounds, as quantified by 
{classic worst-case}
convergence bounds in terms of the spectral properties of the topology~\cite{ndeic_834019_Tradeoffs,Duchi_2012, scaman2018optimal, scaman2017optimal, CooperatedSGDwang, jiang2017collaborative}. Second, a more connected topology increases the duration of a communication round (e.g., it may cause network congestion), motivating the use of degree-bounded topologies where every client sends and receives a small number of messages at each  round~\cite{DBLP:journals/corr/abs-1811-10792, NIPS2017_7117}. 
Recent experimental and theoretical work suggests that, in practice, \emph{the first effect has been over-estimated by classic worst-case convergence bounds}. Reference \cite{neglia:hal-02430485} partially explains the phenomenon and overviews theoretical results proving asymptotic topology-independence \cite{NIPS2017_7117, olshevsky2019asymptotic, DBLP:journals/corr/abs-1811-10792}.  \cite[Sect. 6.3]{koloskova2020} extends some of the conclusions in \cite{neglia:hal-02430485} to dynamic topologies and multiple local updates. Experimental evidence on image classification tasks (\cite[Fig. 2]{neglia:hal-02430485}, \cite[Fig 20.]{Luo_2019}, \cite[Fig. 3]{NIPS2017_7117}) and natural language processing tasks (\cite[Figs.~13-16]{NIPS2017_7117}) confirms this finding. Motivated by these observations, this paper focuses on the effect of topology on the duration of communication rounds.



Only a few studies have designed topologies taking into account the duration of a communication round. Under the simplistic assumption that the communication time is proportional to node degree, MATCHA~\cite{wang2019matcha} decomposes the set of possible communications into matchings (disjoint pairs of clients) and, at each communication round, randomly selects some matchings and allows their pairs to transmit. MATCHA chooses the matchings' selection probabilities in order to optimize the algebraic connectivity of the expected topology. Reference~\cite{neglia_infocom} studies how to select the degree of a regular topology when the duration of a communication round is determined by stragglers~\cite{karakus17,li18}. Apart from these corner cases,  ``\emph{how to design a [decentralized] model averaging policy that achieves the fastest convergence remains an open problem}''~\cite{kairouz2019advances}.

Our paper addresses this open problem. It uses the theory of linear systems in the max-plus algebra~\cite{baccelli} to design cross-silo FL topologies that minimize the duration of communication rounds, or equivalently maximize the system \emph{throughput}, i.e., the number of completed rounds per time unit. The theory holds for synchronous systems and has been successfully applied in other fields (e.g.,~manufacturing~\cite{chandra2003automated}, communication networks~\cite{leboudec01}, biology~\cite{brunsch12}, railway systems~\cite{goverde98}, and road networks~\cite{farhi11}). Synchronous optimization algorithms are often preferred for federated learning~\cite{Bonawitz19}, because they enjoy stronger convergence guarantees than their asynchronous counterparts and can be easily combined with cryptographic secure aggregation protocols~\cite{bonawitz2017practical}, differential privacy techniques~\cite{abadi2016deep}, and model and update compression~\cite{pmlr-v70-zhang17e, wang2018atomo, sattler2019robust, caldas2018expanding}.

To the best of our knowledge, this paper is the first work to take explicitly in consideration all delay components contributing to the total training time including computation times, link latencies, transmission times, and queueing delays. It complements the topology design approaches  listed above that only account for congestion at access links~\cite{wang2019matcha} and straggler effect~\cite{neglia_infocom}.  

The algorithms we propose (Sect.~\ref{sec:theoritcal_results}) are either optimal or enjoy guaranteed approximation factors. Numerical results in Sect.~\ref{sec:experiments} show significant training speed-up in realistic network settings; the slower the access links, the larger the speedups.

\section{Problem Formulation}
\label{sec:notation}
\subsection{Machine Learning Training}
We consider a network of $N$ siloed data centers  who collaboratively train a global machine learning model, solving the following optimization problem:
\begin{equation}
    \label{eq:optimization_problem_constraints}
    \underset{\vw \in \mathbb{R}^{d}}{\textrm{minimize}}   
    \sum_{i=1}^{N} p_i \mathbb{E}_{\xi_{i}}\!\left[ f_{i}(\vw, \xi_{i})\right],\;\;
\end{equation}
where  $f_{i}(\vw,\xi_{i})$ is the loss of model $\vw$ at a sample $\xi_{i}$ drawn from  data distribution at silo $i$  
and the coefficient $p_i>0$ specifies the relative importance of each silo,
with two natural settings being $p_i$ equal to $1$ or to the size of silo $i$'s local dataset~\cite{li2019federated}. In the rest of the paper we consider $p_i=1$, but our analysis is not affected by the choice of $p_i$.




In order to solve Problem~\eqref{eq:optimization_problem_constraints} in an FL scenario, silos do not share the local datasets, but periodically transmit model updates, and different distributed algorithms have been proposed \cite{li2018federated, mcmahan2016communicationefficient, Li2019CommunicationEfficientLD, karimireddy2019scaffold, wang2019matcha, DBLP:journals/corr/KonecnyMRR16,CooperatedSGDwang}. In this paper we consider as archetype the decentralized periodic averaging stochastic gradient descent (DPASGD)~\cite{CooperatedSGDwang}, where silos are represented as vertices of a communication graph that we call \emph{overlay}. Each silo $i$  maintains a local model $\vw_i$ and  performs $s$~mini-batch gradient updates before sending its model to a subset of silos $\mathcal N_i^-$ (its out-neighbours in the overlay). It then aggregates its model with those received by a (potentially different) set of silos $\mathcal N_i^+$ (its in-neighbours). Formally, the algorithm is described by the following equations:
\begin{align}
    \label{equ:gossip_updates}
    \vw_i(k+1) & = \begin{cases}
        \sum_{j \in \mathcal N^{+}_i \cup \{i\}} \mA_{i,j} \vw_j\!\left(k \right), & \mbox{ if }k \equiv 0 \pmod{s+1},\\
        \vw_i(k) - \alpha_k \frac{1}{m} \sum_{h=1}^m \nabla f_{i}\!\left(\vw_i\!\left(k \right),  \xi_i^{(h)}(k)\right), & \mbox{ otherwise}.
    \end{cases}
\end{align}
where $m$ is the batch size, $\alpha_{k} > 0$ is a potentially  varying learning rate, and $\mA \in \mathbb{R}^{N\times N}$ is a  matrix of non-negative weights, referred to as the \emph{consensus matrix}. For particular choices of the matrix $\mA$ and the number of local updates $s$, DPASGD reduces to other schemes previously proposed~\cite{NIPS2017_7117,Li2019CommunicationEfficientLD, yuan19}, including FedAvg~\cite{mcmahan2016communicationefficient}, where the orchestrator just performs the averaging step (this corresponds to its local loss function $f_i(.)$ being a constant). Convergence of~\eqref{equ:gossip_updates} was proved in \cite{CooperatedSGDwang}.

In this paper we study how to design the overlay in order to minimize the training time. While we consider DPASGD, our results are applicable to any synchronous iterative algorithm where each silo alternates a local computation phase and a communication phase during which it needs to receive inputs from a given subset of silos before moving to the next computation phase. This includes the distributed algorithms already cited, as well as push-sum training schemes~\cite{DBLP:journals/corr/abs-1811-10792, shi2014extra, DBLP:journals/oms/RamNV12, nedich2016achieving, 7472453, 6483403, Yin_2017} and in general the \emph{black-box optimization procedures} as defined in~\cite{scaman2017optimal}.

\subsection{Underlay, Connectivity graph, and Overlay} 
\label{s:underlay_and_co}
\begin{figure}
    \centering
    \begin{subfigure}[b]{0.3\textwidth}
        \centering 
        \includegraphics[scale=0.15]{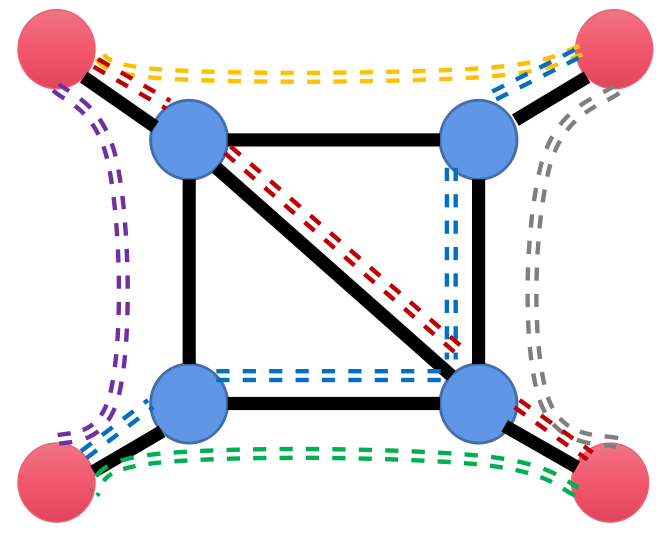}
        \caption{\small Underlay  $\mathcal{G}_{u} = (\mathcal{V} \cup \mathcal{V'}, \mathcal{E}_{u})$}
        \label{fig:underlay}
    \end{subfigure}
    \hfill
    \begin{subfigure}[b]{0.33\textwidth}
        \centering
        \includegraphics[scale=0.15]{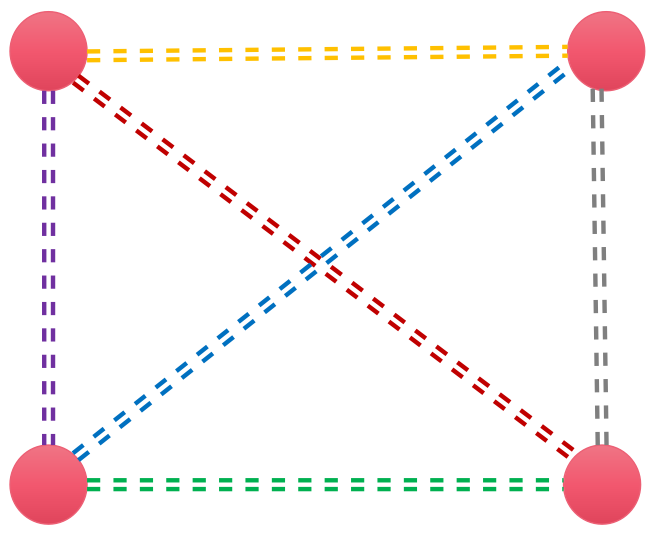}
        \caption{\small Connectivity graph $\mathcal G_c = (\mathcal V, \mathcal E_c)$}
        \label{fig:connectivity}
    \end{subfigure}
    \hfill
    \begin{subfigure}[b]{0.3\textwidth}
        \centering
        \includegraphics[scale=0.16]{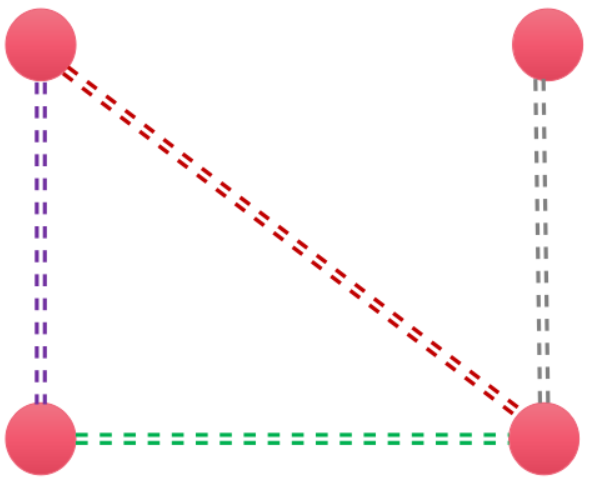}
        \caption{\small Overlay  $\mathcal G_o = (\mathcal V, \mathcal E_o)$}
        \label{fig:overlay}
    \end{subfigure}
    \caption{Examples for underlay, connectivity graph, and overlay, with routers (blue nodes), silos (red nodes), underlay links (solid black lines), and  information exchanges (dashed lines).}
    \label{fig:topologies}
\end{figure}

FL silos are connected by a communication infrastructure (e.g., the Internet or some private network), which we call  \emph{underlay}. The underlay can be represented as a directed graph (digraph) \mbox{$\mathcal{G}_{u} = (\mathcal{V} \cup \mathcal{V'}, \mathcal{E}_{u})$}, where $\mathcal V$ denotes the set of silos, $\mathcal V'$ the set of other nodes (e.g., routers) in the network, and $\mathcal E_u$ the set of communication links. For simplicity, we consider that each silo $i \in \mathcal V$ is connected to the rest of the network through a single link $(i,i')$, where $i' \in \mathcal V'$, with uplink capacity $\upcap(i)$ and downlink capacity $\dncap(i)$. The example in Fig.~\ref{fig:topologies} illustrates  the underlay and the other concepts we are going to define. 

The \emph{connectivity graph} $\mathcal G_c = (\mathcal V, \mathcal E_c)$ captures the possible direct communications among silos. Often the connectivity graph  is fully connected, but specific NAT or firewall configurations may prevent some pairs of silos to communicate. If $(i,j) \in \mathcal E_c$, $i$ can transmit its updated model to $j$. The message experiences a delay that is the sum of two contributions: 1)~an end-to-end delay $l(i,j)$ accounting for link latencies, and queueing delays long the path, and 2)~a term depending on the model size $M$ and the \emph{available bandwidth}\footnote{
    The available bandwidth of a path is the maximum rate that the path can provide to a flow, taking into account the rest of the traffic~\cite{carter96,jain02}; it is then smaller than the minimum link capacity of the path.
}
$A(i,j)$ of the path. Each pair of silos $(i,j)$ can use probing packets~\cite{jain02,prasad03,gaia_10.5555/3154630.3154682} to measure end-to-end delays and available bandwidths and communicate them to the orchestrator, which then designs the topology. We assume that in the stable cross-silo setting these quantities do not vary or vary slowly, so that the topology is recomputed only occasionally, if at all.

The training algorithm in~\eqref{equ:gossip_updates} does not need to use all potential connections. The orchestrator can  select a connected subgraph of $\mathcal G_c$. We call such subgraph \emph{overlay} and denote it by $\mathcal G_o = (\mathcal V, \mathcal E_o)$, where $\mathcal E_o \subset \mathcal E_c$.  Only nodes directly connected in $\mathcal G_o$ will exchange messages. We can associate a delay to each link $(i,j) \in \mathcal E_o$, corresponding to the time interval between the beginning of a local computation at node $i$, and the receiving of $i$'s updated model by $j$:
\begin{equation}
    \label{eq:do}
    d_o(i,j) =  s \times T_c(i) + l(i,j) + \frac{M}{A(i,j)} = s \times T_c(i) + l(i,j) + \frac{M}{\min\left(\frac{\upcap(i)}{|\mathcal N^-_i|},\frac{\dncap(j)}{|\mathcal N^+_j|}, A(i',j')\right)}, 
\end{equation}
where $T_c(i)$ denotes the time to compute one local update of the model. We also define $d_o(i,i) = s \times T_c(i)$. Equation~\eqref{eq:do} holds under the following assumptions. First, each silo $i$ uploads its model in parallel to its out-neighbours in $\mathcal N^-_i$ (with a rate at most $\upcap(i)/|\mathcal N^-_i|$). Second, downloads at $j$ happen in parallel too. While messages from different in-neighbours may not arrive at the same time at $j$'s downlink, their transmissions are likely to partially overlap. Finally, different messages do not interfere significantly in the core network, where they are only a minor component of the total network traffic ($A(i',j')$ does not depend on $\mathcal G_o$).

{Borrowing the terminology from P2P networks~\cite{massoulie07} we call a network \emph{edge-capacitated} if access links delays can be neglected,
otherwise we say that it is \emph{node-capacitated}}. While in cross-device FL the network is definitely node-capacitated, in cross-silo FL---the focus of our work---silos may be geo-distributed data centers or branches of a company and then have high-speed connections, so that neglecting access link delays may be an acceptable approximation.

Our model is more general than those considered in related work:  \cite{wang2019matcha} considers $d_o(i,j) = M \times |\mathcal N^-_i| / \upcap(i)$ and \cite{neglia_infocom} considers $d_o(i,j) = T_c(i)$ (but it accounts for random computation times).

\subsection{Time per Communication Round (Cycle Time)}
Let $t_i(k)$ denote the time at which worker $i$ starts computing $w_i((s+1)k+1)$ according to~\eqref{equ:gossip_updates} with $t_i(0)=0$. As $i$ needs to wait for the inputs $w_j((s+1)k)$ from its in-neighbours, the following recurrence relation holds
\begin{equation}
    t_i(k+1)= \max_{j \in \mathcal N^+_i \cup \{i\}} (t_j(k) + d_o(j,i)).
\end{equation}
This set of relations generalizes the concept of a linear system in the max-plus algebra, where the $\max$ operator replaces the usual sum and the $+$ operator replaces the usual product. We refer the reader to~\cite{baccelli} for the general theory of such systems and we present here only the key results for our analysis.

We call the time interval between $t_i(k)$ and $t_i(k+1)$ a \emph{cycle}. The average cycle time for silo $i$ is defined as $\tau_i = \lim_{k \to \infty} t_i(k)/k$. The cycle time 1)  does not depend on the specific silo (i.e., $\tau_i = \tau_j$)~\cite[Sect.~7.3.4]{baccelli}, and 2) can be computed directly from the graph $\mathcal G_o$~\cite[Thm.~3.23]{baccelli}. In fact:
\begin{equation}
    \label{eq:cycle_time}
     \tau(\mathcal{G}_o) = \max_{\gamma}\frac{d_o(\gamma)}{|\gamma|},
\end{equation}
where $\gamma$ is a generic circuit, i.e., a path $(i_{1}, \dots, i_{p}=i_{1})$ where the initial node and the final node coincide,  $|\gamma|=p$ is the length of the circuit, and  $d_o(\gamma)=\sum_{k=1}^{p-1}d_o(i_{k}, i_{k+1})$ is the sum of delays on~$\gamma$. A circuit $\gamma$ of $\mathcal{G}_o$ is called \emph{critical} if $\tau(\mathcal{G}_o) = d_o(\gamma)/{|\gamma|}$. There exist algorithms with different complexity to compute the cycle time~\cite{karp78,dasdan98}.

The cycle time is a key performance metric for the system because the difference $|t_i(k) - \tau(\mathcal G_o) \times k|$ is bounded for all $k \ge 0$ so that, for large enough  $k$, $t_i(k) \approx \tau(\mathcal G_o) \times k$. In particular, the inverse of the cycle time is the \emph{throughput} of the system, i.e., the number of communication rounds per time unit. An overlay with minimal cycle time minimizes the time required for a given number of communication rounds. This observation leads to our optimization problem.

\subsection{Optimization Problem}
Given a connectivity graph $\mathcal{G}_{c}$, we want the overlay $\mathcal{G}_{o}$ to be a strong digraph (i.e., a strongly connected directed graph) with minimal cycle time.  Formally, we define the following \textit{Minimal Cycle Time} problem: 
\begin{problem}
  \label{main_problem}
  \problemtitle{Minimal Cycle Time ($\MCT$)}
  \probleminput{A strong digraph $\mathcal{G}_{c}\! =\! (\mathcal{V},\mathcal{E}_{c})$, $\{\upcap(i), \dncap(j), l(i,j), A(i',j'), T_c(i), \forall (i,j) \in \mathcal E_c\}$.} 
  \problemquestion{A strong  spanning subdigraph of $\mathcal{G}_{c}$ with minimal cycle time.}
    \end{problem}
Note that the input does not include detailed information about the underlay $\mathcal G_u$, but only information available or measurable at the silos (see Sect.~\ref{s:underlay_and_co}). To the best of our knowledge, our paper is the first effort to study $\MCT$. The closest problem considered in the literature is, for a given overlay,  to select the largest delays that guarantee a minimum throughput~\cite{Gaubert95, Hardouin2014}.

\section{Theoretical Results and Algorithms}
\label{sec:theoritcal_results}
\begin{table}
    \caption{Algorithms to design the overlay $\mathcal{G}_o$ from the connectivity graph $\mathcal{G}_c$.}
    \centering
    \resizebox{\columnwidth}{!}{%
        \begin{tabular}{l l l l l}
            \toprule
            Network & Conditions & Algorithm & Complexity & Guarantees \\
            \midrule
            Edge-capacitated  & Undirected  $\mathcal{G}_o$ & Prim's Algorithm \cite{doi:10.1002/j.1538-7305.1957.tb01515.x} & $\mathcal{O}(|\mathcal E_c| + |\mathcal V| \log |\mathcal V|)$ & Optimal solution (Prop.~\ref{prop:symmetric_underlay}) \\
            Edge/Node-capacitated & Euclidean $\mathcal{G}_c$ & Christofides' Algorithm \cite{monnot:hal-00003997} & $\mathcal{O}(|\mathcal V|^{2} \log |\mathcal V|)$ & 3$N$-approximation (Prop.~\ref{prop:mctz_3-approx},\ref{prop:mct_3_approx})\\
            Node-capacitated & \begin{tabular}{@{}l@{}}Euclidean  $\mathcal{G}_c$ \\  and undirected $\mathcal{G}_o$ \end{tabular} & Algorithm \ref{alg:our_algorithm} (Appendix~\ref{sec:appendix_alg}) & $\mathcal{O}(|\mathcal E_c| |\mathcal V| \log |\mathcal V|)$ & 6-approximation (Prop.~\ref{prop:6-approx}) \\
            \bottomrule
        \end{tabular}%
    }
    \label{tab:solutions} 
\end{table}

In this section we present complexity results for $\MCT$ and algorithms to design the optimal topology in different settings. Table~\ref{tab:solutions} lists these algorithms, their time-complexity, and their guarantees. We note that in some cases we adapt known algorithms to solve \MCT. All proofs and auxiliary lemmas are in Appendix~\ref{sec:appendix_proofs}.



\subsection{Edge-capacitated networks}
\label{sec:throughput_maximization}
Remember that we call a network edge-capacitated if access links delays can be neglected, as it is for example the case whenever $\frac{1}{N}\times\min \left(\upcap(i), \dncap(j) \right ) \geq A(i',j')$ for each $(i, j) \in \mathcal{E}_c$. 
In this setting \eqref{eq:do} becomes
\begin{equation}
    \label{eq:do_edge}
    d_o(i,j) = s \times T_c(i) + l(i,j) + \frac{M}{A(i',j')},
\end{equation}
and then the delay between two silos does not depend on the selected overlay $\mathcal G_o$.

FL algorithms often use an \emph{undirected} overlay with symmetric communications, i.e., $(i, j) \in \mathcal{E}_{o} \Rightarrow (j, i) \in \mathcal{E}_{o}$. This is the case of centralized schemes, like FedAvg, but is also common for other consensus-based optimization schemes where the consensus matrix $\mA$ is required to be doubly-stochastic \cite{nedic09,DBLP:journals/oms/RamNV12,CooperatedSGDwang}---a condition simpler to achieve when $\mathcal G_o$ is undirected.

When building an undirected overlay, we can restrict ourselves to consider trees as solutions of $\MCT$ {(Lemma~\ref{lem:spanning_tree})}. In fact, additional links can only increase the number of circuits and then increase the cycle time (see~\eqref{eq:cycle_time}). Moreover, we can prove that the overlay has simple critical circuits of the form $\gamma=(i,j,i)$, for which $d_o(\gamma)/|\gamma| = (d_o(i,j)+d_o(j,i))/2$ {(Lemma~\ref{lem:symmetric_underlay})}. Intuitively, if we progressively build a tree using the links in $\mathcal G_c$ with the smallest average of delays in the two directions, we obtain the overlay with minimal cycle time. This construction corresponds to finding a minimum weight spanning tree (MST) in an opportune undirected version of $\mathcal{G}_c$:
\begin{prop}
    \label{prop:symmetric_underlay}
    Consider an undirected weighted graph $\mathcal G_c^{(u)} = (\mathcal V, \mathcal E_c^{(u)})$, where $(i,j) \in \mathcal E_c^{(u)}$ iff  $(i,j) \in \mathcal E_c$ and $(j,i) \in \mathcal E_c$ and where $(i,j) \in \mathcal E_c^{(u)}$ has weight $d_c^{(u)}(i,j)= (d_o(i,j)+d_o(j,i))/2$. A~minimum weight spanning tree of $\mathcal{G}_{c}^{(u)}$ is a solution of $\MCT$ when $\mathcal G_c$ is edge-capacitated and $\mathcal G_o$ is required to be undirected.
\end{prop}
Prim's algorithm~\cite{doi:10.1002/j.1538-7305.1957.tb01515.x} is an efficient algorithm to find an MST with complexity $\mathcal O( |\mathcal E_c| + |\mathcal V| \log |\mathcal V|)$ and then suited for the usual cross-silo scenarios with at most a few hundred nodes \cite{kairouz2019advances}.

We have pointed out a simple algorithm when the overlay is undirected, but directed overlays can have arbitrarily shorter cycle times than undirected ones even in simple settings where all links in the underlay are bidirectional with identical delays in the two directions (see Appendix~\ref{sec:appendix_ring_beats_mts}). Unfortunately, computing optimal directed overlays is NP-hard:
\begin{prop}
	\label{prop:mctz_nphard}
	$\MCT$ is NP-hard even when $\mathcal G_c$ is a complete Euclidean edge-capacitated graph.
\end{prop}
We call a connectivity graph $\mathcal G_c$ \emph{Euclidean} if its delays  $d_c(i,j)\triangleq s \times T_{c}(i) + l(i,j) + M/A(i',j')$ are symmetric ($d_c(i,j) = d_c(j,i), \forall i,j \in \mathcal V$) and satisfy the triangle inequality ($d_c(i,j) \le d_c(i,k) + d_c(k,j), \forall i, j, k \in \mathcal V$). These assumptions are roughly satisfied for geographically distant computing clusters with similar computation times, as the delay to transmit a message between two silos is roughly an affine function of the geodesic distance between them~\cite{10.1145/1028788.1028828}. Under this condition $\MCT$ can be approximated:
\begin{prop}
    \label{prop:mctz_3-approx}
    Christofides' algorithm~\cite{monnot:hal-00003997} is a $3 N$-approximation algorithm for $\MCT$ when $\mathcal G_c$ is edge-capacitated and Euclidean.
\end{prop} 
The result follows from Christofides' algorithm being a 1.5-approximation algorithm for the Travelling Salesman Problem~\cite{monnot:hal-00003997}, and our proof shows that a solution of the Travelling Salesman Problem provides a $2N$-approximation of $\MCT$. Note that Christofides' algorithm finds \emph{ring} topologies.

\subsection{Node-capacitated networks}
\label{sec:congestion}
When silos do not enjoy high-speed connectivity, congestion at access links can become the dominant contribution to network delays, especially when one silo communicates with many others. Intuitively, in this setting, good overlays will exhibit small degrees. 

If $\mathcal G_o$ is required to be undirected, $\MCT$ can be reduced from the problem of finding the minimum bottleneck spanning tree with bounded degree $\delta>1$ ($\delta$-$\MBST$ for short),\footnote{
A $\delta$-MBST is a spanning tree with degree at most $\delta$ in which the largest edge delay is as small as possible.
} which is NP-hard.
\begin{prop}
    \label{prop:mctu_nphard}
    In node-capacitated networks $\MCT$ is NP-hard even when the overlay is required to be undirected.
\end{prop}

We propose Algorithm~\ref{alg:our_algorithm} (see Appendix~\ref{sec:appendix_alg}), which combines existing approximation algorithms for $\delta$-$\MBST$ on a particular graph built from $\mathcal G_c$.
\begin{prop}
    \label{prop:6-approx}
    Algorithm \ref{alg:our_algorithm} is a $6$-approximation algorithm for $\MCT$ when $\mathcal{G}_c$ is node-capacitated and Euclidean with  $\upcap(i)\leq \min\left(\frac{\dncap(j)}{N}, A(i',j') \right)$, $\forall (i,j)\in \mathcal{E}_c$, and $\mathcal G_o$ is required to be undirected.
\end{prop}

Finding directed overlays is obviously an NP-hard problem also for node-capacitated networks. Christofides' algorithm holds its approximation factor also in this more general case:
\begin{prop}
    \label{prop:mct_3_approx}
    Christofides' algorithm is a $3N$-approximation algorithm for $\MCT$ when $\mathcal{G}_c$ is node-capacitated and Euclidean.
\end{prop}

\section{Numerical Experiments}
\label{sec:experiments}
We adapted PyTorch with the MPI backend to run DPASGD (see~\eqref{equ:gossip_updates}) on a GPU cluster. We also developed a separate network simulator that takes as input an arbitrary underlay topology described in the Graph Modelling Language~\cite{himsolt1997gml} and silos' computation times and calculates the time instants at which local models $\vw_i(k)$ are computed according to~\eqref{equ:gossip_updates} (Appendix~\ref{sec:appendix_time_simulator}).
While PyTorch trains the model as fast as the cluster permits, the network simulator reconstructs the real timeline on the considered underlay. {The code is available at \url{https://github.com/omarfoq/communication-in-cross-silo-fl}.}

We considered three real topologies from \emph{Rocketfuel engine}~\cite{rocketfuel_10.1109/TNET.2003.822655} (Exodus and Ebone) and from \emph{The Internet Topology Zoo}~\cite{zoo_6027859} (G\'eant), and two synthetic topologies (AWS North-America and Gaia) built from the geographical locations of AWS data centers~\cite{gaia_10.5555/3154630.3154682,amazon} (Table~\ref{tab:topologies_cycle_time}). These topologies have between 11 and 87 nodes located in the same continent with the exception of Gaia, which spans four continents. We considered that each node is connected to a geographically close silo by a symmetric access link. See Appendixes~\ref{sec:appendix_exp_details} and~\ref{sec:appendix_full_experiments} for a detailed description of the experiments and additional results. 
\begin{table}
    \caption{Datasets and Models. Mini-batch gradient computation time with NVIDIA Tesla P100.}
    \centering
    \scriptsize
    \begin{tabular*}{\textwidth}{ l @{\extracolsep{\fill}} l @{\extracolsep{\fill}} r @{\extracolsep{\fill}} c  @{\extracolsep{\fill}} l @{\extracolsep{\fill}}r@{\extracolsep{\fill}} r@{\extracolsep{\fill}} r}
        \toprule
        \multirow{2}{*}{\textbf{Dataset}} & \multirow{2}{*}{\textbf{Task}}  &  \textbf{Samples} & \textbf{Batch} & \multirow{2}{*}{\textbf{Model}} & \textbf{Parameters}& \textbf{Model Size}  & \textbf{Computation}\\
         & & (x $10^3$)  & \textbf{Size}& &  (x $10^3$) & (Mbits)& \textbf{Time} (ms)\\
         \midrule
         Shakespeare \cite{caldas2018leaf, mcmahan2016communicationefficient} & Next-Character Prediction  & $4,226$ & 512 & Stacked-GRU \cite{cho2014properties}& $840$ & $3.23$  & $389.6$\\
         FEMNIST \cite{caldas2018leaf} & Image classification  & $805$ & 128 & 2-layers CNN  & $1,207$ & 4.62 &  4.6\\
         Sentiment140 \cite{Go_Bhayani_Huang_2009} & Sentiment analysis & $ 1,600$ & 512 & GloVe \cite{pennington2014glove}+ LSTM \cite{HochSchm97} & $4,810$ & $18.38$ & $9.8$\\
         iNaturalist \cite{DBLP:journals/corr/HornASSAPB17} & Image classification  & $450$ & 16 & ResNet-18 \cite{he2015deep} & $11,217$ & $42.88$ & $25.4$\\
          \bottomrule
    \end{tabular*}%
    \label{tab:datasets_models} 
\end{table}

We evaluated our solutions on three standard federated datasets from LEAF~\cite{caldas2018leaf} and on iNaturalist dataset~\cite{DBLP:journals/corr/HornASSAPB17} with geolocalized images from over 8,000 different species of plants and animals (Table~\ref{tab:datasets_models}). For LEAF datasets, we generated non-iid data distributions following the procedure in~\cite{li2018federated}. For iNaturalist we assigned half of the images uniformly at random and half to the closest silo obtaining local datasets different in size and in the species represented {(Appendix~\ref{sec:appendix_exp_details})}.

\setlength{\tabcolsep}{2pt}
\begin{table}
    \caption{iNaturalist training over different networks. $1$~Gbps core links capacities, $10$ Gbps access links capacities. One local computation step ($s=1$).}
    \centering
    \resizebox{\columnwidth}{!}{%
        \begin{tabular}{l| c | c | c | c | r | r | r | c | c}
            \toprule
            \multirow{2}{*}{\textbf{Network name}} & \multirow{2}{*}{\textbf{Silos}} & \multirow{2}{*}{\textbf{Links}} & \multicolumn{5}{c|}{\textbf{Cycle time (ms)}} & \multicolumn{2}{c}{\textbf{Ring's training speed-up}}\\
            &  &  & {\scriptsize STAR}  &  {\scriptsize MATCHA$^{(+)}$} &{\scriptsize  MST} & {\scriptsize $\delta$-MBST} & {\scriptsize RING}  & {\scriptsize vs STAR} &   {\scriptsize vs MATCHA$^{(+)}$}\\
            \midrule 
            Gaia \cite{gaia_10.5555/3154630.3154682} & $11$ & $55$ & $391$ & $228~(228)$ &  $138$ &  $138$ &  $\mathbf{118}$  & $2.65$ & $1.54~(1.54)$\\ 
            AWS North America \cite{amazon} & $22$ & $231$ & $288$&  $124~(124)$ &  $90$ &  $90$ &  $\mathbf{81}$  & $3.41$ & $1.47~(1.47)$\\
            Géant \cite{geant} & $40$ & $61$ & $634$ & $452~(106)$ &  $\mathbf{101}$ &  $\mathbf{101}$ &  $109$  & $4.85$ & $3.46~(0.81)$\\
            Exodus~\cite{mahajan2002inferring} & $79$ & $147$ & $912$ &  $ 593~(142)$ &  $145$ &  $145$ &  $\mathbf{103}$ & $8.78$ & $5.71~(1.37)$\\
            Ebone~\cite{mahajan2002inferring} & $87$ & $161$ & $902$& $580~(123)$ & $122$ & $122$ & $\mathbf{95}$ & $8.83$ & $6.09~(1.29)$\\
            \bottomrule
        \end{tabular}%
    }
    \label{tab:topologies_cycle_time}
\end{table}

\interfootnotelinepenalty=10000
Table~\ref{tab:topologies_cycle_time} shows the effect of 6 different overlays when training ResNet-18 over iNaturalist in networks with  capacities equal to 1~Gbps and 10~Gbps for core links and access links, respectively.\footnote{
    The delay in the core network is determined by the available bandwidth as in~\eqref{eq:do}. Available bandwidths are often limited to tens or hundreds of Mbps even over inter-datacenter links with capacities between 100~Gbps and 1~Tbps~\cite{gaia_10.5555/3154630.3154682,liu17,persico17,kathiravelu18}. By selecting $1$~Gbps core links in our simulator, which ignores other traffic, we obtain available bandwidth distributions comparable to those observed in experimental studies like~\cite{gaia_10.5555/3154630.3154682} (Appendix~\ref{sec:appendix_exp_details}).
} 
These overlays are \textit{(1)}~the STAR, corresponding to the usual {server-client} setting, where the orchestrator (located at the node with the highest load centrality~\cite{brandes08}) averages all models at each communication round, \textit{(2)}~a dynamic topology built from MATCHA starting from the connectivity graph, \textit{(3)} one built starting from the underlay and denoted as MATCHA$^{+}$ (in both cases MATCHA's parameter $C_b$ equals $0.5$ as in experiments in ~\cite{wang2019matcha}\footnote{
    {Additional experiments fine tuning $C_{b}$ were carried out, conclusions remain the same (Appendix~\ref{sec:appendix_c_b}).}
}), \textit{(4)}~the minimum spanning tree (MST) from Prop.~\ref{prop:symmetric_underlay}, \textit{(5)}~the $\delta$-minimum bottleneck tree ($\delta$-MBST) from Prop.~\ref{prop:6-approx}, and \textit{(6)}~the directed RING from Prop.~\ref{prop:mct_3_approx}. In this particular setting, $\delta$-MBST selects the same overlay as MST. The consensus matrix $\mA$ is selected according to the local-degree rule~\cite{Boyd2004_1272421}.\footnote{{Additional experiments were conducted selecting the matrix $\mA$ as solution of the fastest distributed linear averaging problem defined in~\cite{Boyd2004_1272421} (Appendix~\ref{sec:appendix_full_inaturalist}).}}

The overlays found by our algorithms achieve a higher throughput (smaller cycle time) than the STAR (the {server-client} architecture) and, in most cases, than state-of-the-art MATCHA$^{(+)}$.
\footnote{
    As MATCHA and MATCHA$^{(+)}$ select random overlays at each iteration, we compute their average cycle time.
}
In particular, the RING is between 3.3 ($\approx 391/118$ on Gaia) and 9.4 ($\approx 902/95$ on Ebone) times faster than the STAR and between 1.5 and 6 times faster than MATCHA. MATCHA$^{+}$ relies on the knowledge of the underlay---probably an unrealistic assumption in an Internet setting---while our algorithms only require information about the connectivity graph. Still, the RING is also faster than MATCHA$^+$ but on G\'eant network (where MST is the fastest overlay). From now on, we show only the results for MATCHA$^+$, as it outperforms MATCHA.

\begin{figure*}
    \centering
    \begin{subfigure}[b]{0.24\textwidth}  
        \centering 
        \includegraphics[width=\textwidth, height=0.8\textwidth]{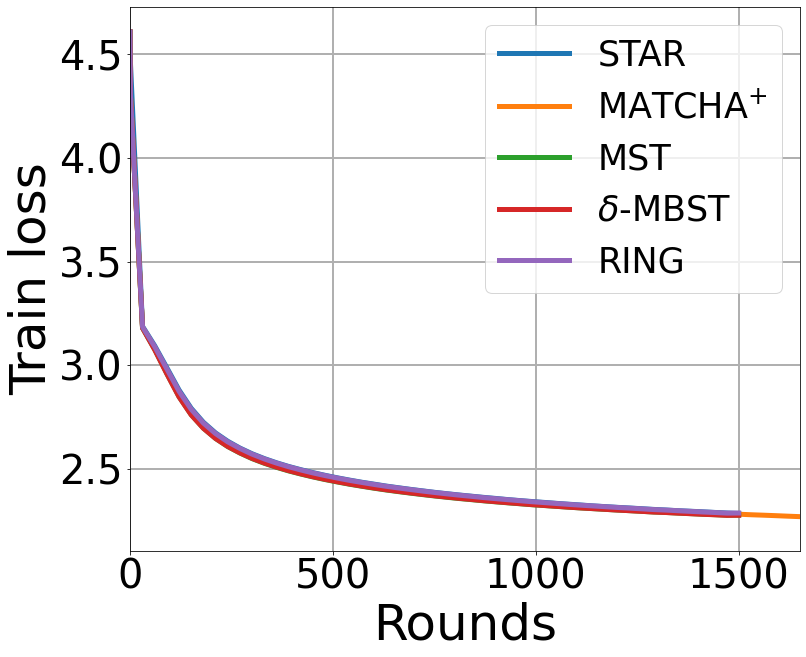}
    \end{subfigure}
    \hfill
    \begin{subfigure}[b]{0.24\textwidth}
        \centering
        \includegraphics[width=\textwidth, height=0.8\textwidth]{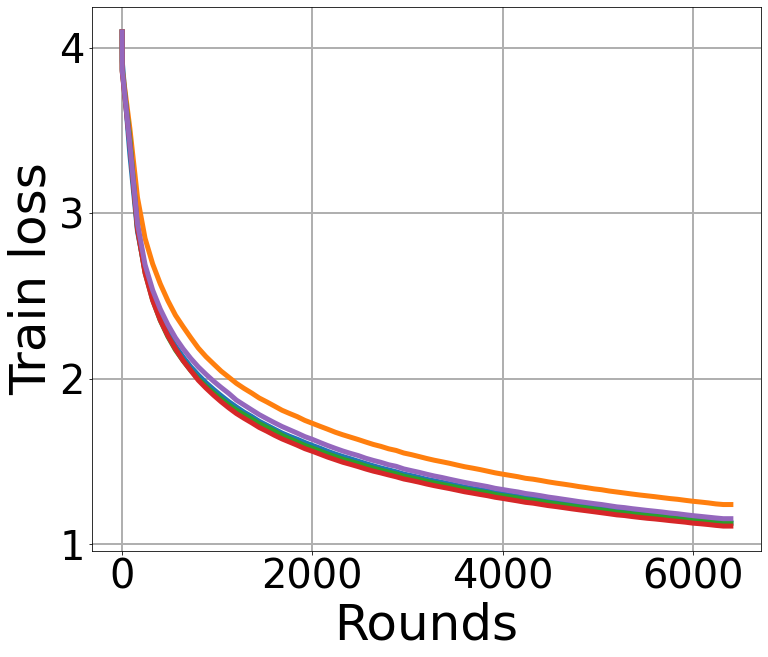}
    \end{subfigure}
    \hfill
    \begin{subfigure}[b]{0.24\textwidth}   
        \centering 
        \includegraphics[width=\textwidth, height=0.8\textwidth]{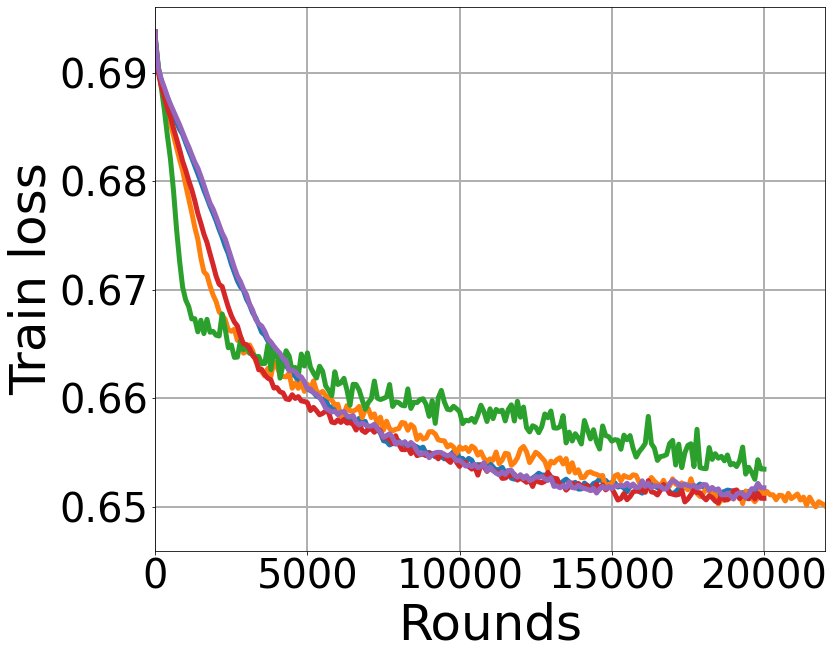}
    \end{subfigure}
    \hfill
    \begin{subfigure}[b]{0.24\textwidth}   
        \centering 
        \includegraphics[width=\textwidth, height=0.8\textwidth]{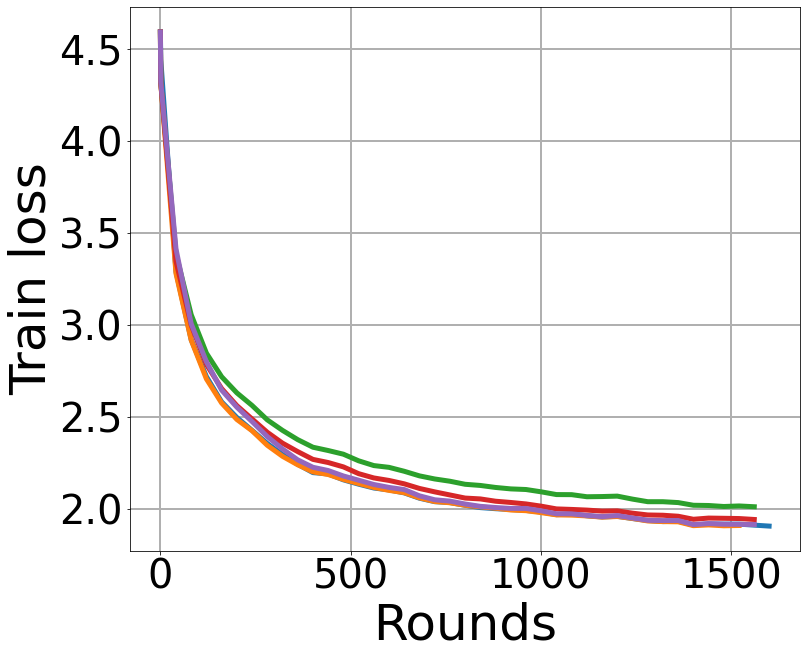}
    \end{subfigure}
        
    \begin{subfigure}[b]{0.24\textwidth}  
        \centering 
        \includegraphics[width=\textwidth, height=0.8\textwidth]{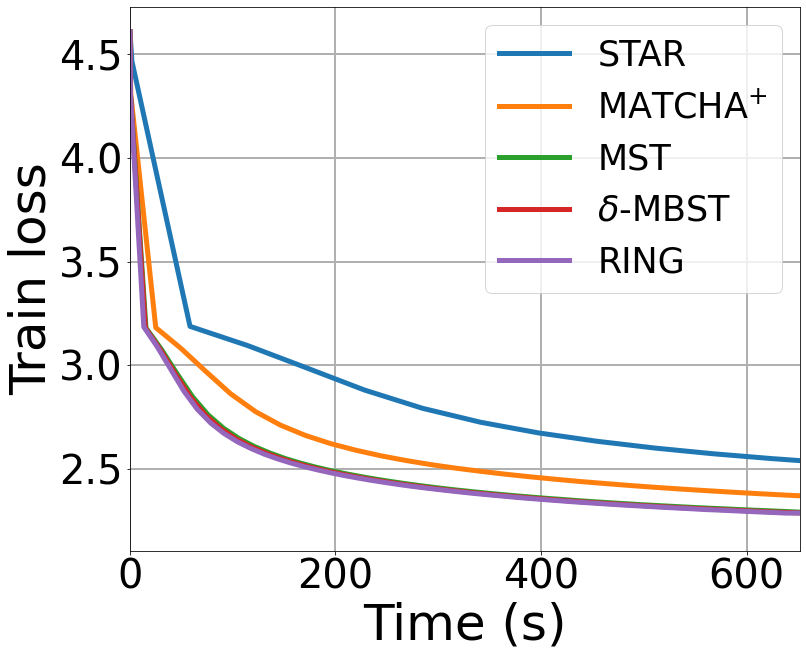}
        \caption[]%
        {{\small Shakespeare}}    
        \label{f:shakespeare}
    \end{subfigure}
    \hfill
    \begin{subfigure}[b]{0.24\textwidth}
        \centering
        \includegraphics[width=\textwidth, height=0.8\textwidth]{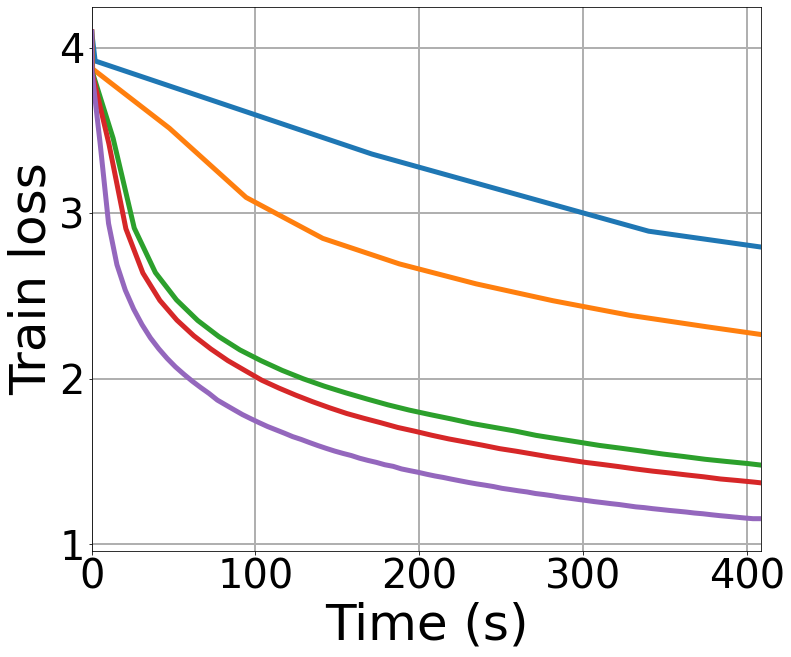}
        \caption[]%
        {{\small FEMNIST}}    
        \label{f:femnist}
    \end{subfigure}
    \hfill
    \begin{subfigure}[b]{0.24\textwidth}   
        \centering 
        \includegraphics[width=\textwidth, height=0.8\textwidth]{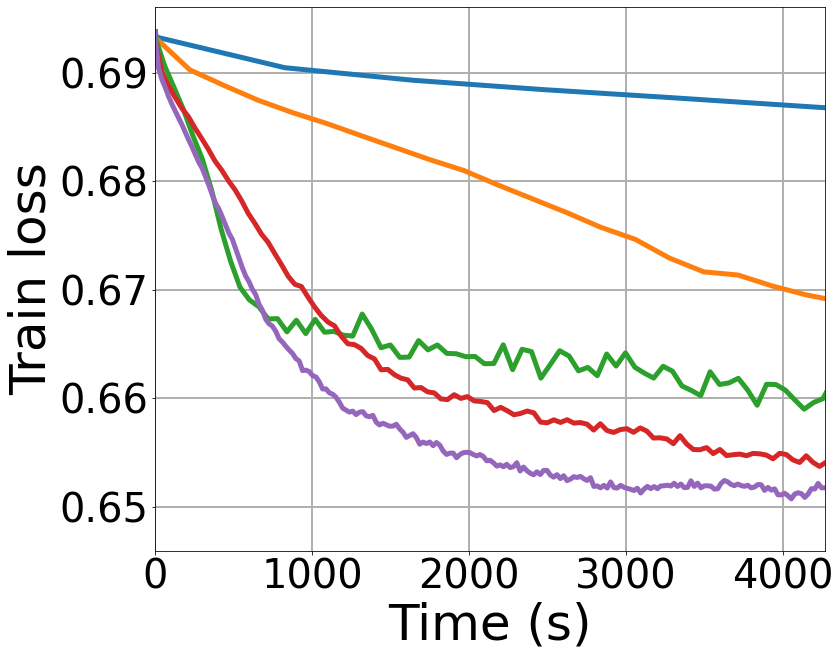}
        \caption[]%
        {{\small Sentiment140}}    
        \label{f:sent140}
    \end{subfigure}
    \hfill
    \begin{subfigure}[b]{0.24\textwidth}   
        \centering 
        \includegraphics[width=\textwidth, height=0.8\textwidth]{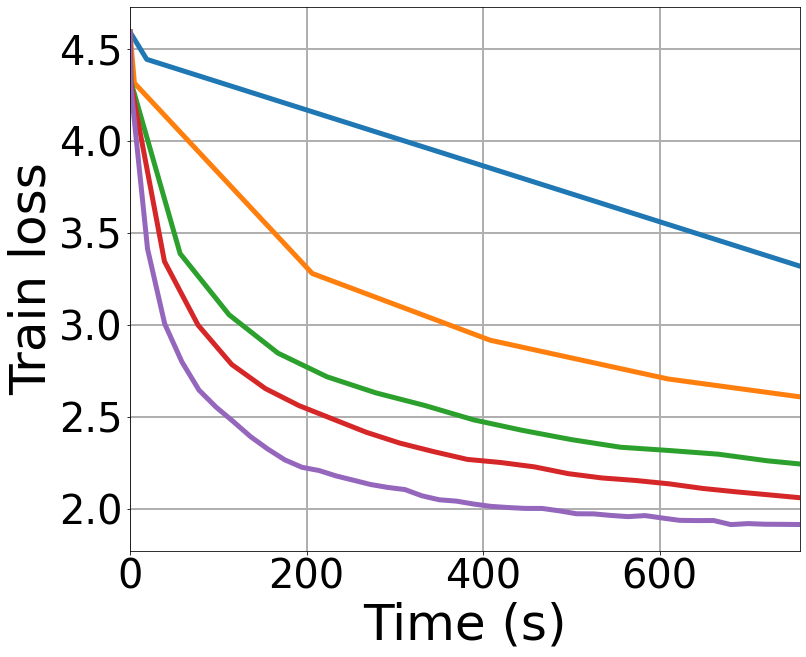}
        \caption[]%
        {{\small iNaturalist}}    
        \label{f:inaturalist}
    \end{subfigure}
    \caption[]
    {\small Effect of overlays on the convergence w.r.t.~communication rounds (top row) and wall-clock time (bottom row) when training four different datasets on AWS North America underlay. $1$~Gbps core links capacities, $100$~Mbps access links capacities, $s=1$.} 
    \label{f:training_for_all_data_sets}
\end{figure*}

The final training time is the product of the cycle time and the number of communication rounds required to converge. The overlay also influences the number of communication rounds, with sparser overlays demanding more rounds~\cite{ndeic_834019_Tradeoffs,Duchi_2012}. The last two columns in Table \ref{tab:topologies_cycle_time} show that this is a second order effect: the RING requires at most 20\% more communication rounds than the STAR and then maintains almost the same relative performance in terms of the training time.\footnote{
    Training time is evaluated as the time to reach a training accuracy equal to $65\%$, $55\%$, $55\%$, $50\%$ and $50\%$ for Gaia, AWS North America, Géant, Exodus, and Ebone networks, respectively. Note that data distribution is different in each network, so that a different global model is learned when solving Problem~\eqref{eq:optimization_problem_constraints} {(see explanations in Appendix~\ref{sec:appendix_data_partiotion})}.
} 
These results (and those in Fig.~\ref{f:training_for_all_data_sets}) confirm that the number of communication rounds to converge is weakly sensitive to the topology (as already observed in~\cite{NIPS2017_7117,pmlr-v80-lian18a,pmlr-v97-koloskova19a,Luo_2019} and partially explained in~\cite{olshevsky2019nonasymptotic,,DBLP:journals/corr/abs-1811-10792,neglia:hal-02430485}). The conclusion is that overlays should indeed be designed for throughput improvement rather than to optimize their spectral properties: the topologies selected by our algorithms achieve faster training time than the STAR, which has optimal spectral properties, and MATCHA/MATCHA$^{(+)}$, which optimize spectral properties given a communication budget.

The same qualitative results hold for other datasets and Fig.~\ref{f:training_for_all_data_sets} shows the training loss versus the number of communication rounds (top row) and versus time (bottom row)  when training on AWS North America with 100 times slower access links. {Other metrics for model evaluation (e.g., training/test accuracy) are shown in Appendix~\ref{sec:appendix_other_metrics_aws}.} The advantage of designing the topology on the basis of the underlay characteristics is evident also in this setting.

\begin{figure*}
    \centering
    \begin{subfigure}[b]{0.49\textwidth}   
        \centering 
        \includegraphics[width=0.49\linewidth, height=0.40\linewidth]{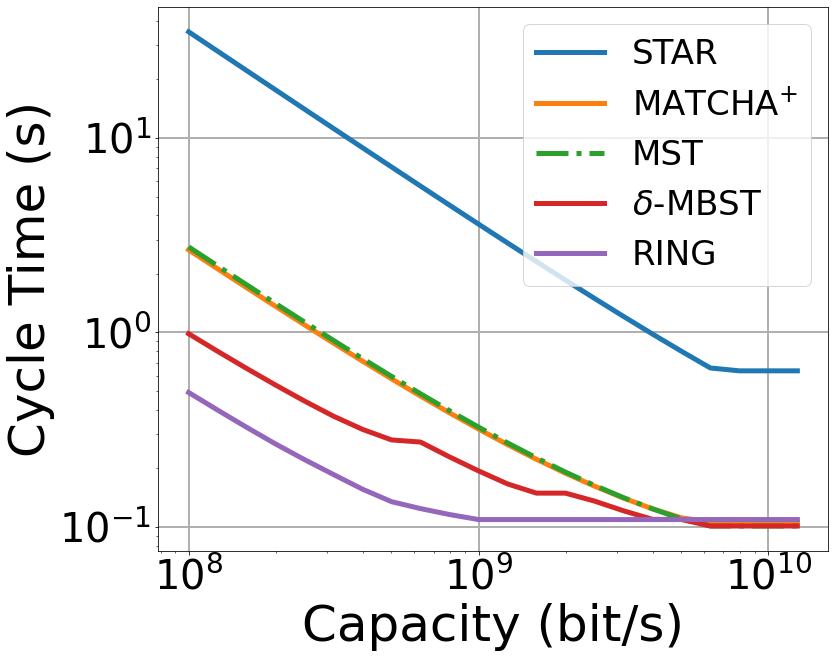}
        \label{f:uniform_capcity_1}
        \hfill
        \includegraphics[width=0.49\linewidth, height=0.40\linewidth]{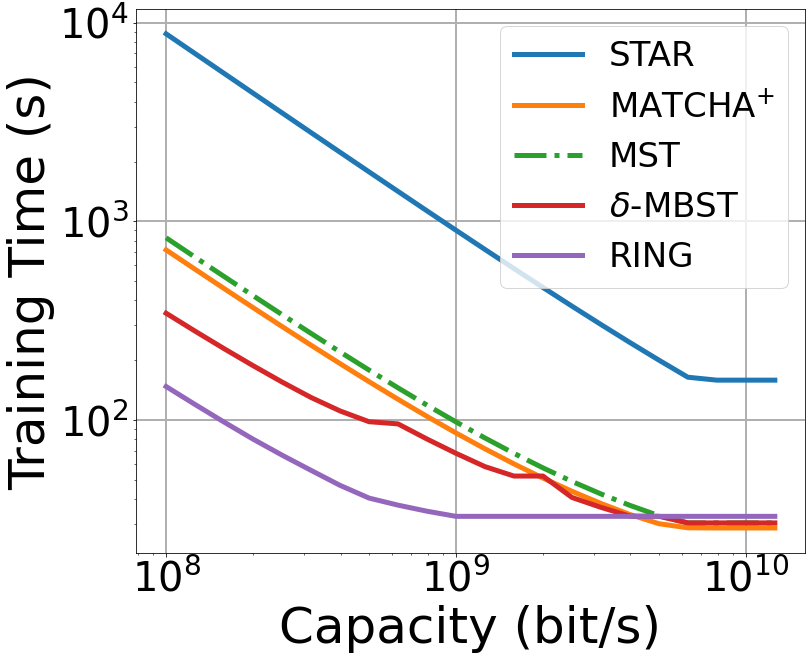}
        \caption[]%
        {{\small Homogeneous access link capacities}}    
        \label{f:uniform_capcity}
    \end{subfigure}
    \hfill
    \begin{subfigure}[b]{0.49\textwidth}   
        \centering 
        \includegraphics[width=0.49\linewidth, height=0.40\linewidth]{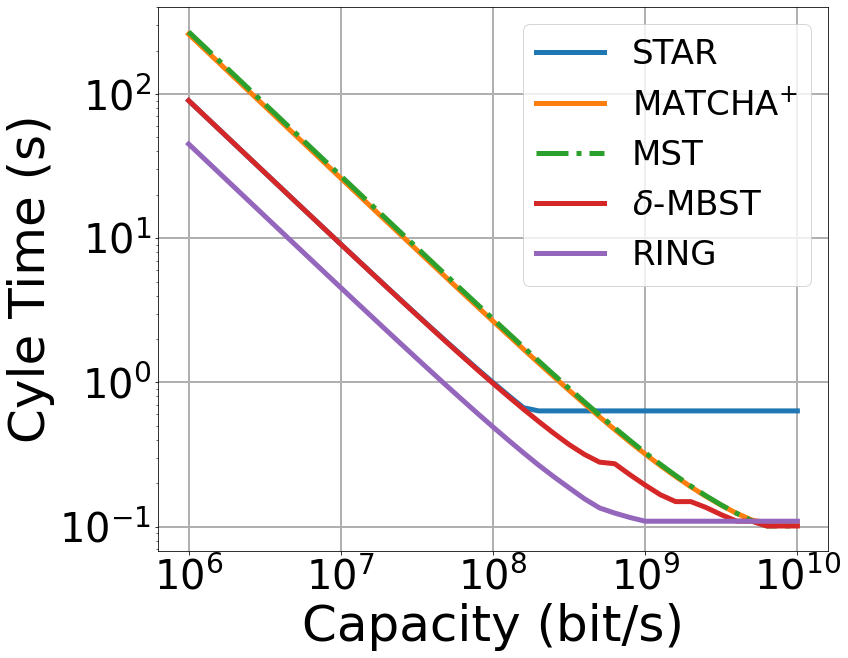}
        \label{f:het_capacity}
        \hfill
        \includegraphics[width=0.49\linewidth, height=0.40\linewidth]{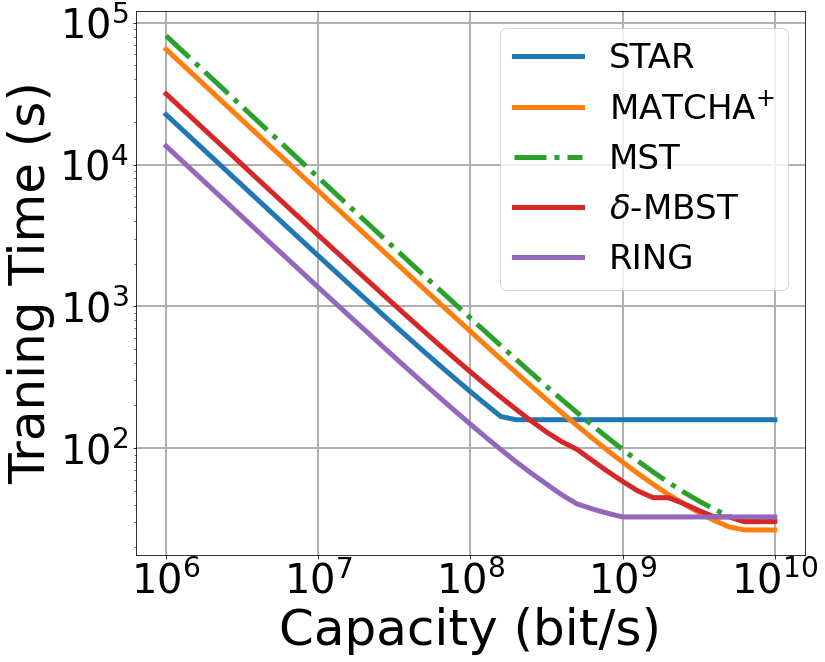}
        \caption[]%
        {{\small Central node with $10$~Gbps access link capacity}}    
        \label{f:het_capacity_1}
    \end{subfigure}
    \caption[]
    {\small Effect of access link capacities on the cycle time and the training time when training iNaturalist  on Géant network. $1$~Gbps core links capacities, $s=1$. (\ref{f:uniform_capcity}): All access links have the same capacity. (\ref{f:het_capacity_1}): One node (the center of the star) has a fixed 10~Gbps access link capacity. The training time is the time when training accuracy reaches $55\%$.} 
    \label{f:cpacity_effect}
\end{figure*}

Figure~\ref{f:cpacity_effect} illustrates the effect of access link speeds on the cycle time and the training time.
When all silos have the same access link capacity (Fig.~\ref{f:uniform_capcity}), for capacity values smaller than 6~Gbps, the RING has the largest throughput followed by $\delta$-MBST, MST and MATCHA$^{+}$ almost paired, and finally the STAR. The advantage of topologies  with small nodes' degrees (like $\delta$-MBST and the RING) is someway expected in the slow access link regime, as access link delays become the dominant term in~\eqref{eq:do}. In particular, Eq.~\eqref{eq:cycle_time} and some simple calculations in  Appendix~\ref{sec:appendix_ring_vs_star}  show that, with $N$ silos, the RING is up to $2 N$ (=80 for G\'eant) times faster than the STAR and $C_b \times \max(\textrm{degree}(\mathcal G_u))$ (= 5 for G\'eant) times faster then MATCHA$^{(+)}$ for slow access links as confirmed in Fig.~\ref{f:uniform_capcity} (left plot). 
What is less expected (but aligned with our observations above about the importance to design overlays for  throughput improvement) is that 
RING's throughput speedups lead to almost as large training time speedups, even larger than those in Table~\ref{tab:topologies_cycle_time}: e.g.~72x in comparison to the STAR and 5.6x in comparison to MATCHA$^{+}$ for 100~Mbps access link capacities.


When the most central node (which is also the center of the STAR) maintains a fixed capacity value equal to $10$~Gbps (Fig.~\ref{f:het_capacity_1}), the STAR performs better, but still is twice slower than the RING and only as fast as $\delta$-MBST. This result may appear surprising at first, but it is another consequence of  Eq.~\eqref{eq:cycle_time} discussed in Appendix~\ref{sec:appendix_ring_vs_star}. Again the relative performance of different overlays in terms of throughput is essentially maintained when looking at the final training time, with differences across topologies emerging only for those with very close throughputs, i.e., MST and MATCHA$^{+}$, and STAR and $\delta$-MBST in the heterogeneous setting of Fig.~\ref{f:het_capacity_1}.

\begin{figure*}
    \centering
    \includegraphics[scale=0.17]{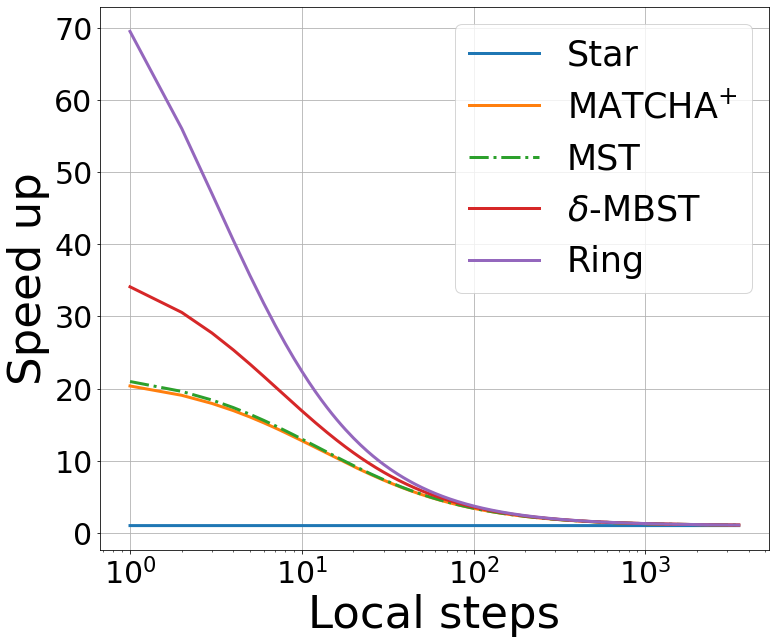}
    \caption{Throughput speedup in comparison to the STAR, when training iNaturalist over Exodus network. All links with $1$~Gbps capacity.}
    \label{fig:cycle_time_vs_local_ste}
\end{figure*}

When local computation requires less time than transmission of model updates, the silo may perform $s$~local computation steps before a communication round. As $s$ increases, the total computation time ($s \times T_c(i)$) becomes dominant in~\eqref{eq:do} and the throughput of different overlays become more and more similar (Fig.~\ref{fig:cycle_time_vs_local_ste}).\footnote{
    In Appendix~\ref{sec:appendix_local_steps}, we show tables similar to Table~\ref{tab:topologies_cycle_time} for different values of $s$.
}
Too many local steps may degrade the quality of the final model, and how to tune~$s$ is still an open research area~\cite{wang2018adaptive,wang2019slowmo, adaptiveJianyuWang,lin2020,woodworth2020,koloskova2020}. Our next research goal is to study this aspect in conjunction with topology design. Intuitively, a faster overlay reduces the number of local steps needed to amortize the communication cost and may lead to better models given the available time budget for training.

\section{Conclusions}
\label{sec:conclusion}
We used the theory of max-plus linear systems to propose topology design algorithms that can significantly speed-up federated learning training by maximizing the system throughput. Our results show that this approach is more promising  than targeting topologies with the best spectral properties, as MATCHA$^{(+)}$ does. {
In future work,
we will explore how to further speed-up training, e.g., by enriching the topologies found by our algorithms with additional links that improve connectivity without decreasing the throughput, and by carefully optimizing the weights of the consensus matrix.
}

\section{{Broader Impact}}
\label{sec:impact}
We have proposed  topology design algorithms that can significantly speed-up federated learning in a cross-silo setting. Improving the efficiency of federated learning can foster its adoption, allowing different entities to share datasets that otherwise would not be available for training.

Federated learning is intended to protect data privacy, as the data is not collected at a single point. At the same time a federated learning system, as any Internet-scale distributed system, may be more vulnerable to different attacks aiming to jeopardize training or to infer some characteristics of the local dataset by looking at the different messages~\cite{fredrikson2015model,shokri2017membership}. Encryption \cite{bost2015machine, nikolaenko2013privacy, bonawitz2017practical} and differential privacy \cite{abadi2016deep} techniques may help preventing such attacks. 

Federated learning is less efficient than training in a highly-optimized computing cluster. It may in particular increase energy training costs, due to a more discontinuous usage of local computing resources and the additional cost of transmitting messages over long distance links. To the best of our knowledge, energetic considerations for federated learning have not been adequately explored, but for a few papers considering FL for mobile devices~\cite{kang19,tran19}.

\section{{Acknowledgements}}
\label{sec:acknowledgement}
{The authors are grateful to the OPAL infrastructure from Universit\'e C\^ote d'Azur for providing computational resources and technical support.}

This work was carried out and partially funded in the framework of a common lab
agreement between Inria and Nokia Bell Labs (ADR `Rethinking the Network').

The authors thank Damiano Carra, Alain Jean-Marie, Marco Lorenzi, and Pietro Michiardi  for their feedback on early versions of this paper,
Fran\c{c}ois Baccelli, Bruno Gaujal, Laurent Hardouin, and Enrico Vicario for pointers to the literature of max-plus linear systems, and the Italian networking community (in particular Mauro Campanella, Marco Canini, Claudio Cicconetti, Francesca Cuomo, Paolo Giaccone, Dario Maggiorini, Marco Mellia, Antonio Pescap\'e, Tommaso Pecorella, and Luca Valcarenghi) for their suggestions to select realistic network scenarios for federated learning.
 Obviously, the authors keep the responsibility  for any error in this paper.

\newpage
\printbibliography

\newpage
\appendix

\newpage
\section{Graph Theory}
\label{sec:appendix_graph}
We now list concepts of graph theory which will be used later on.

\begin{itemize}
    \item \textbf{Predecessor, successor, neighbour}: If in a graph $(i, j) \in \mathcal{E}$, then $i$ is called a predecessor of $j$, $j$ is called a successor of $i$ and $j$, resp. $i$ is called a neighbour of $i$ , resp. $j$. The set of predecessors of $j$ is indicated by $\pi(j)$ (or $\mathcal{N}^{+}_{j}$), the set of all successors of $i$ is denoted $\sigma(i)$ (or $\mathcal{N}^{-}_{i}$) and the set of neighbours of $i$ is denoted $\mathcal{N}_{i}$. Note that in the case of undirected graphs, $\mathcal{N}_{i} = \pi(i) = \sigma(i)$.
    
    \item \textbf{Path, circuit:} A path is a sequence of nodes $(i_{1}, \dots, i_{p}), p > 1$, such that $i_{j} \in \pi(i_{j+1}), j=1, \dots, p-1$. An elementary path is a path where no node appears more then once. When the initial node and the final node coincide, we call the path a circuit. A circuit $C = (i_{1}, \dots, i_{p}=i_{1})$ is an elementary circuit if the path $(i_{1}, \dots, i_{p-1})$ is elementary, an elementary circuit is sometimes referred to as a cycle. If a cycle spans all  vertices of the graph it is called a \emph{Hamiltonian cycle}. The length of circuit $C = (i_{1}, \dots, i_{p})$ is the number of the arcs of which it is composed, i.e., $|C| = p$, and its weight is the sum of the weights of its arcs, i.e, $d(C)=\sum_{k =1}^{p-1}d(i_{k}, i_{k+1})$. 
    
    \item \textbf{Subgraph, spanning subgraph:} Given a graph $\mathcal{G}=(\mathcal{V}, \mathcal{E})$, a graph $\mathcal{G}'=(\mathcal{V}', \mathcal{E}')$ is said to be a subgraph of $\mathcal{G}$ if $\mathcal{V}' \subset \mathcal{V}$ and 
    $\mathcal E \subset \mathcal E'$.
    $\mathcal{G}'$ is said to be a spanning subgraph if $\mathcal{V}' = \mathcal{V}$.
    
    \item \textbf{Strongly connected graph:} A digraph is said to be \emph{strongly connected} or \emph{strong} if for any two different nodes $i$ and $j$ in $\mathcal{V}$ there exists a path from $i$ to $j$. 
    
    \item \textbf{Optimal tour}: In a Hamiltonian graph (i.e., a graph having a Hamiltonian cycle) a Hamiltonian cycle with minimum weight is called an \emph{optimal tour}. Finding the optimal tour in a complete graph is a well known problem and is referred to as the Traveling Salesman Problem (TSP), see for example \cite{10.5555/1374811}.
    
    \item \textbf{Tree, acyclic graph, and Minimum Spanning Tree (MST)}: A tree, or equivalently a connected acyclic undirected graph, is an undirected graph in which any two vertices are connected by exactly one path. An acyclic graph $\mathcal T$ is said to be a spanning tree of an undirected graph $\mathcal{G}$ if $\mathcal T$ is a connected spanning subgraph of $\mathcal{G}$. $\mathcal T$ is said to be an MST of $\mathcal{G}$ if it has minimal weight (the weight of a tree is the sum of the weights of all its edges) among all spanning trees of $\mathcal{G}$.
    
    \item \textbf{Cut, cut-set, and cut property}: A \emph{cut} is a partition of the vertices of a graph into two disjoint subsets. For a cut $c$, the cut-set is the set of edges connecting two nodes from the two disjoint subsets. In a tree, deleting an edge, induces a partition of the set of vertices. For any cut $c$ of the graph, if the weight of an edge $e$ in the cut-set of $c$ is strictly smaller than the weights of all other edges of the cut-set of $c$, then this edge belongs to all MSTs of the graph. 
\end{itemize}

\newpage
\section{On STAR and \texorpdfstring{MATCHA$^{(+)}$}{mat}  Cycle Times}
\label{sec:appendix_ring_vs_star}
For a graph $\mathcal G$, let $\textrm{degree}(i,\mathcal G)$ denote the degree node $i$ in $\mathcal G$ and $\max(\textrm{degree}(\mathcal G))$ denote the maximum degree of the nodes in $\mathcal G$. We show that, with $N$ silos, the RING is up to $2 N$ times faster than the STAR and approximately $C_b \times \max(\textrm{degree}(\mathcal G_u))$ times faster then MATCHA$^{(+)}$ for slow homogeneous access links as shown also in Fig.~\ref{f:uniform_capcity}.

Since access links are homogeneous, i.e.,~$\upcap(i)=\dncap(i)=\upcap(j)=\dncap(j) = C, \forall i,j\in \mathcal{V}$, and slow access links determine the delays, i.e.,~$\upcap(i) \ll A(i',j')$ and $s\times T_c(i)+l(i,j) \ll \frac{M}{A(i,j)}$, according to~\eqref{eq:do}, we have: 
\begin{equation*}
    d_{o}(i, j) = \max\left(|\mathcal{N}_{i}^{-}|, |\mathcal{N}_{j}^{+}|\right) \times \frac{M}{C}. 
\end{equation*}

Then, the cycle time of the RING can be obtained from~\eqref{eq:cycle_time}:
\begin{equation*}
    \tau_{\text{RING}} = \frac{\sum_{i=1}^{N}d_{o}(i, i+1)}{N} = \frac{\frac{M}{C} \times N}{N} =\frac{M}{C}. 
\end{equation*}
Remember that a cycle is the time interval between two consecutive computations at a given silo. For the STAR, it corresponds to the time interval between when the  central node sends the new aggregate model to all silos and when it receives all updated local models. Therefore, we have:
\begin{equation*}
    \tau_{\textrm{STAR}} = \frac{M}{C} \times N + \frac{M}{C} \times N = 2N \times \frac{M}{C}.
\end{equation*}
For MATCHA$^{+}$, at each communication round, we select a random subgraph $\mathcal{G}$. Let $\textrm{degree}(i, \mathcal G)$ denote the degree of silo $i$ in $\mathcal G$. If $\mathcal G$ is drawn, the duration of the communication round is $M/C \times \max(\textrm{degree}(\mathcal G))$. The cycle time is then 
\begin{equation*}
    \tau_{\textrm{MATCHA}^{+}} = \frac{M}{C} \mathbb{E}_{\mathcal{G}}\left[\max\textrm{degree}(\mathcal{G})\right].
\end{equation*}
Let $j$ be the silo such that $j'$ has the largest degree in $\mathcal G_u$.  MATCHA$^{+}$ uses $\max(\textrm{degree}(\mathcal{G}_u))+1$ matchings. The edges of $j$ belong to different matchings. As MATCHA$^{+}$ activates at any communication round a fraction $C_b$ of all matchings, the average degree of node $j$ is $\mathbb{E}_{\mathcal G}\left[\textrm{degree}(j,\mathcal{G})\right] \approx C_b \times \textrm{degree}(j,\mathcal{G}_u) = C_b \times \max(\textrm{degree}(\mathcal{G}_u))$. Then
\begin{equation*}
\tau_{\textrm{MATCHA}^{+}} \gtrapprox \frac{M}{C} \times C_{b} \times  \max(\textrm{degree}(\mathcal G_u)).
\end{equation*}

\newpage
\section{Directed Overlays may be Faster than Undirected Overlays}
\label{sec:appendix_ring_beats_mts}
\begin{figure}
    \begin{subfigure}[b]{0.3\textwidth}
        \centering
        \includegraphics[scale=0.4]{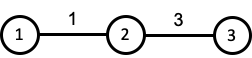}
        \caption{A 3-node example.}
        \label{fig:ring_beats_mts_3}
    \end{subfigure}
    \hfill
    \begin{subfigure}[b]{0.65\textwidth}
        \centering
        \includegraphics[scale=0.4]{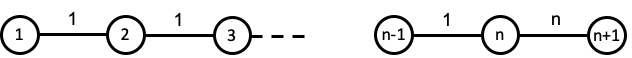}
        \caption{Example with arbitrarily different cycle times.}
        \label{fig:ring_beats_mts_n}
    \end{subfigure}
    \hfill
    \caption{Networks where a directed topology outperforms an undirected one.}
    \label{fig:ring_beats_mts}
\end{figure}

We provide two examples where the underlay network is undirected and still a directed overlay can have shorter cycle time than directed overlays. Examples are in Fig.~\ref{fig:ring_beats_mts}, where numbers associated to links are the corresponding delays (in the two directions).

The network in Fig.~\ref{fig:ring_beats_mts_3} has only three nodes, $\mathcal V = \{1,2,3\}$. We have $d_c(1,2) = d_c(2,1) = 1$, $d_c(2,3) = d_c(3,2) = 3$, and $d_c(1,3) = d_c(3,1) = 4$. The fastest undirected overlay is $G_o^{(u)} = (\mathcal V, \{(1,2),(2,3)\})$. Consider the directed ring $G_o = (\mathcal V, \{(1,2),(2,3),(3,1)\})$.
We have:
\begin{align}
    & \tau\left(\mathcal G_o^{(u)}\right)  = \max\left(\frac{1+1}{2},\frac{3+3}{2},\frac{1+3+1+3}{4}\right) = 3,\\
    & \tau\left(\mathcal G_o\right)  = \frac{1 + 3 + (3 + 1)}{3} = \frac{8}{3} < 3.
\end{align}

The network in Fig.~\ref{fig:ring_beats_mts_n} shows that a directed ring can be arbitrarily faster than an undirected one. Similarly to above, the fastest undirected overlay is $G_o^{(u)}$ and coincides with the underlay. The directed overlay is the ring $(1 \to 2 \to 3 \to \dots n \to n+1 \to 1)$. We have
\begin{align}
    & \tau\left(\mathcal G_o^{(u)}\right)  = n,\\
    & \tau\left(\mathcal G_o\right)   = \frac{(n-1) \times 1 + n + (n+ (n-1) \times 1)}{ n + 1 } = \frac{4 n - 2}{n+1} < 4.
\end{align}
The ratio of the two cycle times can be made arbitrarily large.


\newpage
\section{Approximation Algorithm for \texorpdfstring{$\MCT$}{MCT} on Node-Capacitated Networks}
\label{sec:appendix_alg}
\newcommand{\Algodelta}{\textsc{$\delta$-Prim}}
In this section, we describe  Algorithm~\ref{alg:our_algorithm} that provides an approximate solution for $\MCT$ when the network is node-capacitated and $\mathcal G_c$ is complete. Algorithm~\ref{alg:our_algorithm} combines existing approximation algorithms for $\delta$-$\MBST$ on a particular undirected graph built from $\mathcal G_c$ and denoted by $\mathcal{G}^{(u)}_c$ (lines~\ref{algoline:gprime}-\ref{algoline:d}). Lemma~\ref{lem:bottleneck_vs_mct} establishes a connection  between the bottleneck of the $\MBST$ of $\mathcal{G}^{(u)}_c$ and the cycle time of $\MCT$ on $\mathcal{G}_c$ when the overlay is required to be undirected. To get an approximated $2$-$\MBST$ on $\mathcal{G}^{(u)}_c$, we apply the best known $3$-approximation algorithm from~\cite[Sect.~3.2.1]{doi:10.1002/net.21710} (lines~\ref{algoline:2BSTstart}-\ref{algoline:2BSTend})  which requires $\mathcal{G}^{(u)}_c$ to be Euclidean (Lemma~\ref{lem:g_prime_is_euclidean}), and take its result as one candidate for our solution (line~\ref{algoline:2candidate}). The cube of a graph $\mathcal{G}$, denoted by $\mathcal{G}^3$, is the super-graph of $\mathcal{G}$ such that the edge ($u$, $v$) is in $\mathcal{G}^3$ if and only if there is a path between $u$ and $v$ in $\mathcal{G}$ with three or fewer edges. It has been proved that the cube of a connected graph is Hamiltonian and to find a Hamiltonian path in such a cube can be done in polynomial time~\cite{Karaganis1968OnTC}.
Other $\delta$-BSTs built by Algorithm~\ref{alg:3_mbst} for $3\leq \delta \leq N$ are considered as candidates (lines~\ref{algoline:candidate}-\ref{algoline:otherdelta}) and we finally provide as solution the overlay with the smallest cycle time (line~\ref{algoline:best}).

\IncMargin{2.0em}
\begin{algorithm}[H]
    \label{alg:our_algorithm}
    \SetAlgoLined
    \SetKwInOut{Input}{Input}
    \Indm\Indmm
    \Input{$\mathcal{G}_c = (\mathcal{V}, \mathcal{E}_c)$, uplink capacity $\upcap(i)$, end-to-end delay $l(i,j)$, computation time $T_c(i)$ and model size $M$.}
    \KwResult{Undirected overlay $\mathcal{G}_o$.}
    \Indp\Indpp
    \BlankLine
    Create $\mathcal{G}^{(u)}_{c} = (\mathcal{V}, \mathcal{E}^{(u)}_c)$  where $(i,j)\in \mathcal{E}^{(u)}_c$ iff $(i,j)\in \mathcal{E}_c$ and $(j,i)\in \mathcal{E}_c$ 
    \; \label{algoline:gprime}   
    \For{$(i,j) \in \mathcal{E}^{(u)}_c$}{
        $d_c^{(u)}(i, j) = \large[s\times (T_c(i)+T_c(j)) + l(i, j) + l(j, i) + \frac{M}{\upcap(i)} + \frac{M}{\upcap(j)} \large]/2 $  \label{algoline:d}
    }

    $\mathbb{S} \leftarrow \emptyset$ \tcp*[l]{the set of candidate solutions} 
    \tcc{consider $2$-$\MBST$ approximate solution on $\mathcal{G}^{(u)}_c$ as one candidate}
    $\mathcal{T} \leftarrow$ a minimum weight spanning tree of $\mathcal{G}^{(u)}_c$ \;\label{algoline:2BSTstart}
    $\mathcal{T}^{3}\leftarrow$ the cube of $\mathcal{T}$ \;
    $\mathcal{H} \leftarrow$  a Hamiltonian path in $\mathcal{T}^{3}$
    \;\label{algoline:2BSTend}
    $\mathbb{S} \leftarrow \{\mathcal{H}\}$\;\label{algoline:2candidate}
    \tcc{ consider other $\delta$-BST for $3\leq  \delta \leq N$ as candidates}
    \For{$\delta \in \{3,4,5,...,N\}$\label{algoline:candidate}}{
        $\mathbb{S}\leftarrow \mathbb{S} \cup \{\Algodelta(\mathcal{G}^{(u)}_c)\}$ \tcp{$\Algodelta(\mathcal{G}^{(u)}_c)$  gives a  $\delta$-BST on $\mathcal{G}^{(u)}_c$ \label{algoline:otherdelta}}
    }
    \tcc{ choose the one with the minimum cycle time as output overlay}
    $\mathcal{G}_o \leftarrow \arg\min_{G\in \mathbb{S}} \tilde{\tau}(G)$ \label{algoline:best}
    \caption{Approximation algorithm for $\MCT$ on node-capacitated networks.}
\end{algorithm}
\DecMargin{2.0em}

\IncMargin{2.0em}
\begin{algorithm}[H]
    \label{alg:3_mbst}
    \SetAlgoLined
    \SetKwInOut{Input}{Input}
    \SetKwFunction{Function}{\Algodelta}
    \SetKwProg{Fn}{Function}{:}{}

    \Fn{\Function{$\mathcal{G} = (\mathcal{V}, \mathcal{E})$}}{
        $\mathcal{V}_{T} \coloneqq \left\{v_{0}\right\}$ for some $v_{0} \in \mathcal{V}$\;
        $\mathcal{E}_{T} \coloneqq \left\{\right\}$\;
        $T = (\mathcal{V}_{T}, \mathcal{E}_{T})$\;
        \While{$\left|\mathcal{E}_{T}\right| < \left|\mathcal{V}\right| - 1$}{
            Find the smallest weight edge $(u, v)$ such that $u \in \mathcal{V}_{T}$, $v \not\in \mathcal{V}_{T}$, and $\textsc{degree}_{T}(u) < \delta$\;
            Add $v$ to $\mathcal{V}_{T}$\;
            Add $(u, v)$ to $\mathcal{E}_{T}$\;
    }
    \KwRet $T$ \;
}
\caption{\Algodelta \cite{andersen2018algorithms}}
\end{algorithm}
\DecMargin{2.0em}

\newpage
\section{Proofs}
\label{sec:appendix_proofs}
We use some graph terminology and notation introduced in Appendix~\ref{sec:appendix_graph}.

\subsection{Proof of Proposition \ref{prop:symmetric_underlay}}
\label{sec:appendix_proof_symmetric_underlay}
When we require the overlay $\mathcal G_o$ to be undirected, if we include link $(i,j) \in \mathcal G_c$ then we will also include link $(j,i)$. It is then convenient to consider the undirected graph $\mathcal G_c^{(u)} = (\mathcal V, \mathcal E_c^{(u)})$, where $(i,j) \in \mathcal E_c^{(u)}$ iff  $(i,j) \in \mathcal E_c$ { and $(j,i)\in \mathcal{E}_c$}, from which we want to extract an undirected strong subgraph $\mathcal G_o$ with minimal cycle time. We also associate to each edge $(i,j) \in \mathcal{G}_{c}^{(u)}$ the weight $d_c^{(u)}(i,j)= (d_c(i,j)+d_c(j,i))/2$. Remember that $d_c(i,j)$ is defined as follows \[d_c(i,j)\triangleq s \times T_{c}(i) + l(i,j) + M/A(i',j').\]

Note that an undirected weighted graph can be also seen as a particular directed graph where for each link $(i,j)$ in one direction, there exists a link $(j,i)$ with the opposite direction and the same weight. The concept of cycle time can then immediately be extended to undirected graphs.

\begin{lem}
    \label{lem:spanning_tree}
    Consider the undirected weighted graph $\mathcal G_c^{(u)} = (\mathcal V, \mathcal E_c^{(u)})$, where $(i,j) \in \mathcal E_c^{(u)}$ iff  $(i,j) \in \mathcal E_c$ { and $(j,i)\in \mathcal{E}_c$}. {When $\mathcal G_c$ is edge-capacitated and $\mathcal G_o$ is required to be undirected, the set of solutions $\MCT$ includes a spanning tree of $\mathcal{G}_{c}^{(u)}$.}
\end{lem}

\begin{proof}
    $\MCT$ is a discrete optimization problem on a finite set,\footnote{
        The set of subgraphs of an undirected graph $\mathcal{G}_{c}$ is finite.
    }
    thus the set of solutions of $\MCT$ is non-empty. Suppose by contradiction that the set of solutions does not contain any spanning tree of $\mathcal{G}_{c}$ and consider $\mathcal{G}_o^{*}$ {to be one of such solutions.}

    As $\mathcal{G}_o^*$ is not a spanning tree and it is strongly connected, there exist circuits in $\mathcal{G}_o^*$. For any circuit $C=(i_{1}, i_{2}, \dots, i_{p}=i_{1})$ in $\mathcal{G}_o^*$, we consider the edge $e_{C}$, such that $d_c^{(u)}(e_{C}) = \max_{k=1,\dots , p-1} d_c^{(u)}(i_{k}, i_{k+1})$. The graph $\hat{\mathcal{G}}_o^{*}$ obtained from $\mathcal{G}_o^{*}$ by deleting $e_{C}$ is a connected spanning subgraph of $\mathcal{G}_c^{(u)}$ and its cycle time is not greater then the cycle time of $\mathcal{G}^{*}_o$. We can now proceed in the same way on $\hat {\mathcal G}_o^*$ until the residual graph has no more circuits and it is then a spanning tree of $\mathcal{G}_c^{(u)}$ with cycle time not greater than the cycle time of $\mathcal{G}^{*}_o$. This tree is also a solution of $\MCT$  contradicting the fact that no spanning tree is in the set of solutions. 
\end{proof}

\begin{lem}
    \label{lem:symmetric_underlay}
    Consider an undirected tree $\mathcal{T} = (\mathcal{V}, \mathcal{E})$, weighted with a delay function $d^{(u)}_{c}: \mathcal{V} \times \mathcal{V} \mapsto \mathbb{R}_{+}$. Its cycle time is $\tau(\mathcal{T}) = \max_{\{i, j\} \in \mathcal{E}} d_{c}^{(u)}(i,j)$.
\end{lem}

\begin{proof}
    The cycle time of $\mathcal{T}$ is given by Equation~\eqref{eq:cycle_time}. $\tau(\mathcal{T})   =  \max_{C}\frac{w(C)}{|C|}$,   where the maximum is taken over all the elementary circuits of $\mathcal{T}$. Since $\mathcal{T}$ is acyclic, the only elementary circuits of $\mathcal{T}$ are of the form $(i, j, i)$ for some $\{i, j\} \in \mathcal{E}$. By definition $|(i, j, i)| = 2$ and $w((i, j, i)) = d^{({u})}_{c}(i,j)$. It follows that $\tau(\mathcal{T}) = \max_{\{i, j\} \in \mathcal{E}}\frac{d_{c}^{(u)}(i,j) + d_{c}^{(u)}(j,i)}{2} = \max_{\{i, j\} \in \mathcal{E}} d_{c}^{(u)}(i,j)$.
\end{proof}

\begin{repprop}{prop:symmetric_underlay}
    Consider an undirected weighted graph $\mathcal G_c^{(u)} = (\mathcal V, \mathcal E_c^{(u)})$, where $(i,j) \in \mathcal E_c^{(u)}$ iff  $(i,j) \in \mathcal E_c$ and $(j,i) \in \mathcal E_c$ and where $(i,j) \in \mathcal E_c^{(u)}$ has weight $d_c^{(u)}(i,j)= (d_o(i,j)+d_o(j,i))/2$. A~minimum weight spanning tree of $\mathcal{G}_{c}^{(u)}$ is a solution of $\MCT$ when $\mathcal G_c$ is edge-capacitated and $\mathcal G_o$ is required to be undirected.
\end{repprop}

\begin{proof}
    Denote by $\mathcal{G}^{*}$ the solution of $\MCT$ when $\mathcal G_c$ is edge-capacitated and $\mathcal G_o$ is required to be undirected, and denote $\mathcal{T}^{*}$ an MST of $\mathcal{G}_{c}^{(u)}$ weighted with $d_c^{(u)}$, and suppose by contradiction that $\tau(\mathcal{T}^{*}) > \tau(\mathcal{G}^{*})$.  By Lemma \ref{lem:symmetric_underlay}, it follows that there is an edge $e_{\mathcal{T}^{*}}$ of $\mathcal{T}^{*}$ such that $d_c^{(u)}(e_{\mathcal{T}^{*}}) = \tau(\mathcal{T}^{*})$. Moreover,  it follows that $\forall e \in \mathcal{E}(\mathcal{G}^{*})$, $d_c^{(u)}(e) \leq  \tau(\mathcal{G}^{*}) < \tau(\mathcal{T}^{*}) = d_c^{(u)}(e_{\mathcal{T}^{*}})$. If we remove $e_{\mathcal{T}^{*}}$ from $\mathcal{T}^{*}$, the two components define a cut of $\mathcal{G}_{c}$. The edge of $\mathcal{G}^{*}$, say $e_{cut}$ belonging to the cut-set is such that $d_c^{(u)}(e_{cut}) < d_c^{(u)}(e_{\mathcal{T}^{*}})$, and this is a contradiction with the cut property satisfied by minimum cost spanning trees.
\end{proof}

\subsection{Proof of Proposition \ref{prop:mctz_nphard}}
\label{sec:appendix_proof_nphard_directed}

\begin{repprop}{prop:mctz_nphard}
	$\MCT$ is NP-hard even when $\mathcal{G}_{c}$ is a complete Euclidean edge-capacitated graph.
\end{repprop}

\begin{proof}
    When $\mathcal{G}_{c}$ is an edge-capacitated graph, $d_{c}(i, j) =  s \times T_{c}(i) + l(i, j) + \frac{M}{A(i', j')}$. $\mathcal{G}_{c}$ is complete and Euclidean means that  $d_{c}(i, j) = d_{c}(j, i)$, for all $(i,j) \in \mathcal{V} \times \mathcal{V}$ and that $d_{c}$ verifies triangular inequality, i.e., $d_{c}(i, j) \leq d_{c}(i, k) + d_{c}(k, j)$, for every $i, j, k \in \mathcal{V}$.
    
    We  consider the decision problem $\MCTD$ associated to the particular case of $\MCT$ when $\mathcal{G}_{c}$ is an Euclidean edge-capacitated graph
    and we prove that it is NP-complete.
    
    \begin{problem}
        \label{mct_decisionn}
        \problemtitle{Euclidean Edge-Capacitated Minimal Cycle Time - Decision ($\MCTD$)}
        \probleminput{A strong digraph $\mathcal{G}_{c}\! =\! (\mathcal{V},\mathcal{E}_{c})$, delays function $d_{c}$ and a real number $\tau_{0}$} 
        \problemquestion{Is there a  strong  spanning subdigraph of $\mathcal{G}_c$ with cycle time at most $\tau_{0}$?}
    \end{problem}

    We  first prove that $\MCTD$ is NP.\footnote{A decision problem is NP if we can verify in a polynomial time that the answer for a given instance is YES.} Several algorithms (e.g., Karp's Algorithm \cite{dasdan98}) determines the cycle time of a given graph in a polynomial time. Thus for a proposed solution of  $\MCTD$, we can compute its cycle time in polynomial time, and we can verify if the graph is strongly connected using for example depth first search. It follows that $\MCTD$ is NP.
    
	To prove that $\MCTD$ is NP-complete, we show that Hamiltonian Cycle ($\HC$) can be reduced in a polynomial time to  $\MCTD$, i.e., $\HC \leq_{p} \MCTD$.
	
	Hamiltonian cycle problem is the following decision problem:
	
	\begin{problem}
			\problemtitle{Hamiltonian Cycle ($\HC$)}
			\probleminput{A  connected (undirected) graph $\mathcal{D} = (\mathcal{V},\mathcal{E})$.}
			\problemquestion{Is there a Hamiltonian cycle in $\mathcal{D}$?}
	\end{problem}
	
	Given an instance of $\HC$ with an undirected graph $\mathcal{D}=(\mathcal{V},\mathcal{E})$, we construct an instance of $\MCTD$ with a complete digraph $\mathcal{G}_{c}=(\mathcal{V}, \mathcal{V}\times \mathcal{V})$,  a real number $\tau_{0} = \frac{N+2}{N}$ where $N$ is the size of $\mathcal{V}$, and delay function $d_{c}$, where for a given arbitrary choice of vertex $v_{0}$, $d_{c}$ is defined as:
	\begin{equation*}
		d_{c}(i,j) = \left\{\begin{matrix*}[l]
		1 & \text{if $\left((i,j) \in \mathcal{E}\right) \land \left(j\not = v_{0}\right) \land \left(i \not= v_{0}\right)$, }\\ 
		2 & \text{if $\left( \left((i,j) \in \mathcal{E}\right) \land \left(\left(j = v_0\right) \lor \left(i = v_0\right)\right) \right)$ $\lor$ $\left( \left((i,j) \notin \mathcal{E} \right)\land\left( j \not = v_0\right){\land (i\not = v_{0})}\right),$}\\
		3 &  \text{if ${((i,j) \notin \mathcal{E}) \land ((j=v_0) \lor (i=v_0))}$}.
		\end{matrix*}\right.
	\end{equation*}
	{
	    The constructed digraph $\mathcal{G}_{c}$ is complete and the delays are symmetric and verify triangular inequality. In fact for three  distinct nodes $i, j$, and $k$ in $\mathcal{V}$, we prove that $d_{c}(i, j) \leq d_{c}(i,k) + d_{c}(k, j)$ by distinguishing three possible cases:
	    \begin{enumerate}
	        \item If  $i \not = v_{0}$ and $j\not = v_{0}$, then $d_{c}(i, j)\leq 2$, but every delay is at least equal to one and then $2\leq d_{c}(i,k) + d_{c}(k,j)$; it follows that $d_{c}(i, j) \leq d_{c}(i,k) + d_{c}(k,j)$.
	        \item If $i=v_{0}$, then $d_c(v_{0}, k) \geq 2$, thus $d_{c}(v_{0}, k) + d_{c}(k, j) \geq 3$. It follows that $d_c(v_{0}, j) \leq 3\leq  d_c(v_0, k) + d_c(k, j)$.
	        \item The case when $j=v_{0}$ is analogous to the case when $i=v_0$.
	    \end{enumerate}
	}
	

	If $\mathcal{D}$ has a Hamiltonian cycle, then the {(directed)} graph induced by this cycle is a strong spanning subdigraph of $\mathcal{G}_{c}$ and its cycle time is $\tau_{\HC} = \frac{1\times(N-2) + 2 + 2}{N} = \frac{N+2}{N} \leq \tau_{0}$.

	If $\mathcal{G}_{c}$ has a strong spanning sub-digraph, say $\mathcal{G}^{*}$, having a cycle time $\tau^{*} \leq \frac{N+2}{N}$, let $C$ be an elementary  circuit of $\mathcal{G}^*$ containing $v_{0}$ (such a circuit always exists because the graph is strongly connected). By definition of cycle time, $\frac{d_{c}(C)}{|C|} \leq \tau^{*} = 1 + \frac{2}{N}$ .  We are going to prove that $C$ is a Hamiltonian cycle of $\mathcal D$.
	
	We prove first by contradiction that $C$ contains only the arcs from $\mathcal{E}$. Suppose by contradiction that there exists an arc $(i,j) \notin \mathcal{E}$ in $C$, two cases are possible: 
	\begin{enumerate}
	   \item If $j \not = v_{0}$, and $i \not = v_{0}$ then  ${d_{c}(i,j) = 2}$ and since $v_{0} \in C$, there exist two nodes $v^{-}_{0} \in \sigma(v_{0})$ and $v^{+}_{0} \in \pi(v_{0})$ in $C$. It follows that  $d_{c}(C) \geq d_{c}(i, j)  + d(v^{+}_{0},v_{0}) + d_{c}(v_{0},v^{-}_{0}) + 1\times (|C| - 3) \geq 2 + 2 + 2 + |C| - 3 = |C| + 3 $. Since $C$ is an elementary circuit, it follows that $|C| \leq N$, thus $\frac{d_{c}(C)}{|C|} \geq 1 + \frac{3}{N}$, and this contradicts $\frac{d_{c}(C)}{|C|} \leq 1 + \frac{2}{N}$. 
	   \item If $i = v_{0}$, let $v_{0}^{+}$ be the predecessor of $v_{0}$ in $C$, it follows that $d_{c}(C) \geq d_{c}(v_0^+, v_{0}) + d(v_{0}, j) + 1 \times (|C| - 2) \geq 3 + 2 + |C| -2 = 3 + |C|$, thus $\frac{d_{c}(C)}{|C|} \geq 1 + \frac{3}{|C|}$, and using the same argument as for the first case we get a contradiction. 
	   \item The case when $j = v_0$ is analogous to the case when $i=v_0$.
	\end{enumerate}
	
	It follows that any arc of $C$ is in $\mathcal{E}$. 
	
	We prove next that $C$ is a Hamiltonian Cycle, i.e., $|C| = N$. Since $v_{0} \in C$, there exist two nodes $v^{+}_{0} \in \sigma(v_{0})$ and $v^{-}_{0} \in \pi(v_{0})$ in $C$, it follows that $d_{c}(C) = d_{c}(v^{-}_{0}, v_{0}) + d_{c}(v_{0}, v^{+}_{0}) + 1 \times (|C| - 2) = 2 + 2 + |C| - 2  = 2 + |C|$.
	
	Since $\frac{d_{c}(C)}{|C|} \leq \tau^{*} = 1 + \frac{2}{N}$,  it follows that $1 + \frac{2}{|C|} \leq 1 + \frac{2}{N}$, thus $|C| \geq N$. As $C$ is an elementary circuit it follows that  $|C| = N$, i.e., $C$ is a Hamiltonian cycle. Since $C$ is a circuit containing only arcs from  $\mathcal{D}$, it follows that $\mathcal{D}$ has a Hamiltonian cycle.
	
	So we have proved that $\mathcal{D}$ has a Hamiltonian cycle if and only if $\mathcal{G}_{c}$ has strong spanning subdigraph of cycle time at most $\tau_{0} = \frac{N+2}{N}$. It follows that $\MCTD$ is NP-complete, thus $\MCT$ is NP-hard even when $\mathcal{G}_{c}$ is a complete Euclidean edge-capacitated graph.
\end{proof}

\subsection{Proof of Proposition \ref{prop:mctz_3-approx}}
\label{sec:appendix_proof_3_approx}

Under the assumption that the connectivity topology is Euclidean (delays are symmetric and verify triangular inequality), we first show that the solution of Travelling Salesman Problem ($\TSP$) \cite{gutin2006traveling} is guaranteed to be within a $2N$-multiplicative factor of the solution of $\MCT$ (Lemma~\ref{lem:tsp_mct}).  As a result, the Christofides algorithm~\cite{monnot:hal-00003997}  which is a 1.5-approximation algorithm for $\TSP$, is a $3N$-approximation algorithm for $\MCT$ (Prop.~\ref{prop:mctz_3-approx}).

\begin{lem}
    \label{lem:tsp_mct}
    Consider an Euclidean digraph $\mathcal{G}_{c}$ with $N$ nodes and let $\mathcal H^{*}$ denote its optimal tour. Then $\frac{d_{c}(\mathcal H^{*})}{|\mathcal H^{*}|} \leq 2 N \times \tau_{*}$, where $\tau_{*}$ is the optimal cycle time that can be achieved by a strong spanning subdigraph of $\mathcal{G}_{c}$. 
\end{lem}

\begin{proof}
    Let $\mathcal{G}^{*}$ be a spanning digraph of $\mathcal{G}_{c}$ with optimal cycle time $\tau^{*}$.
    
    Let  $\left\{\mathcal C_{i} \right\}_{i=1, \dots, c}$ be a minimal set of elementary circuits of $\mathcal{G}_{*}$, so that $\cup_{i=1}^c \mathcal C_{i} = \mathcal{G}_{*}$ and $\cup_{i \neq j} \mathcal C_{i} \neq \mathcal{G}_{*}$ for each $j$ {(as illustrated in Fig.~\ref{fig:ring_approximation_proof_decomposition})}. Consider an auxiliary graph $\mathcal G'$ whose $c$ nodes represent the $c$ circuits and whose links correspond to two circuits sharing a node. Let $\mathcal T$ be a spanning tree of $\mathcal G'$. Starting from the root of $\mathcal T$, we can define an order of the nodes in each circuit and an order of the children of each circuit as follows. Given the orientation of the circuit corresponding to the root, consider the first node they share with each child. We order the children according to such order (solving arbitrarily possible ties). For each child we reorder its nodes starting from the node they share with the father and following the orientation of the circuit. We consider then the ordered traversal of the circuits $\Gamma= (\mathcal C_{i_1}, \mathcal C_{i_2},\dots, \mathcal C_{i_{2c+1}} = \mathcal  C_{i_{1}})$ obtained using DFS on $\mathcal T$ and visiting the children according to the order introduced above {(as illustrated in Fig.~\ref{fig:ring_approximation_proof_ordering})}. 
    
    From $\Gamma$ we can build two closed walks $\mathcal W_1$ and $\mathcal W_2$, both spanning all nodes of $\mathcal G^*$. The walk $\mathcal W_1$ is built by considering all circuits in the order they appear in $\Gamma$, and then  concatenating their nodes  as follows. The first time we visit one circuit we take all nodes in the circuit in their order (but the last one in each circuit that coincides with the first one). When we come back to the circuit, we only pick the nodes needed to move to the following circuit in $\Gamma$. The walk $\mathcal W_2$ is built by considering the $c$ circuits in the order they first appear in $\Gamma$, and then again concatenating their nodes (but the last one in each circuit that coincides with the first one). Both sequences of nodes define walks as $\mathcal G_c$ is Euclidean and then complete. The length of $\mathcal W_2$ is $|\mathcal W_2| = \sum_{i=1}^c |\mathcal C_i| \le N^2$, as we can have at most $N-1$ elementary circuits and each of them has length at most $N$. {See Figs.~\ref{fig:ring_approximation_proof_walk_1} and~\ref{fig:ring_approximation_proof_walk_2} for the examples of $\mathcal W_1$ and $\mathcal W_2$.
    }
    
    \begin{figure}
        \begin{subfigure}[b]{0.49\textwidth}
            \centering
            \includegraphics[scale=0.32]{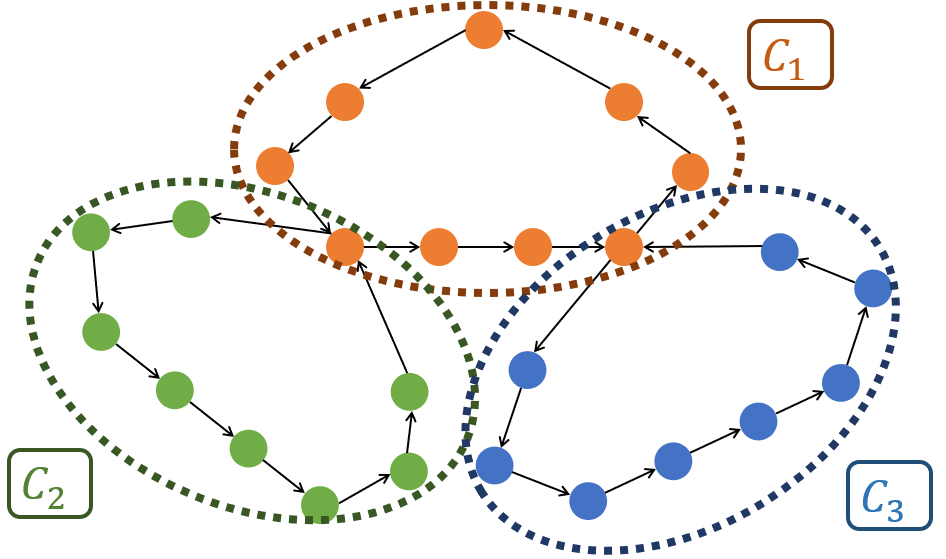}
            \caption{Circuits decomposition}
            \label{fig:ring_approximation_proof_decomposition}
        \end{subfigure}
        \hfill
        \begin{subfigure}[b]{0.49\textwidth}
            \centering
            \includegraphics[scale=0.32]{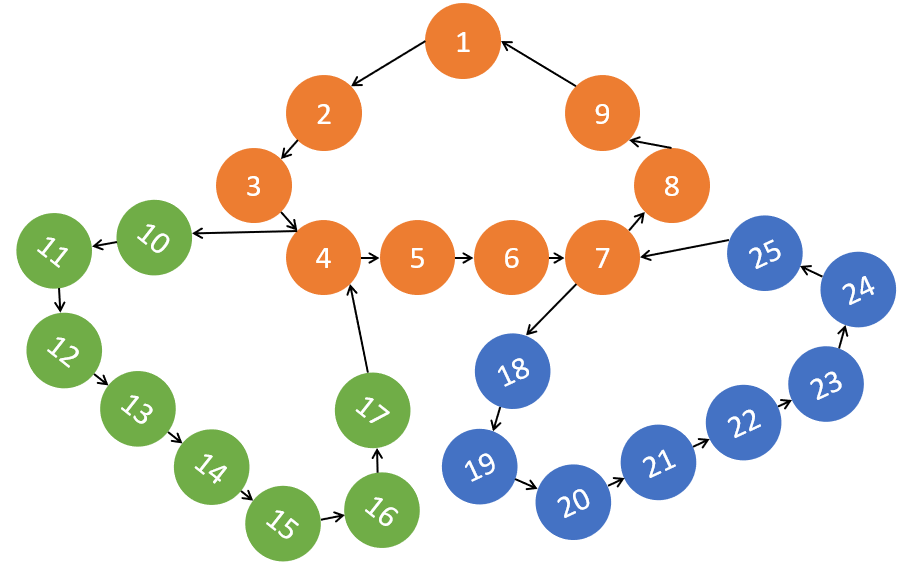}
            \caption{Nodes ordering}
            \label{fig:ring_approximation_proof_ordering}
        \end{subfigure}
        
        \begin{subfigure}[b]{0.49\textwidth}
            \centering
            \includegraphics[scale=0.32]
            {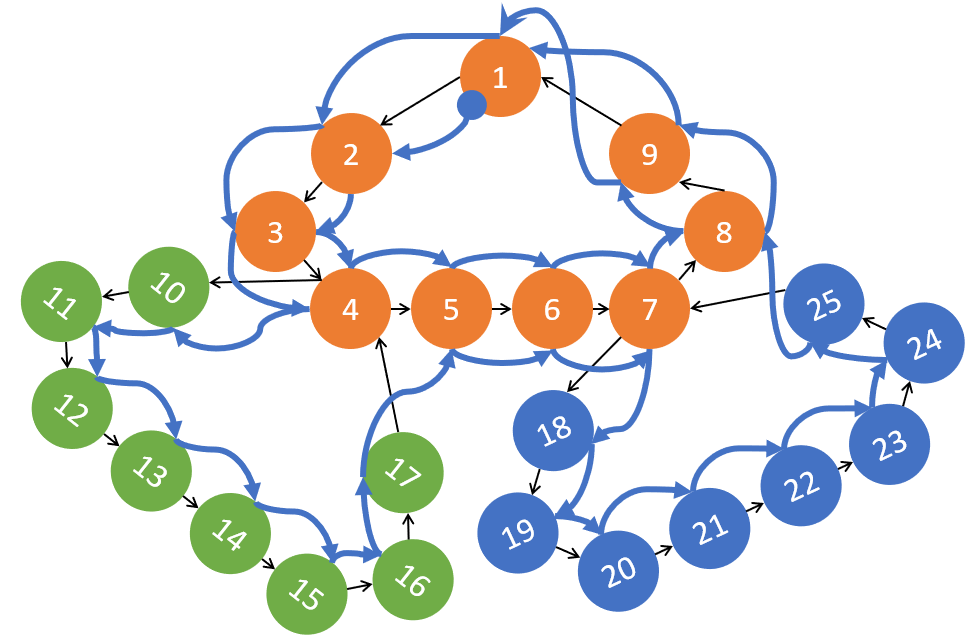}
            \caption{Walk $\mathcal{W}_{1}$}
            \label{fig:ring_approximation_proof_walk_1}
        \end{subfigure}
        \hfill
        \begin{subfigure}[b]{0.49\textwidth}
            \centering
            \includegraphics[scale=0.32]{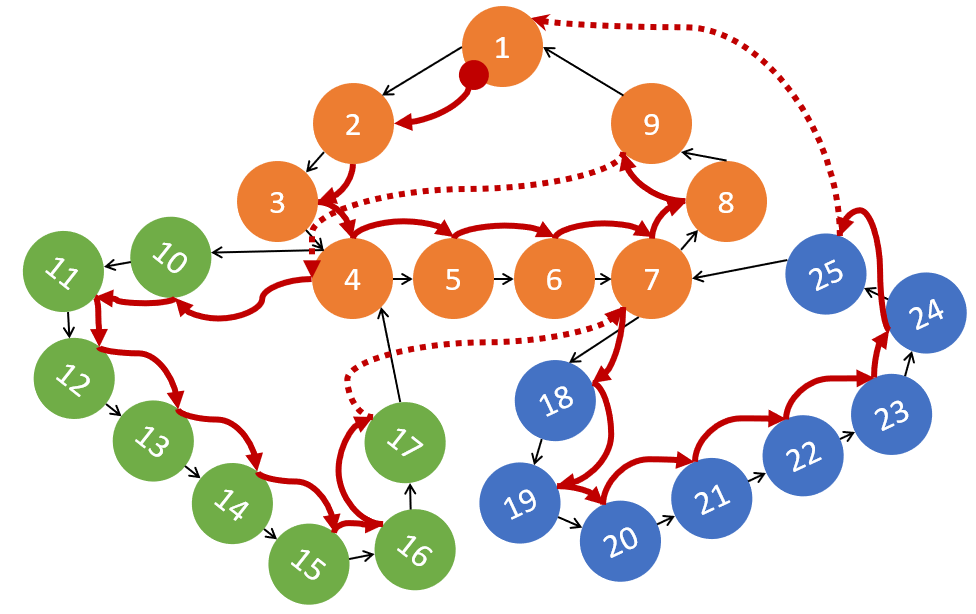}
            \caption{Walk $\mathcal{W}_2$}
            \label{fig:ring_approximation_proof_walk_2}
        \end{subfigure}
        \caption{Illustration of building walks used in the proof of Lemma~\ref{lem:tsp_mct}. }
        \label{fig:ring_approximation_proof}
    \end{figure}
    
    We observe that $d_c(\mathcal W_1) \le 2 \sum_{i=1}^c d_c(\mathcal C_i)$ as the walk $\mathcal W_1$ passes through each link in each circuit $\mathcal C_i$ at most twice: it walks through the first $|\mathcal C_i|-1$ edges of $\mathcal C_i$ the first time it visits $\mathcal C_i$, and uses once more the edges in $\mathcal C_i$ to visit the other circuits and go back to the root. As $\mathcal W_2$ is a sublist of the nodes in $\mathcal W_1$ and delays satisfy the triangle inequality, it holds $d_c(\mathcal W_2) \le d_c(\mathcal W_1)$.
    
    Finally, from the walk $\mathcal W_2$ we can extract a Hamiltonian cycle $\mathcal H$ that has an even smaller delay. Let $\mathcal H^*$ be an optimal tour. It follows
    
    \begin{align}
        \tau(\mathcal H^*) & = \frac{d_c(\mathcal H^*)}{|\mathcal H^*|} \le \frac{d_c(\mathcal H)}{|\mathcal H^*|} \le \frac{d_c(\mathcal W_2)}{|\mathcal H^*|} \\
        & = \frac{|\mathcal W_2|}{|\mathcal H^*|} \frac{d_c(\mathcal W_2)}{|\mathcal W_2|}\\
        & \le \frac{N^2}{N} \frac{d_c(\mathcal W_1)}{\sum_{i=1}^c |\mathcal C_i|}\\
        & \le 2 N \frac{\sum_{i=1}^c d_c(\mathcal C_i)}{\sum_{i=1}^c |\mathcal C_i|}\\
        & \le 2 N \max_{i=1,\dots, c} \frac{ d_c(\mathcal C_i)}{|\mathcal C_i|} = 2 N \tau^*.
    \end{align}
\end{proof}

\begin{repprop}{prop:mctz_3-approx}
Christofides' algorithm~\cite{monnot:hal-00003997} is a $3N$-approximation algorithm for $\MCT$ when $\mathcal G_c$ is edge-capacitated and Euclidean.
\end{repprop} 

\begin{proof}
Christofides algorithm provides a $\frac{3}{2}$-approximation for the traveling salesman problem $\TSP$ defined in \cite{10.5555/1374811}.\footnote{See \cite{monnot:hal-00003997} for the proof.} Given an instance of $\MCT$ let $\hat{C}$ denote  the output of Christofides algorithm and  $C^{*}$ denote the optimal tour of $\mathcal{G}_{c}$. It follows that $d_{c}(\hat{C}) \leq \frac{3}{2} d_{c}(C^{*})$. Since both $\hat{C}$ and $C^{*}$ are Hamiltonian cycles, $|\hat{C}| = |C^{*}|$. Using Lemma \ref{lem:tsp_mct}. it follows that $\frac{d_{c}(\hat{C})}{|\hat{C}|} \leq 2N \times \frac{3}{2} \times \tau_{*} = 3N\times \tau_{*}$. Thus the graph obtained using only the edges of $\hat{C}$ is a $3N$-approximation of the $\MCT$ problem when $\mathcal G_c$ is edge-capacitated and Euclidean.
\end{proof}

\begin{observation}
\label{ob:omegaN_tight}
Christofides' algorithm~\cite{monnot:hal-00003997} is a $\Omega(N)$-approximation algorithm for $\MCT$ when $\mathcal G_c$ is edge-capacitated and Euclidean.
\end{observation}

\begin{proof}
    Christofides' algorithm returns a ring as solution.
    We provide an example of an Euclidean underlay where any ring has cycle time at least $N/4$ times larger than the optimal overlay. We consider a complete connectivity graph $\mathcal{G}_{c} = \left(\mathcal{V}, \mathcal{V} \times \mathcal{V}\right)$ to which we associate a delay function $d_{c}$ verifying

\begin{equation}
    \forall (i, j) \in \mathcal{V} \times \mathcal{V};~~ d_c(i,j) = \left\{\begin{matrix*}[l] 0 &\text{if} ~~ i,j \in \{1, \dots, N\},
\\  1 &\text{if} ~~ i \in \{N+1, \dots, 2N\} ~~ \text{or} ~~ j \in \{N+1, \dots, 2N\}. 
\end{matrix*}\right.
\end{equation}

$\mathcal{G}_{c}$ is clearly an Euclidean graph.

A Hamiltonian cycle $\mathcal{H}$ of $\mathcal{G}_{c}$ needs to use exactly $2N$ different edges and in particular $N$ different edges with delay $1$ to connect nodes $i\in \{N+1, \ldots,2N\}$. Therefore, the total delay of the cycle is at least $N \times 0 + N \times 1 = N$, and its cycle time $\tau(\mathcal{H}) \geq \frac{N}{2N} = \frac{1}{2}$.

Consider a directed overlay $\mathcal{G}_{o} = \left(\mathcal{V}, \mathcal{E}_{o}\right)$, with

\begin{equation}
    \mathcal{E}_{o} = \left\{(i, i+1); ~i \in \{1, \dots, N-1\}\right\} \cup \underset{K \in \{N+1, \dots, 2N\}}{\bigcup}\left\{(N, K), (K, 1)\right\}.
\end{equation}

The set of elementary circuits of $\mathcal{E}_{o}$ is exactly the set \[\mathcal{C} = \left\{C_{K} = (1,\dots, N, K, 1): K \in \{N+1, 2N\} \right\}.\]
For any circuit $C_{K} \in \mathcal{C}$,  \[\tau(C_{K}) = \frac{0 \times (N-1) + 2 \times 1}{N+1} = \frac{2}{N+1}.\]

It follows that the minimal cycle time $\tau_{\text{OPT}} = \frac{2}{N+1}$, and $\tau(\mathcal{H}) \geq \frac{N+1}{4} \tau_{\text{OPT}}$ for any Hamiltonian cycle $\mathcal{H}$ of $\mathcal{G}_{c}$. 
\end{proof}

\subsection{Proof of Proposition \ref{prop:mctu_nphard}}
\label{sec:appendix_proof_mct_nphard}
We prove that in a node-capacitated network, $\MCT$ is NP-hard even when $\mathcal{G}_{o}$ is required to be undirected. We start introducing the associated decision problem:
\begin{problem}
		\problemtitle{$\MCTU$-Decision}
		\probleminput{A strongly connected directed graph \mbox{$\mathcal{G}_{c} = (\mathcal{V},\mathcal{E}_{c})$}, model size $M$,\qquad $\{\upcap(i), \dncap(j), l(i,j), A(i',j'), T_c(i), \forall (i,j) \in \mathcal E_c\}$, and a constant $\tau_{0} > 0$.}
		\problemquestion{Is there a strong spanning undirected subgraph $\mathcal{G}_{o}$ of $\mathcal{G}_{c}$, such that $\tau(\mathcal{G}_{o}) \leq \tau_{0}$?}
\end{problem}
$\MCTU$-Decision is closely related to  the \emph{degree-constrained spanning tree} ($\DCST$) defined below:
\begin{problem}
	\problemtitle{Degree-constrained spanning tree ($\DCST$)}
	\probleminput{ An $N$-node connected undirected graph $\mathcal{G} = (\mathcal{V},\mathcal{E})$; positive integer $k\leq N$.}
	\problemquestion{Does $\mathcal{G}$ have a spanning tree in which no node has degree greater than $k$?}
\end{problem}
\DCST{} is a simpler version of $\delta$-\MBST, where we look for a spanning tree with degree at most $k$ and minimum bottleneck.

$\DCST$ is NP-complete~\cite{garey1979computers}.
For example for $k=2$ it can be shown by a reduction from \HC.

\begin{repprop}{prop:mctu_nphard}
    In node-capacitated networks $\MCT$ is NP-hard even when the overlay is required to be undirected.
\end{repprop}
\begin{proof}
    Our proof is based on a reduction of \DCST{} to \MCTU-Decision.
 
    Given an instance of $\DCST$ with an $N$-node connected undirected graph $\mathcal{G}=(\mathcal{V}, \mathcal{E})$ and a positive integer $k \leq N$, we define an instance of $\MCTU$-Decision on a connected graph $\mathcal{G}_{c} = (\mathcal{V}_{c}, \mathcal{E}_{c})$ built from $\mathcal{G}$ according to the following mapping $\Pi$: 
    For each node $v$ in $\mathcal{V}$, there are two nodes $v^{(1)}$ and $v^{(2)}$ in $\mathcal{V}_{c}$ and $(v^{(1)}, v^{(2)}) \in \mathcal{E}_{c}$, and for an arc $(v_{i}, v_{j}) \in \mathcal{E}$, there is an arc $(v_{i}^{(1)}, v_{j}^{(1)})$ in $\mathcal{E}_{c}$. We set $\frac{M}{C_{\text{UP}}(v^{{(1)}})} = 1$, $\frac{M}{C_{\text{UP}}(v^{{(2)}})} = k+1$ for all $v \in \mathcal V$,  $T_c(i)=0$, $C_{DN}(i) = \infty$ for all $i \in \mathcal V_c$, and $l(i,j)=0, A(i',j') = \infty$ for all $(i,j) \in \mathcal E_c$ . Finally, we consider $\tau_0 = k + 1$.

    Suppose that $\mathcal{G}$ has a spanning tree $\mathcal{T} = (\mathcal{V}, \mathcal{E}_{\mathcal{T}})$ in which no node has degree greater than $k$,  and denote $\mathcal{T}_{c} = \Pi(\mathcal{T})$ (i.e., we apply the same mapping described above).  $\mathcal{T}_{c}$ is a spanning tree of $\mathcal{G}_{c}$ (it is acyclic and spans all  nodes of $\mathcal{G}_{c}$). All elementary circuits of $\mathcal{T}_{c}$ are either of the form $(v_{i}^{(1)}, v_{i}^{(2)}, v_{i}^{(1)})$ for some $v_{i} \in \mathcal{V}$, or of the form $(v_{i}^{(1)}, v_{j}^{(1)}, v_{i}^{(1)})$ for some $(v_{i}, v_{j}) \in \mathcal{E}_{\mathcal{T}}$. Moreover, $\tau((v_{i}^{(1)}, v_{i}^{(2)}, v_{i}^{(1)}))  = \frac{k + 1 + \texttt{degree}_{\mathcal{T}}(v_{i}) + 1}{2} \leq k+1$ and   $\tau((v_{i}^{(1)}, v_{j}^{(1)}, v_{i}^{(1)})) =  \frac{\texttt{degree}_{\mathcal{T}}(v_{i}) + 1 + \texttt{degree}_{\mathcal{T}}(v_{j}) + 1}{2} \leq k+1$.
    It follows that $\tau(\mathcal{T}_{c}) \leq k+1 = \tau_{0}$.

    Inversely, suppose that $\mathcal{G}_{c}$ has an MST $\mathcal{T}_{c}$ having a cycle time at most $\tau_{0}$, and let $\mathcal{T} = \Pi^{-1}(\mathcal{T}_{c})$, where $\Pi^{-1}(\mathcal{T})$ is obtained by deleting all the vertices of the form $v^{(2)}_{i}$ for $v_{i} \in \mathcal{V}$. $\mathcal{T}$ is a spanning tree of $\mathcal{G}$ (it contains all  nodes of $\mathcal{G}$ and is acyclic). We prove by contradiction that $\texttt{degree}(\mathcal{T}) \leq k$. Suppose that there exists a node $v \in \mathcal{V}$ such that $|\mathcal{N}^{-}_{v}(\mathcal{T})| > k$, it follows that circuit $\{v_{i}^{(1)}, v_{i}^{(2)}, v_{i}^{(1)}\}$ is a circuit of  $\mathcal{T}_{c}$, and $\tau((v_{i}^{(1)}, v_{i}^{(2)}, v_{i}^{(1)})) = \frac{k + 1 + |\mathcal{N}^{-}_{v}(\mathcal{T})| +1 }{2} > k+1$. It follows that $\tau(\mathcal{T}_c) > k+1$, thus $ k + 1  < \tau_{0} = k + 1$ (contradiction). 
    
    Then the answer to $\DCST$ is positive if and only if the answer to $\MCTU$-Decision is positive. In addition, we have a polynomial reduction algorithm. It follows that $\MCTU$-Decision is NP-hard.
\end{proof}

\subsection{Proof of Proposition \ref{prop:6-approx}}
\label{sec:appendix_proof_6_approx}
The bottleneck of a tree $\mathcal T$ is its maximum edge weight, denoted by $B(\mathcal T)$.
To prove Prop.~\ref{prop:6-approx}, we start by proving that the bottleneck of the $\MBST$ of the undirected graph $\mathcal{G}^{(u)}_c$ (considered in lines~\ref{algoline:gprime}-\ref{algoline:d} of Algo.~\ref{alg:our_algorithm}) is smaller than or equal to the minimal cycle time of the connectivity graph $\mathcal{G}_c$. 

We consider a node-capacitated case where  $\upcap(i)\leq \min\left(\frac{\dncap(j)}{N}, A(i',j') \right)$, $\forall (i,j)\in \mathcal{E}_c$. Thus, according to $\eqref{eq:do}$, the overlay $\mathcal{G}_o$ has weights
\begin{align}
    d_o(i,j) & =  s \times T_c(i) + l(i,j) + \frac{M |\mathcal{N}^-_i|}{\upcap(i)},~~\forall (i,j)\in \mathcal{E}_c.
    \label{eq:weight_6_approx}
\end{align}

Note that the weights defined for the undirected graph $\mathcal{G}^{(u)}_c = (\mathcal{V}, \mathcal{E}^{(u)}_c)$ are
\begin{equation}
      d_c^{(u)}(i, j) = \frac{s\times (T_c(i)+T_c(j)) + l(i, j) + l(j, i) + \frac{M}{\upcap(i)} + \frac{M}{\upcap(j)}}{2},~~\forall  (i,j) \in \mathcal{E}^{(u)}_c.
      \label{eq:built_weight_6_approx}
\end{equation}

\begin{lem}
    \label{lem:bottleneck_vs_mct}
    Consider the case where $\mathcal{G}_c$ is node-capacitated  with  $\upcap(i)\leq \min\left(\frac{\dncap(j)}{N}, A(i',j') \right)$, $\forall (i,j)\in \mathcal{E}_c$, and the overlay is required to be undirected. Let  $\tau^{*}(\mathcal{G}_c)$ be the cycle time of $\MCT$ on $\mathcal{G}_{c}$ and  $\mathcal{T}_{MBST}(\mathcal{G}^{(u)}_c)$ be the $\MBST$ of $\mathcal{G}^{(u)}_c$. The bottleneck of $\mathcal{T}_{MBST}(\mathcal{G}^{(u)}_c)$ is smaller than or equal to $\tau^{*}(\mathcal{G}_c)$, i.e.~ $B(\mathcal{T}_{MBST}(\mathcal{G}^{(u)}_c)) \leq \tau^{*}(\mathcal{G}_{c})$. 
\end{lem}

\begin{proof}
    Denote $\mathcal{T}^*(\mathcal{G}_c)$ the undirected overlay of $\mathcal{G}_c$ with minimal cycle time. We consider the edge \[(w, v) =  \argmax_{(i,j) \in \mathcal{E}(\mathcal{T}^*(\mathcal{G}_c))}d_c^{(u)}(i,j).\] By definition, $B(\mathcal{T}_{MBST}(\mathcal{G}^{(u)}_c)) = \min_{\mathcal{T}\in ST(\mathcal{G}^{(u)}_c)} \max_{(i,j)\in \mathcal{E}(\mathcal{T})} d_c^{(u)}(i,j)$, where $ST(\mathcal{G}^{(u)}_c)$ is the set of spanning trees of $\mathcal{G}^{(u)}_c$. Since  $\mathcal{T}^*(\mathcal{G}_c)\in ST(\mathcal{G}^{(u)}_c)$,  we have:
    \begin{align}
        B(\mathcal{T}_{MBST}& (\mathcal{G}^{(u)}_c))  \leq  d_c^{(u)}(w,v) \nonumber \\
        & \displaystyle_{=}^{\eqref{eq:built_weight_6_approx}} \frac{s\times (T_{c}(w)+T_{c}(v)) + l(w,v) + l(v,w) + M/\upcap(w) + M/\upcap(v)}{2} \nonumber \\
        & \le \frac{s\times (T_{c}(w) + T_{c}(v)) + l(w,v) + l(v,w) + |\mathcal{N}^-_w| M/ \upcap(w)+ |\mathcal{N}^-_v| M/ \upcap(v)}{2} \nonumber \\
        & \displaystyle_{=}^{\eqref{eq:weight_6_approx}} \frac{d_o(w,v) + d_o(v,w)}{2} \nonumber \\
        & \le \tau^{*}(\mathcal{G}_{c}), \nonumber
    \end{align}
    where the second inequality follows from $|\mathcal{N}^{-}_{w}|,|\mathcal{N}^{-}_{v}|  \geq 1$, and the last inequality comes from the definition of cycle time. 
\end{proof}

{
Lemma~\ref{lem:bottleneck_vs_mct} establishes a connection  between the bottleneck of the $\MBST$ of $\mathcal{G}^{(u)}_c$ and the cycle time of $\MCT$ on $\mathcal{G}_c$ when the overlay is required to be undirected. 
To get an approximated $2$-$\MBST$ on $\mathcal{G}^{(u)}_c$, we apply the best known $3$-approximation algorithm from~\cite[Sect.~3.2.1]{doi:10.1002/net.21710} (see lines~\ref{algoline:2BSTstart}-\ref{algoline:2BSTend} in Algo.~\ref{alg:our_algorithm})  which requires $\mathcal{G}^{(u)}_c$ to be Euclidean. So in the following, we show that indeed $\mathcal{G}^{(u)}_c$ is Euclidean.
}

\begin{lem}
    \label{lem:g_prime_is_euclidean}
    If $\mathcal{G}_c$ is Euclidean, then $\mathcal{G}^{(u)}_c$ is Euclidean.
\end{lem}

\begin{proof}
    Remind that the connectivity graph $\mathcal{G}_c$ is Euclidean on a node-capacitated network, if its delays $d_c(i,j) = s\times T_c(i) + l(i,j)$ are symmetric ($d_c(i,j) = d_c(j,i), \forall i,j\in \mathcal{V})$ and satisfy the triangle inequality. From~$\eqref{eq:built_weight_6_approx}$ it is easy to check that $d_c^{(u)}(i,j) = d_c^{(u)}(j,i)$.  Consider three nodes $i, j, k \in \mathcal{V}$, we have:
    \begin{align*}
        d_c^{(u)}(i,j) &= \frac{d_c(i,j) + d_c(j,i) + M/\upcap(i) + M/\upcap(j)}{2}\\
        &\le \frac{d_c(i,k) + d_c(k,j) + d_c(j,k) + d_c(k,i) + M/\upcap(i) + M/\upcap(j)}{2}\\
        &\le \frac{d_c(i,k) + d_c(k,j) + d_c(j,k) + d_c(k,i) + M/\upcap(i) + M/\upcap(j) + 2 M/\upcap(k)}{2}\\
        &= d_c^{(u)}(i,k) + d_c^{(u)}(k,j),
    \end{align*}
    where the first inequality follows from the triangle inequality for $d_c(i,j)$ and the second inequality from $\upcap(k) \geq 0$.
\end{proof}

\begin{repprop}{prop:6-approx}
    Algorithm~\ref{alg:our_algorithm} is a $6$-approximation algorithm for $\MCT$ when $\mathcal{G}_c$ is node-capacitated and Euclidean with   $\upcap(i)\leq \min\left(\frac{\dncap(j)}{N}, A(i',j') \right)$, $\forall (i,j)\in \mathcal{E}_c$, and $\mathcal{G}_o$ is required to be undirected.
\end{repprop}

\begin{proof}
    Algorithm~\ref{alg:our_algorithm} considers, as a candidate solution, an opportune Hamiltonian path $\mathcal{H}$ (line~\ref{algoline:2BSTend}) for which reference~\cite[Thm.~8]{doi:10.1002/net.21710} proves that 
    \begin{equation}
        B(\mathcal{H}) \leq 3 \times B(\mathcal T_{MBST}(\mathcal{G}^{(u)}_c))
        \label{eq:MBST}
    \end{equation}
    as $\mathcal{G}^{(u)}_c$ is Euclidean (Lemma~\ref{lem:g_prime_is_euclidean}). Moreover,
    \begin{align}
        \label{eq:xyz}
        \tau(\mathcal{H}) & = \max_{(i,j) \in \mathcal{E}(\mathcal{H})} \frac{d_o(i,j) + d_o(j,i)}{2} \nonumber\\
        & = \max_{(i,j) \in \mathcal{E}(\mathcal{H})}\frac{s \times T_{c}(i) + s\times T_{c}(j) + l(i,j) + l(j,i) + \frac{M |\mathcal{N}^-_i|}{\upcap(i)} + \frac{M |\mathcal{N}^-_j|}{\upcap(j)}}{2}\nonumber\\
        & \leq \max_{(i,j) \in \mathcal{E}(\mathcal{H})}\frac{s \times T_{c}(i) + s\times  T_{c}(j) + l(i,j) + l(j,i) + 2\frac{M}{\upcap(i)} + 2\frac{M}{\upcap(j)}}{2}\nonumber\\
        & \leq \max_{(i,j) \in \mathcal{E}(\mathcal{H})}s \times T_{c}(i) + s\times  T_{c}(j) + l(i,j) + l(j,i) + \frac{M}{\upcap(i)} + \frac{M}{\upcap(j)}\nonumber\\
        & = 2 \max_{(i,j) \in \mathcal{E}(\mathcal{H})} d_c^{(u)}(i,j)\nonumber\\
        & = 2  B(\mathcal H),
    \end{align}
    where the first inequality follows from nodes in a path having degree at most 2. Combining $\eqref{eq:MBST}$, $\eqref{eq:xyz}$, and Lemma~\ref{lem:bottleneck_vs_mct}, it follows that $\tau(\mathcal{H}) \leq 6 \times \tau^{*}(\mathcal{G}_{c})$.
\end{proof}

\subsection{Proof of Proposition \ref{prop:mct_3_approx}}

\begin{repprop}{prop:mct_3_approx}
    Christofides' algorithm is a $3N$-approximation algorithm for $\MCT$ when $\mathcal{G}_c$ is node-capacitated and Euclidean.
\end{repprop} 

\begin{proof}
    Let $\mathcal{G}'_c$ be a weighted graph with the same topology as $\mathcal{G}_c$ with weights $d'(i,j) = s \times T_c(i) +  l(i,j)+\frac{M}{\min\left(\upcap(i), \dncap(j), A(i', j') \right)}$. Denote $\hat{C}$ the output of Christofides' algorithm when used on $\mathcal{G}'_c$, and denote $C^{*}$ the optimal tour of $\mathcal{G}'_{c}$. Since Christofides' algorithm provides a $\frac{3}{2}$-approximation to $\TSP$, it follows that $d'(\hat{C}) \leq \frac{3}{2} d'(C^*)$. As $\hat C$ and $C^*$ are directed rings, it holds $d'(\hat C)= d_o(\hat C)$ and $d'(C^*) = d_o(C^*)$. Using Lemma \ref{lem:tsp_mct} it follows that \[\tau(\hat C) = \frac{d_o(\hat{C})}{|\hat{C}|} =  \frac{d'(\hat{C})}{|\hat{C}|}\leq \frac{3}{2} \frac{d'(C^*)}{|C^*|} = \frac{3}{2} \frac{d_o(C^*)}{|C^*|} = \frac{3}{2} \tau(C^*)  \le 3 N \tau^*.\] Thus the graph obtained using only the edges of $\hat{C}$ is a $3N$-approximation algorithm for $\MCT$ when $\mathcal{G}_c$ is node-capacitated and Euclidean.
\end{proof}

\newpage
\section{Time Simulator}
\label{sec:appendix_time_simulator}
The time simulator reconstructs the wall-clock time. It requires the complete knowledge about the underlay topology, i.e., the capacities of all physical links and the upload and download capacities for each silo. For a given overlay topology $\mathcal{G}_{o} = (\mathcal{V}, \mathcal{E}_{o})$, the purpose of the proposed time simulator (Alg. \ref{alg:time_simulator}) is to compute $t(k) = \left(t_{i}(k)\right)_{1\leq i \leq N}$, i.e., the time at which each silo starts computing for the $k$-th time. The simulator needs to compute the delay required to send a message with a known size on each physical link of the underlay. This delay is the sum of two terms~\cite{geant_doi:10.1002/dac.645}:

\begin{itemize}
    \item Latency: it is the time required by the first transmitted bit to travel from the source to the destination. The latency of a link $(i, j)$ essentially depends on the length of the link and the speed of the light in the link's transmission medium. We have estimated the latency using the formula proposed in  \cite{10.1145/1028788.1028828}: $0.0085 \times \textrm{distance}(i, j)+ 4$, where the distance is expressed in kilometers and the latency in milliseconds. The latency of a path is the sum of the link latencies.
    \item Transmission Delay: it is the time between the reception of the first bit of the message and the reception of the last bit. It depends on the minimum available bandwidth along the path. We compute it as $M/\min\left(\frac{\upcap(i)}{|\mathcal N^-_i|},\frac{\dncap(j)}{|\mathcal N^+_j|}, A(i',j')\right)$.
\end{itemize}

Finally, the simulator also accounts for the total time spent in computation by each node, that is the product of the number of local steps $s$ and the time needed to perform one local step (in milliseconds), i.e., $s \times T_{c}(i)$.

\IncMargin{2.0em}
\begin{algorithm}[H]
    \label{alg:time_simulator}
    \SetAlgoLined
    \SetKwInOut{Input}{Input} 
    \SetKwFunction{Function}{compute\_round\_end\_times}
    \SetKwProg{Fn}{Function}{:}{}

    \Indm\Indmm
    \Input{$\left(l_{i,j}\right)_{(i,j) \in \mathcal{G}_{o}}$, $\left(T^{c}_{i}\right)_{i \in \mathcal{V}}$, $\left(\dncap(i)\right)_{i \in \mathcal{V}}$ and $\left(\upcap(i)\right)_{i \in \mathcal{V}}$}
    \KwResult{$t \in \mathbb{R}^{N\times K}$}

    \Indp\Indpp

    \For{ $i \in \mathcal{V}$}{
        $t_{i}(0) = 0$\;
    }

    \For{ $k \in \{1, \dots, K\}$}{
        $t_i(k)= \max_{j \in \mathcal N^+_i} \left(t_j(k-1) + l(i, j) + \frac{M}{\min\left(\frac{\upcap(i)}{|\mathcal N^-_i|},\frac{\dncap(j)}{|\mathcal N^+_j|}, A(i',j')\right)}\right)$\;
        $t_{i}(k)  = t_{i}(k) + s \times T_{c}(i)$\;
    }
    \caption{Time Simulator}
\end{algorithm}

\newpage
\section{Experiments Detailed Description}
\label{sec:appendix_exp_details}
\subsection{Networks and Communication model}
\label{app:networks}
\nocite{SciPyProceedings_11}
We considered three real topologies from \emph{Rocketfuel engine}~\cite{rocketfuel_10.1109/TNET.2003.822655} (Exodus and Ebone) and from \emph{The Internet Topology Zoo}~\cite{zoo_6027859} (G\'eant), and two synthetic topologies (AWS North-America and Gaia) built from  AWS data centers~\cite{gaia_10.5555/3154630.3154682,amazon} (Table~\ref{tab:topologies_cycle_time}). For the synthetic topologies, we consider a full-meshed underlay. We assume all underlays support a shortest path routing with the geographical distance (or equivalently the latency) as link cost. These topologies have between 11 and 87 nodes located in the same continent with the exception of Gaia, which spans four continents. The Géant and Ebone network connect European cities and Exodus network connect American cities. We considered that each network node is connected to a geographically close silo by a symmetric access link. 

Some underlays and examples of overlays are shown in Figures~\ref{fig:geant},~\ref{fig:gaia}, and~\ref{fig:awsna}.

\begin{figure}
    \centering
    \begin{subfigure}[b]{0.49\textwidth}  
        \centering 
        \includegraphics[width=\textwidth, height=0.8\textwidth]{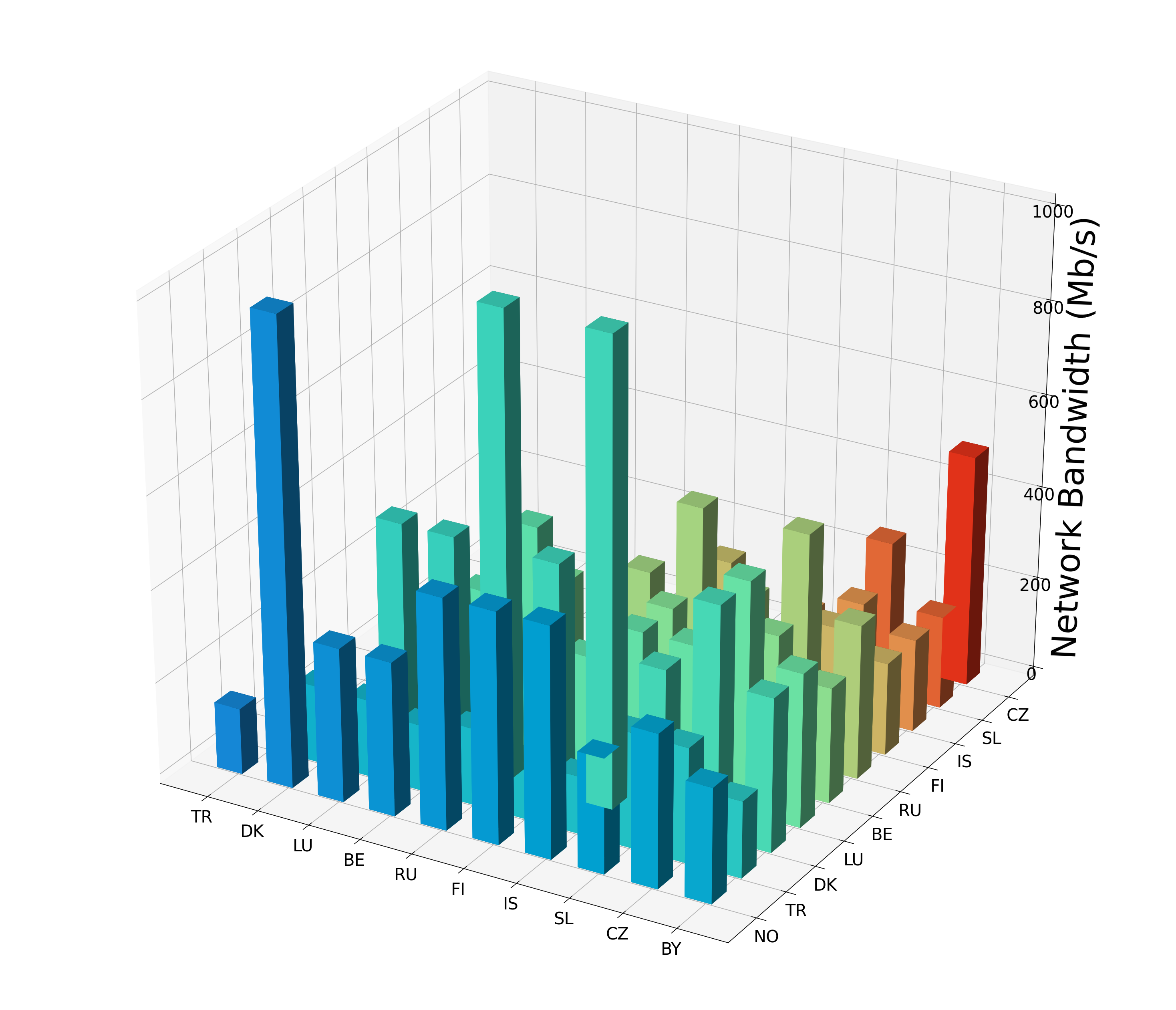}
        \caption[]%
        {{ Available bandwidth  between some pairs of silos in G\'eant as computed through our model.}} 
    \end{subfigure}
    \hfill
    \begin{subfigure}[b]{0.49\textwidth}
        \centering
        \includegraphics[width=\textwidth, height=0.8\textwidth]{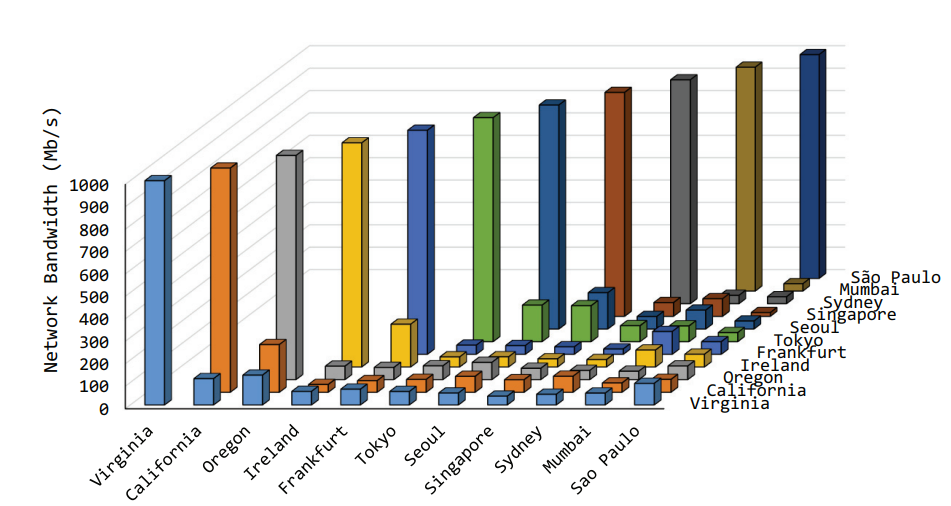}
        \caption[]%
        {{Available bandwidth measurements between Gaia sites \cite[Fig.~2]{gaia_10.5555/3154630.3154682}.}} 
    \end{subfigure}
    \caption[]
    {Our simulator with 1~Gbps capacity links generates a distribution of available bandwidths with the same variability observed in real networks.}
    \label{fig:available_bandwidth}
\end{figure}{}

\begin{figure}
    \centering
    \begin{subfigure}[b]{0.24\textwidth}  
        \centering 
        \includegraphics[width=\textwidth, height=0.8\textwidth]{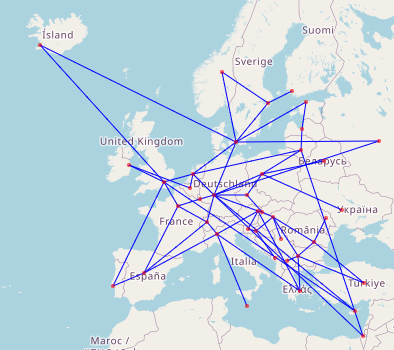}
        \caption[]%
        {{\small Underlay}} 
    \end{subfigure}
    \hfill
    \begin{subfigure}[b]{0.24\textwidth}
        \centering
        \includegraphics[width=\textwidth, height=0.8\textwidth]{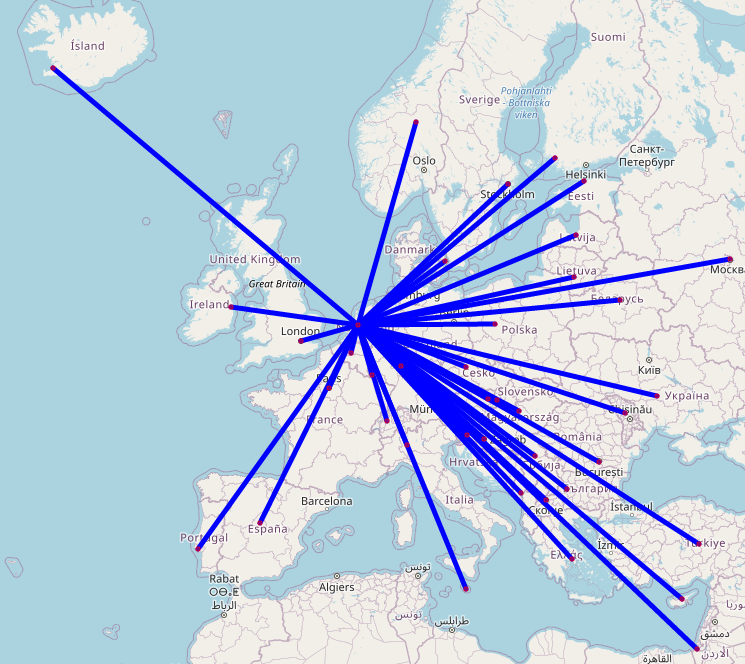}
        \caption[]%
        {{\small Star}} 
    \end{subfigure}
    \hfill
    \begin{subfigure}[b]{0.24\textwidth}   
        \centering 
        \includegraphics[width=\textwidth, height=0.8\textwidth]{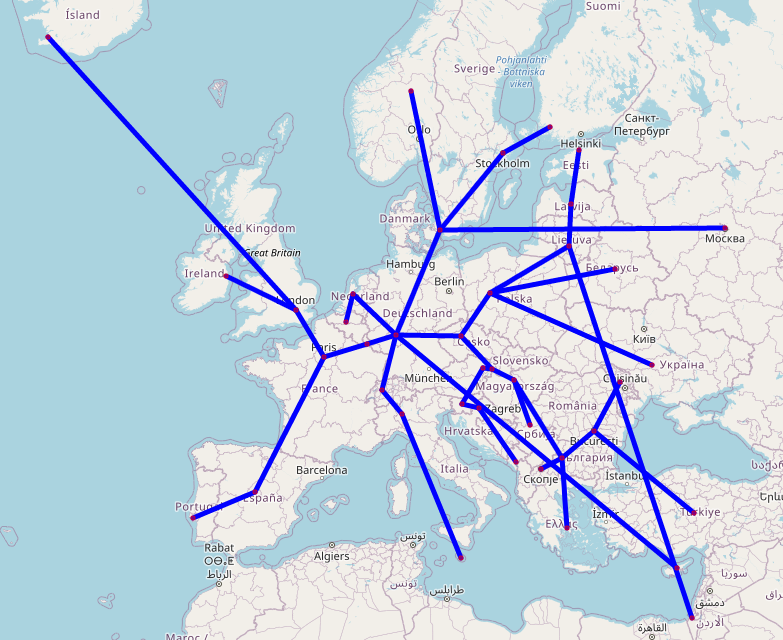}
        \caption[]%
        {{\small MST}} 
    \end{subfigure}
    \hfill
    \begin{subfigure}[b]{0.24\textwidth}   
        \centering 
        \includegraphics[width=\textwidth, height=0.8\textwidth]{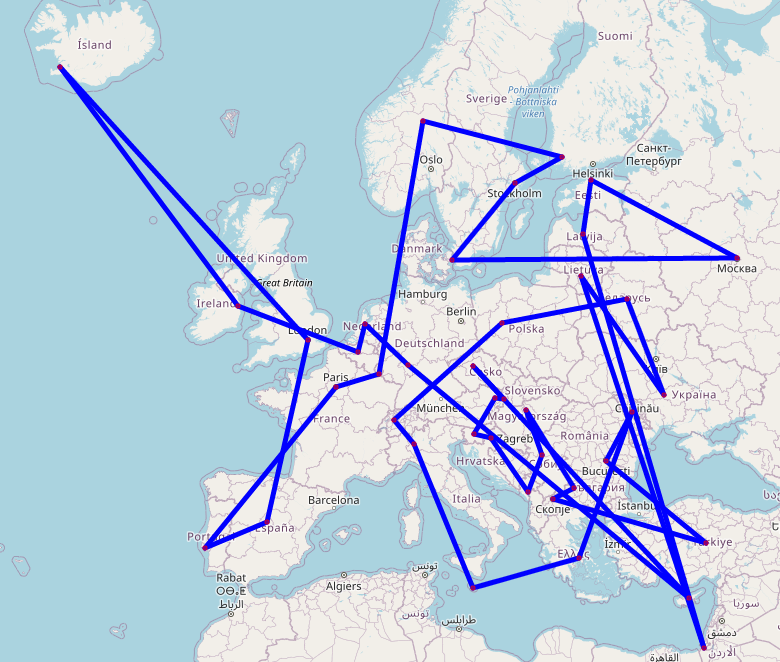}
        \caption[]%
        {{\small Ring}} 
    \end{subfigure}
    \caption[]
    {\small G\'eant Network: the underlay (a) and selected overlays computed when core links have 1~Gbps capacity and access links have 10~Gbps capacity (b-d).}
    \label{fig:geant}
\end{figure}

\begin{figure}
    \centering
    \begin{subfigure}[b]{0.24\textwidth}  
        \centering 
        \includegraphics[width=\textwidth, height=0.8\textwidth]{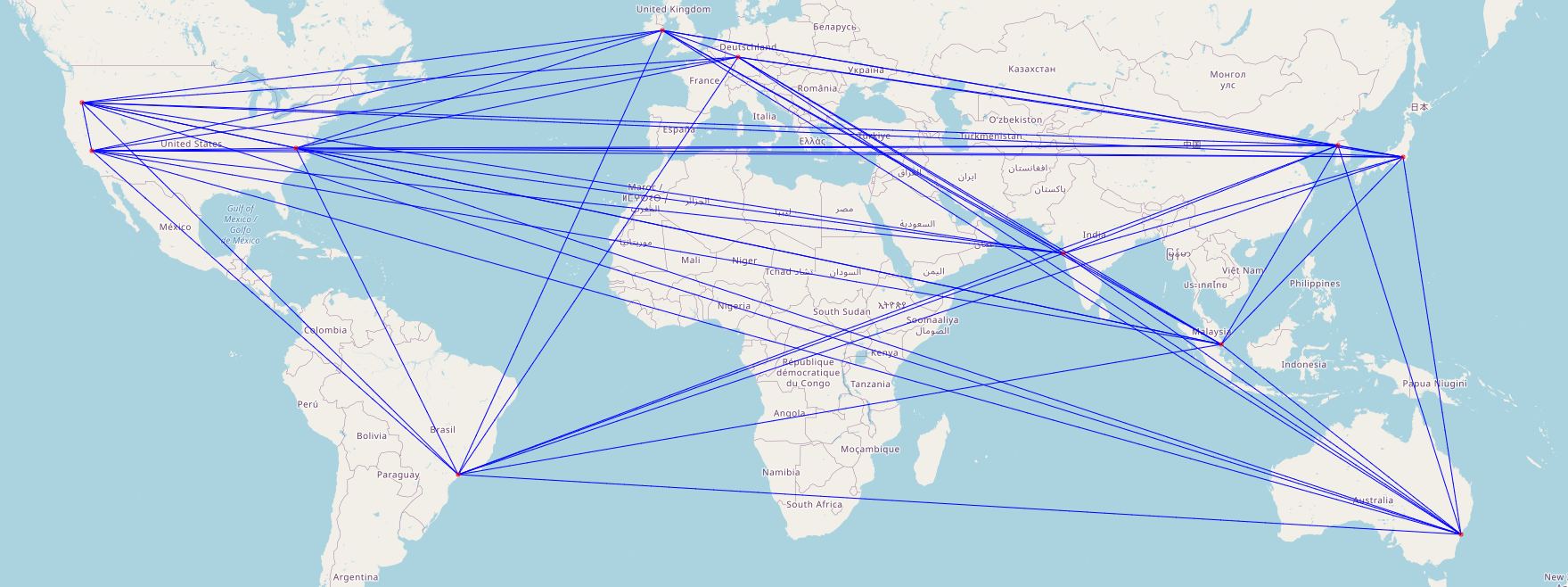}
        \caption[]%
        {{\small Underlay}} 
    \end{subfigure}
    \hfill
    \begin{subfigure}[b]{0.24\textwidth}
        \centering
        \includegraphics[width=\textwidth, height=0.8\textwidth]{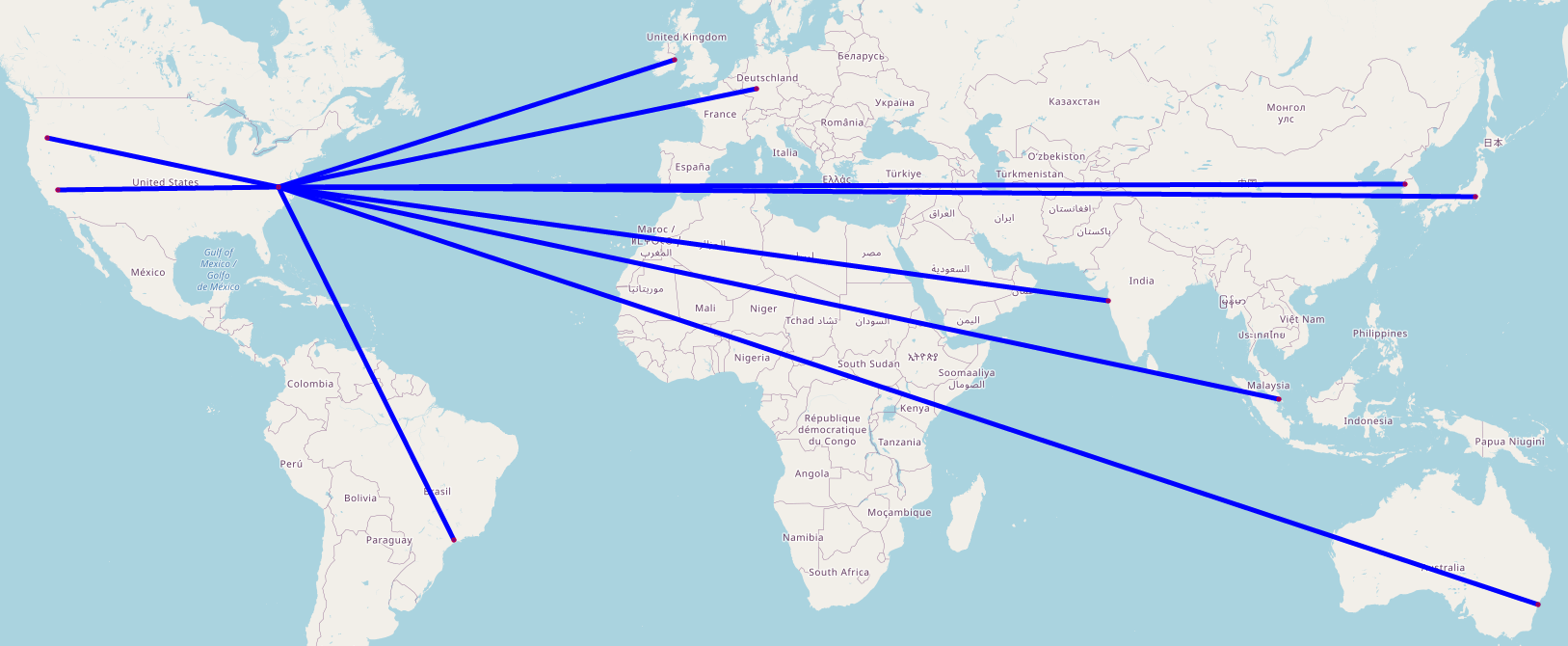}
        \caption[]%
        {{\small Star}} 
    \end{subfigure}
    \hfill
    \begin{subfigure}[b]{0.24\textwidth}   
        \centering 
        \includegraphics[width=\textwidth, height=0.8\textwidth]{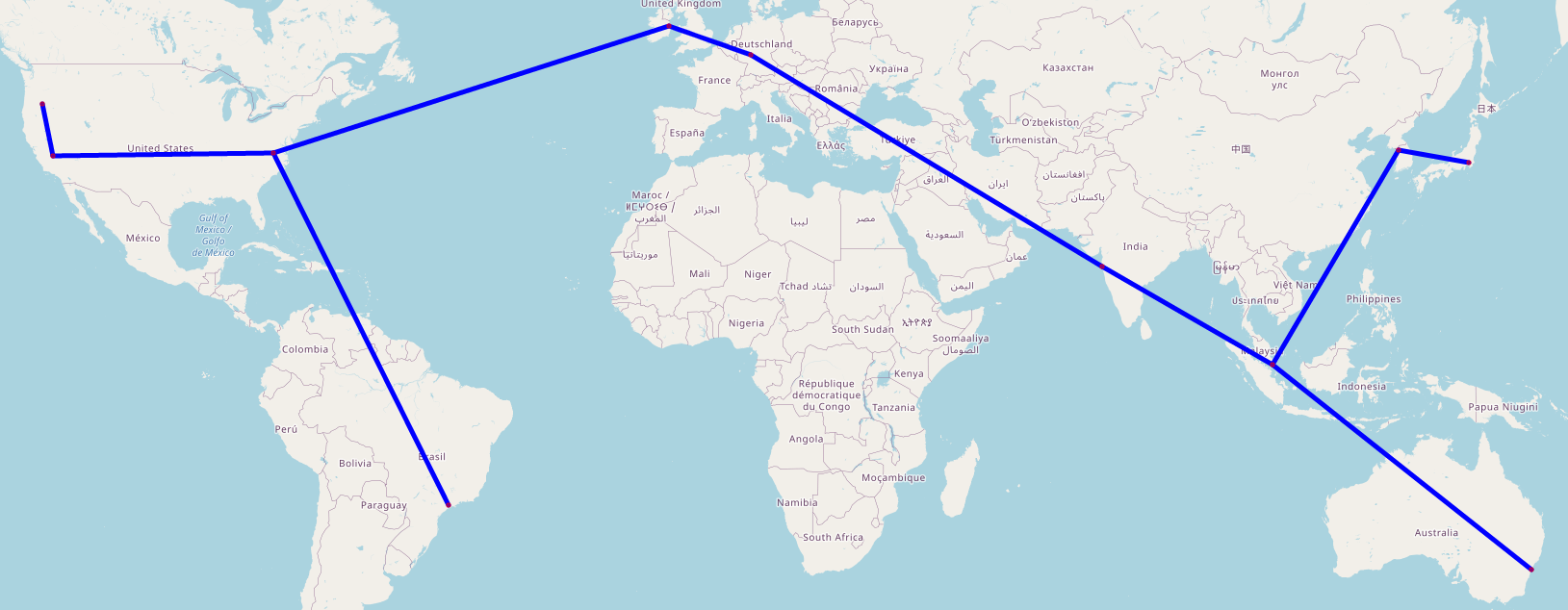}
        \caption[]%
        {{\small MST}} 
    \end{subfigure}
    \hfill
    \begin{subfigure}[b]{0.24\textwidth}   
        \centering 
        \includegraphics[width=\textwidth, height=0.8\textwidth]{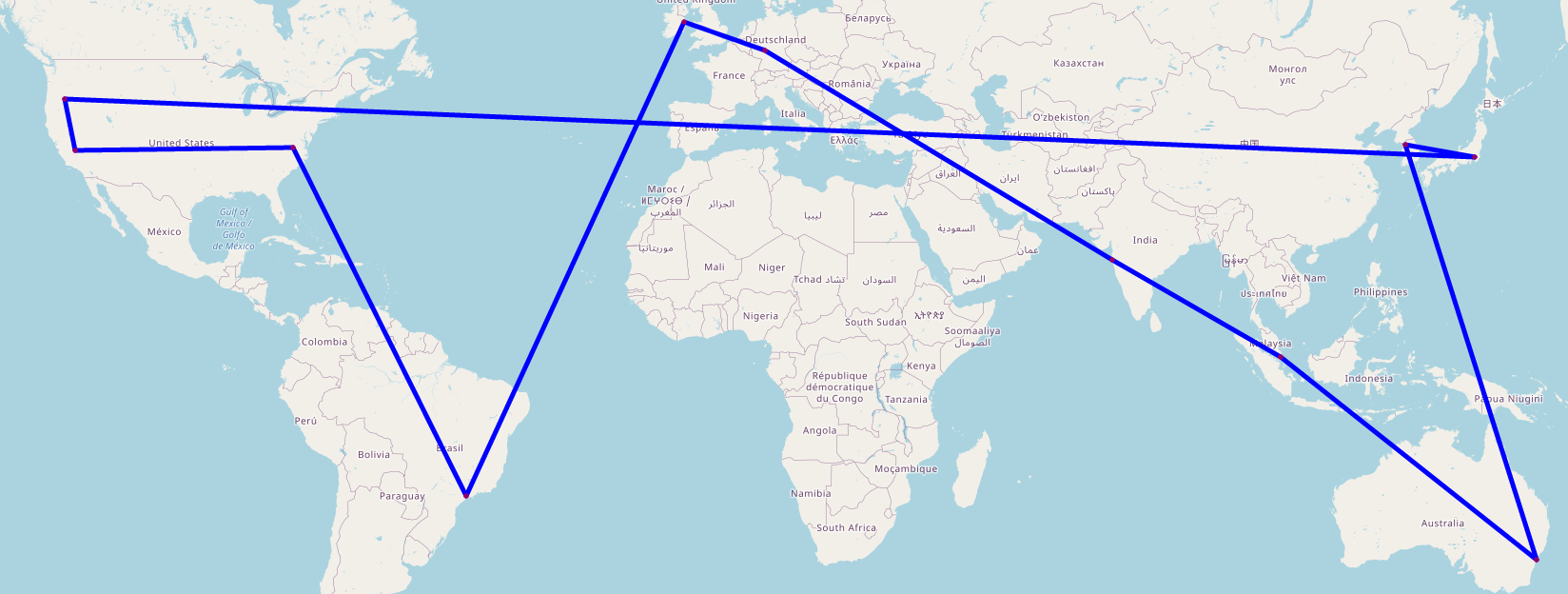}
        \caption[]%
        {{\small Ring}} 
    \end{subfigure}
    \caption[]
    {\small Gaia Network: the underlay (a) and selected overlays computed when core links have 1~Gbps capacity and access links have 10~Gbps capacity (b-d).}
    \label{fig:gaia}
\end{figure}

\begin{figure}
    \centering
    \begin{subfigure}[b]{0.24\textwidth}  
        \centering 
        \includegraphics[width=\textwidth, height=0.8\textwidth]{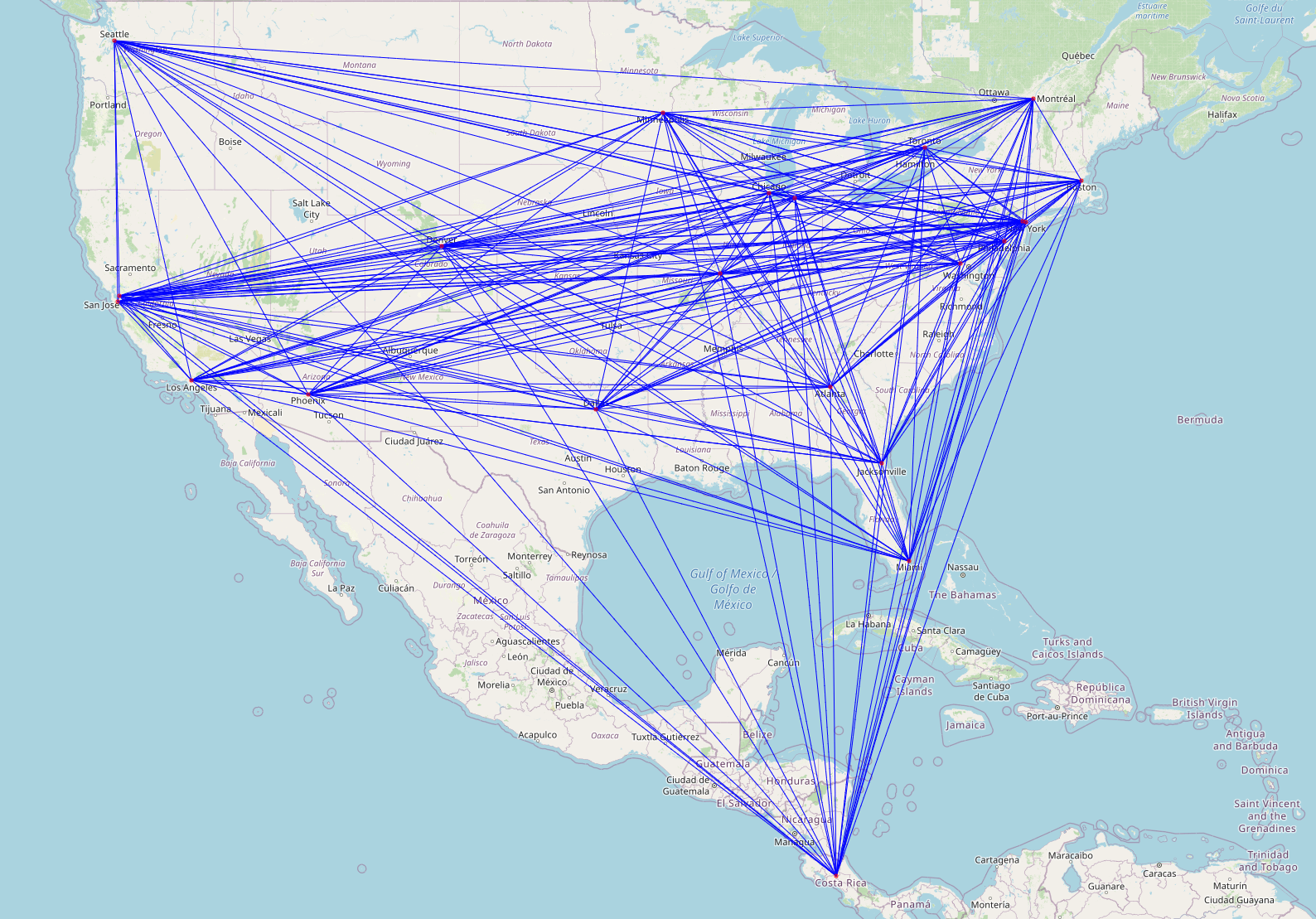}
        \caption[]%
        {{\small Underlay}} 
    \end{subfigure}
    \hfill
    \begin{subfigure}[b]{0.24\textwidth}
        \centering
        \includegraphics[width=\textwidth, height=0.8\textwidth]{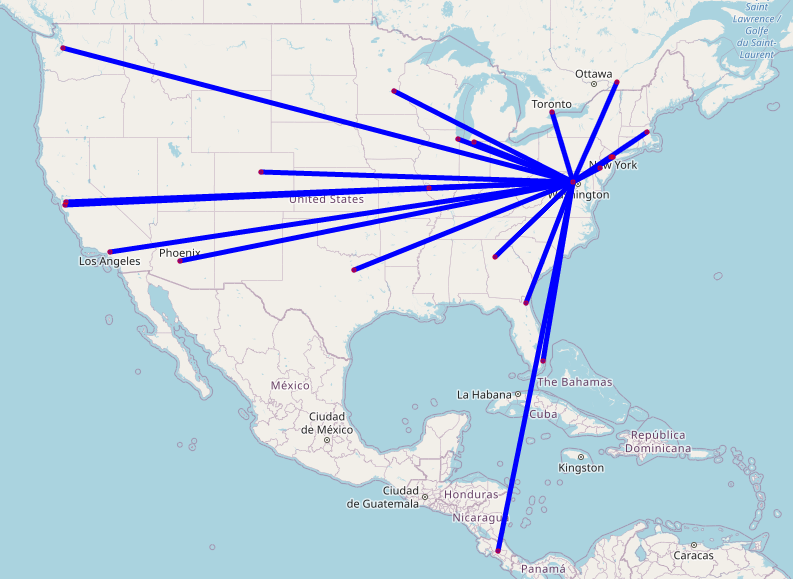}
        \caption[]%
        {{\small Star}} 
    \end{subfigure}
    \hfill
    \begin{subfigure}[b]{0.24\textwidth}   
        \centering 
        \includegraphics[width=\textwidth, height=0.8\textwidth]{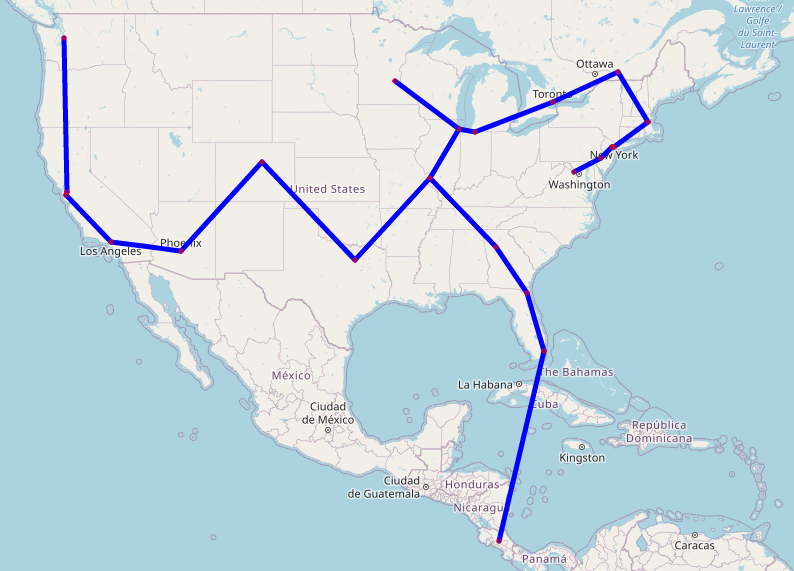}
        \caption[]%
        {{\small MST}} 
    \end{subfigure}
    \hfill
    \begin{subfigure}[b]{0.24\textwidth}   
        \centering 
        \includegraphics[width=\textwidth, height=0.8\textwidth]{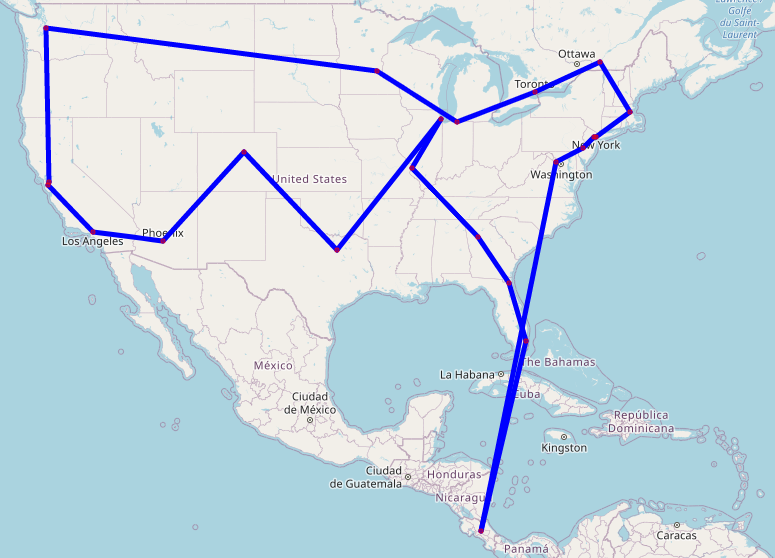}
        \caption[]%
        {{\small Ring}} 
    \end{subfigure}
    \caption[]
    {\small AWS-North America Network: the underlay (a) and selected overlays computed when core links have 1~Gbps capacity and access links have 10~Gbps capacity (b-d).}
    \label{fig:awsna}
\end{figure}

\subsection{Datasets and Models}
We provide full details on datasets and models used in our experiments. We use multiple datasets spanning a wide range of machine learning tasks (sentiment analysis, language modeling, image classification, handwritten character recognition), including those used in prior work on federated learning \cite{mcmahan2016communicationefficient}, and in LEAF \cite{caldas2018leaf} benchmark, and a cross-silo specific dataset based on iNaturalist~\cite{DBLP:journals/corr/HornASSAPB17}. 

\paragraph{iNaturalist dataset.}
    iNaturalist~\cite{DBLP:journals/corr/HornASSAPB17} consists of images from over 8,000 different species of plants and animals. We choose the dataset from iNaturalist 2018 competition which contains 450,000 images\footnote{iNaturalist 2018 competition is part of the $FGVC^5$ workshop at CVPR (\url{https://github.com/visipedia/inat_comp/blob/master/2018/README.md}).} where the geo-locations of these images are provided. 
    Due to a large class imbalance, iNaturalist species classification is a tough learning task, which requires large computation resources.
    In our experiments, we started by using a subset of the original iNaturalist dataset, selecting images containing the $80$ most popular species.\footnote{The dataset size is reduced from 120GB to 18GB containing 67,000 images. We subsampled then $20\%$ from this dataset for training.}
    We have also conducted additional experiments on the full iNaturalist dataset, whose corresponding results are presented in Appendix~\ref{sec:appendix_full_inaturalist}.
    We refer to the complete dataset as Full-iNaturalist.

    In order to simulate a realistic cross-silo environment with non-iid local datasets, one can assign the images to the geographically closest silo obtaining local datasets different in size and in the species represented. This distribution would lead some silos to have no point. We decided then to assign half of the images uniformly at random and half to the closest silo. Moreover, since most of the images in iNaturalist are from North America, for European networks such as Ebone and Géant, we mapped the European cities westward by reducing their longitude by 90 degrees. Table~\ref{tab:statistics_data_distribution} shows that our method  generates quite unbalanced data distribution (e.g., for Ebone, one silo can have up to 50 times more images than another one). 
    
    To classify iNaturalist images we finetuned a pretrained ResNet-18 on ImageNet~\cite{imagenet_cvpr09}. In particular we used the torchvision \cite{10.1145/1873951.1874254} implementation of ResNet-18.

\paragraph{LEAF datasets.}
    LEAF \cite{caldas2018leaf} is a benchmark framework for learning in federated settings. We used three LEAF datasets in our experiments on AWS North America network where we took $20\%$ of the samples randomly as our dataset.\footnote{
        Actually, the amount of data we considered is comparable to the federated learning paper~\cite{li2019federated}: we considered 10 times more data for FEMNIST and the same amount of data for Sentiment140 and Shakespeare.
    }
    Statistics for the corresponding data distributions are in Table~\ref{tab:statistics_data_distribution_2}.
    
    \begin{itemize}
        \item  \textbf{FEMNIST} (\emph{Federated Extended} MNIST): A $62$-class image classification dataset built by partitioning the data of Extended MNIST based on the writer of the digits/characters. In our experiments, we associate each silo with a random number of writers following a lognormal distribution with mean equal to $5$ and standard deviation equal to $1.5$.
        
         We train a  convolutional neural network, similar to LeNet, with two convolutional layers followed by a max-pooling layer and two fully connected layers. 
        
        \item \textbf{Shakespeare}: A dataset built from \emph{The Complete Works of William Shakespeare}, which is partitioned by the speaking roles~\cite{mcmahan2016communicationefficient}. In our experiment, we associate each silo with a random number of speaking roles following a lognormal distribution with mean equal to $5$ and standard deviation equal to $1.5$.
        
        We consider character-level based language modeling on this dataset. 
        The model takes as input a sequence of $200$ English characters and predicts the next character. The model embeds the $200$ characters into a learnable $16$ dimensional embedding space, and uses two stacked-GRU layers with $256$ hidden units, followed by a densely-connected layer.
        
        \item \textbf{Sentiment140}~\cite{Go_Bhayani_Huang_2009}: An automatically generated sentiment analysis dataset that annotates tweets based on their emoticons. In our experiment, we associate each silo with a random number of Twitter accounts following a lognormal distribution with mean equal to $5$ and standard deviation equal to $1.5$.
        
        We use a two layer bi-directional LSTM binary classifier containing $256$ hidden units with pretrained $100$ dimensional GloVe embedding \cite{pennington2014glove}. 
    \end{itemize}
  
    \begin{table}
        \centering
        \caption{Statistics of iNaturalist dataset distribution for different networks.}
        \begin{tabular}{l c r r r r }
            \toprule
            \multirow{2}{*}{Network name} &  \multirow{2}{*}{Silos} & \multicolumn{4}{c}{Samples/silo} \\
            & & Mean & Stdev & Min & Max \\
            \midrule
            Gaia & 11 & 1213 & 1143 & 610 & 3981\\ 
            AWS North America& 22& 606& 731 & 113 & 3216\\
            Géant& 40& 333 & 644 & 152 & 4261\\
            Exodus & 79 & 168& 96& 92& 576\\
            Ebone & 87 & 153 & 394& 68& 3389\\
            \bottomrule
        \end{tabular}
        \label{tab:statistics_data_distribution}
    \end{table}

    \begin{table}
        \centering
        \caption{Statistics of LEAF dataset distribution for AWS North America network (22 silos).}
        \begin{tabular}{l  r r r r }
            \toprule
            \multirow{2}{*}{Dataset}  & \multicolumn{4}{c}{Samples/silo} \\
                            & Mean & Stdev & Min & Max \\
            \midrule
            Shakespeare & 36359 & 6837 &24207& 50736 \\ 
            FEMNIST & 6847& 7473 & 196 & 26469\\
            Sentiment140 &  13101& 14273 & 424 & 50562 \\
            \bottomrule
        \end{tabular}
        \label{tab:statistics_data_distribution_2}
    \end{table}

\subsection{Implementation Details}
\paragraph{Machines.} The experiments have been run on a CPU/GPU cluster, with different GPUs available (e.g., Nvidia Tesla V100, GeForce GTX 1080 Ti, and Titan X).

\paragraph{Libraries.} All code is implemented in PyTorch Version 1.4.0. We offer two possibilities for running the code: \emph{sequential} (using only one GPU) and \emph{parallel} (using multiple GPUs). In the parallel setting MPI backend is used for  inter-GPU communications. 

\paragraph{Hyperparameters.} The dataset is randomly split into an $80\%$ training set and a $20\%$ testing set. When training on Gaia, AWS North America, and G\'eant networks, the initial learning rate is set to 0.001 with Adam optimizer. When training on Exodus and Ebone networks, the initial learning rate is set to $0.1$ with SGD optimizer. We decay the learning rate based on the inverse square root of the number of communication rounds. The batch size is set to $512$ for Sentiment140 and Shakespeare datasets, to $128$ for Femnist dataset and to $16$ for iNaturalist dataset. 

\paragraph{Consensus Matrix.} For a given overlay $\mathcal{G}_o=(\mathcal{V},\mathcal{E}_o)$, the consensus matrix $\mA$ is selected similarly to the local-degree rule in \cite{Boyd2004_1272421}. The weight on an arc is based on the larger in-degree of its two incident nodes:
\begin{align}
    \label{eq:mixing_matrix}
    \mA_{i,j} &=\frac{1}{1 + \max\left(|\mathcal{N}^{-}_{i}|, |\mathcal{N}^{-}_{j}|\right)}, ~~\forall (i ,j) \in \mathcal{E}_{o}.\\ 
    \mA_{i,i} &= 1-\sum_{j\in \mathcal{N}^{-}_i} \mA_{i,j},~~\forall i\in \mathcal{V}.
\end{align}
The matrix $\mA$ so-built is symmetric doubly stochastic. The weights can be determined in a fully-distributed way: every node just needs to exchange degree information with its neighbours.

\paragraph{MATCHA.} We implemented MATCHA as described in~\cite{wang2019matcha} but for one difference. 
In MATCHA, each matching $i$ is selected independently with some probability $p_i$. With probability $q = \prod_{i}(1-p_i)$, no matching is selected and then no communication occurs. This is equivalent to perform a random number of local steps $s$ between two communication rounds. In order to compare fairly the different approaches and isolate the effect of $s$, we fixed $s$ also for MATCHA as follows. Silos perform a given number of local steps $s$ and then, when a communication should occur, matchings are independently sampled until at least one of them is selected. In practice, in our experiments, the probability $q$ was close to 0, so that the two approaches are practically undistinguishable. Finally, we observe that MATCHA computes the matchings coloring an initial topology, but it is not explained how this initial topology is selected. MATCHA and MATCHA$^{+}$ operate exactly in the same way but starting from two different initial topologies: the connectivity graph $\mathcal G_c$ and the underlay $\mathcal G_u$, respectively. The silos can easily discover the connectivity graph $\mathcal G_c$; reconstructing the underlay is much more complicated. Nevertheless, as MATCHA$^{+}$ was in general outperforming MATCHA, we showed the results for MATCHA$^{+}$.

\newpage
\section{Complete Set of Experiments}
\label{sec:appendix_full_experiments}
\subsection{Effect of the number of local steps}
\label{sec:appendix_local_steps}
Tables~\ref{tab:topologies_cycle_time_local_step_5} and~\ref{tab:topologies_cycle_time_local_step_10} show the effect of 6 different overlays when training ResNet-18 over iNaturalist in networks with 1~Gbps core links and 10~Gbps access links and local steps equal to 5 and 10, respectively. For $5$ local steps, the training time is evaluated as the time to reach a training accuracy equal to $65\%$, $55\%$, $60\%$, $45\%$, and $45\%$ for Gaia, AWS North America, Géant, Exodus, and Ebone, respectively. For $10$ local steps, the training time is evaluated as the time to reach a training accuracy equal to $65\%$, $50\%$, $50\%$, $45\%$, and $40\%$, respectively. 

\setlength{\tabcolsep}{2pt}
\begin{table}
    \footnotesize
    \caption{iNaturalist training over different networks. $1$~Gbps core links capacities, $10$ Gbps access links capacities. Five local computation steps.}
    \centering
    \resizebox{\columnwidth}{!}{%
        \begin{tabular}{l| c | c | r | c | r | r | r | c | c}
            \toprule
            \multirow{2}{*}{\textbf{Network name}} & \multirow{2}{*}{\textbf{Silos}} & \multirow{2}{*}{\textbf{Links}} & \multicolumn{5}{c|}{\textbf{Cycle time (ms)}} & \multicolumn{2}{c}{\textbf{Ring's training speed-up}}\\
            &  &  & {\scriptsize STAR}  &  {\scriptsize MATCHA$^{(+)}$} & {\scriptsize MST} & {\scriptsize $\delta$-MBST} &  {\scriptsize RING}  & {\scriptsize vs STAR} &   {\scriptsize vs MATCHA$^{(+)}$}\\
            \midrule 
            Gaia \cite{gaia_10.5555/3154630.3154682} & $11$ & $55$ & $492.4$ & $329.3 (329.3)$ &  $239.7$ &  $239.8$ &  $219.7$  & $1.79$ & $1.50 (1.50)$\\ 
            AWS NA \cite{amazon} & $22$ & $231$ & $389.8$ &  $226.0 (226.0)$ &  $191.3$ &  $191.3$ &  $182.9$  & $1.40$ & $1.24 (1.24)$\\
            Géant \cite{geant} & $40$ & $61$ & $736.0$& $ 553.8 (207.4)$ &  $202.6$ &  $202.6$ &  $210.6$  & $3.49$ & $2.63 (2.96)$\\
            Exodus(us)~\cite{mahajan2002inferring} & $79$ & $147$ & $1013.4$&  $695.0 (243.8)$ &  $246.9$ &  $246.9$ &  $205.5$ & $3.95$ & $2.25 (1.18)$\\
            Ebone(eu)~\cite{mahajan2002inferring} & $87$ & $161$ & $1003.2$ & $681.6 (224.9)$ & $223.2$ & $223.2$ & $196.9$ & $3.04$ & $2.29 (1.21)$\\
            \bottomrule
        \end{tabular}%
    }
    \label{tab:topologies_cycle_time_local_step_5}
\end{table}

\setlength{\tabcolsep}{2pt}
\begin{table}
    \footnotesize
    \caption{iNaturalist training over different networks. $1$~Gbps core links capacities, $10$ Gbps access links capacities. Ten local computation steps.}
    \centering
    \resizebox{\columnwidth}{!}{%
        \begin{tabular}{l| c | c | r | c | r | r | r | c | c}
            \toprule
            \multirow{2}{*}{\textbf{Network name}} & \multirow{2}{*}{\textbf{Silos}} & \multirow{2}{*}{\textbf{Links}} & \multicolumn{5}{c|}{\textbf{Cycle time (ms)}} & \multicolumn{2}{c}{\textbf{Ring's training speed-up}}\\
            &  &  & {\scriptsize STAR}  &  {\scriptsize MATCHA$^{(+)}$} & {\scriptsize MST} & {\scriptsize $\delta$-MBST} &  {\scriptsize RING}  & {\scriptsize vs STAR} &   {\scriptsize vs MATCHA$^{(+)}$}\\
            \midrule 
            Gaia \cite{gaia_10.5555/3154630.3154682} & $11$ & $55$ & $619.4$ & $456.4 (456.4)$ &  $366.7$ &  $366.7$ &  $346.7$  & $1.79$ & $1.32 (1.32)$\\ 
            AWS NA \cite{amazon} & $22$ & $231$ & $516.8$ &  $353.2 (353.2)$ &  $318.3$ &  $318.3$ &  $309.9$  & $0.69$ & $0.47 (0.47)$\\
            Géant \cite{geant} & $40$ & $61$ & $609.0$& $ 680.8 (334.7)$ & $329.6$ &  $329.6$ &  $337.6$  & $0.90$ & $ 1.00 (1.98)$\\
            Exodus(us)~\cite{mahajan2002inferring} & $79$ & $147$ & $1140.4$&  $ 822.0(370.9)$ &  $373.9$ &  $373.9$ &  $332.5$ & $1.52$ & $1.10 (1.23)$ \\
            Ebone(eu)~\cite{mahajan2002inferring} & $87$ & $161$ & $1130.2$ & $ 808.6(352.1)$ & $350.4$ & $350.4$ & $323.9$ & $1.74$ & $ 1.25 (1.09)$\\
            \bottomrule
        \end{tabular}%
    }
    \label{tab:topologies_cycle_time_local_step_10}
\end{table}

\subsection{Full results for training every dataset on AWS North America}
\label{sec:appendix_other_metrics_aws}
In Figure~\ref{f:training_for_all_data_sets}, we have shown the training loss w.r.t. communication rounds and wall-clock time when training four different datasets on AWS North America. Here we provide the complete results (Figures~\ref{f:shakespear_aws}--\ref{f:inaturalist_aws}) which include training loss, training accuracy, test loss, and test accuracy w.r.t communication rounds and wall-clock time.

\begin{figure*}
    \centering
    \begin{subfigure}[b]{0.24\textwidth}  
        \centering 
        \includegraphics[width=\textwidth, height=0.8\textwidth]{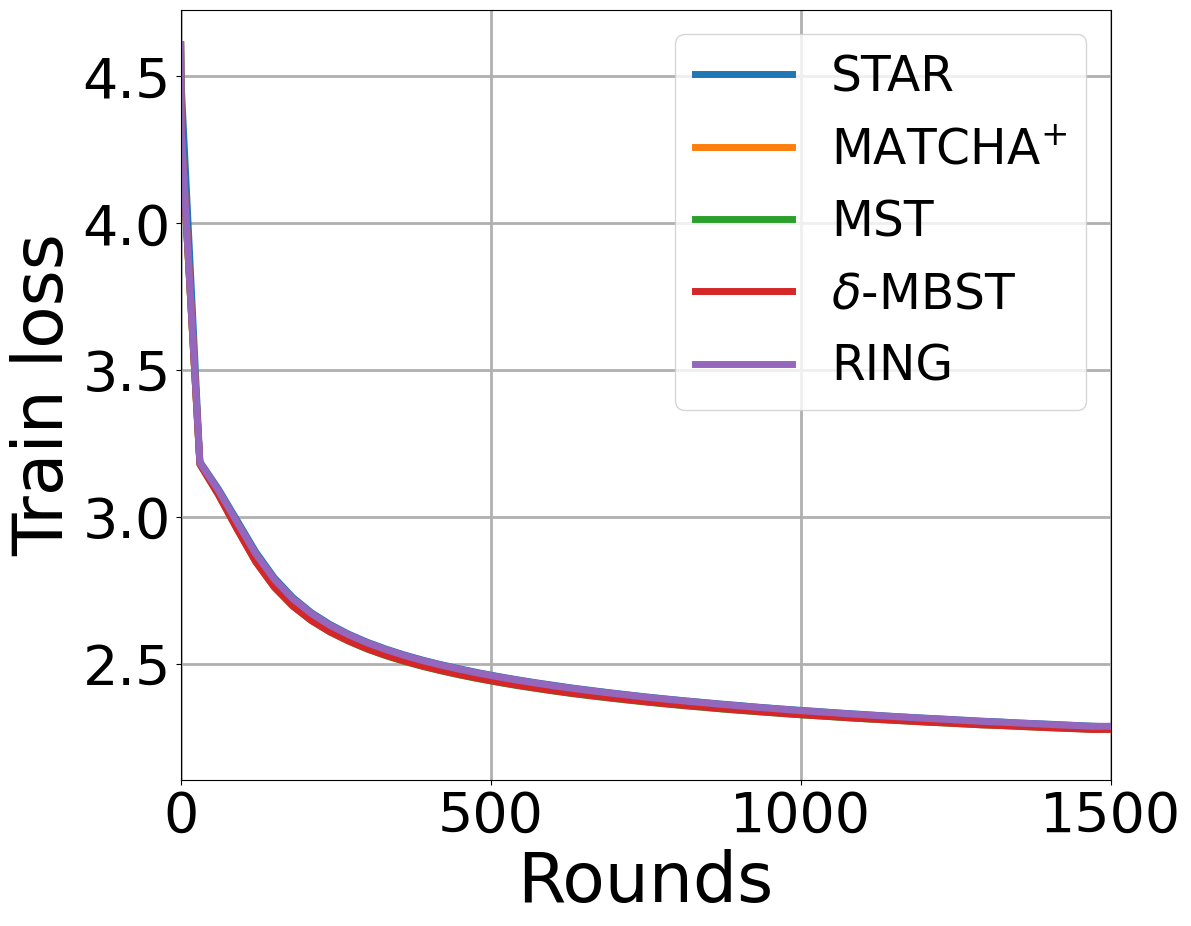}
    \end{subfigure}
    \hfill
    \begin{subfigure}[b]{0.24\textwidth}
        \centering
        \includegraphics[width=\textwidth, height=0.8\textwidth]{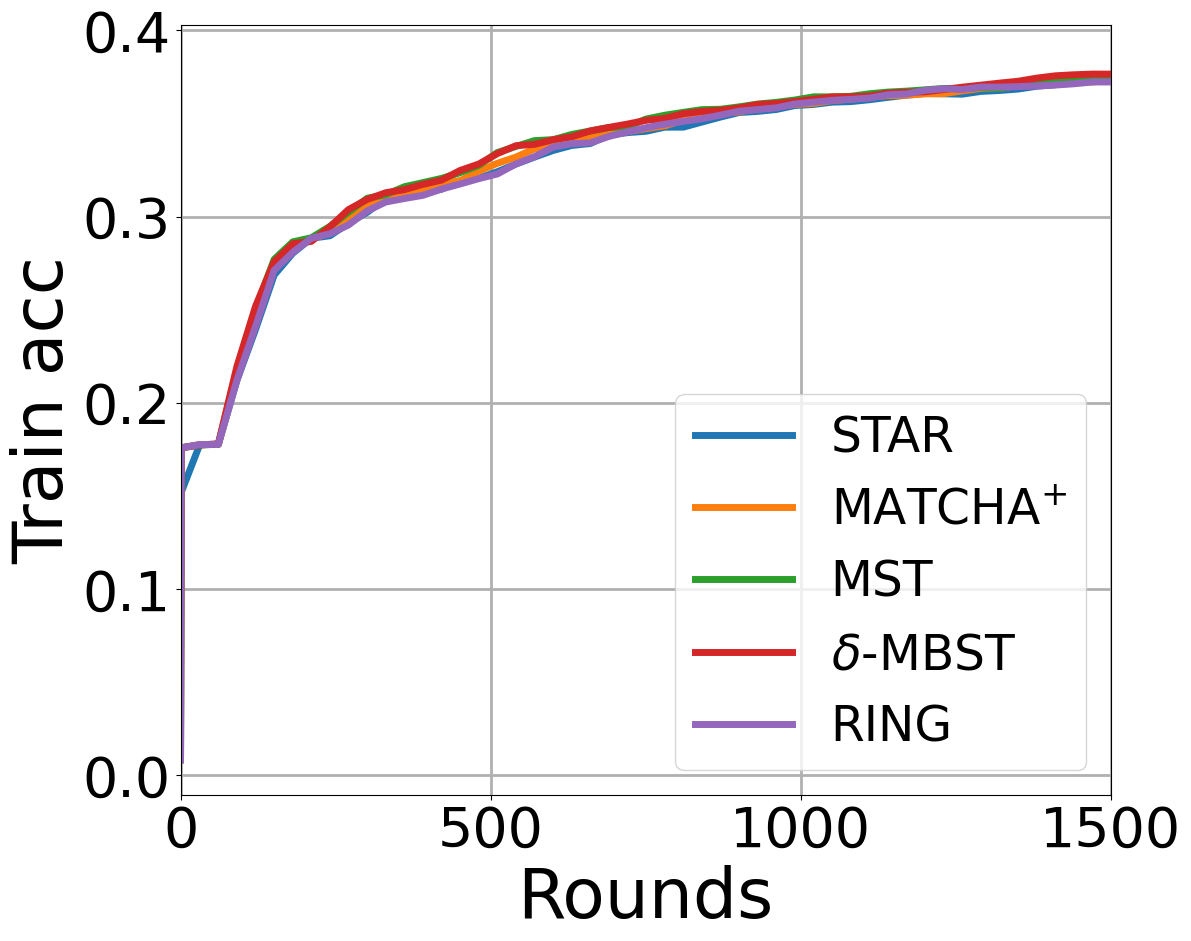}
    \end{subfigure}
    \hfill
    \begin{subfigure}[b]{0.24\textwidth}   
        \centering 
        \includegraphics[width=\textwidth, height=0.8\textwidth]{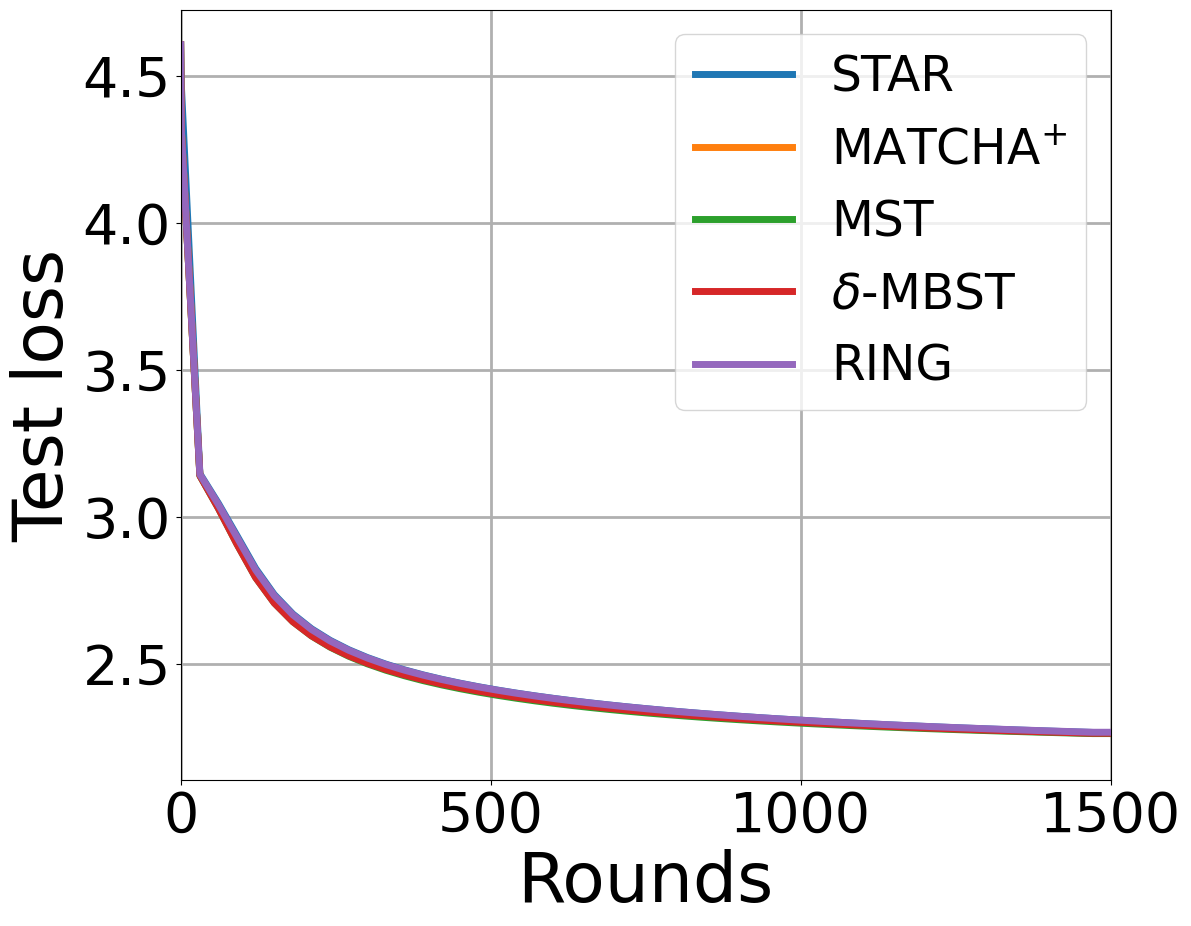}
    \end{subfigure}
    \hfill
    \begin{subfigure}[b]{0.24\textwidth}   
        \centering 
        \includegraphics[width=\textwidth, height=0.8\textwidth]{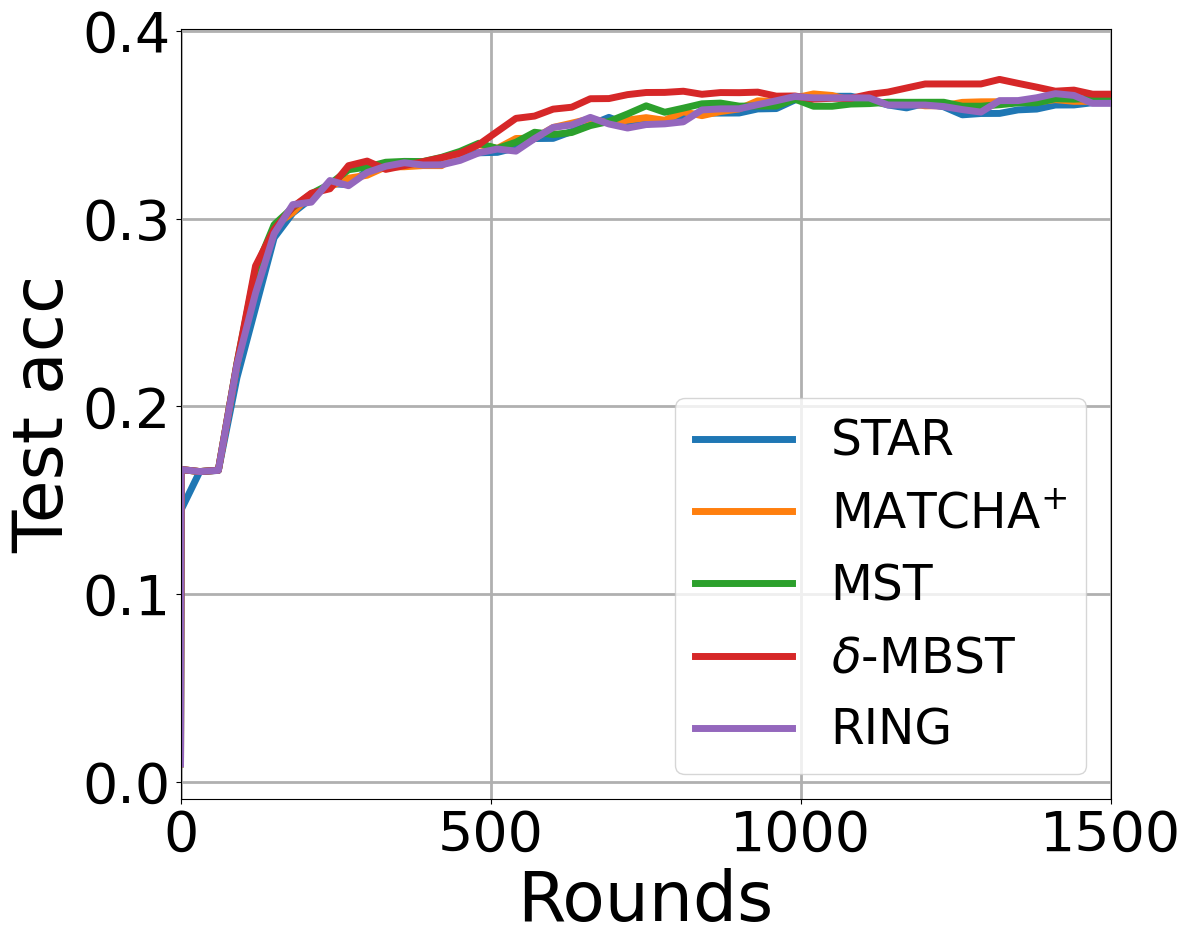}
    \end{subfigure}
    \\    
    \begin{subfigure}[b]{0.24\textwidth}  
        \centering 
        \includegraphics[width=\textwidth, height=0.8\textwidth]{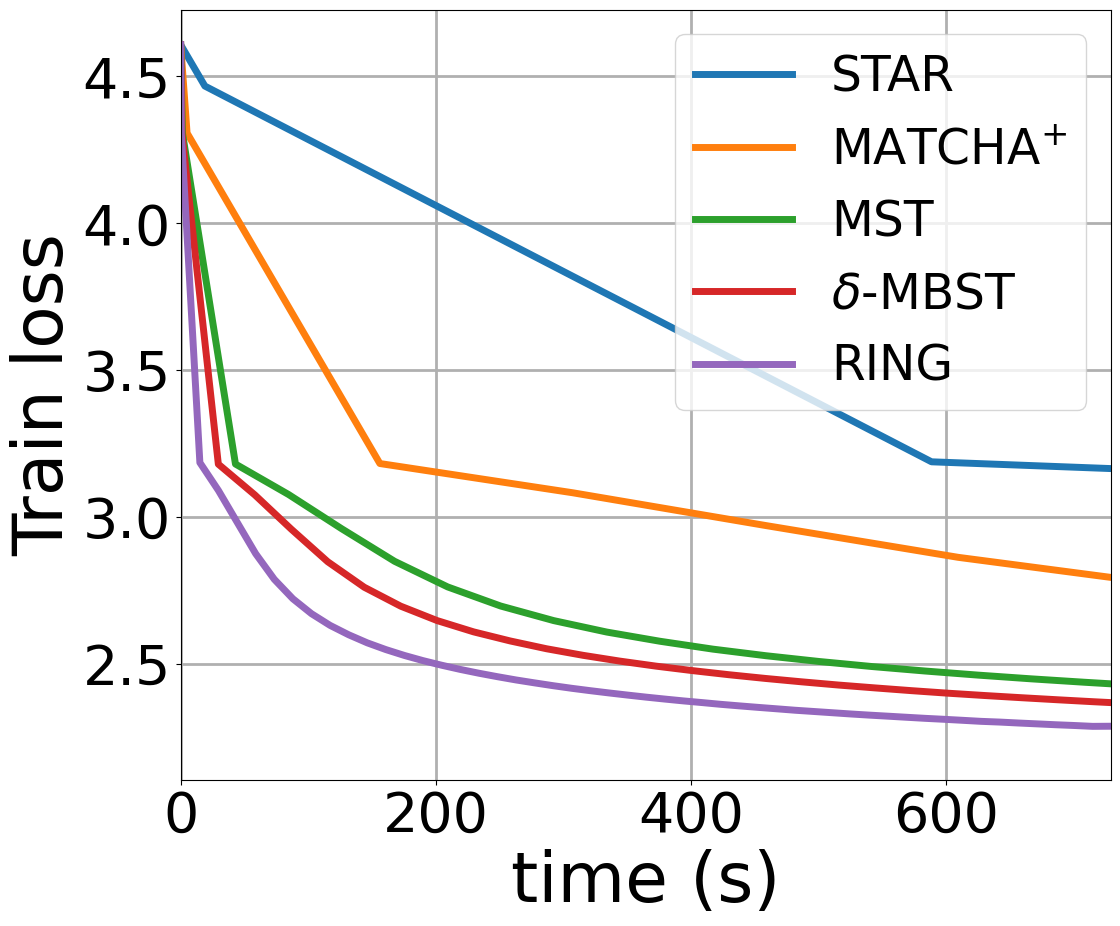}
        \caption[]{{\small Train Loss}}    
    \end{subfigure}
    \hfill
    \begin{subfigure}[b]{0.24\textwidth}
        \centering
        \includegraphics[width=\textwidth, height=0.8\textwidth]{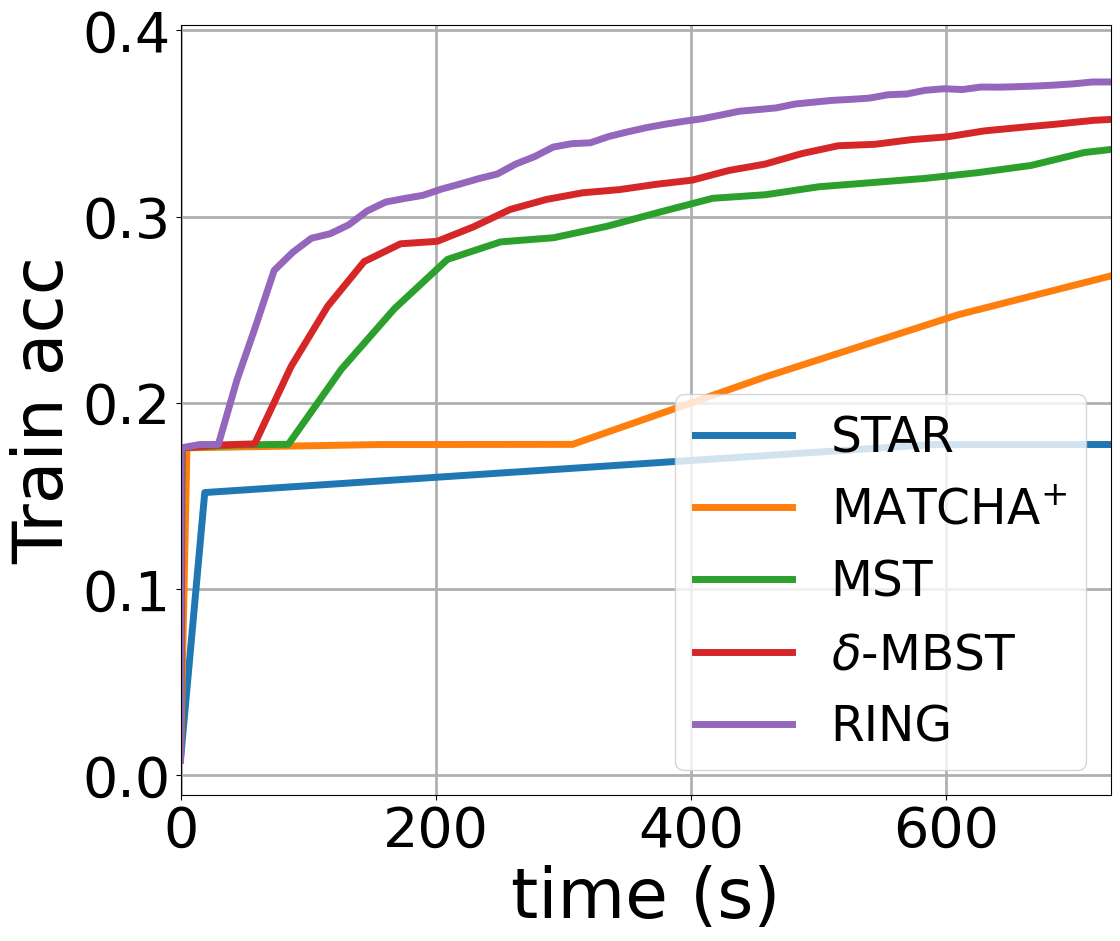}
        \caption[]{{\small Train Accuracy}}    
    \end{subfigure}
    \hfill
    \begin{subfigure}[b]{0.24\textwidth}   
        \centering 
        \includegraphics[width=\textwidth, height=0.8\textwidth]{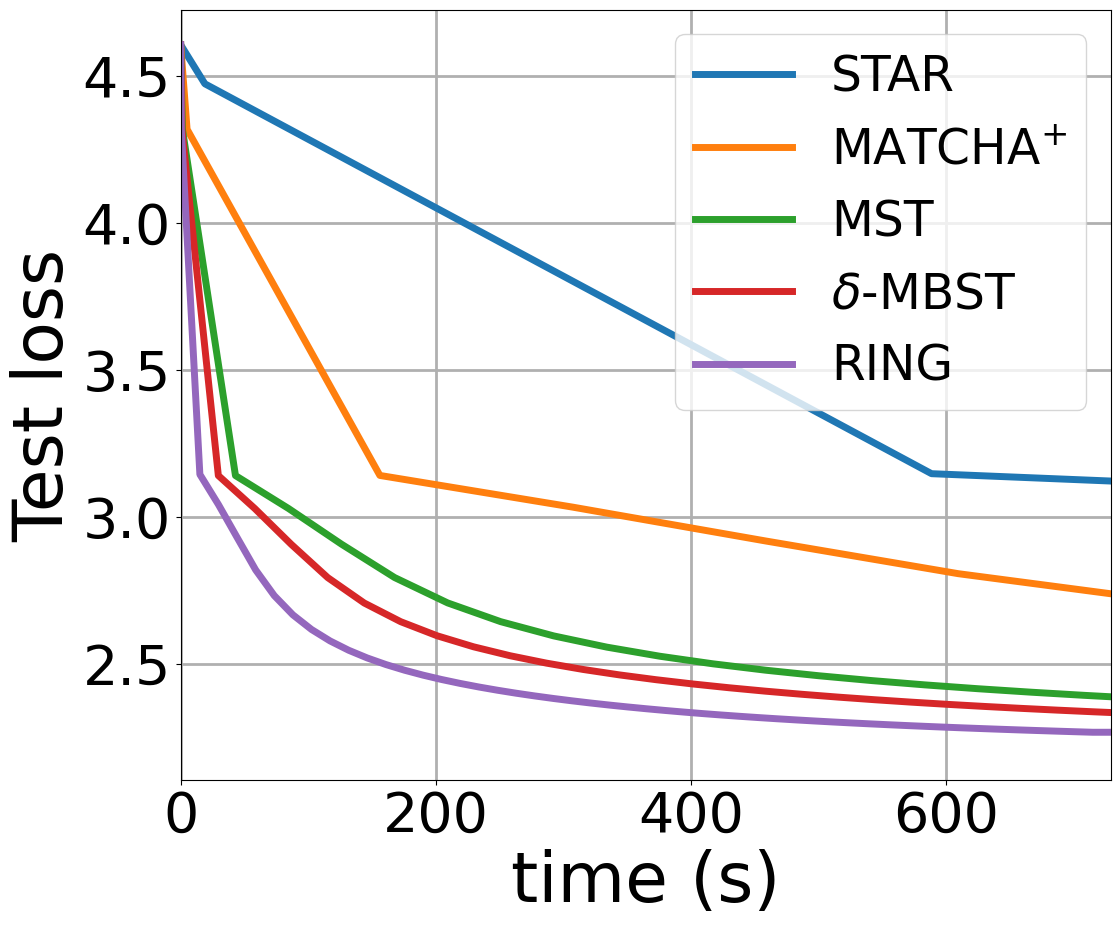}
        \caption[]{{\small Test Loss}}    
    \end{subfigure}
    \hfill
    \begin{subfigure}[b]{0.24\textwidth}   
        \centering 
        \includegraphics[width=\textwidth, height=0.8\textwidth]{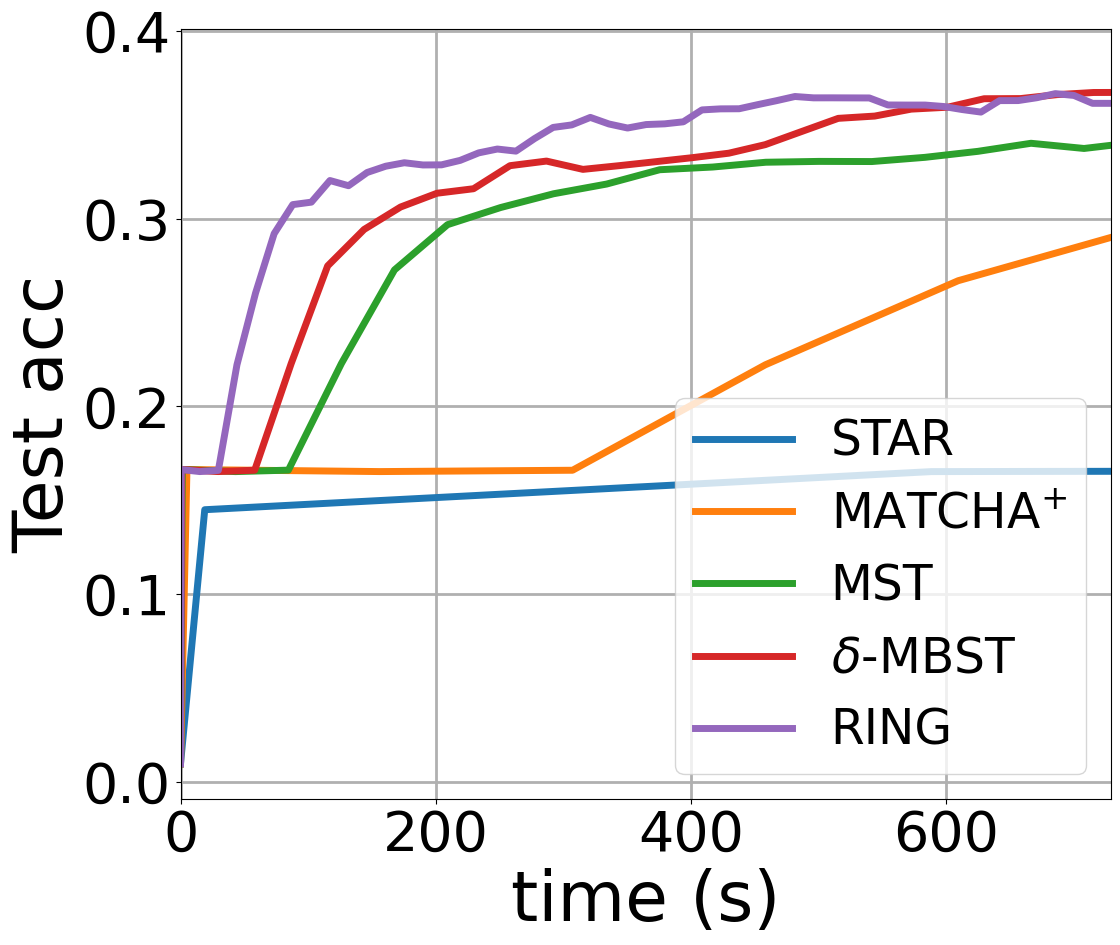}
        \caption[]{{\small Test Accuracy}}    
    \end{subfigure}
    \caption[]
    {\small Effect of overlays on the convergence w.r.t.~communication rounds (top row) and wall-clock time (bottom row) when training Shakespeare on AWS North America underlay. $1$~Gbps core links capacities, $100$~Mbps access links capacities, $s=1$.} 
    \label{f:shakespear_aws}
\end{figure*}

\begin{figure*}
    \centering
    \begin{subfigure}[b]{0.24\textwidth}  
        \centering 
        \includegraphics[width=\textwidth, height=0.8\textwidth]{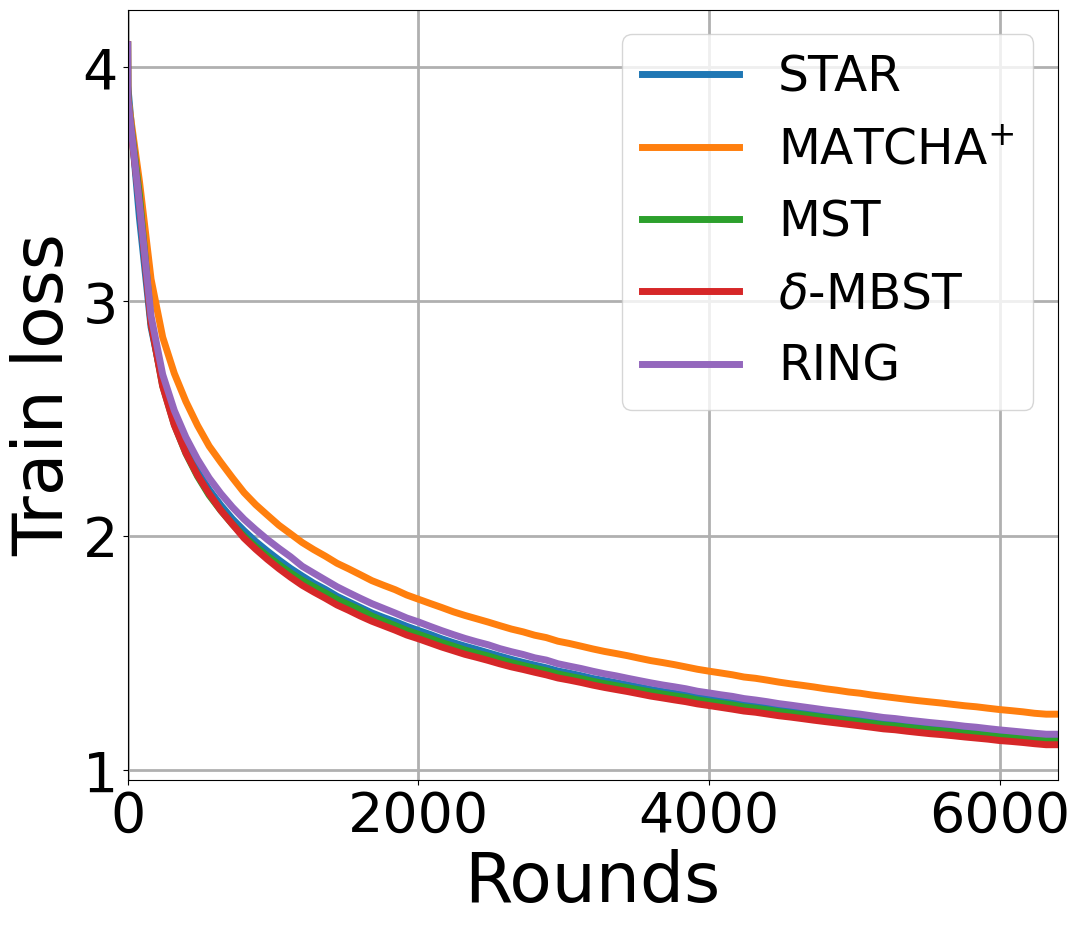}
    \end{subfigure}
    \hfill
    \begin{subfigure}[b]{0.24\textwidth}
        \centering
        \includegraphics[width=\textwidth, height=0.8\textwidth]{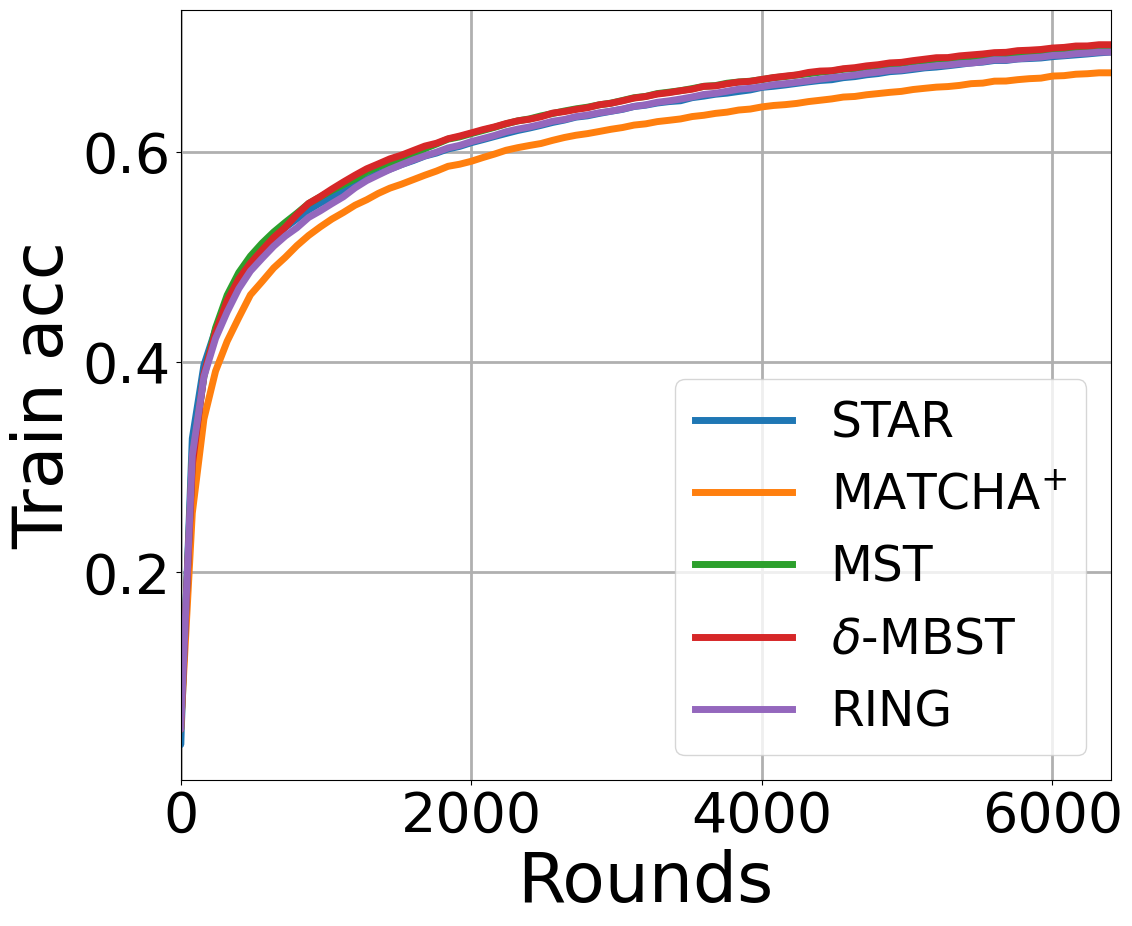}
    \end{subfigure}
    \hfill
    \begin{subfigure}[b]{0.24\textwidth}   
        \centering 
        \includegraphics[width=\textwidth, height=0.8\textwidth]{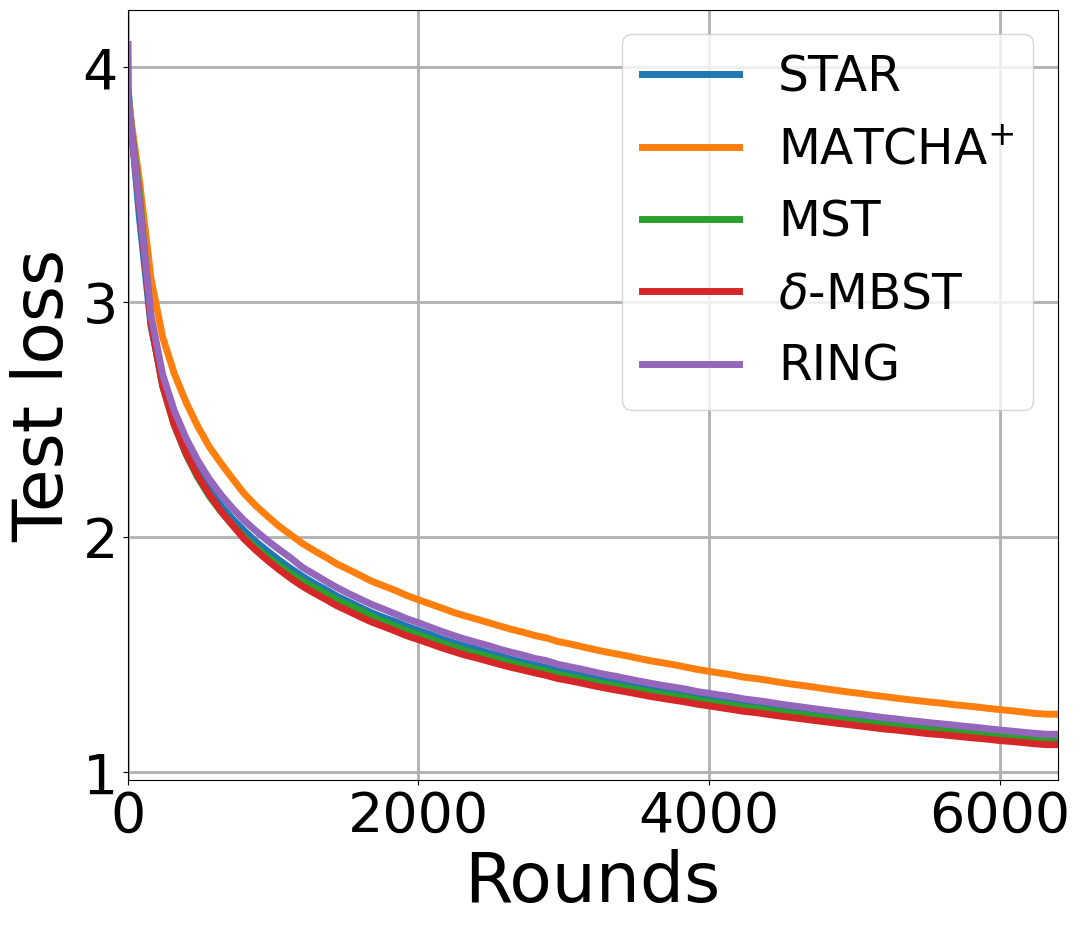}
    \end{subfigure}
    \hfill
    \begin{subfigure}[b]{0.24\textwidth}   
        \centering 
        \includegraphics[width=\textwidth, height=0.8\textwidth]{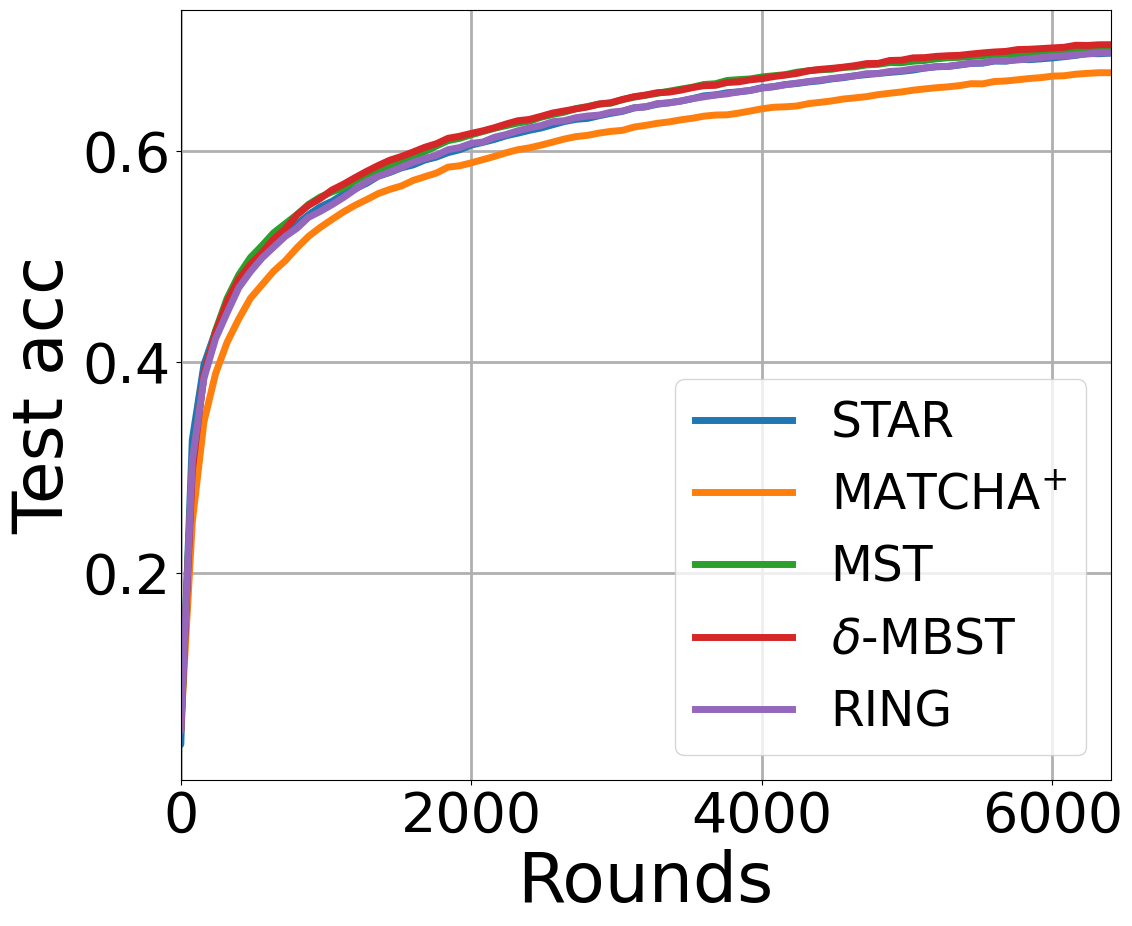}
    \end{subfigure}
    \\
    \begin{subfigure}[b]{0.24\textwidth}  
        \centering 
        \includegraphics[width=\textwidth, height=0.8\textwidth]{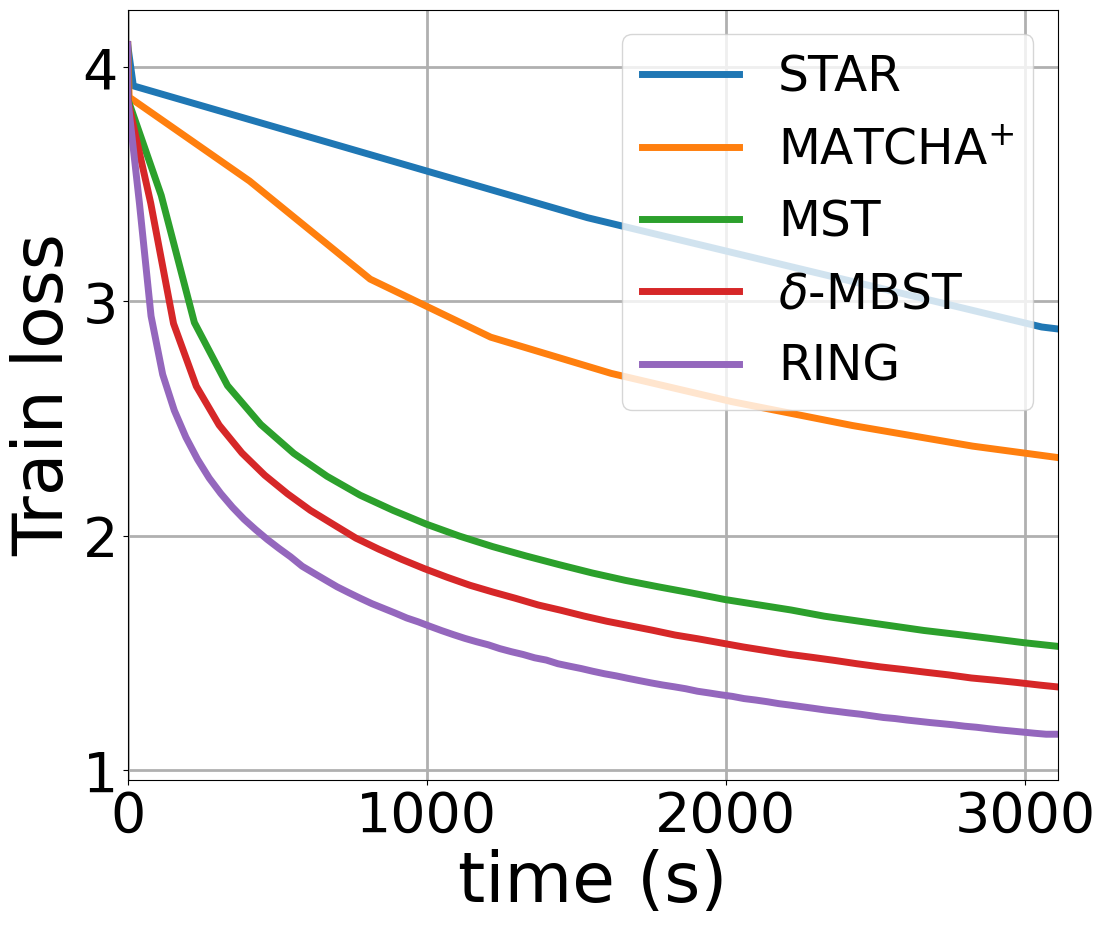}
        \caption[]{{\small Train Loss}}    
    \end{subfigure}
    \hfill
    \begin{subfigure}[b]{0.24\textwidth}
        \centering
        \includegraphics[width=\textwidth, height=0.8\textwidth]{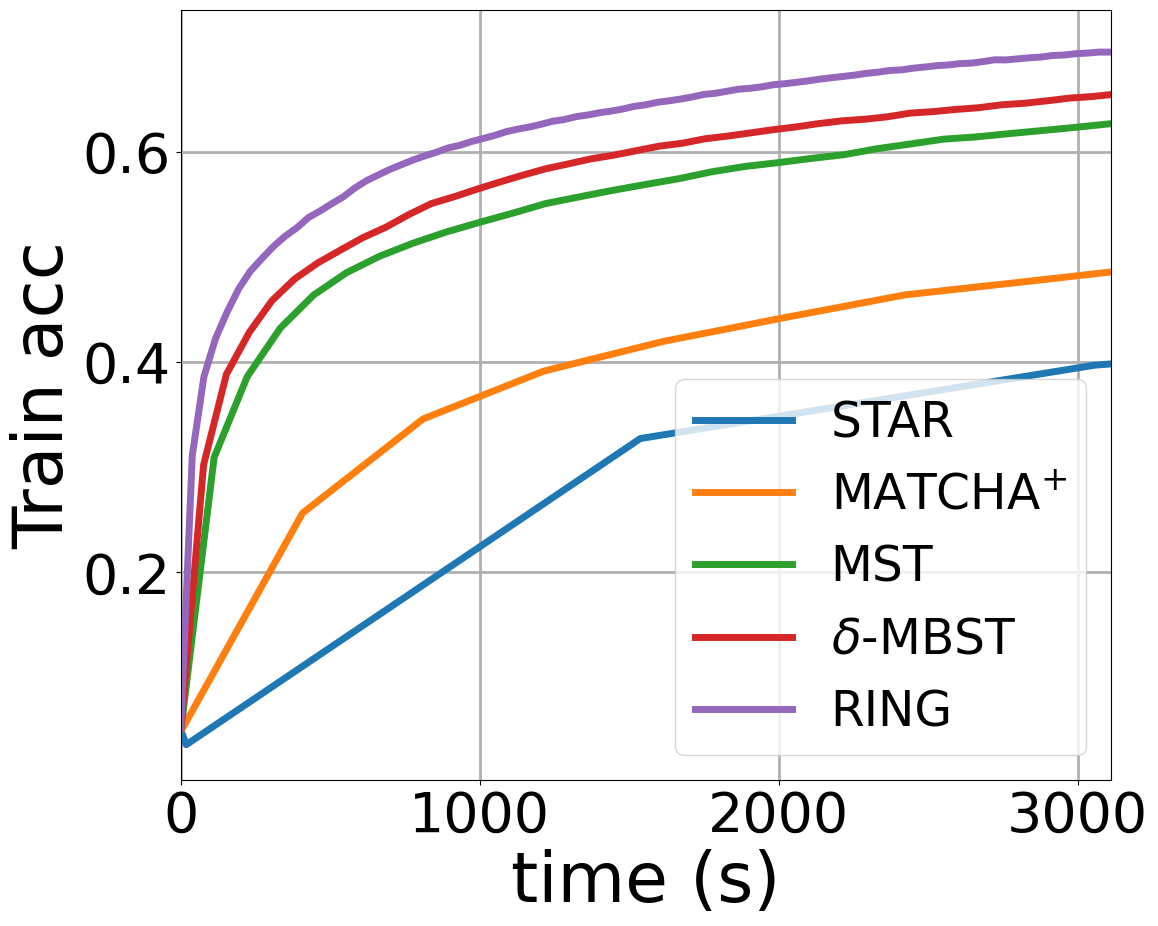}
        \caption[]{{\small Train Accuracy}}    
    \end{subfigure}
    \hfill
    \begin{subfigure}[b]{0.24\textwidth}   
        \centering 
        \includegraphics[width=\textwidth, height=0.8\textwidth]{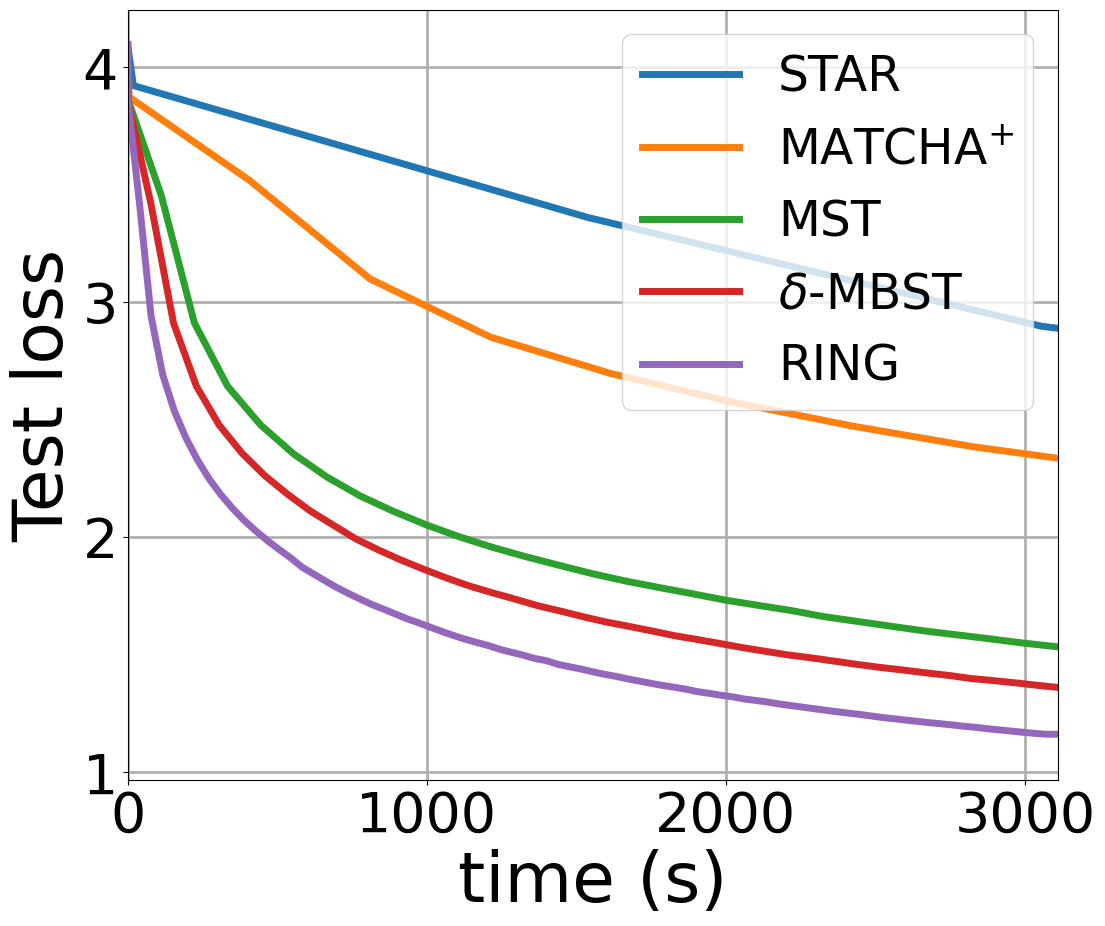}
        \caption[]{{\small Test Loss}}    
    \end{subfigure}
    \hfill
    \begin{subfigure}[b]{0.24\textwidth}   
        \centering 
        \includegraphics[width=\textwidth, height=0.8\textwidth]{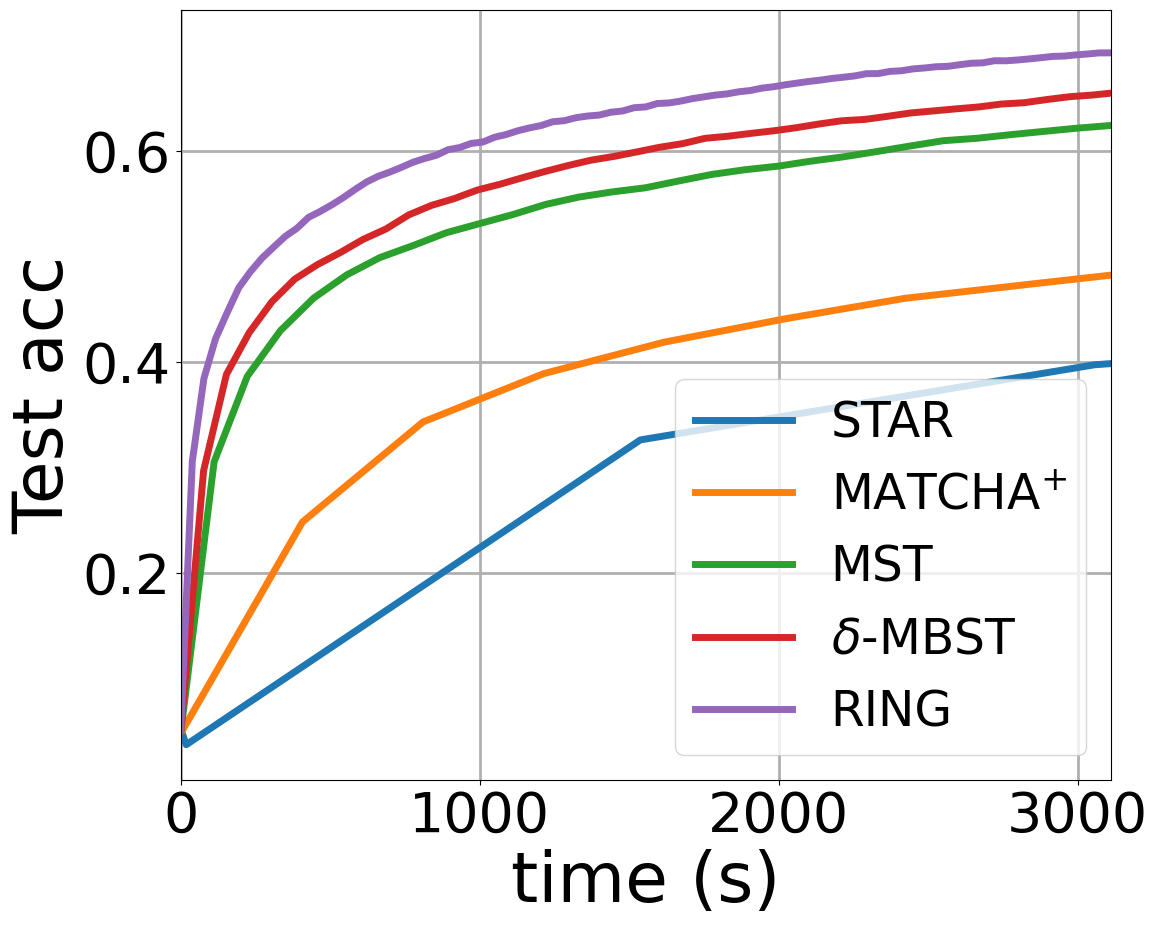}
        \caption[]{{\small Test Accuracy}}    
    \end{subfigure}
    \caption[]
    {\small Effect of overlays on the convergence w.r.t.~communication rounds (top row) and wall-clock time (bottom row) when training FEMNIST on AWS North America underlay. $1$~Gbps core links capacities, $100$~Mbps access links capacities, $s=1$.} 
    \label{f:femnist_aws}
\end{figure*}

\begin{figure*}
    \centering
    \begin{subfigure}[b]{0.24\textwidth}  
        \centering 
        \includegraphics[width=\textwidth, height=0.8\textwidth]{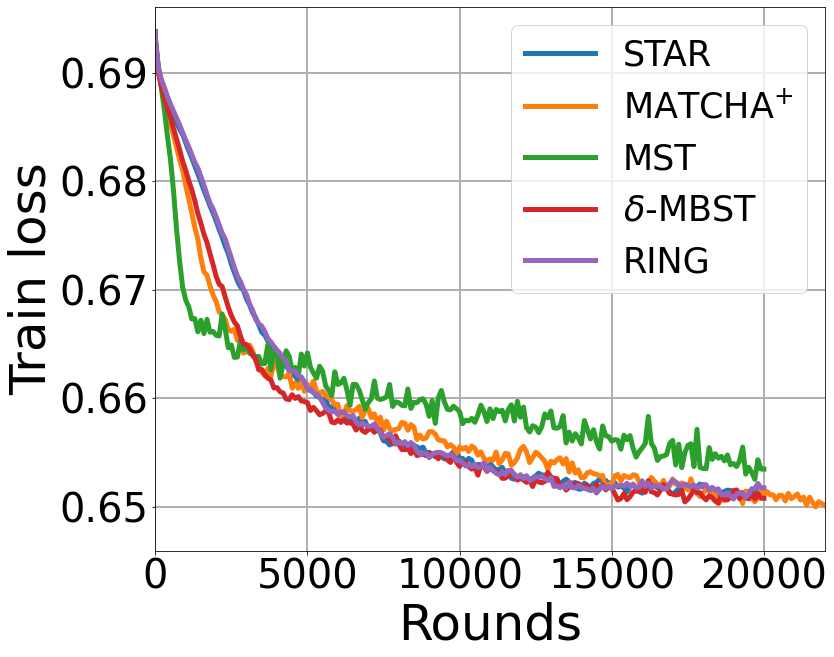}
    \end{subfigure}
    \hfill
    \begin{subfigure}[b]{0.24\textwidth}
        \centering
        \includegraphics[width=\textwidth, height=0.8\textwidth]{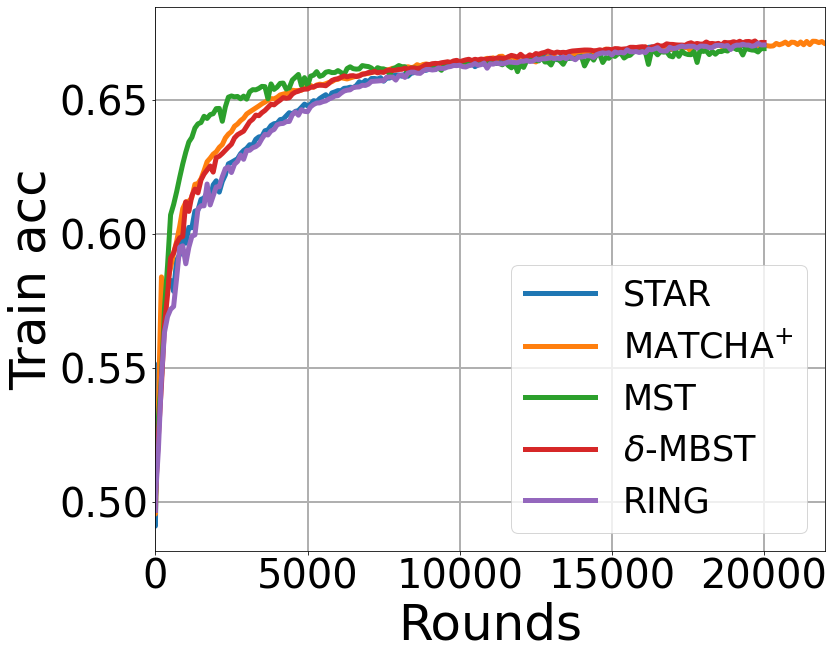}
    \end{subfigure}
    \hfill
    \begin{subfigure}[b]{0.24\textwidth}   
        \centering 
        \includegraphics[width=\textwidth, height=0.8\textwidth]{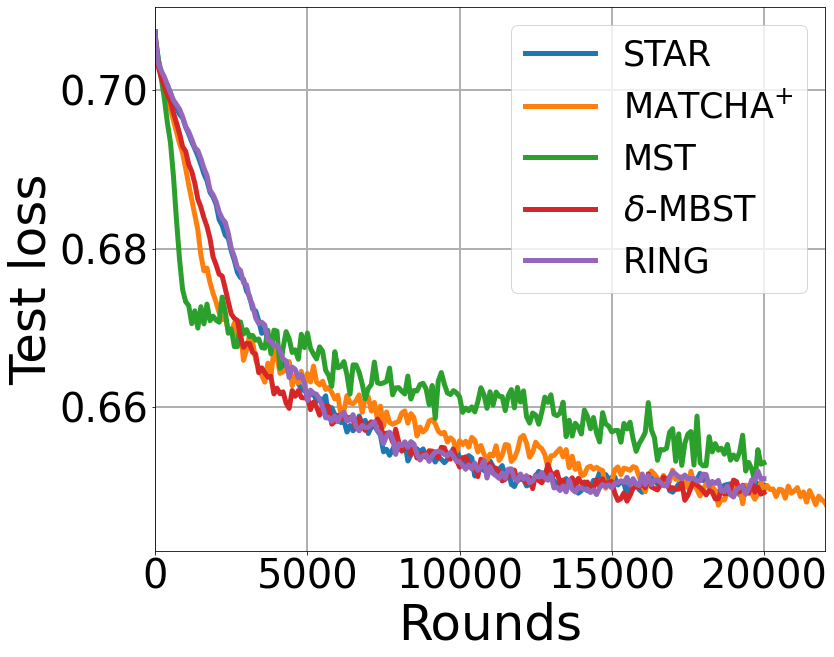}
    \end{subfigure}
    \hfill
    \begin{subfigure}[b]{0.24\textwidth}   
        \centering 
        \includegraphics[width=\textwidth, height=0.8\textwidth]{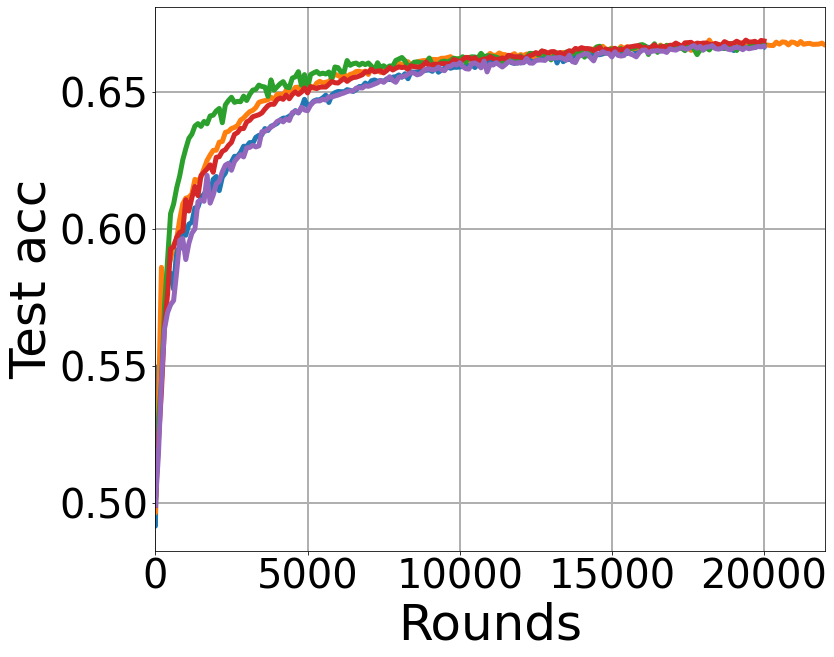}
    \end{subfigure}
    \\    
    \begin{subfigure}[b]{0.24\textwidth}  
        \centering 
        \includegraphics[width=\textwidth, height=0.8\textwidth]{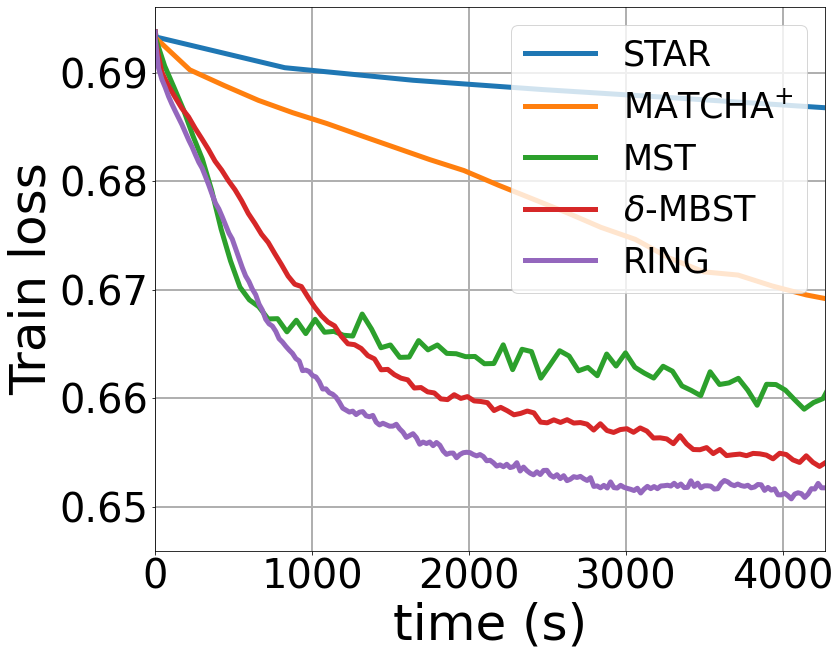}
        \caption[]{{\small Train Loss}}    
    \end{subfigure}
    \hfill
    \begin{subfigure}[b]{0.24\textwidth}
        \centering
        \includegraphics[width=\textwidth, height=0.8\textwidth]{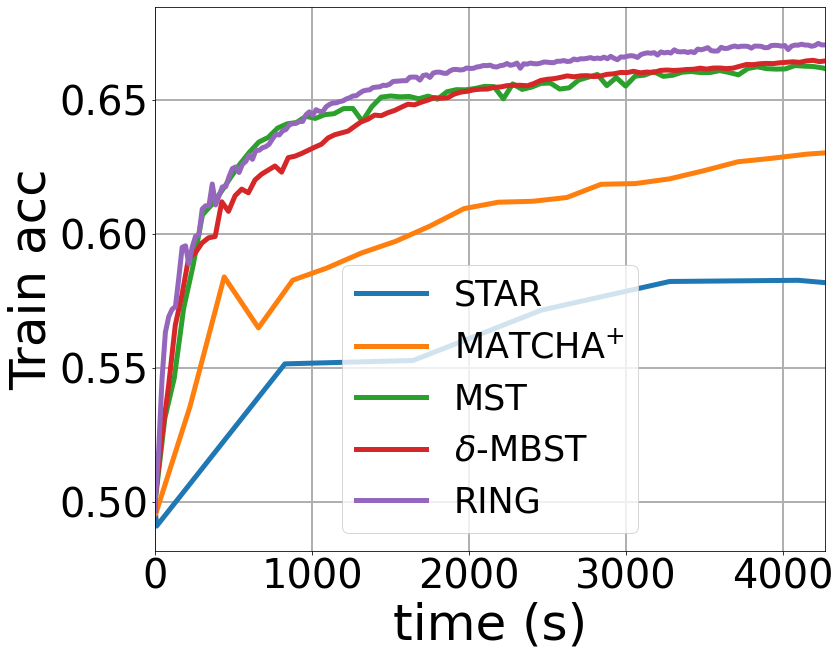}
        \caption[]{{\small Train Accuracy}}    
    \end{subfigure}
    \hfill
    \begin{subfigure}[b]{0.24\textwidth}   
        \centering 
        \includegraphics[width=\textwidth, height=0.8\textwidth]{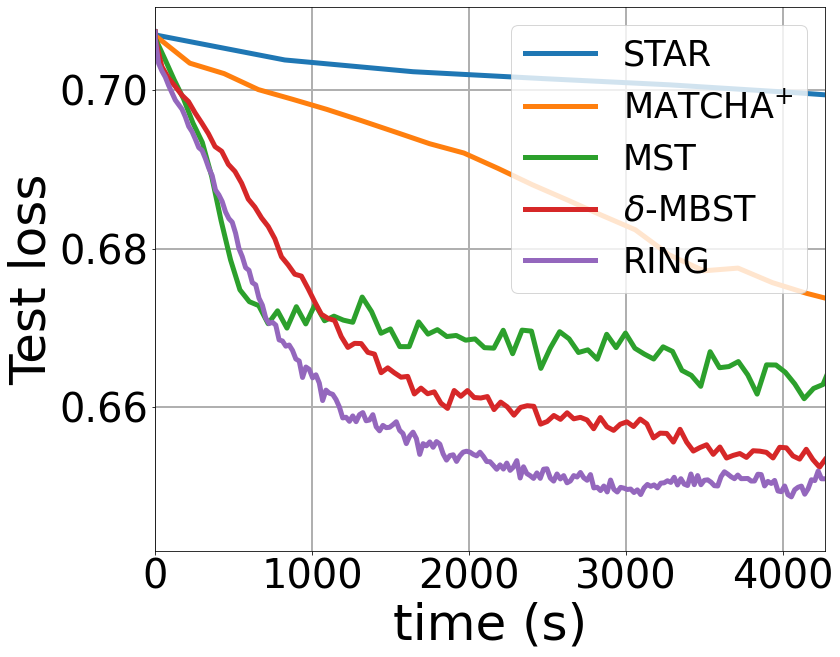}
        \caption[]{{\small Test Loss}}
    \end{subfigure}
    \hfill
    \begin{subfigure}[b]{0.24\textwidth}   
        \centering 
        \includegraphics[width=\textwidth, height=0.8\textwidth]{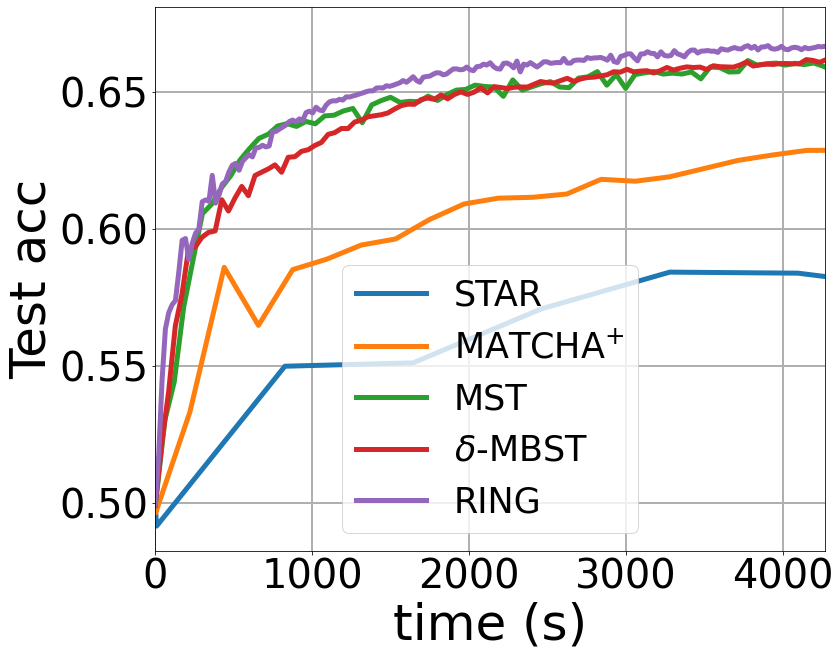}
        \caption[]{{\small Test Accuracy}} 
    \end{subfigure}
    \caption[]
    {\small Effect of overlays on the convergence w.r.t.~communication rounds (top row) and wall-clock time (bottom row) when training Sentiment140 on AWS North America underlay. $1$~Gbps core links capacities, $100$~Mbps access links capacities, $s=1$.} 
    \label{f:senti_aws}
\end{figure*}

\begin{figure*}
    \centering
    \begin{subfigure}[b]{0.24\textwidth}  
        \centering 
        \includegraphics[width=\textwidth, height=0.8\textwidth]{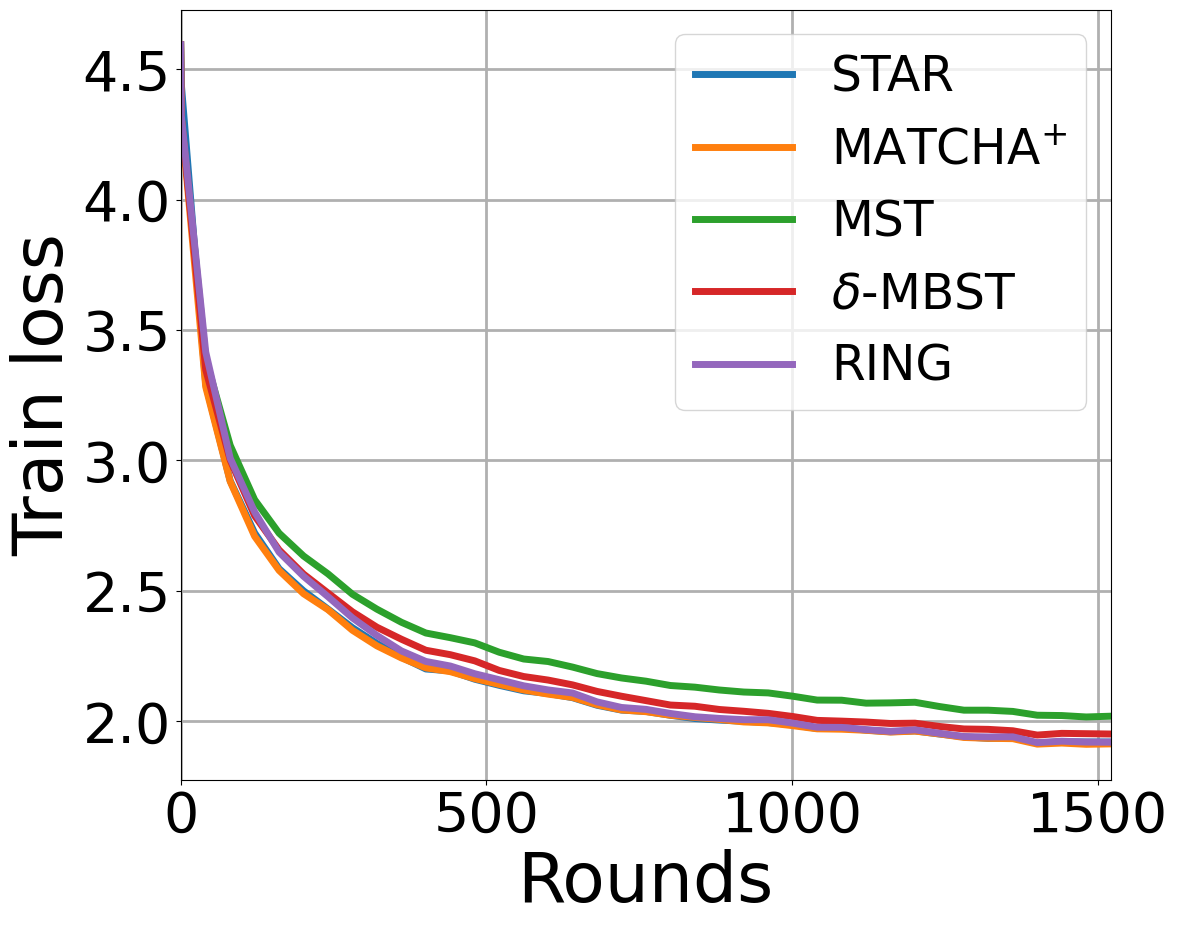}
    \end{subfigure}
    \hfill
    \begin{subfigure}[b]{0.24\textwidth}
        \centering
        \includegraphics[width=\textwidth, height=0.8\textwidth]{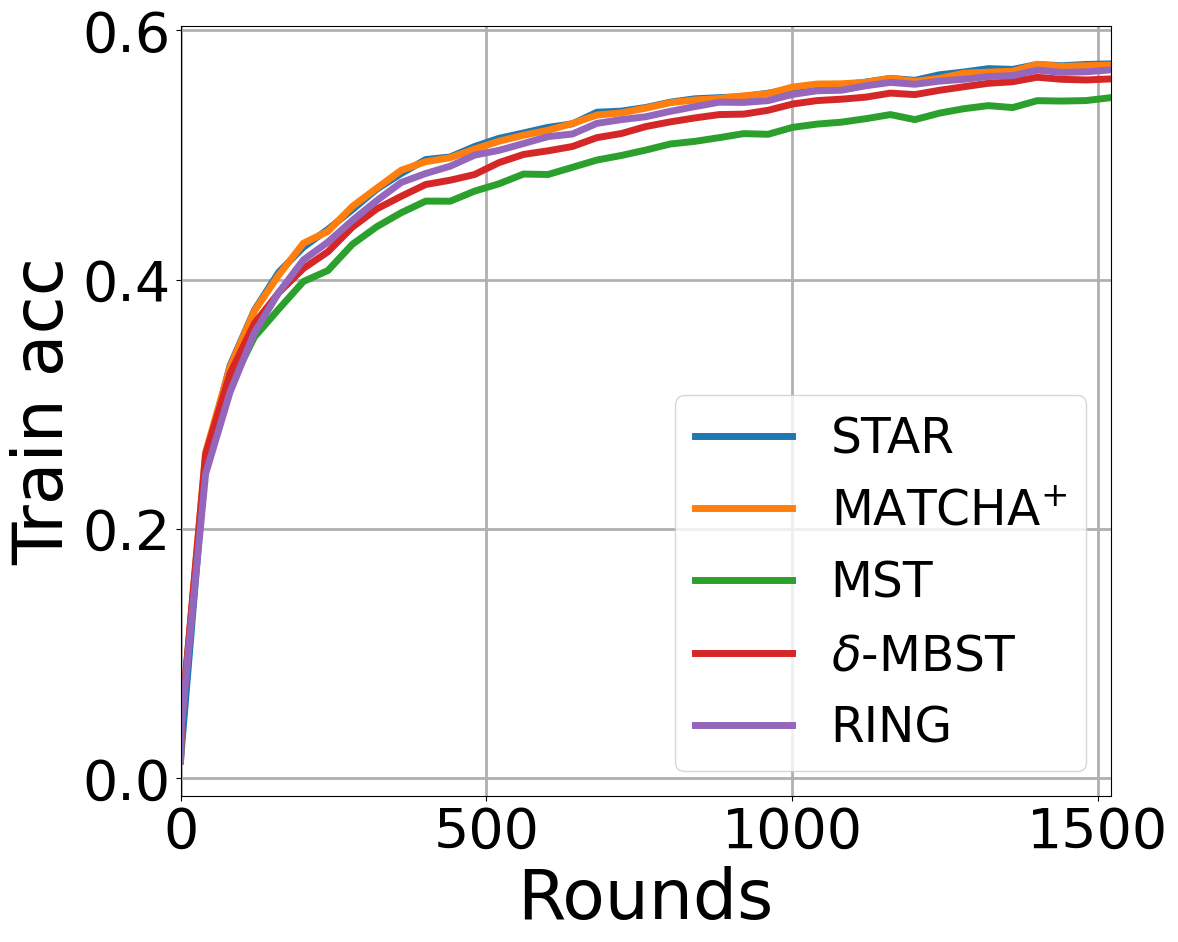}
    \end{subfigure}
    \hfill
    \begin{subfigure}[b]{0.24\textwidth}   
        \centering 
        \includegraphics[width=\textwidth, height=0.8\textwidth]{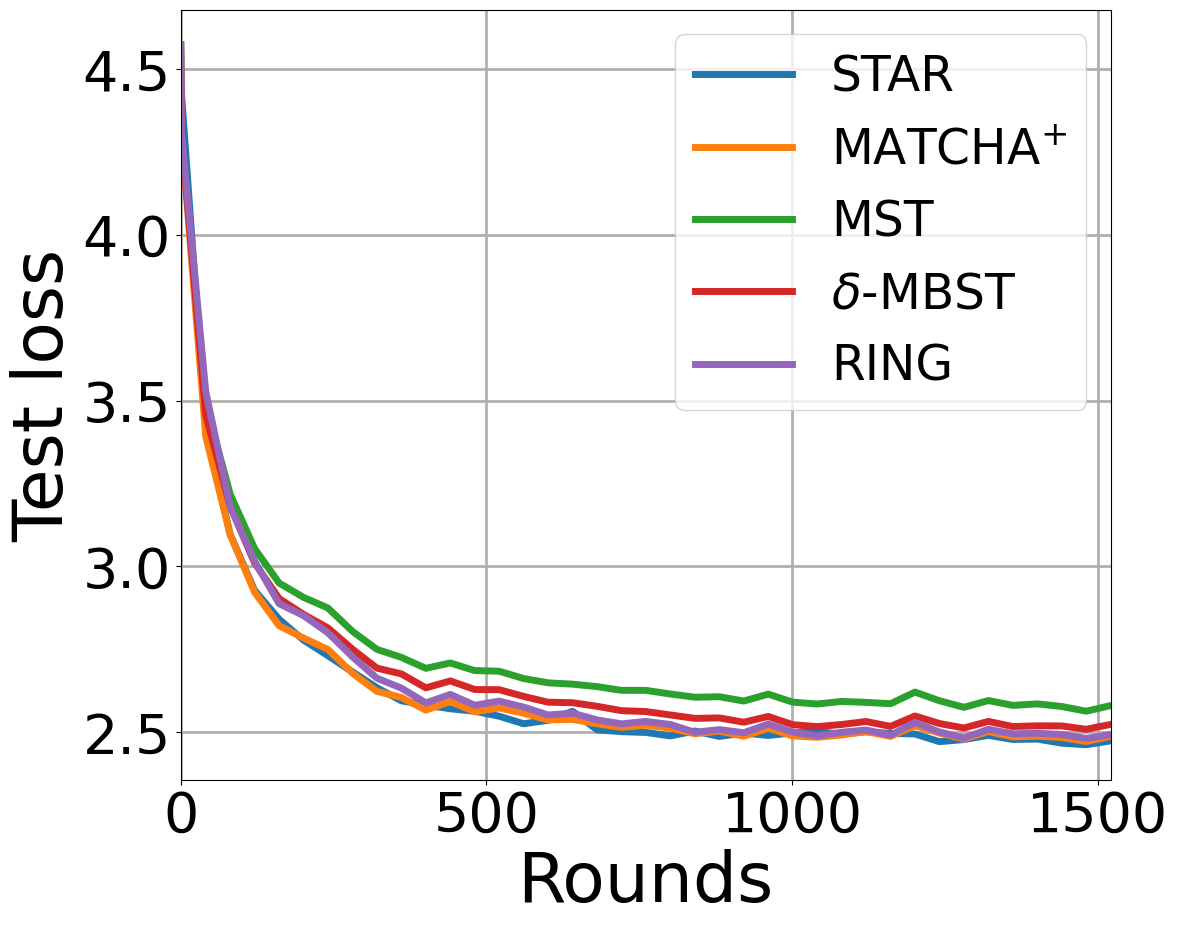}
    \end{subfigure}
    \hfill
    \begin{subfigure}[b]{0.24\textwidth}   
        \centering 
        \includegraphics[width=\textwidth, height=0.8\textwidth]{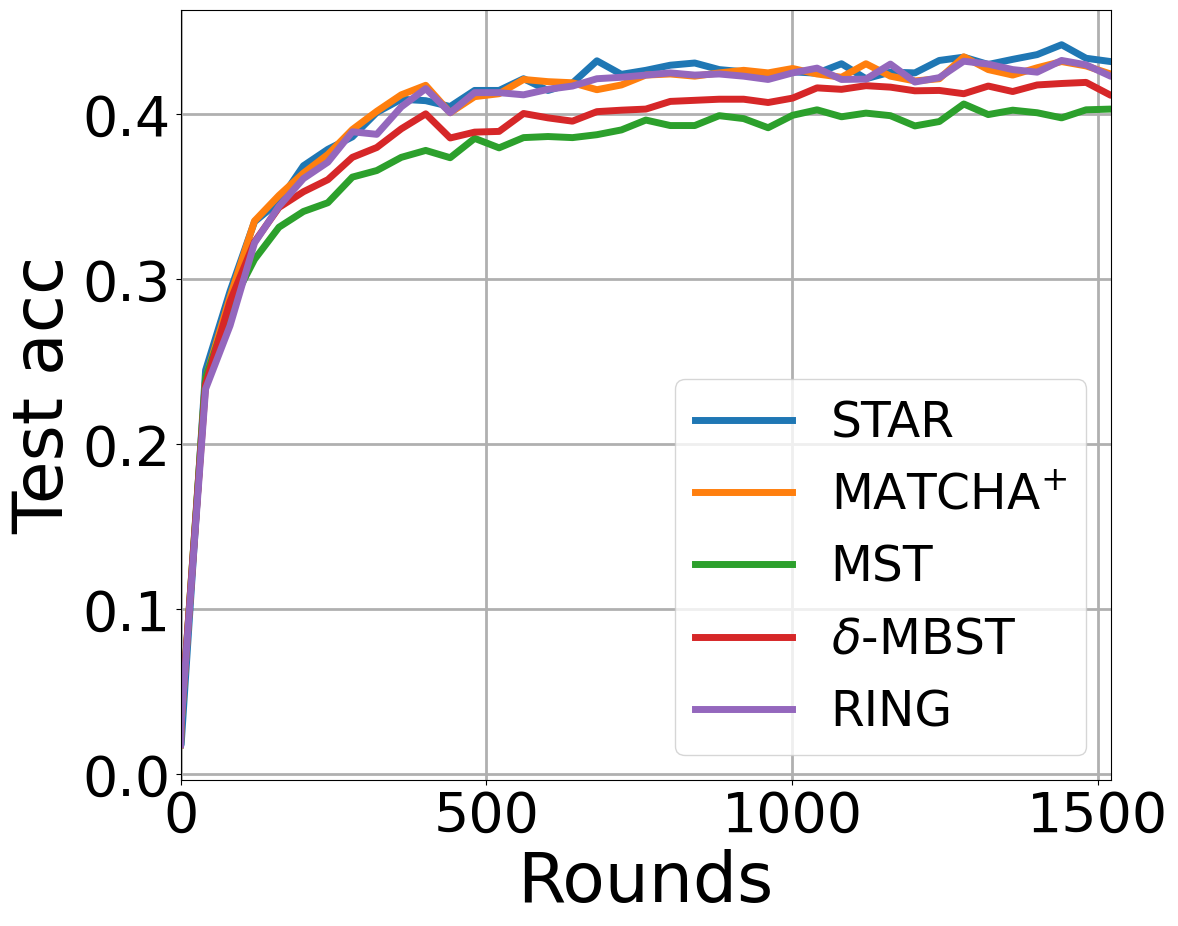}
    \end{subfigure}
    \\    
    \begin{subfigure}[b]{0.24\textwidth}  
        \centering 
        \includegraphics[width=\textwidth, height=0.8\textwidth]{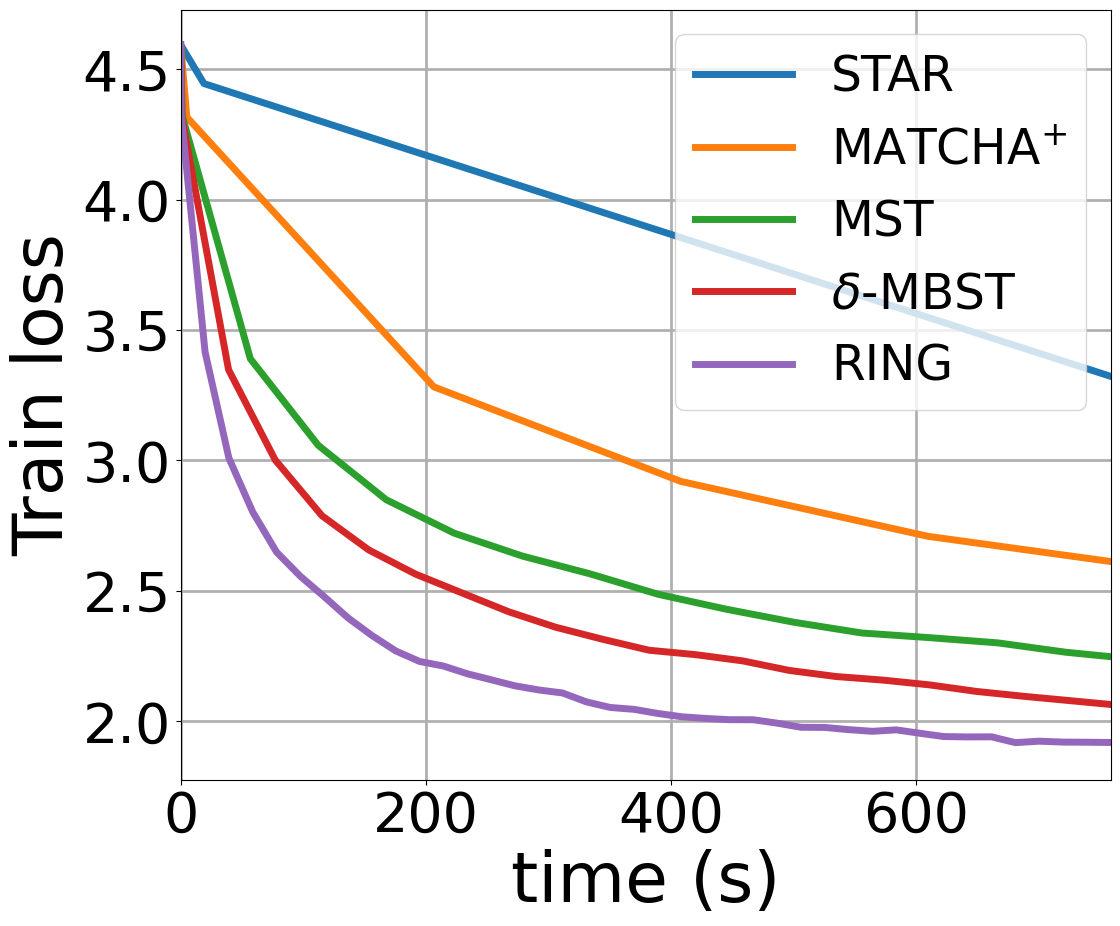}
        \caption[]{{\small Train Loss}}
    \end{subfigure}
    \hfill
    \begin{subfigure}[b]{0.24\textwidth}
        \centering
        \includegraphics[width=\textwidth, height=0.8\textwidth]{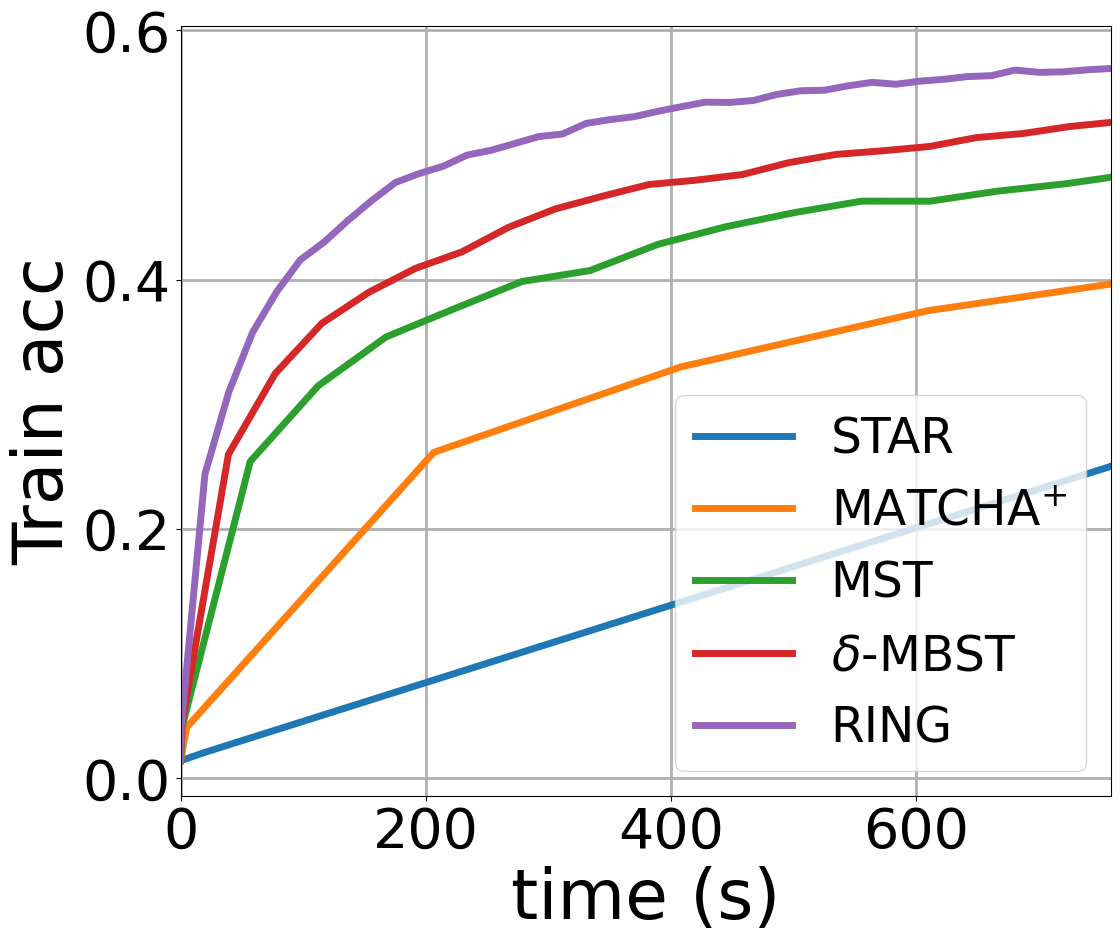}
        \caption[]{{\small Train Accuracy}}
    \end{subfigure}
    \hfill
    \begin{subfigure}[b]{0.24\textwidth}   
        \centering 
        \includegraphics[width=\textwidth, height=0.8\textwidth]{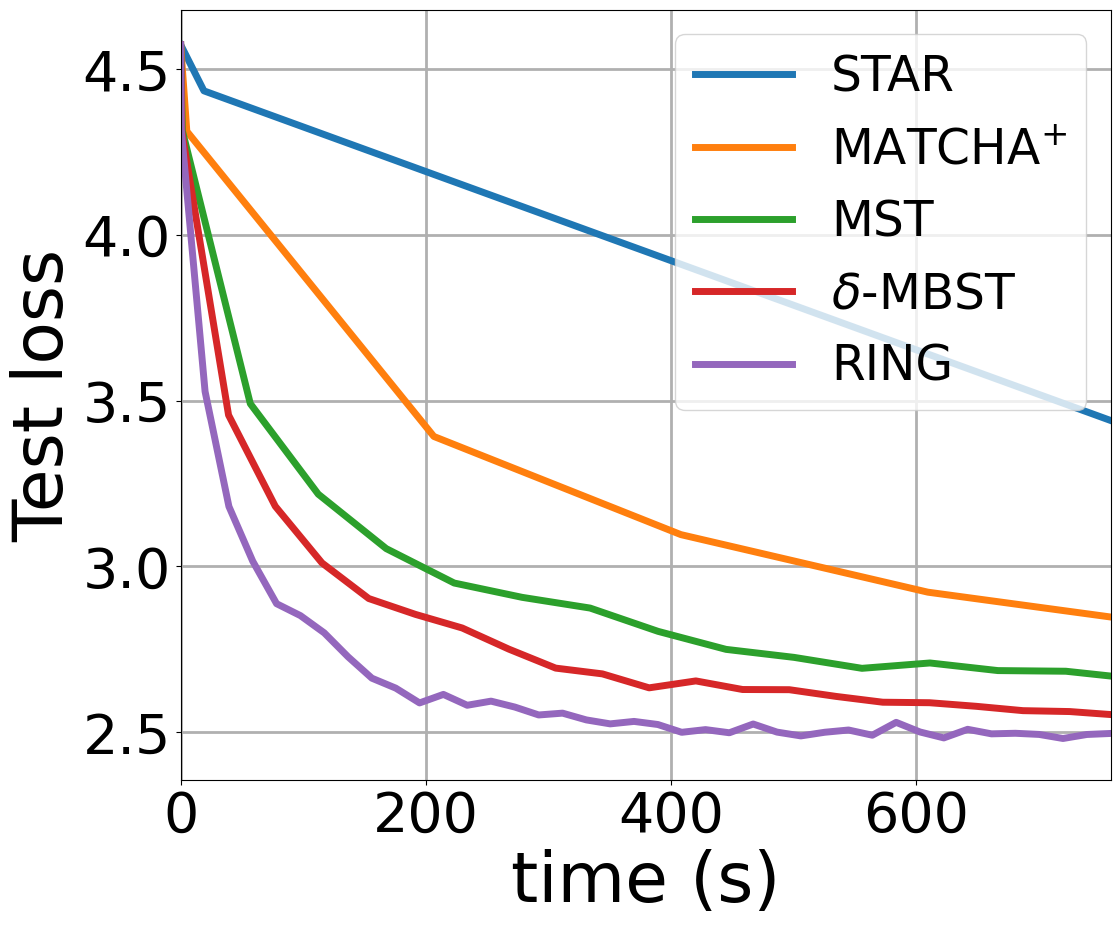}
        \caption[]{{\small Test Loss}} 
    \end{subfigure}
    \hfill
    \begin{subfigure}[b]{0.24\textwidth}   
        \centering 
        \includegraphics[width=\textwidth, height=0.8\textwidth]{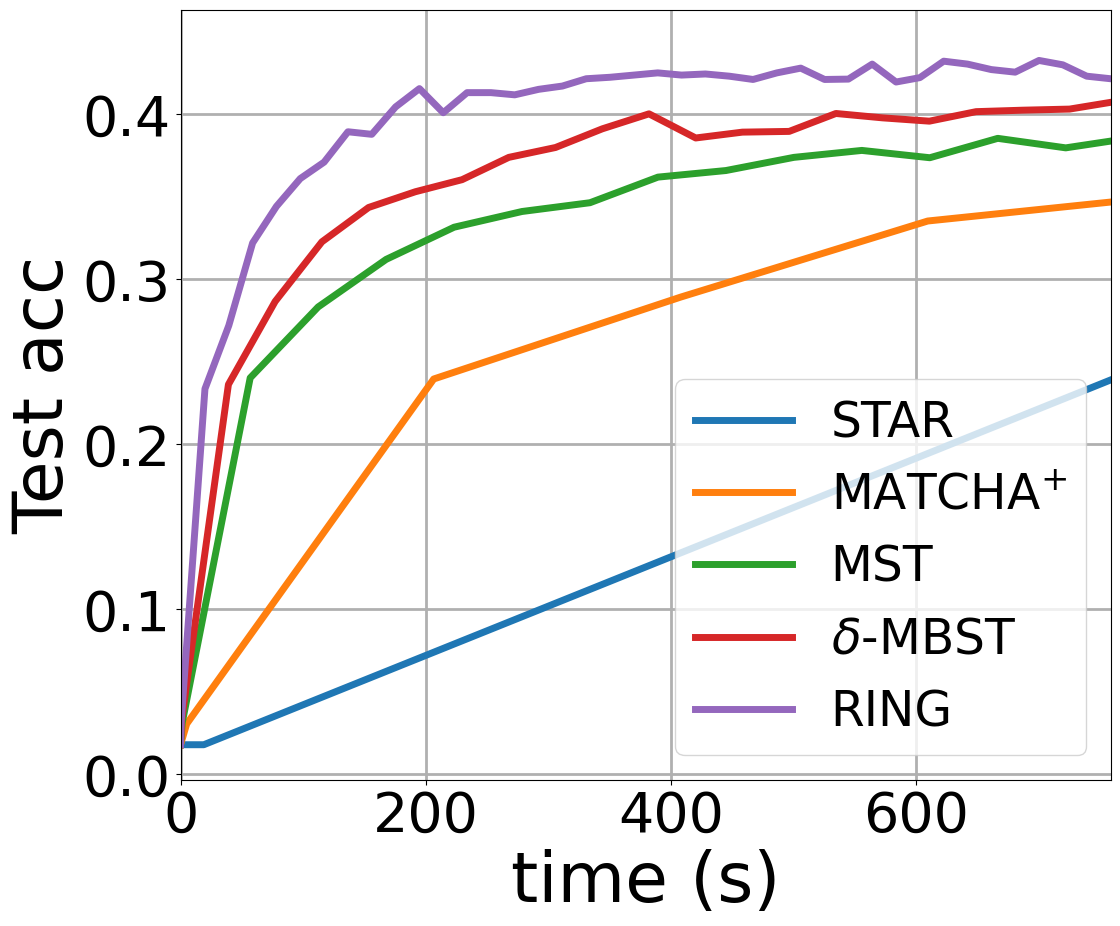}
        \caption[]{{\small Test Accuracy}}    
    \end{subfigure}
    \caption[]
    {\small Effect of overlays on the convergence w.r.t.~communication rounds (top row) and wall-clock time (bottom row) when training iNaturalist on AWS North America underlay. $1$~Gbps core links capacities, $100$~Mbps access links capacities, $s=1$.} 
    \label{f:inaturalist_aws}
\end{figure*}

\subsection{Exploring other scenarios} 
In our experiments, we considered 5  underlays, for which we compared 6 different overlays (e.g.,~Table~\ref{tab:topologies_cycle_time}). Moreover, we tested 4 different datasets (e.g.,~Fig.~\ref{f:training_for_all_data_sets}) and 3 different values for the number of local steps $s=1, 5, 10$ (e.g.,~Tables~\ref{tab:topologies_cycle_time_local_step_5} and~\ref{tab:topologies_cycle_time_local_step_10}). We were not able to run experiments for all 360 possible combinations. In Figures~\ref{f:gaia_s_1}--\ref{f:ebone_s_5}, we show some representative additional results. For each experimental result, four metrics are shown including the train loss, train accuracy, test loss, and test accuracy w.r.t. communication rounds and wall-clock time. The common observation is that the RING converges faster than MATCHA$^+$ and STAR in terms of wall-clock time.  In some cases, the test loss and accuracy of the model learned by the RING start becoming worse after some time, with overfitting being a possible explanation  in some cases (see Figs.~\ref{f:gaia_s_1},~\ref{f:geant_s_1},~\ref{f:gaia_s_5}, and~\ref{f:geant_s_5}). 

\begin{figure*}
    \centering
    \begin{subfigure}[b]{0.24\textwidth}  
        \centering 
        \includegraphics[width=\textwidth, height=0.8\textwidth]{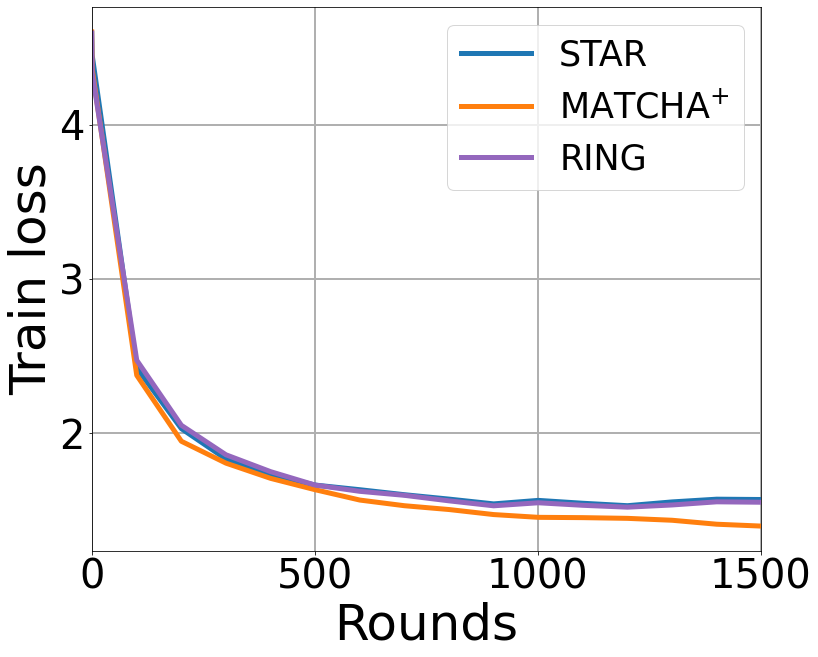}
    \end{subfigure}
    \hfill
    \begin{subfigure}[b]{0.24\textwidth}
        \centering
        \includegraphics[width=\textwidth, height=0.8\textwidth]{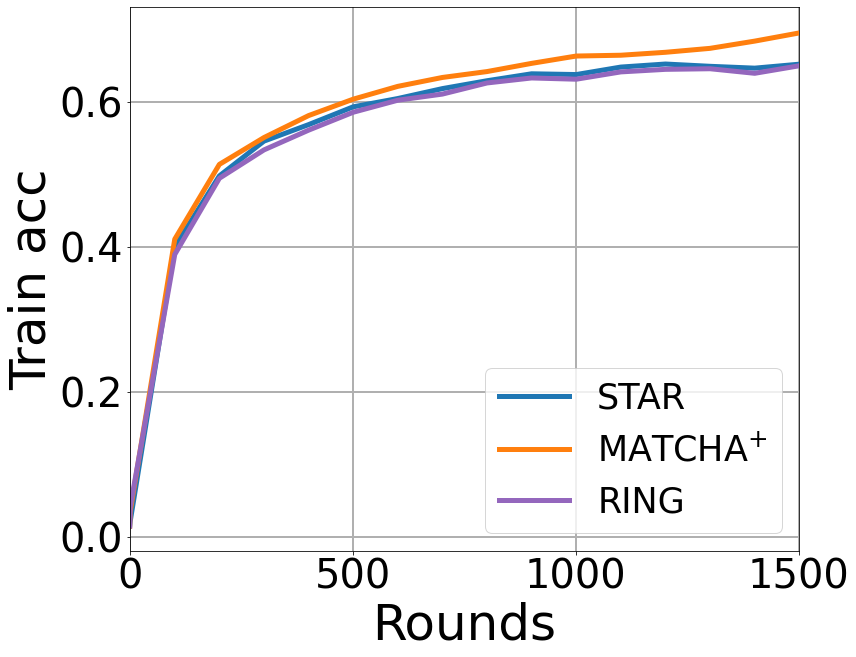}
    \end{subfigure}
    \hfill
    \begin{subfigure}[b]{0.24\textwidth}   
        \centering 
        \includegraphics[width=\textwidth, height=0.8\textwidth]{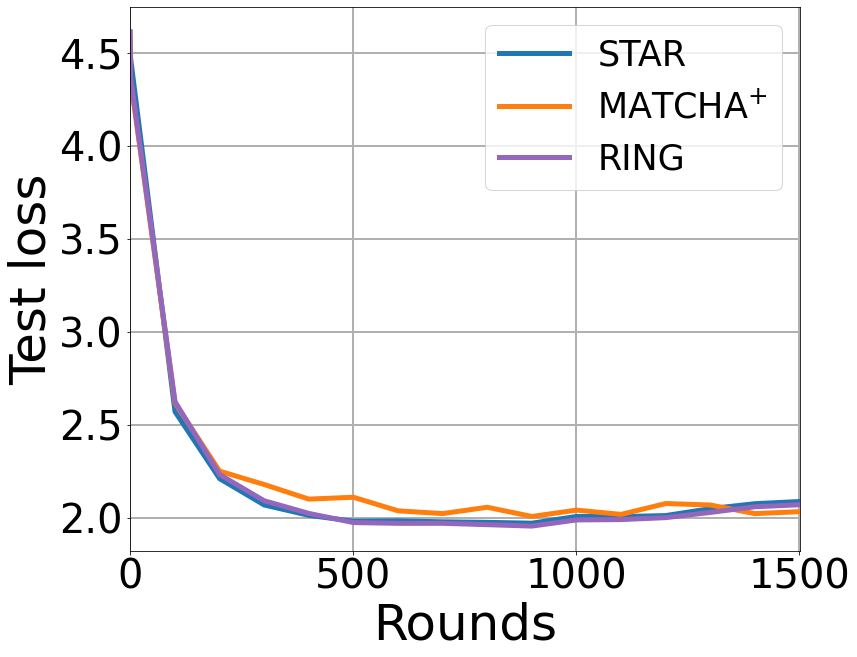}
    \end{subfigure}
    \hfill
    \begin{subfigure}[b]{0.24\textwidth}   
        \centering 
         \includegraphics[width=\textwidth, height=0.8\textwidth]{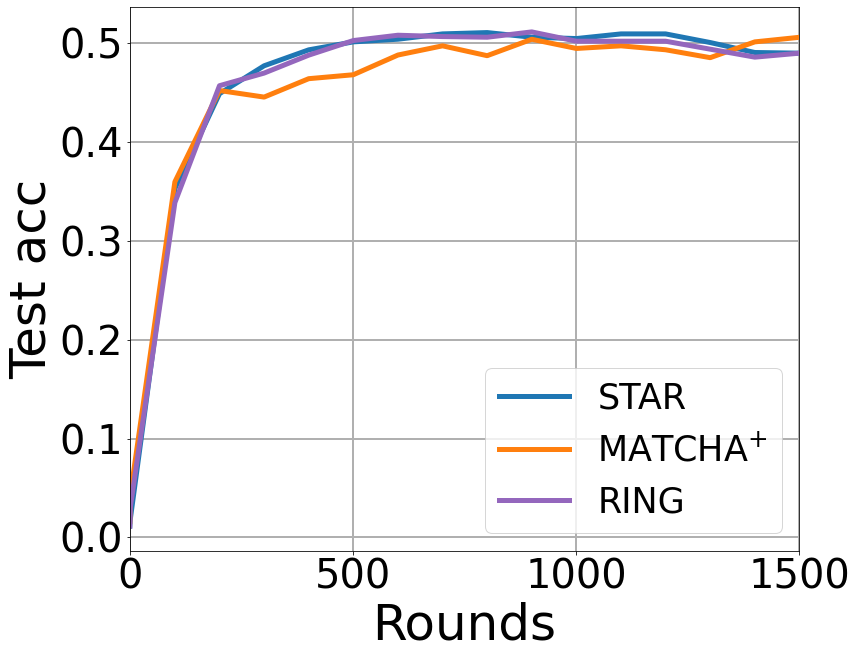}
    \end{subfigure}
    \\    
    \begin{subfigure}[b]{0.24\textwidth}  
        \centering 
        \includegraphics[width=\textwidth, height=0.8\textwidth]{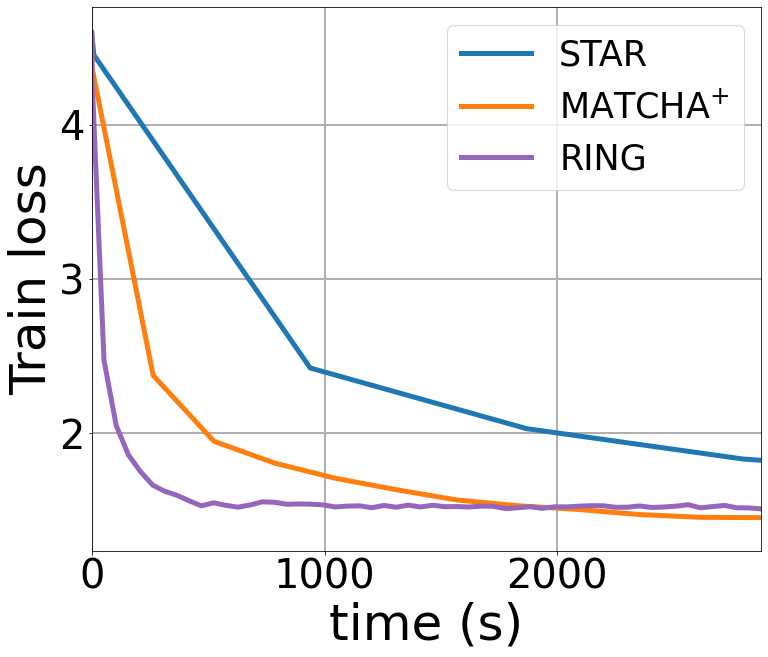}
        \caption[]{{\small Train Loss}}    
    \end{subfigure}
    \hfill
    \begin{subfigure}[b]{0.24\textwidth}
        \centering
        \includegraphics[width=\textwidth, height=0.8\textwidth]{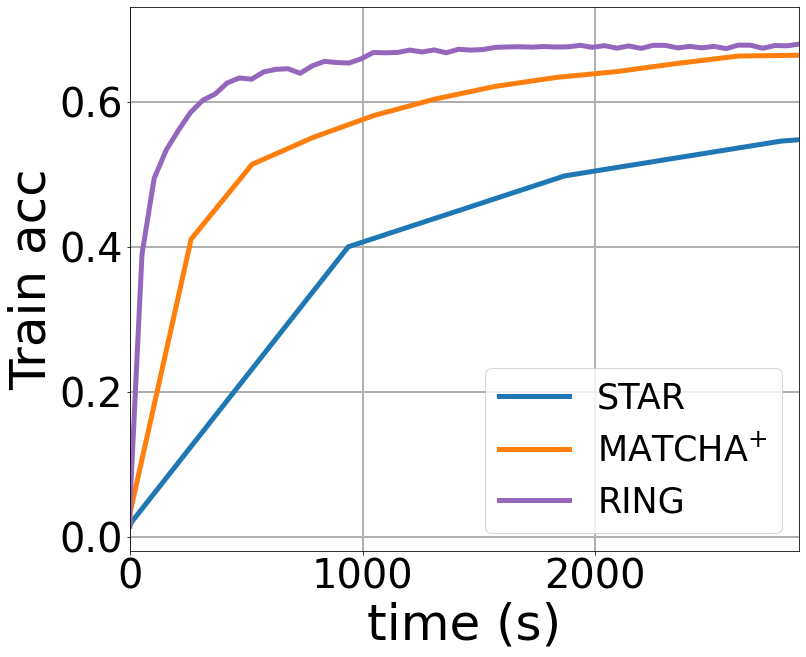}
        \caption[]{{\small Train Accuracy}}    
    \end{subfigure}
    \hfill
    \begin{subfigure}[b]{0.24\textwidth}   
        \centering 
        \includegraphics[width=\textwidth, height=0.8\textwidth]{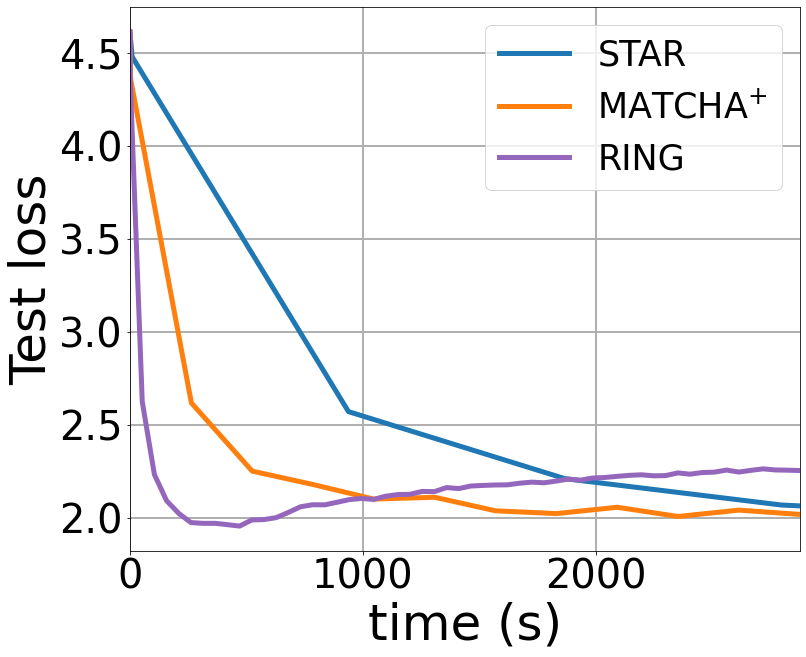}
        \caption[]{{\small Test Loss}}    
    \end{subfigure}
    \hfill
    \begin{subfigure}[b]{0.24\textwidth}   
        \centering 
         \includegraphics[width=\textwidth, height=0.8\textwidth]{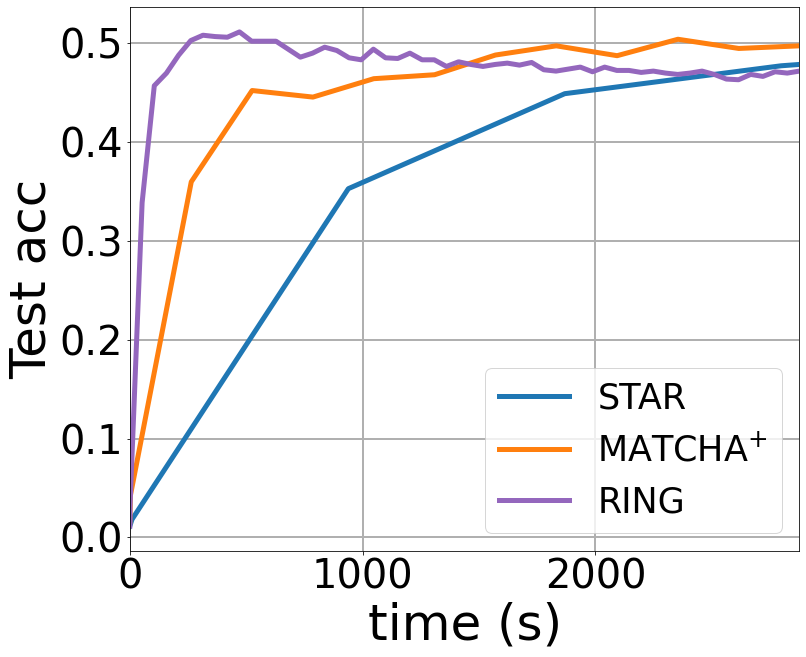}
        \caption[]{{\small Test Accuracy}}    
    \end{subfigure}
    \caption[]
    {\small Effect of overlays on the convergence w.r.t.~communication rounds (top row) and wall-clock time (bottom row) when training ResNet-18 image classification model using iNaturalist on Gaia underlay. $1$~Gbps core links capacities, $100$~Mbps access links capacities, $s=1$.} 
    \label{f:gaia_s_1}
\end{figure*}

\begin{figure*}
    \centering
    \begin{subfigure}[b]{0.24\textwidth}  
        \centering 
        \includegraphics[width=\textwidth, height=0.8\textwidth]{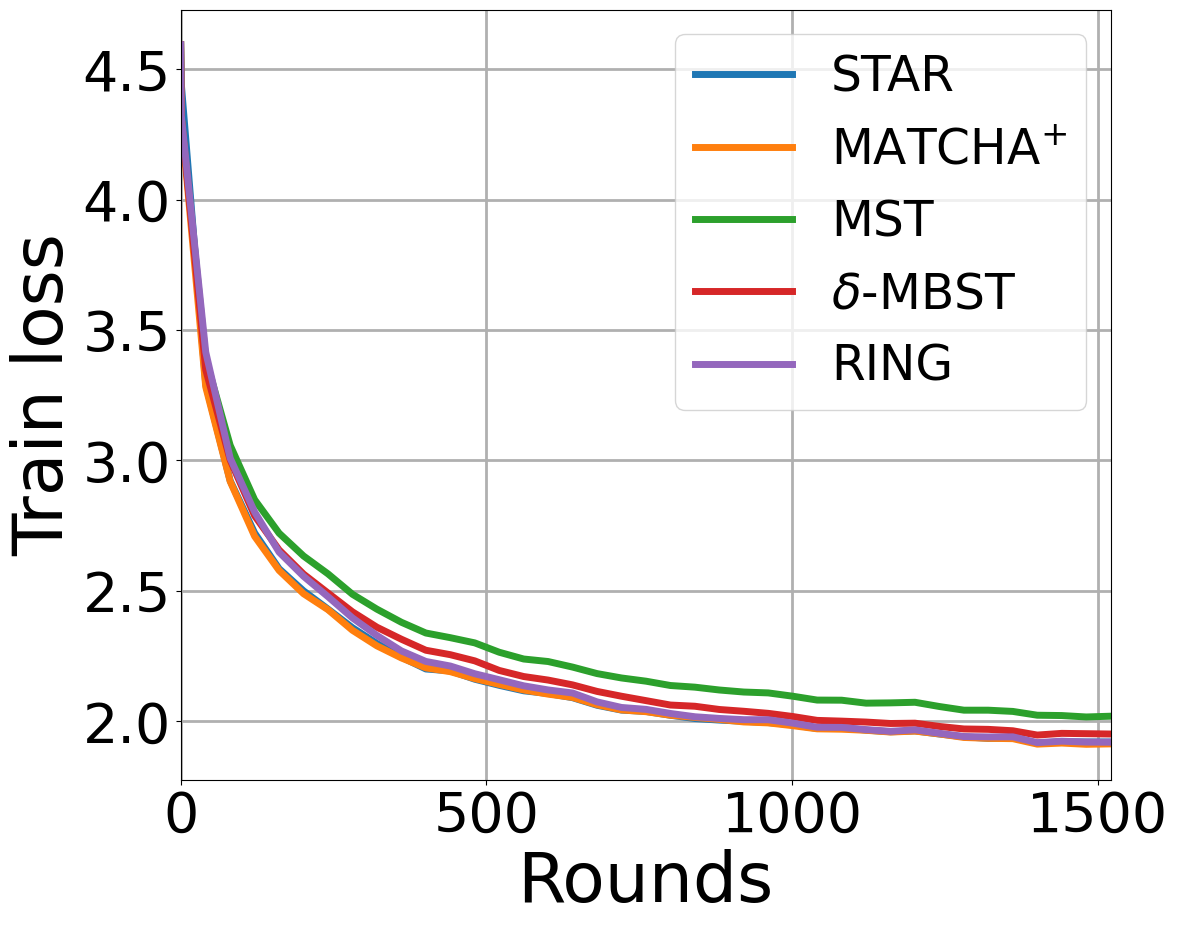}
    \end{subfigure}
    \hfill
    \begin{subfigure}[b]{0.24\textwidth}
        \centering
        \includegraphics[width=\textwidth, height=0.8\textwidth]{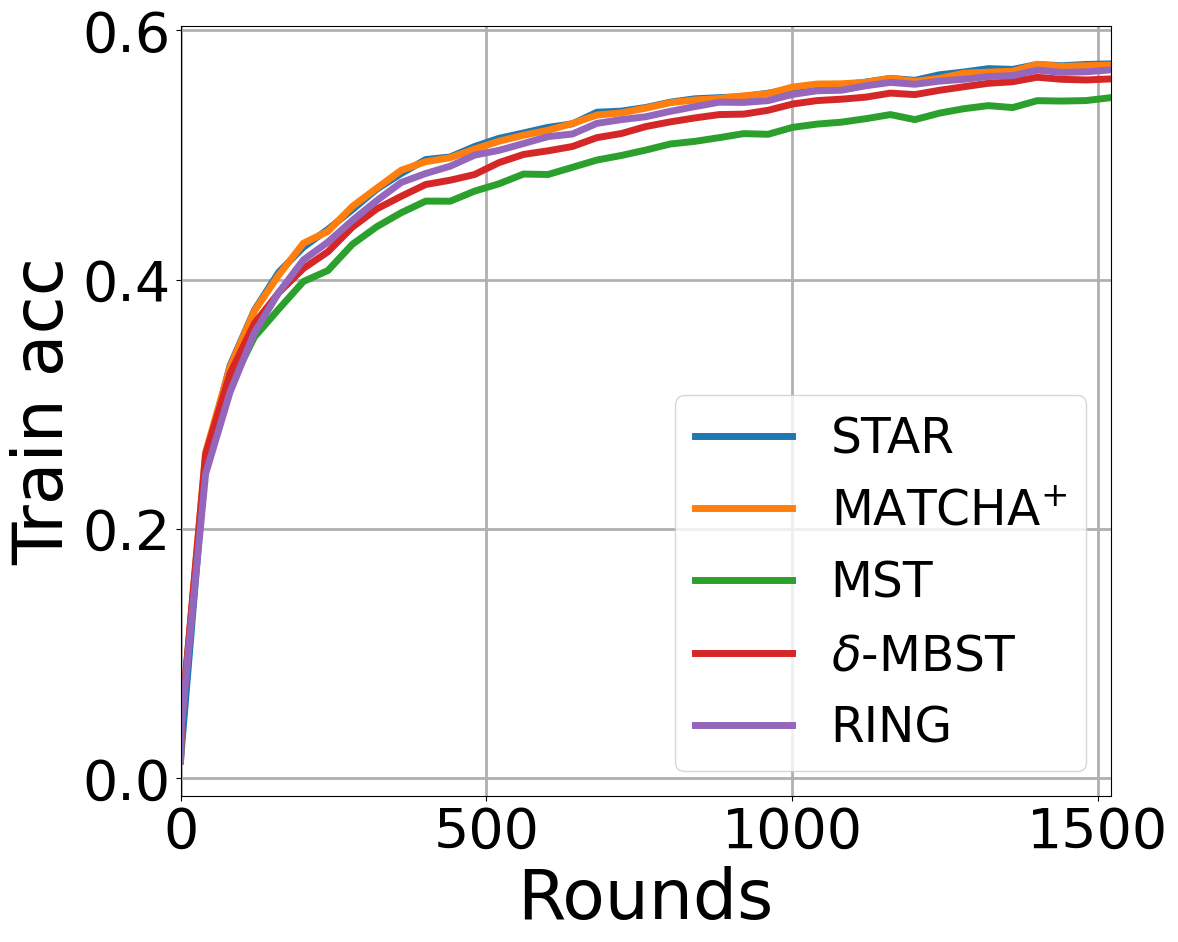}
    \end{subfigure}
    \hfill
    \begin{subfigure}[b]{0.24\textwidth}   
        \centering 
        \includegraphics[width=\textwidth, height=0.8\textwidth]{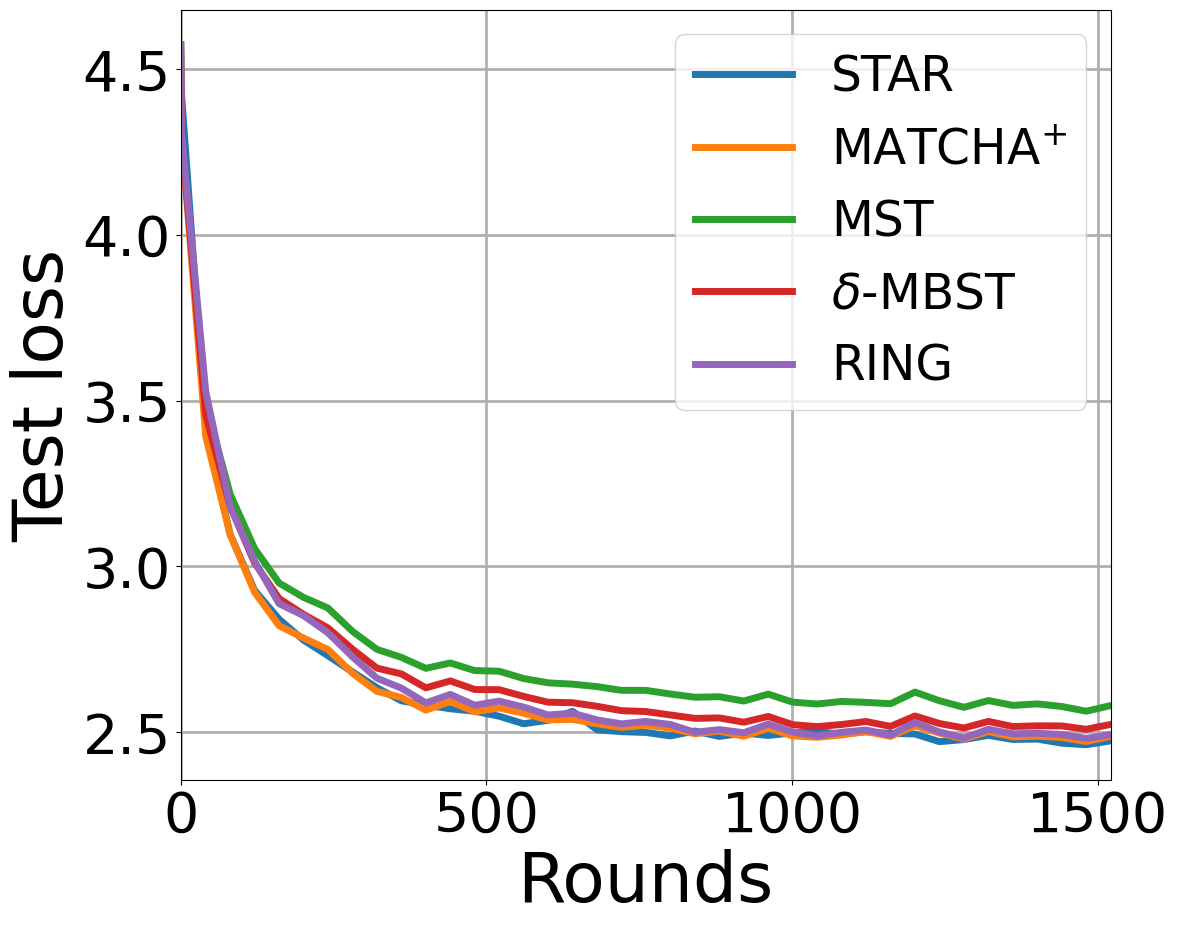}
    \end{subfigure}
    \hfill
    \begin{subfigure}[b]{0.24\textwidth}   
        \centering 
        \includegraphics[width=\textwidth, height=0.8\textwidth]{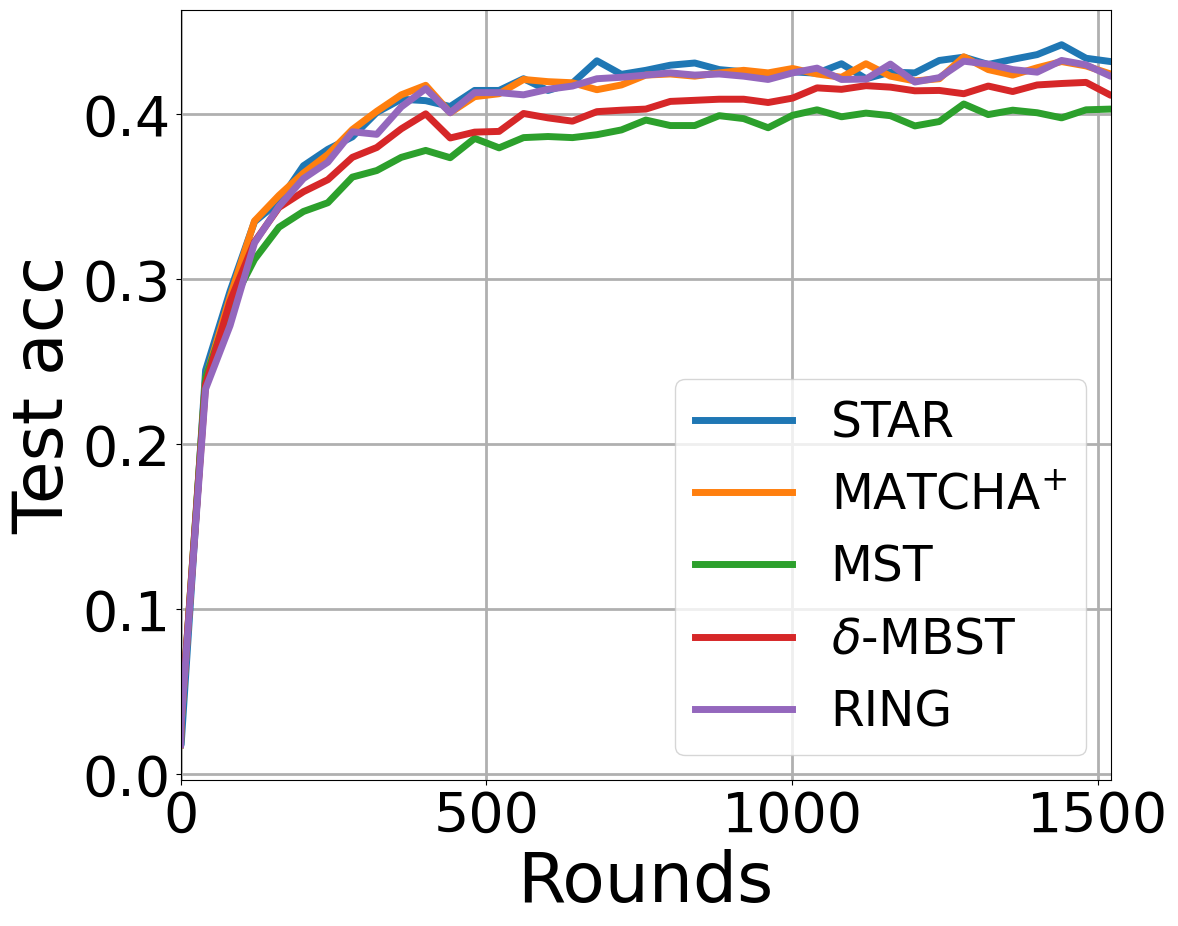}
    \end{subfigure}
    \\   
    \begin{subfigure}[b]{0.24\textwidth}  
        \centering 
        \includegraphics[width=\textwidth, height=0.8\textwidth]{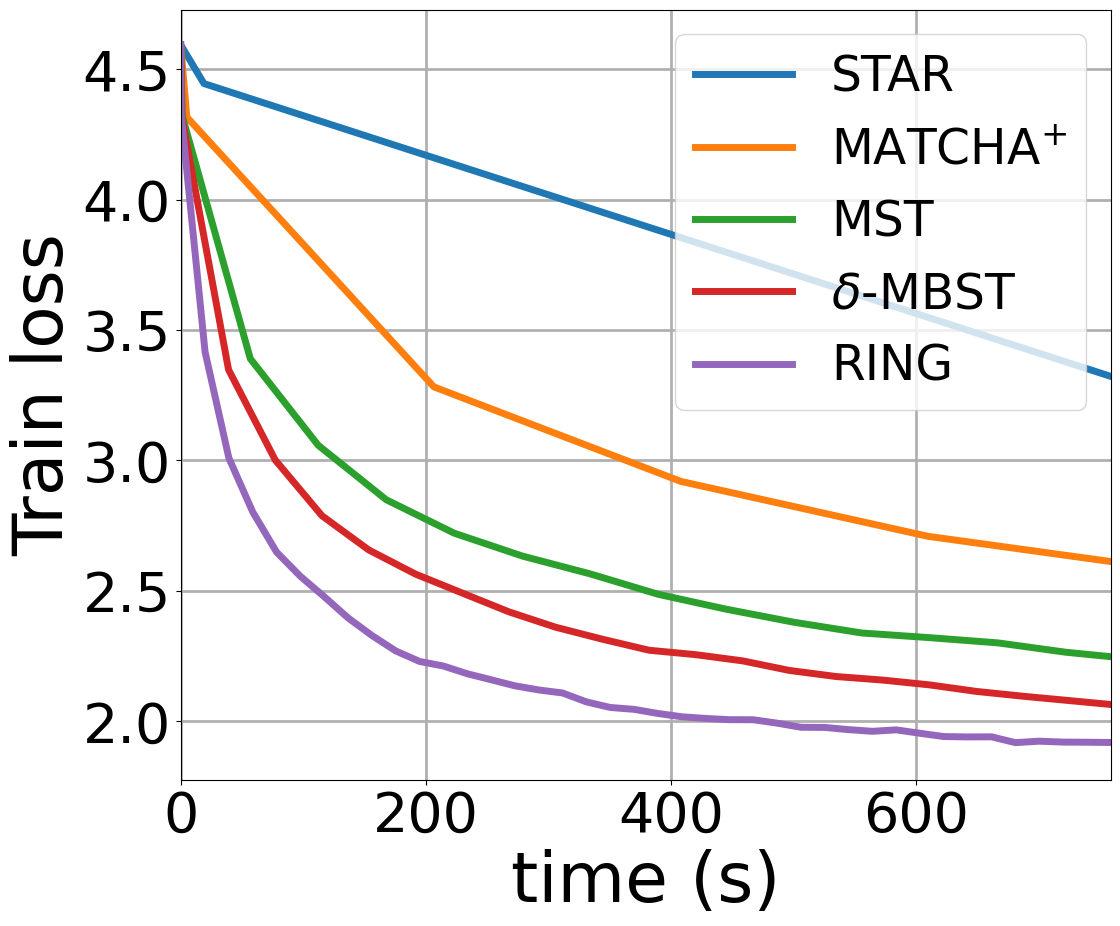}
        \caption[]{{\small Train Loss}}    
    \end{subfigure}
    \hfill
    \begin{subfigure}[b]{0.24\textwidth}
        \centering
        \includegraphics[width=\textwidth, height=0.8\textwidth]{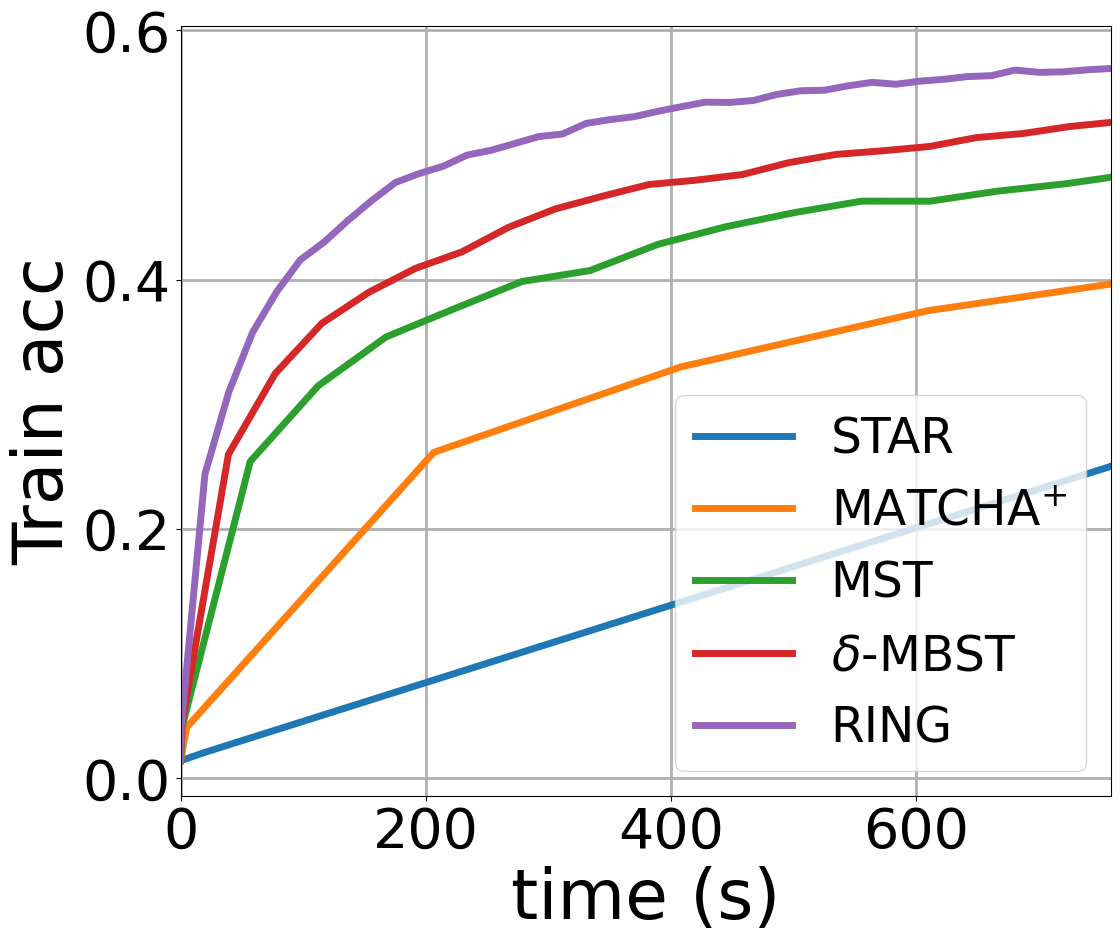}
        \caption[]{{\small Train Accuracy}}    
    \end{subfigure}
    \hfill
    \begin{subfigure}[b]{0.24\textwidth}   
        \centering 
        \includegraphics[width=\textwidth, height=0.8\textwidth]{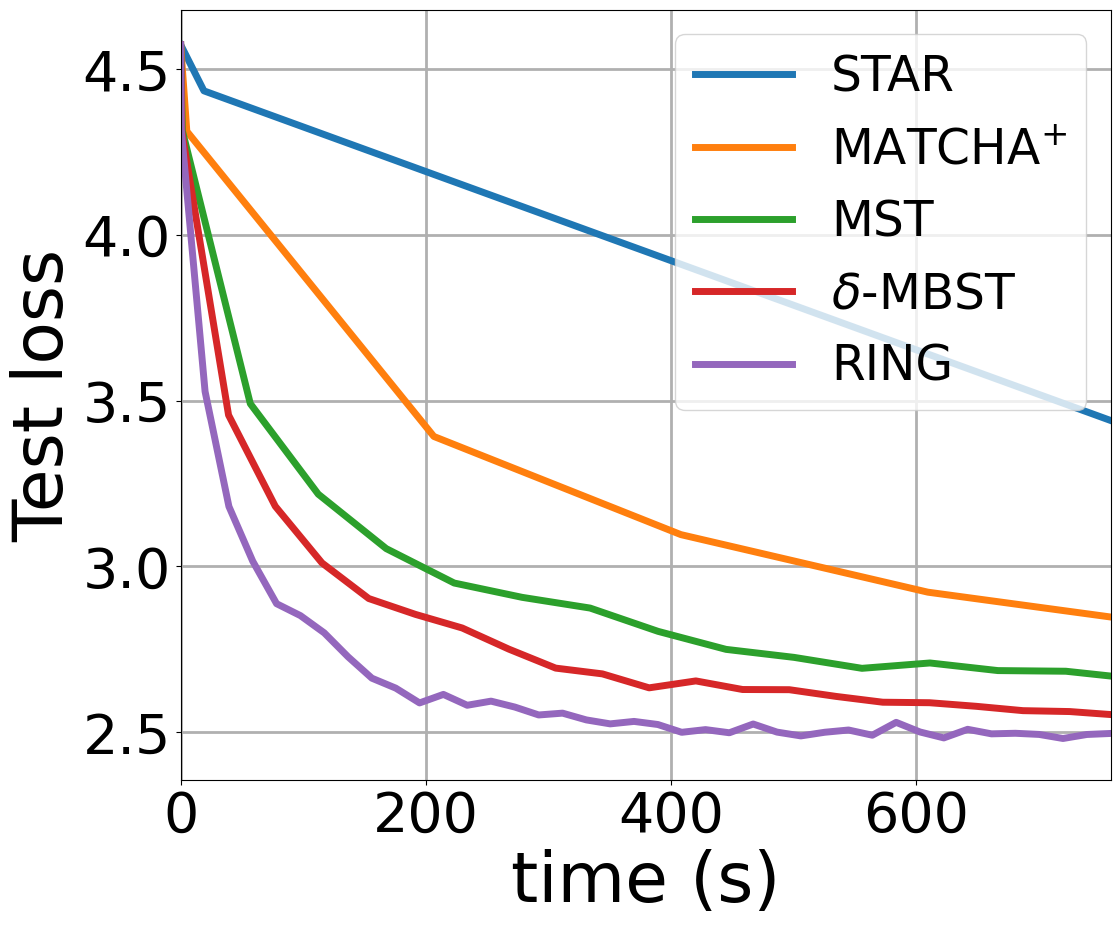}
        \caption[]{{\small Test Loss}}    
    \end{subfigure}
    \hfill
    \begin{subfigure}[b]{0.24\textwidth}   
        \centering 
        \includegraphics[width=\textwidth, height=0.8\textwidth]{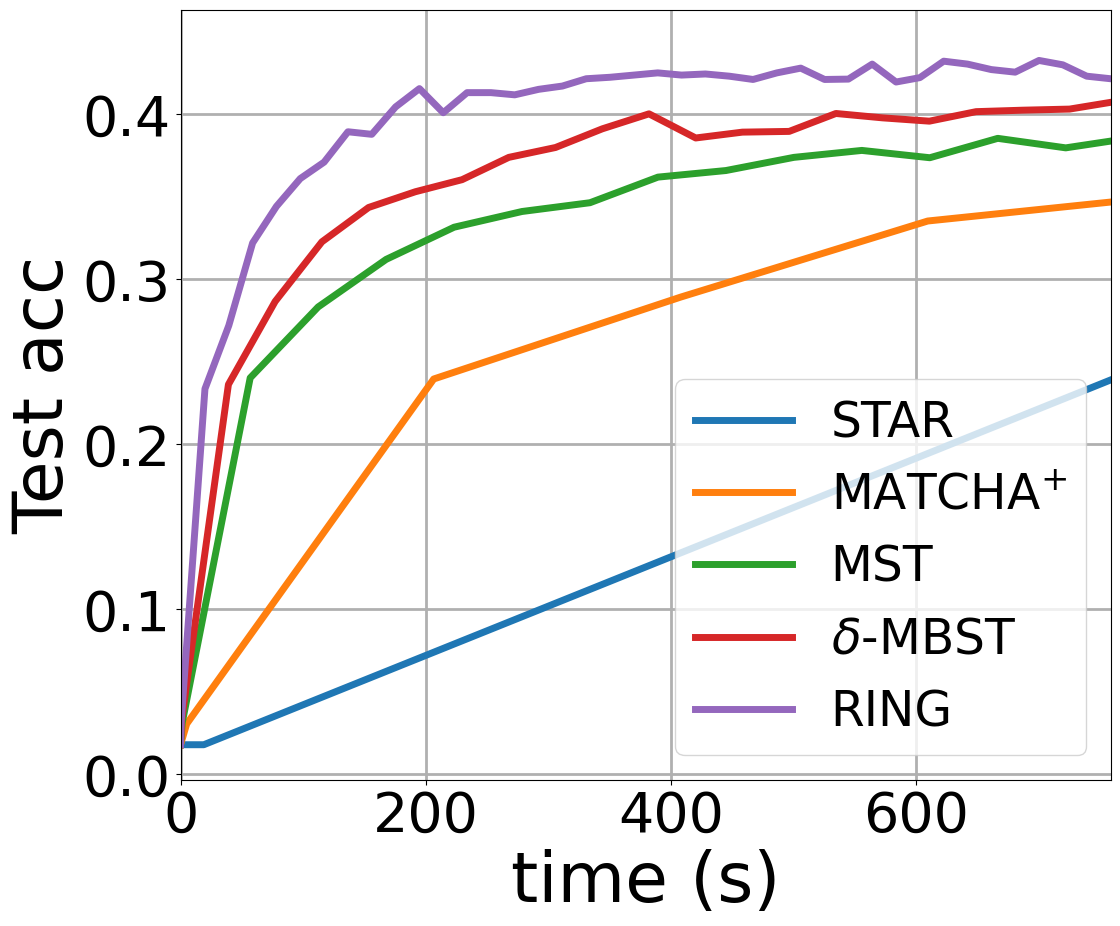}
        \caption[]{{\small Test Accuracy}}    
    \end{subfigure}
    \caption[]
    {\small Effect of overlays on the convergence w.r.t.~communication rounds (top row) and wall-clock time (bottom row) when training ResNet-18 image classification model using iNaturalist on  AWS North America underlay. $1$~Gbps core links capacities, $100$~Mbps access links capacities, $s=1$.} 
\end{figure*}

\begin{figure*}
    \centering
    \begin{subfigure}[b]{0.24\textwidth}  
        \centering 
        \includegraphics[width=\textwidth, height=0.8\textwidth]{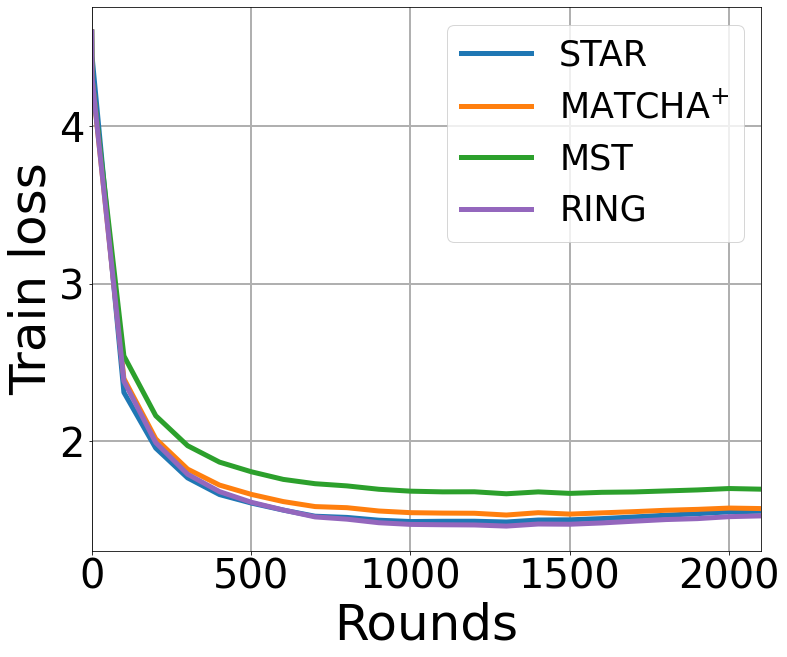}
    \end{subfigure}
    \hfill
    \begin{subfigure}[b]{0.24\textwidth}
        \centering
        \includegraphics[width=\textwidth, height=0.8\textwidth]{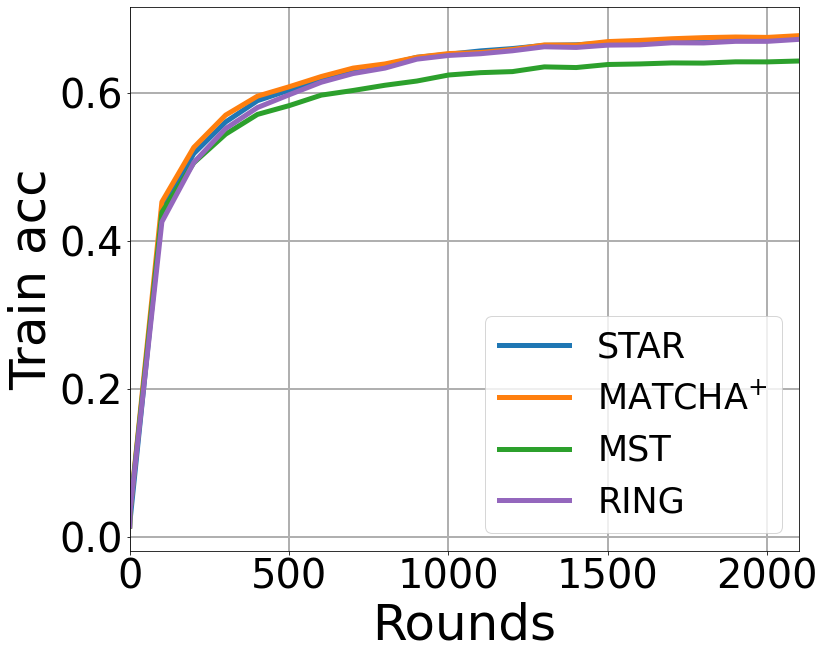}
    \end{subfigure}
    \hfill
    \begin{subfigure}[b]{0.24\textwidth}   
        \centering 
        \includegraphics[width=\textwidth, height=0.8\textwidth]{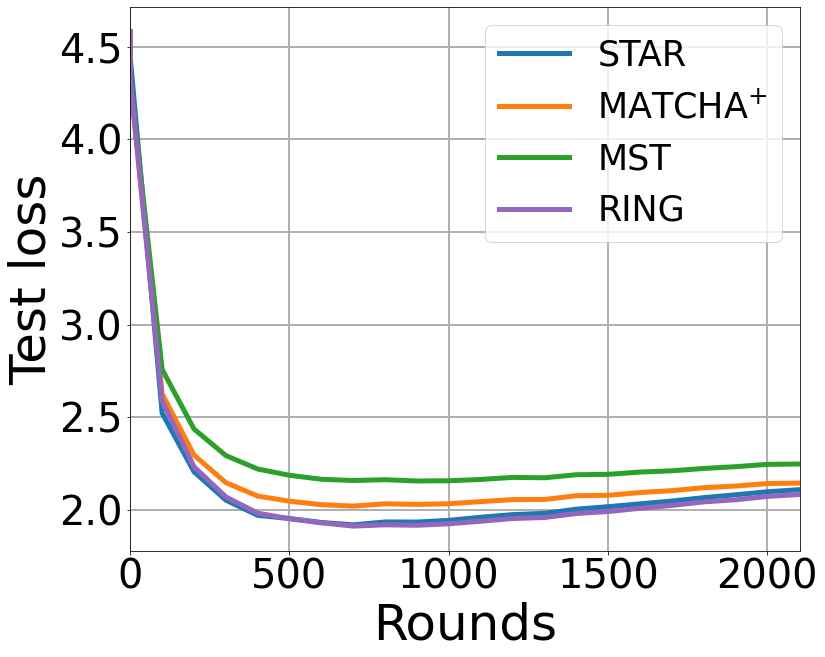}
    \end{subfigure}
    \hfill
    \begin{subfigure}[b]{0.24\textwidth}   
        \centering 
        \includegraphics[width=\textwidth, height=0.8\textwidth]{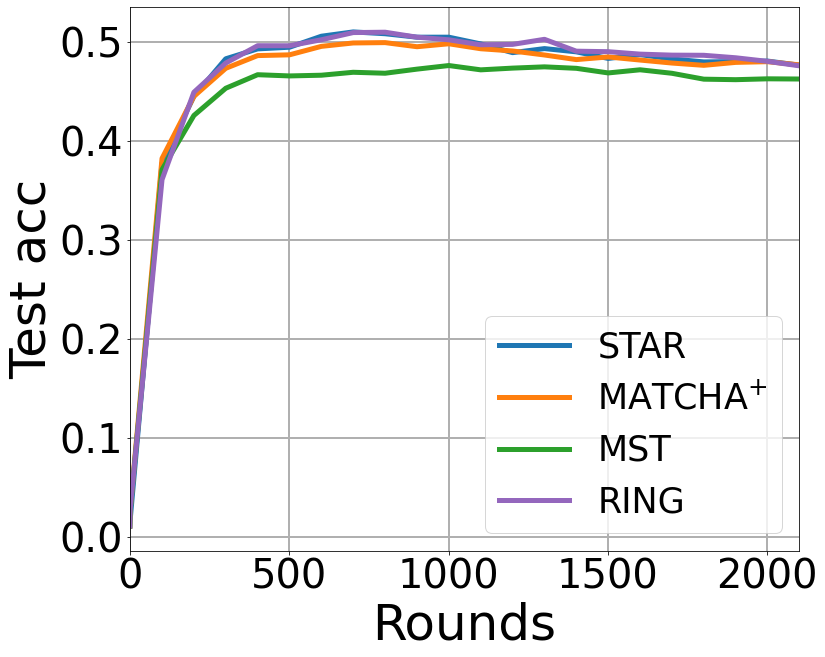}
    \end{subfigure}
    \\    
    \begin{subfigure}[b]{0.24\textwidth}  
        \centering 
        \includegraphics[width=\textwidth, height=0.8\textwidth]{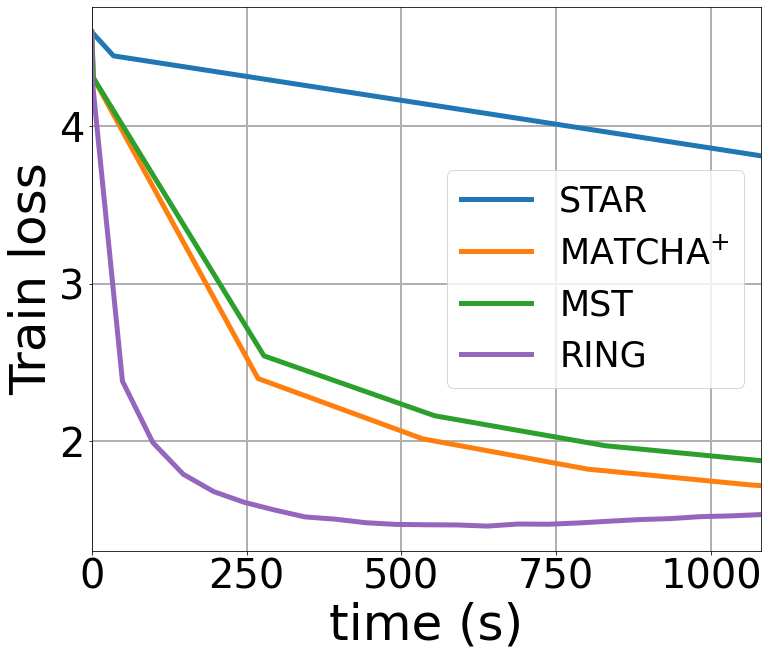}
        \caption[]{{\small Train Loss}}    
    \end{subfigure}
    \hfill
    \begin{subfigure}[b]{0.24\textwidth}
        \centering
        \includegraphics[width=\textwidth, height=0.8\textwidth]{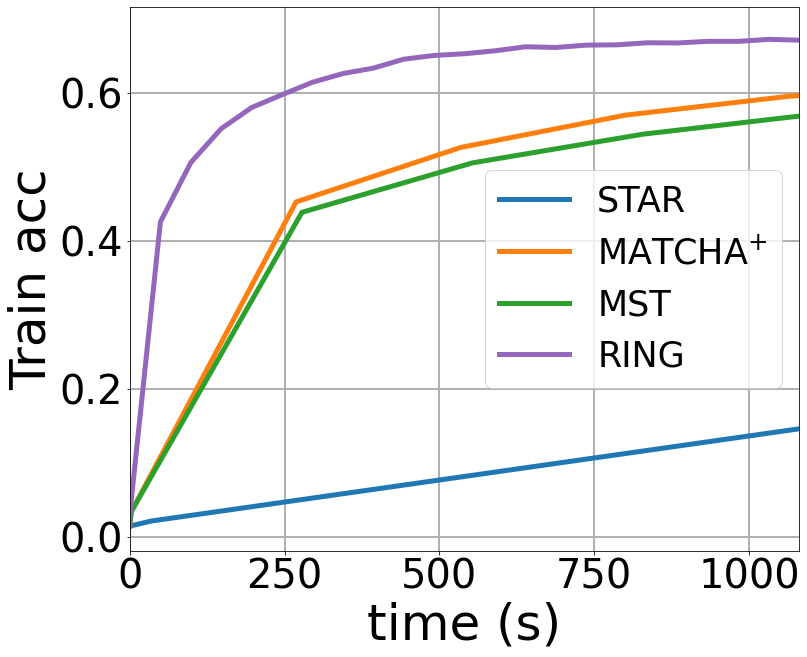}
        \caption[] {{\small Train Accuracy}}    
    \end{subfigure}
    \hfill
    \begin{subfigure}[b]{0.24\textwidth}   
        \centering 
        \includegraphics[width=\textwidth, height=0.8\textwidth]{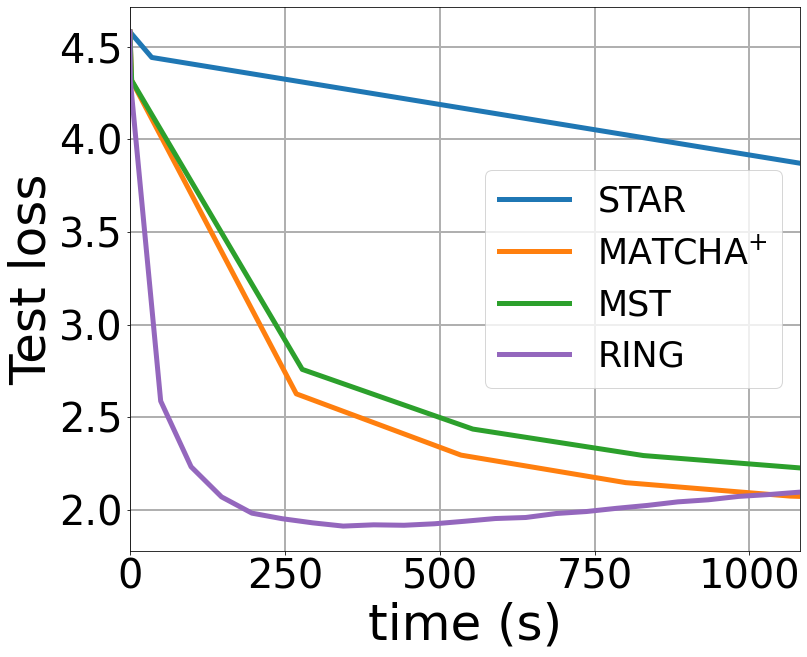}
        \caption[]{{\small Test Loss}}    
    \end{subfigure}
    \hfill
    \begin{subfigure}[b]{0.24\textwidth}   
        \centering 
        \includegraphics[width=\textwidth, height=0.8\textwidth]{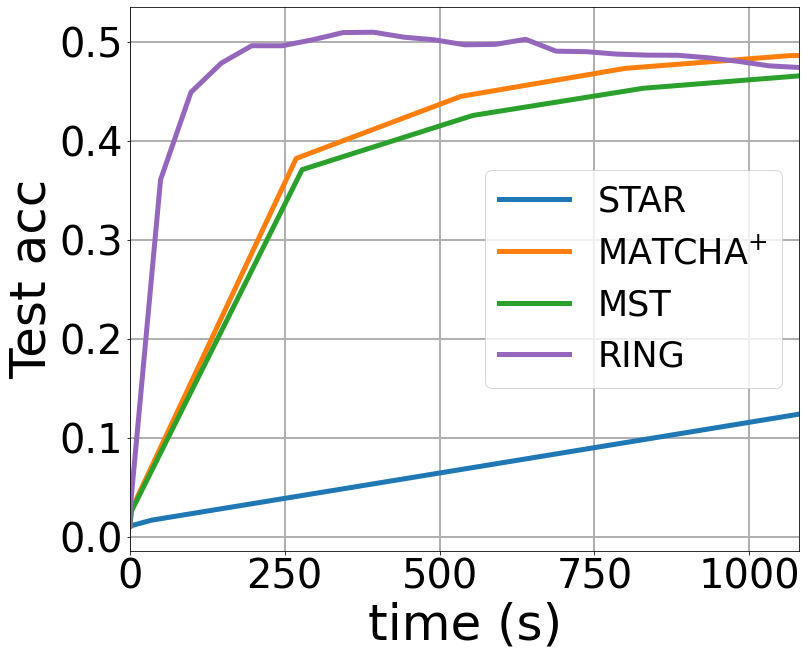}
        \caption[]{{\small Test Accuracy}}    
    \end{subfigure}
    \caption[]
    {\small Effect of overlays on the convergence w.r.t.~communication rounds (top row) and wall-clock time (bottom row) when training ResNet-18 image classification model using iNaturalist on  G\'eant underlay. $1$~Gbps core links capacities, $100$~Mbps access links capacities, $s=1$.} 
    \label{f:geant_s_1}
    \end{figure*}

\begin{figure*}
    \centering
    \begin{subfigure}[b]{0.24\textwidth}  
        \centering 
        \includegraphics[width=\textwidth, height=0.8\textwidth]{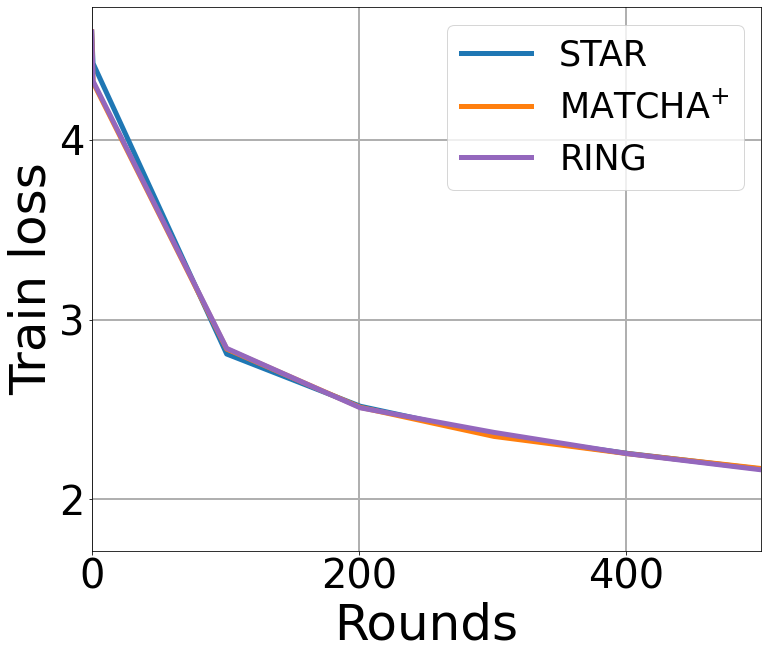}
    \end{subfigure}
    \hfill
    \begin{subfigure}[b]{0.24\textwidth}
        \centering
        \includegraphics[width=\textwidth, height=0.8\textwidth]{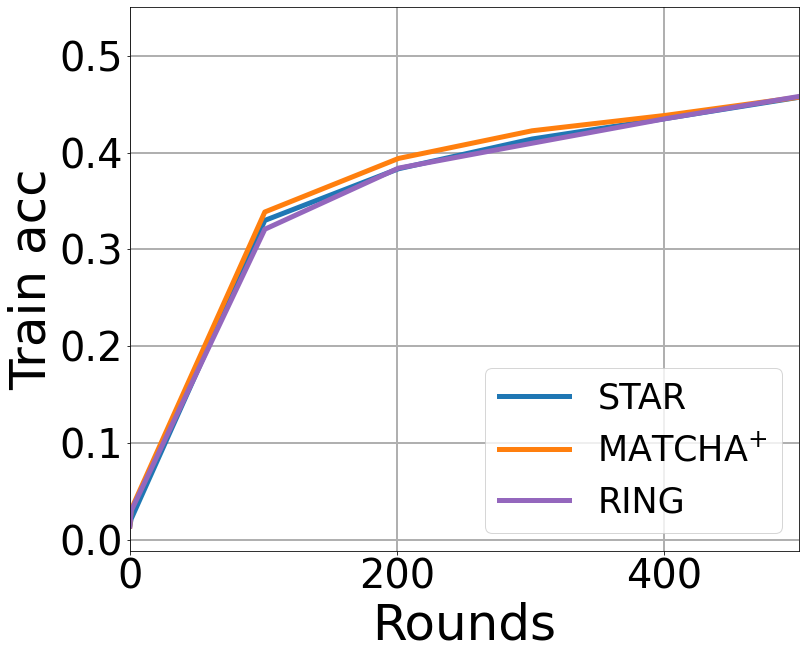}
    \end{subfigure}
    \hfill
    \begin{subfigure}[b]{0.24\textwidth}   
        \centering 
        \includegraphics[width=\textwidth, height=0.8\textwidth]{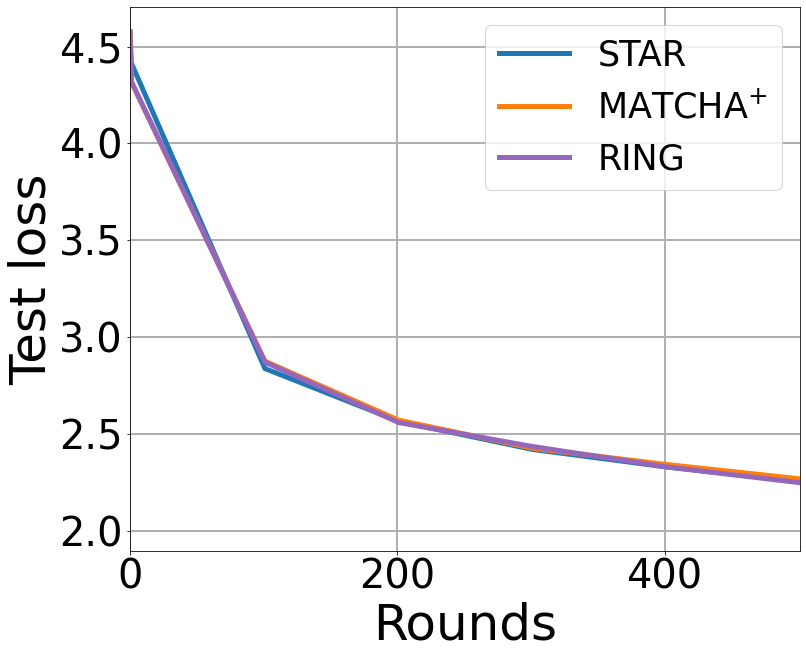}
    \end{subfigure}
    \hfill
    \begin{subfigure}[b]{0.24\textwidth}   
        \centering 
        \includegraphics[width=\textwidth, height=0.8\textwidth]{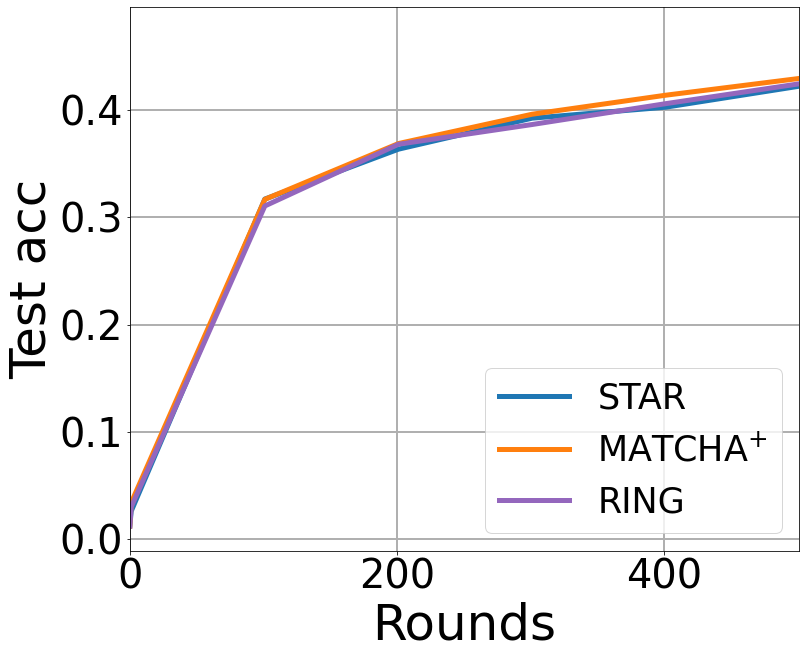}
    \end{subfigure}
    \\
    \begin{subfigure}[b]{0.24\textwidth}  
        \centering 
        \includegraphics[width=\textwidth, height=0.8\textwidth]{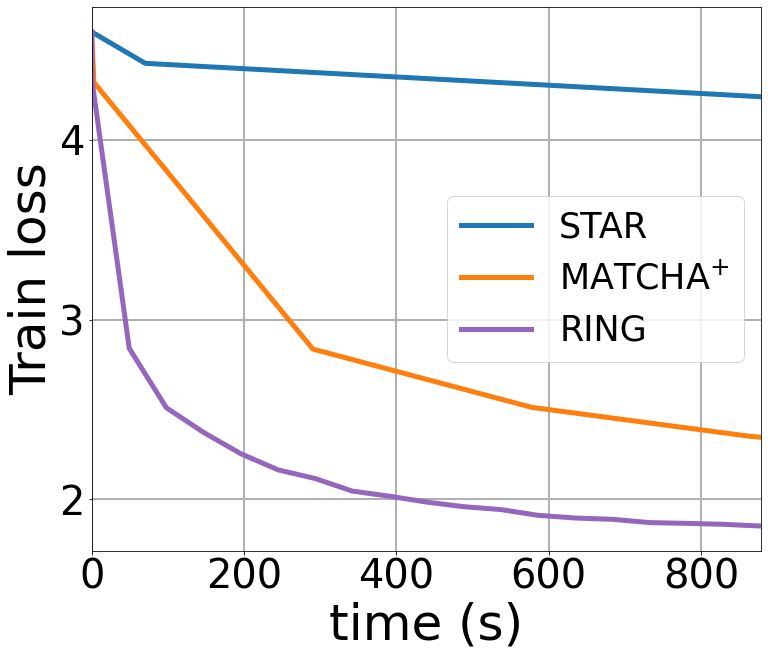}
        \caption[]{{\small Train Loss}}    
    \end{subfigure}
    \hfill
    \begin{subfigure}[b]{0.24\textwidth}
        \centering
        \includegraphics[width=\textwidth, height=0.8\textwidth]{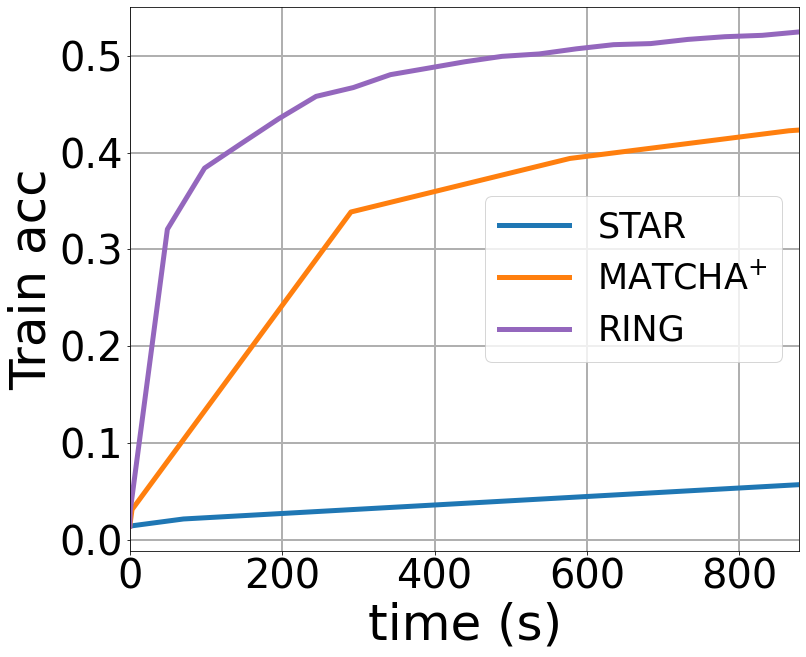}
        \caption[]{{\small Train Accuracy}} 
    \end{subfigure}
    \hfill
    \begin{subfigure}[b]{0.24\textwidth}   
        \centering 
        \includegraphics[width=\textwidth, height=0.8\textwidth]{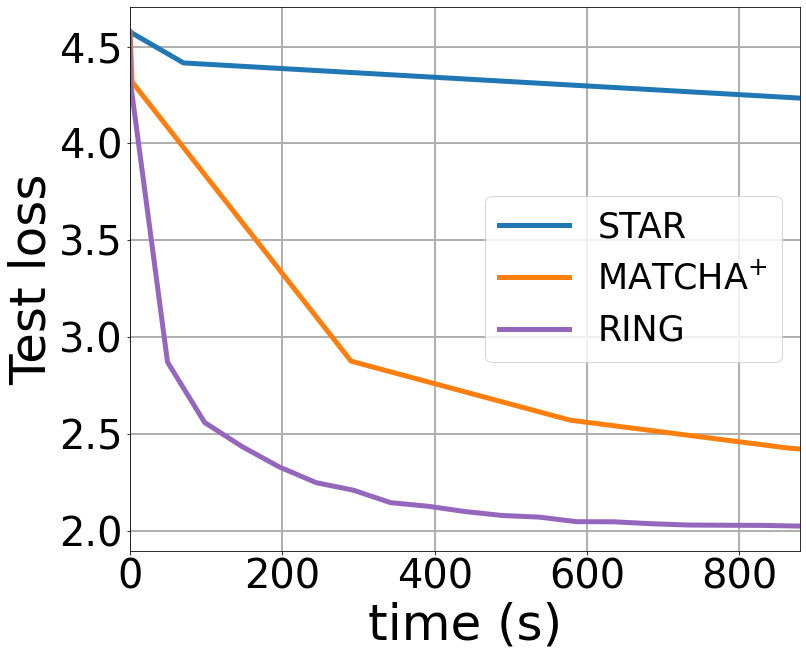}
        \caption[]{{\small Test Loss}}    
    \end{subfigure}
    \hfill
    \begin{subfigure}[b]{0.24\textwidth}   
        \centering 
        \includegraphics[width=\textwidth, height=0.8\textwidth]{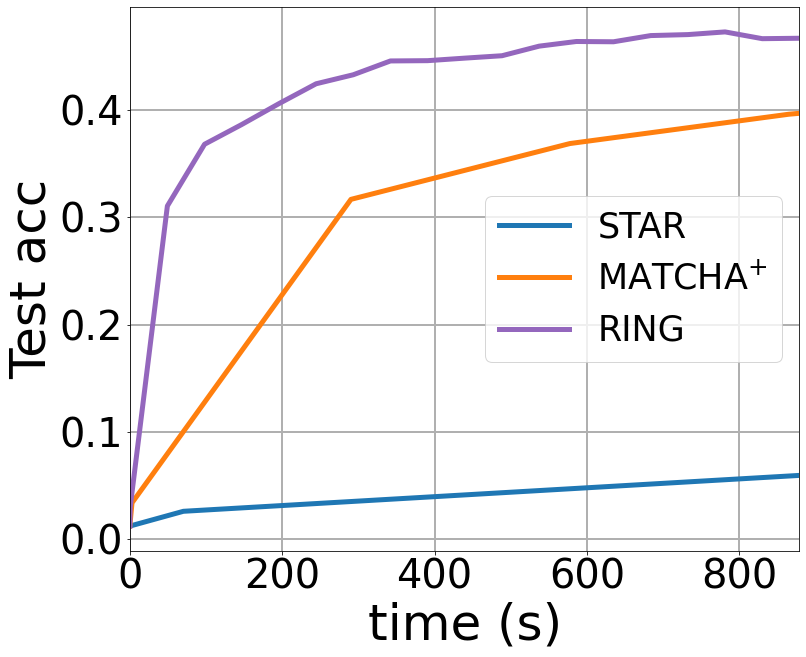}
        \caption[]{{\small Test Accuracy}}    
    \end{subfigure}
    \caption[]
    {\small Effect of overlays on the convergence w.r.t.~communication rounds (top row) and wall-clock time (bottom row) when training ResNet-18 image classification model using iNaturalist on  Exodus underlay. $1$~Gbps core links capacities, $100$~Mbps access links capacities, $s=1$.} 
\end{figure*}
    
\begin{figure*}
    \centering
    \begin{subfigure}[b]{0.24\textwidth}  
        \centering 
        \includegraphics[width=\textwidth, height=0.8\textwidth]{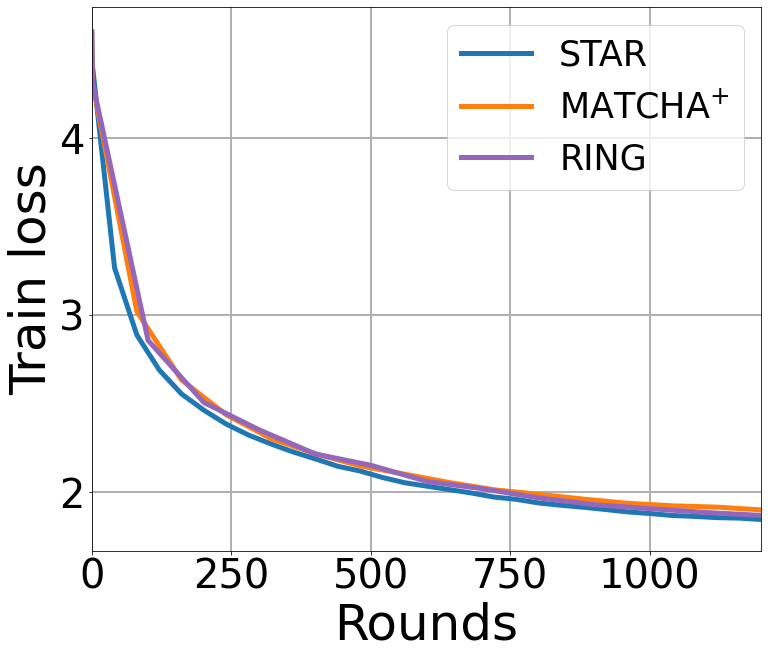}
    \end{subfigure}
    \hfill
    \begin{subfigure}[b]{0.24\textwidth}
        \centering
        \includegraphics[width=\textwidth, height=0.8\textwidth]{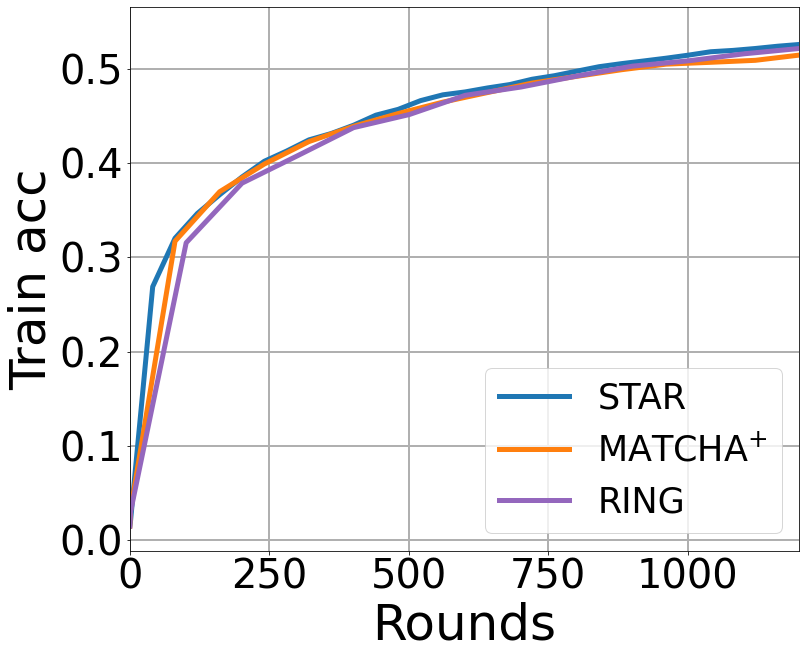}
    \end{subfigure}
    \hfill
    \begin{subfigure}[b]{0.24\textwidth}   
        \centering 
        \includegraphics[width=\textwidth, height=0.8\textwidth]{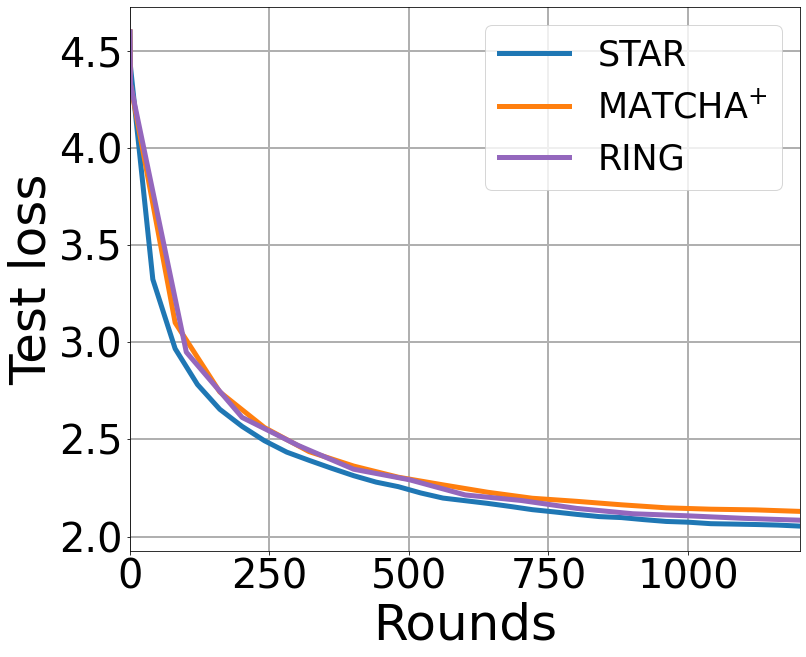}
    \end{subfigure}
    \hfill
    \begin{subfigure}[b]{0.24\textwidth}   
        \centering 
        \includegraphics[width=\textwidth, height=0.8\textwidth]{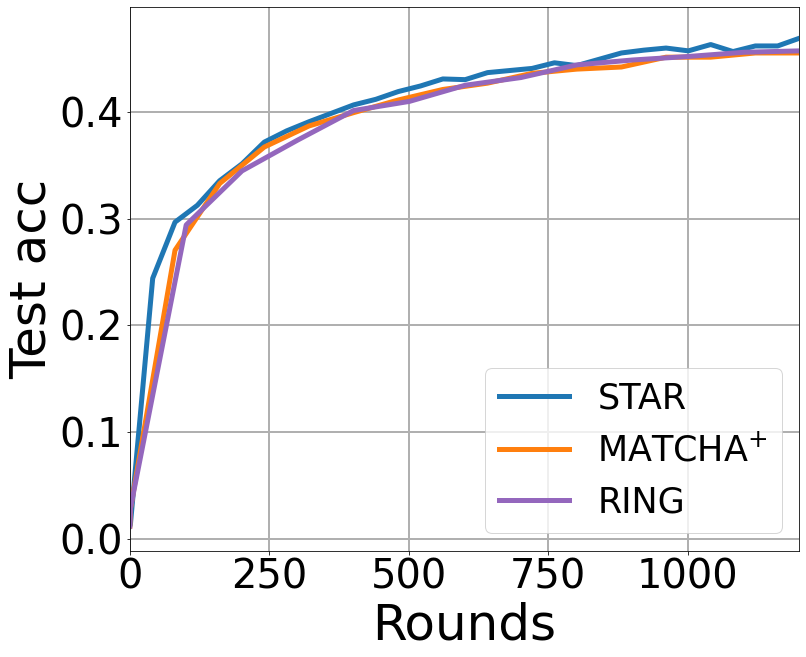}
    \end{subfigure}
    \\    
    \begin{subfigure}[b]{0.24\textwidth}  
        \centering 
        \includegraphics[width=\textwidth, height=0.8\textwidth]{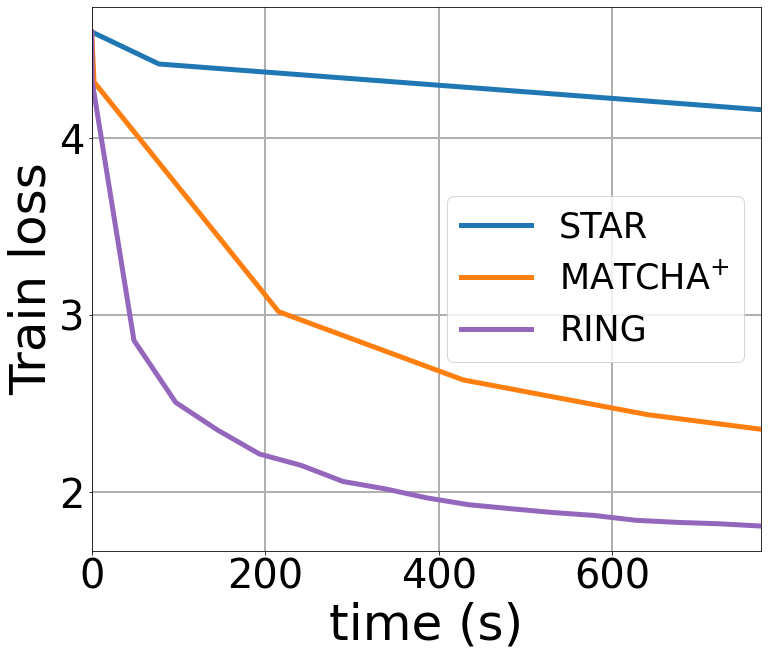}
        \caption[]{{\small Train Loss}}    
    \end{subfigure}
    \hfill
    \begin{subfigure}[b]{0.24\textwidth}
        \centering
        \includegraphics[width=\textwidth, height=0.8\textwidth]{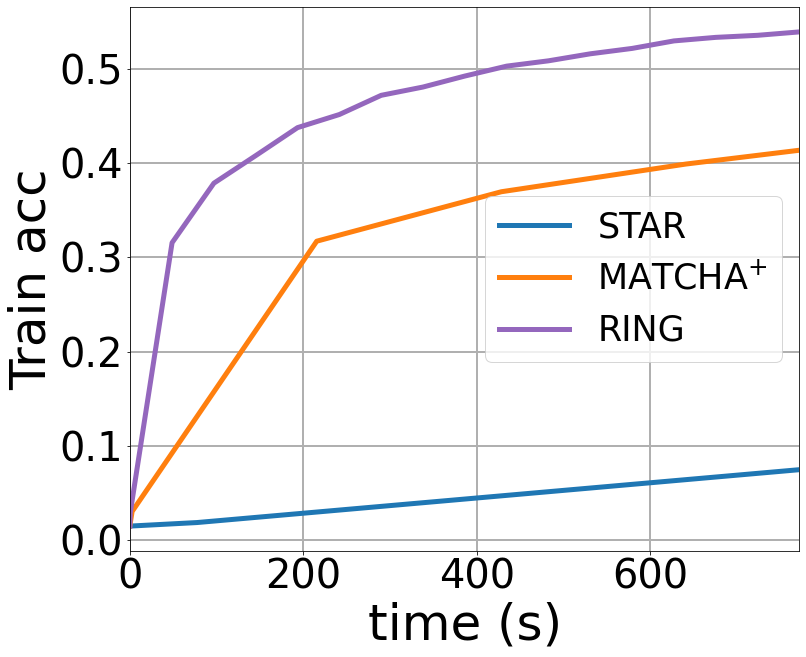}
        \caption[]{{\small Train Accuracy}}    
    \end{subfigure}
    \hfill
    \begin{subfigure}[b]{0.24\textwidth}   
        \centering 
        \includegraphics[width=\textwidth, height=0.8\textwidth]{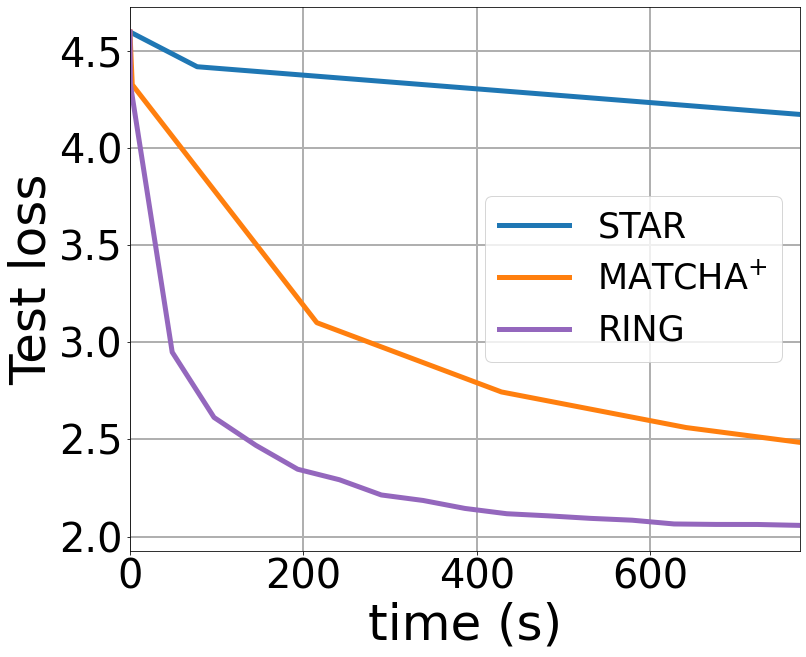}
         \caption[]{{\small Test Loss}}    
    \end{subfigure}
    \hfill
    \begin{subfigure}[b]{0.24\textwidth}   
        \centering 
        \includegraphics[width=\textwidth, height=0.8\textwidth]{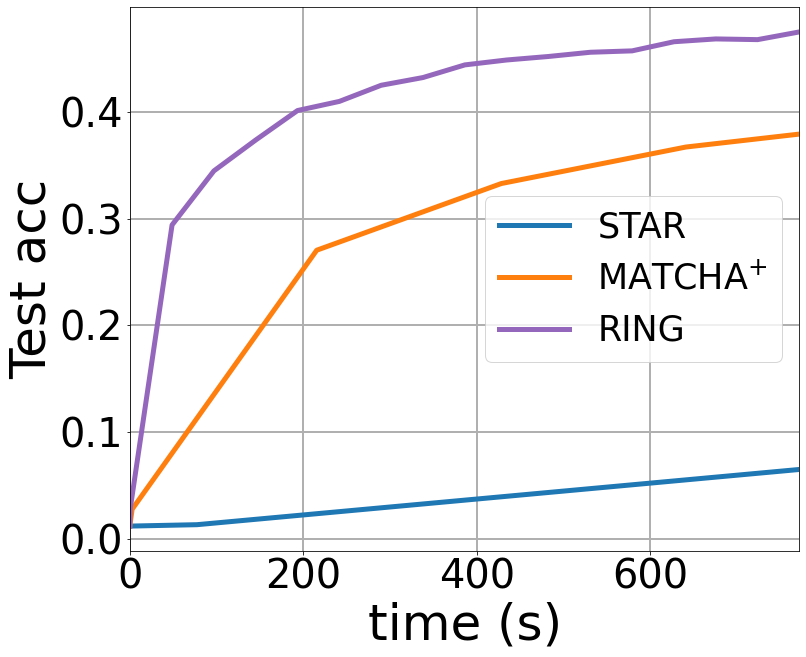}
        \caption[]{{\small Test Accuracy}}    
    \end{subfigure}
    \caption[]
    {\small Effect of overlays on the convergence w.r.t.~communication rounds (top row) and wall-clock time (bottom row) when training ResNet-18 image classification model using iNaturalist on  Ebone underlay. $1$~Gbps core links capacities, $100$~Mbps access links capacities, $s=1$.} 
\end{figure*}
    
\begin{figure*}
    \centering
    \begin{subfigure}[b]{0.24\textwidth}  
        \centering 
        \includegraphics[width=\textwidth, height=0.8\textwidth]{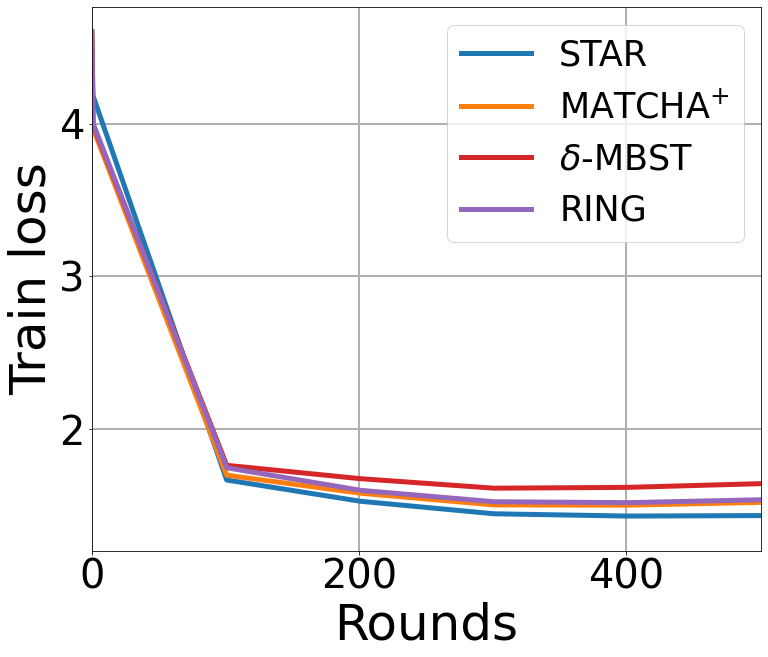}
    \end{subfigure}
    \hfill
    \begin{subfigure}[b]{0.24\textwidth}
        \centering
        \includegraphics[width=\textwidth, height=0.8\textwidth]{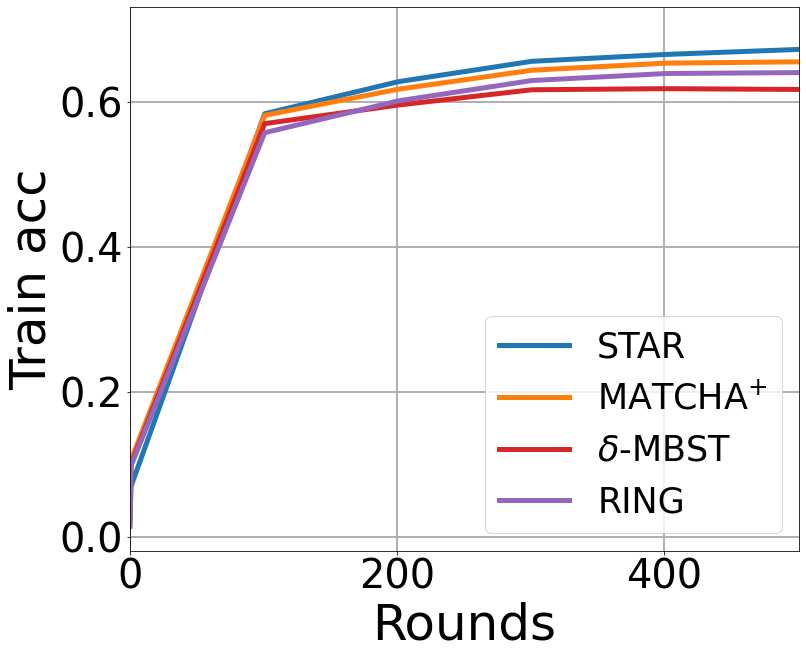}
    \end{subfigure}
    \hfill
    \begin{subfigure}[b]{0.24\textwidth}   
        \centering 
        \includegraphics[width=\textwidth, height=0.8\textwidth]{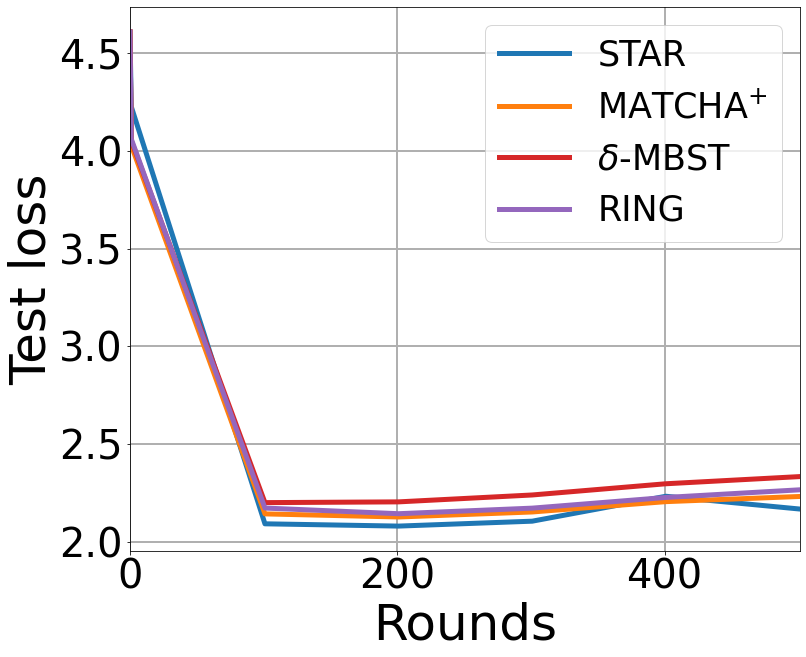}
    \end{subfigure}
    \hfill
    \begin{subfigure}[b]{0.24\textwidth}   
        \centering 
        \includegraphics[width=\textwidth, height=0.8\textwidth]{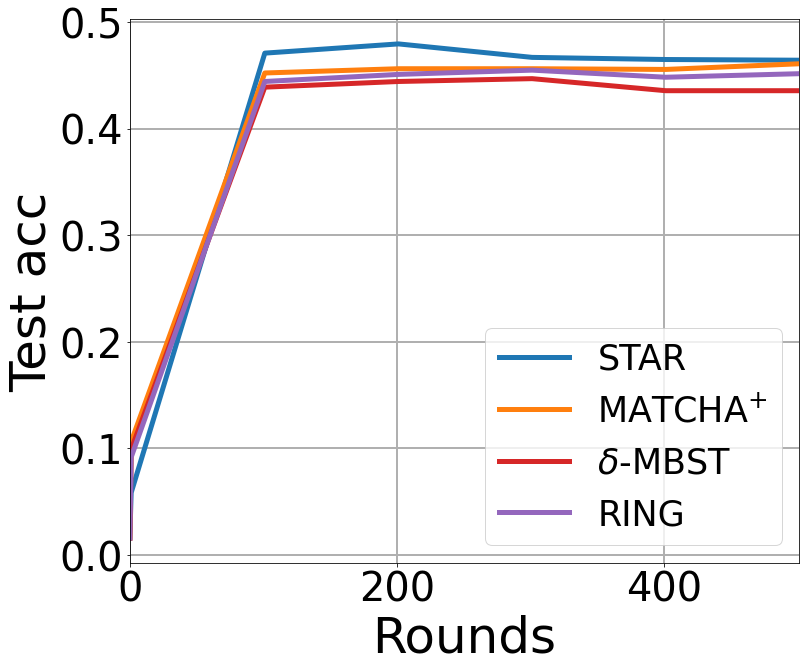}
    \end{subfigure}
    \\    
    \begin{subfigure}[b]{0.24\textwidth}  
        \centering 
        \includegraphics[width=\textwidth, height=0.8\textwidth]{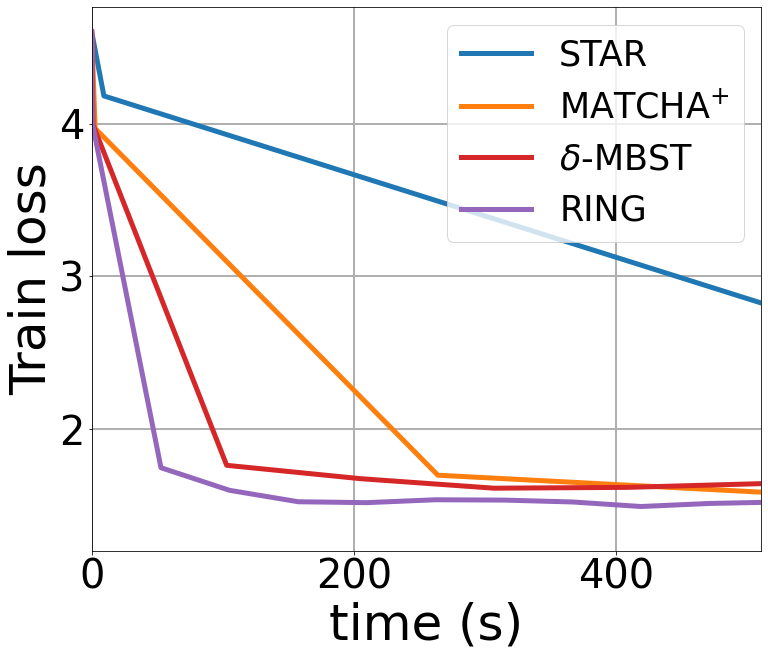}
        \caption[]{{\small Train Loss}}    
    \end{subfigure}
    \hfill
    \begin{subfigure}[b]{0.24\textwidth}
        \centering
        \includegraphics[width=\textwidth, height=0.8\textwidth]{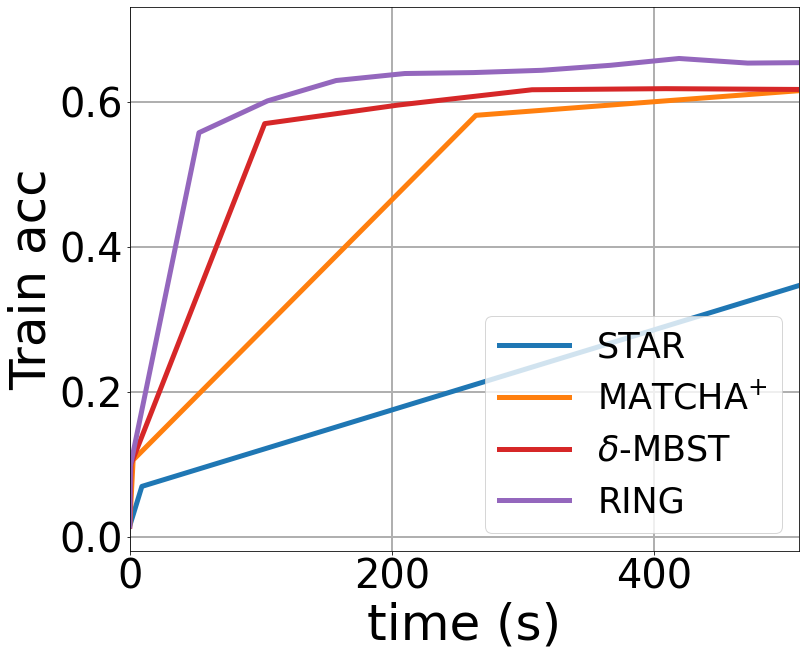}
        \caption[]{{\small Train Accuracy}}    
    \end{subfigure}
    \hfill
    \begin{subfigure}[b]{0.24\textwidth}   
        \centering 
        \includegraphics[width=\textwidth, height=0.8\textwidth]{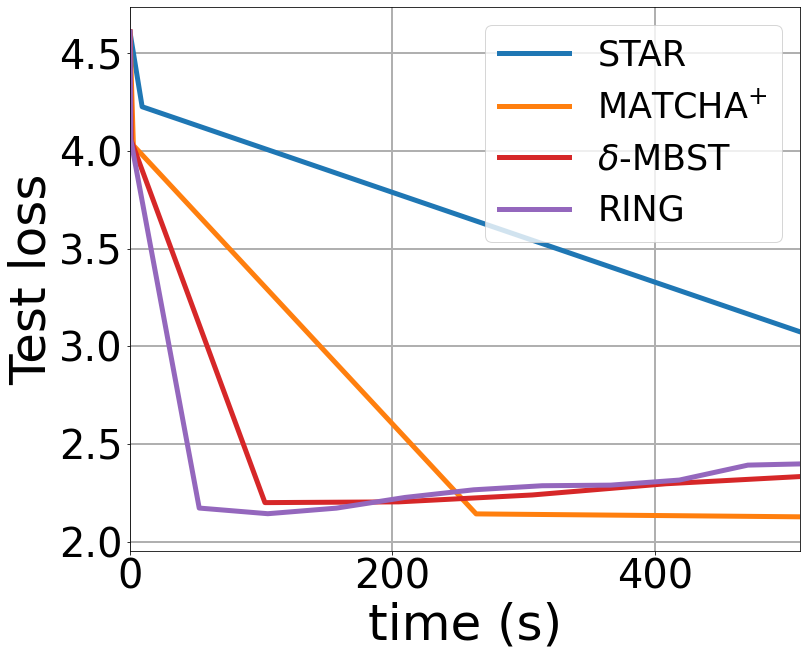}
        \caption[]{{\small Test Loss}}    
    \end{subfigure}
    \hfill
    \begin{subfigure}[b]{0.24\textwidth}   
        \centering 
        \includegraphics[width=\textwidth, height=0.8\textwidth]{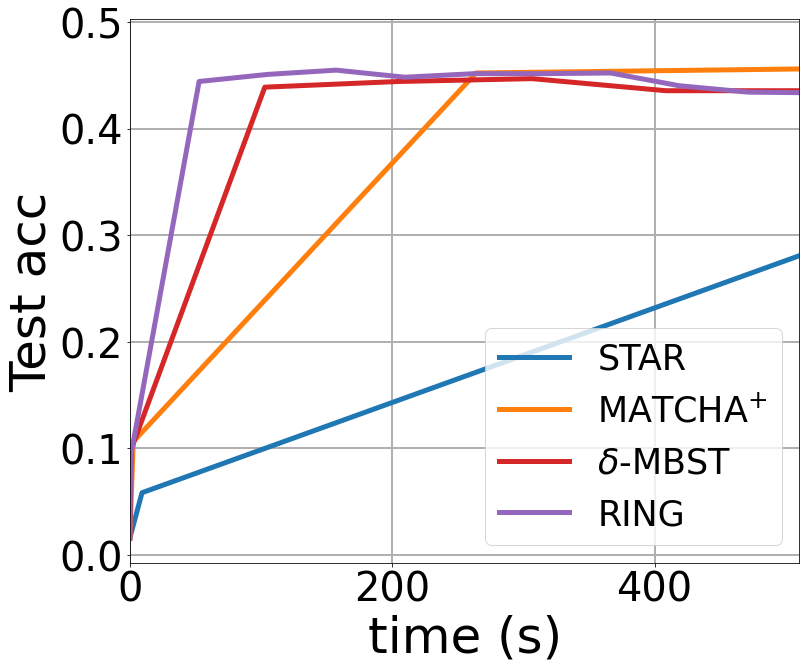}
        \caption[]{{\small Test Accuracy}}    
    \end{subfigure}
    \caption[]
    {\small Effect of overlays on the convergence w.r.t.~communication rounds (top row) and wall-clock time (bottom row) when training ResNet-18 image classification model using iNaturalist on Gaia underlay. $1$~Gbps core links capacities, $100$~Mbps access links capacities, $s=5$.} 
    \label{f:gaia_s_5}
\end{figure*}
    
\begin{figure*}
    \centering
    \begin{subfigure}[b]{0.24\textwidth}  
        \centering 
        \includegraphics[width=\textwidth, height=0.8\textwidth]{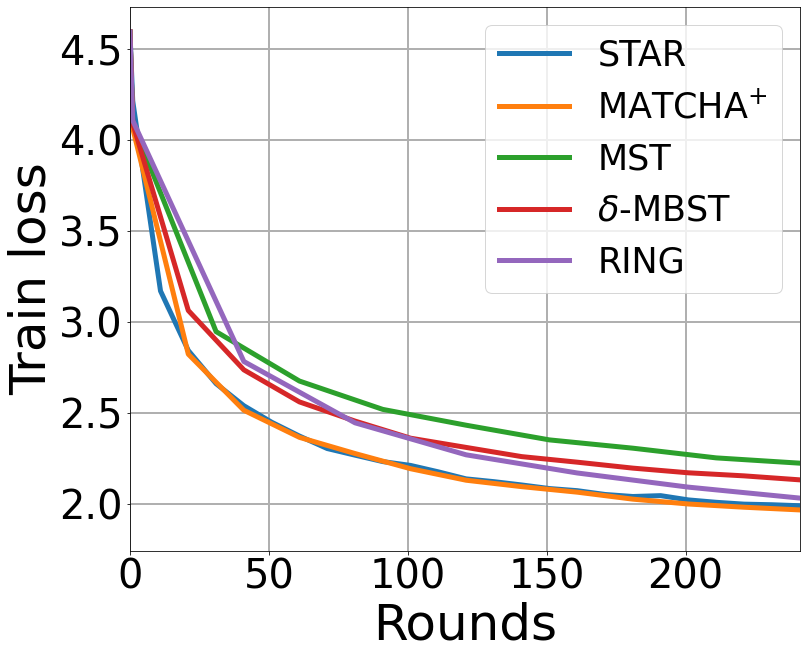}
    \end{subfigure}
    \hfill
    \begin{subfigure}[b]{0.24\textwidth}
        \centering
        \includegraphics[width=\textwidth, height=0.8\textwidth]{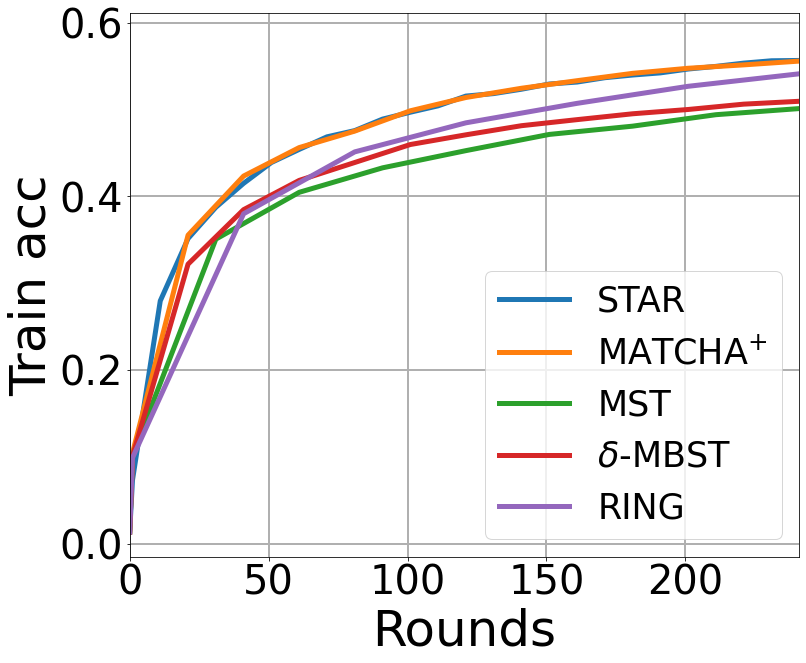}
    \end{subfigure}
    \hfill
    \begin{subfigure}[b]{0.24\textwidth}   
        \centering 
        \includegraphics[width=\textwidth, height=0.8\textwidth]{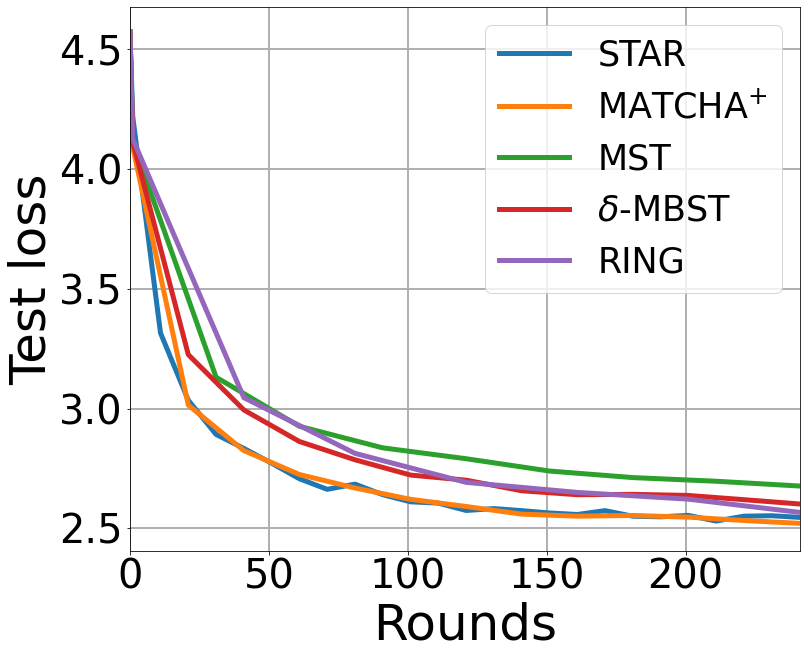}
    \end{subfigure}
    \hfill
    \begin{subfigure}[b]{0.24\textwidth}   
        \centering 
        \includegraphics[width=\textwidth, height=0.8\textwidth]{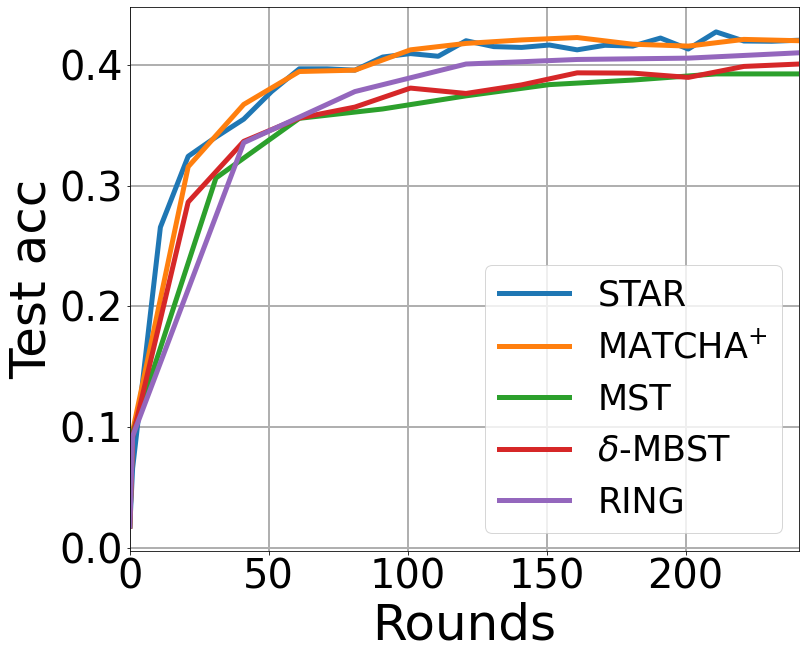}
    \end{subfigure}
    \\    
    \begin{subfigure}[b]{0.24\textwidth}  
        \centering 
        \includegraphics[width=\textwidth, height=0.8\textwidth]{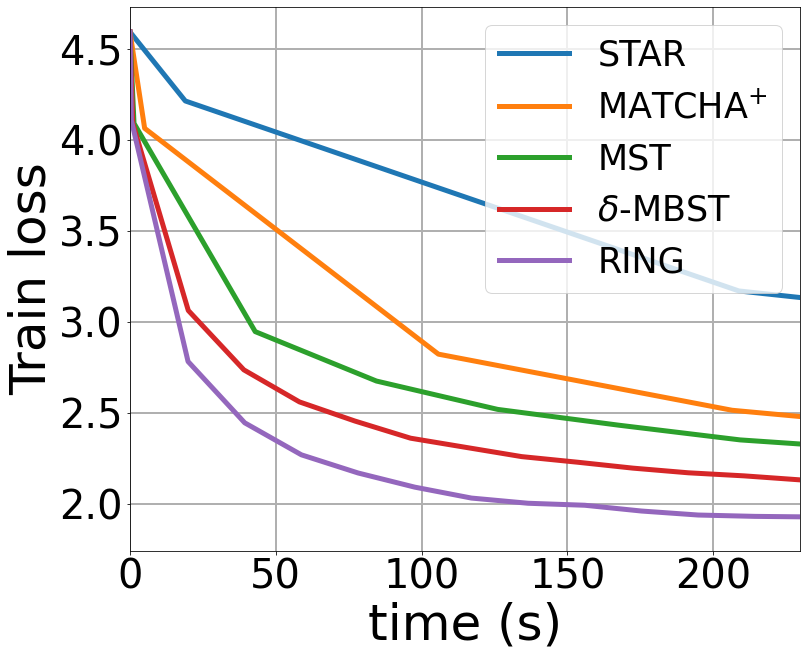}
        \caption[]{{\small Train Loss}}    
    \end{subfigure}
    \hfill
    \begin{subfigure}[b]{0.24\textwidth}
        \centering
        \includegraphics[width=\textwidth, height=0.8\textwidth]{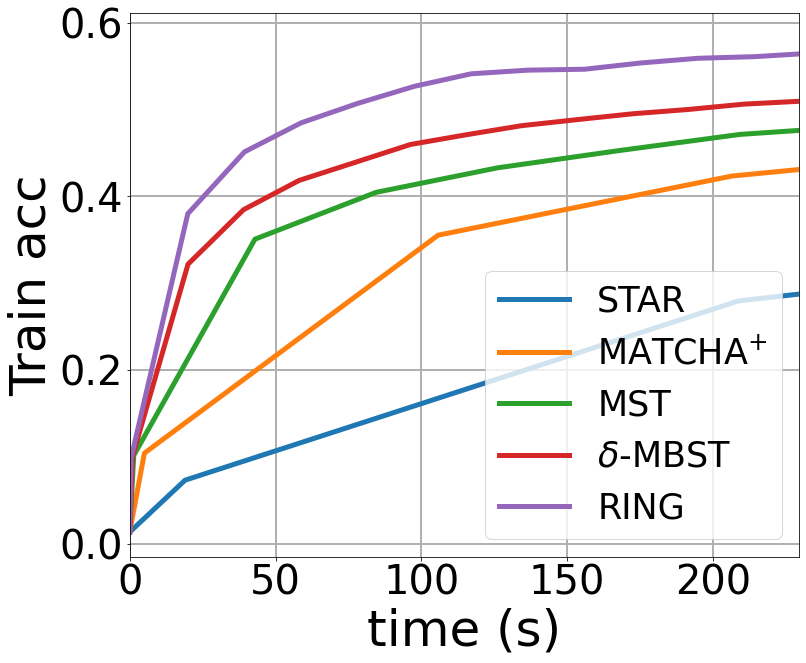}
        \caption[]{{\small Train Accuracy}}    
    \end{subfigure}
    \hfill
    \begin{subfigure}[b]{0.24\textwidth}   
        \centering 
        \includegraphics[width=\textwidth, height=0.8\textwidth]{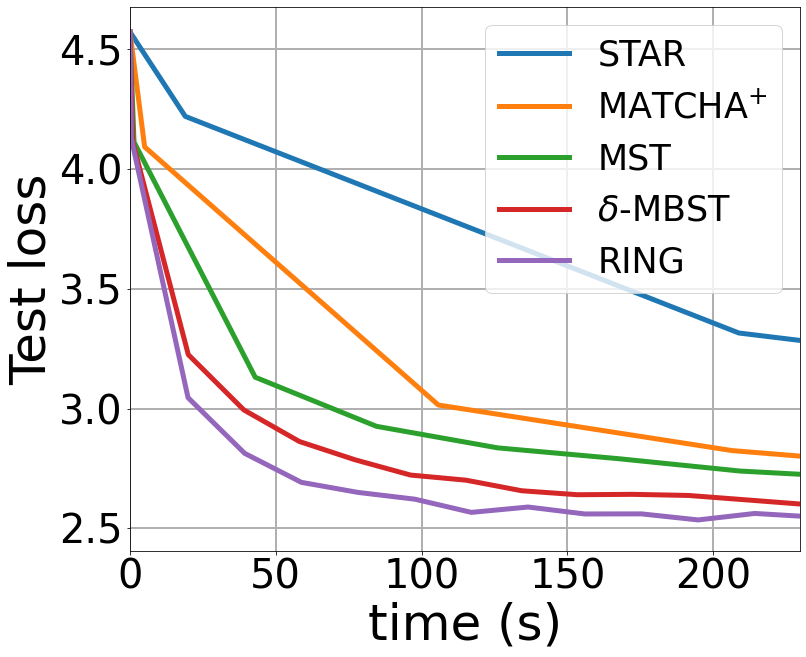}
        \caption[]{{\small Test Loss}}    
    \end{subfigure}
    \hfill
    \begin{subfigure}[b]{0.24\textwidth}   
        \centering 
        \includegraphics[width=\textwidth, height=0.8\textwidth]{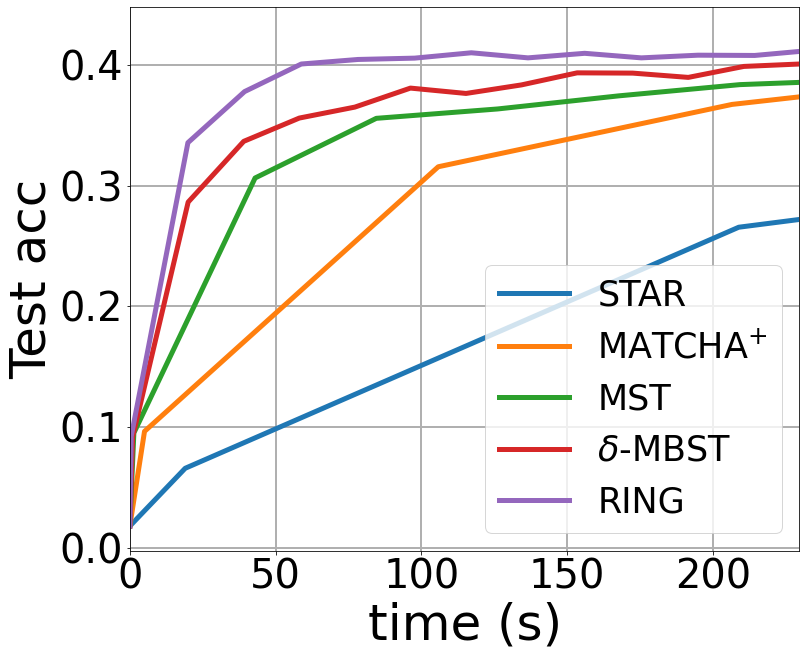}
        \caption[]{{\small Test Accuracy}}    
    \end{subfigure}
    \caption[]
    {\small Effect of overlays on the convergence w.r.t.~communication rounds (top row) and wall-clock time (bottom row) when training ResNet-18 image classification model using iNaturalist on  AWS North America underlay. $1$~Gbps core links capacities, $100$~Mbps access links capacities, $s=5$.} 
\end{figure*}
    
\begin{figure*}
    \centering
    \begin{subfigure}[b]{0.24\textwidth}  
        \centering 
        \includegraphics[width=\textwidth, height=0.8\textwidth]{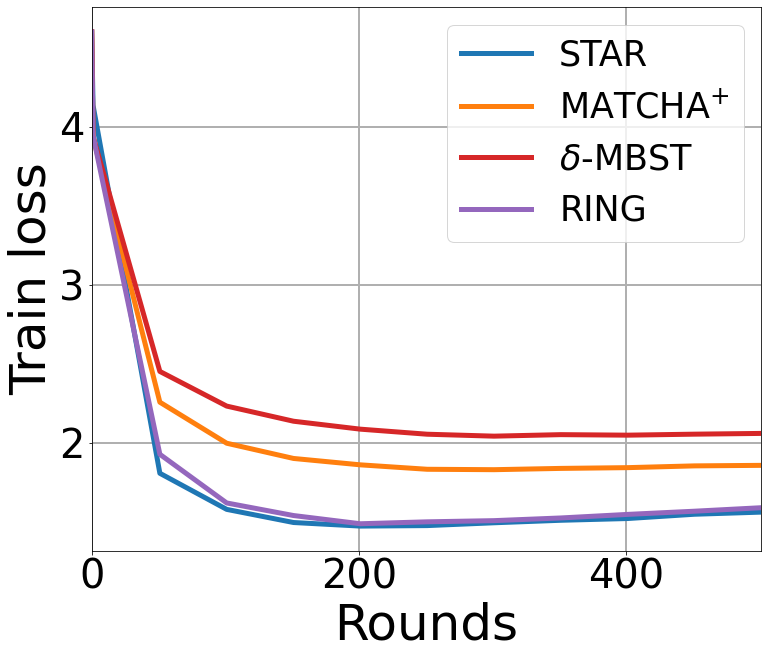}
    \end{subfigure}
    \hfill
    \begin{subfigure}[b]{0.24\textwidth}
        \centering
        \includegraphics[width=\textwidth, height=0.8\textwidth]{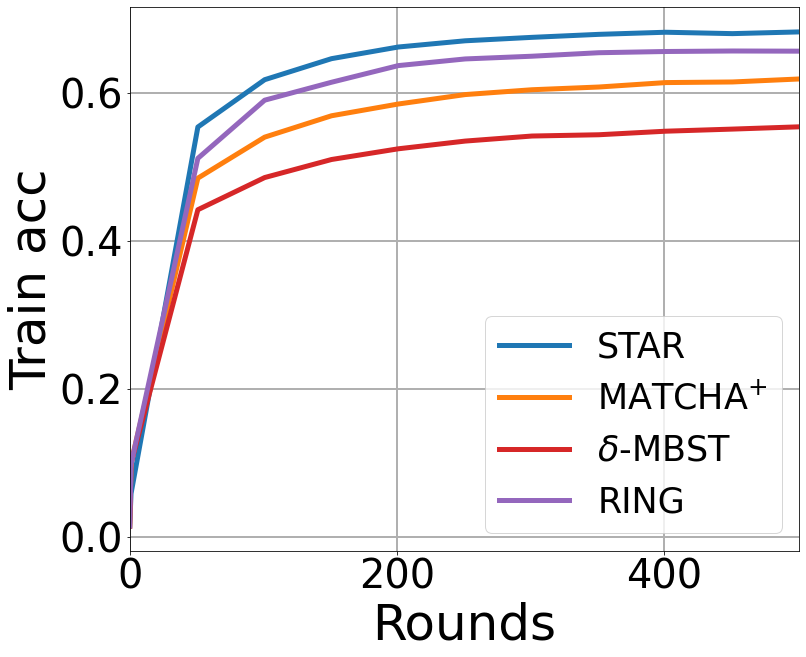}
    \end{subfigure}
    \hfill
    \begin{subfigure}[b]{0.24\textwidth}   
        \centering 
        \includegraphics[width=\textwidth, height=0.8\textwidth]{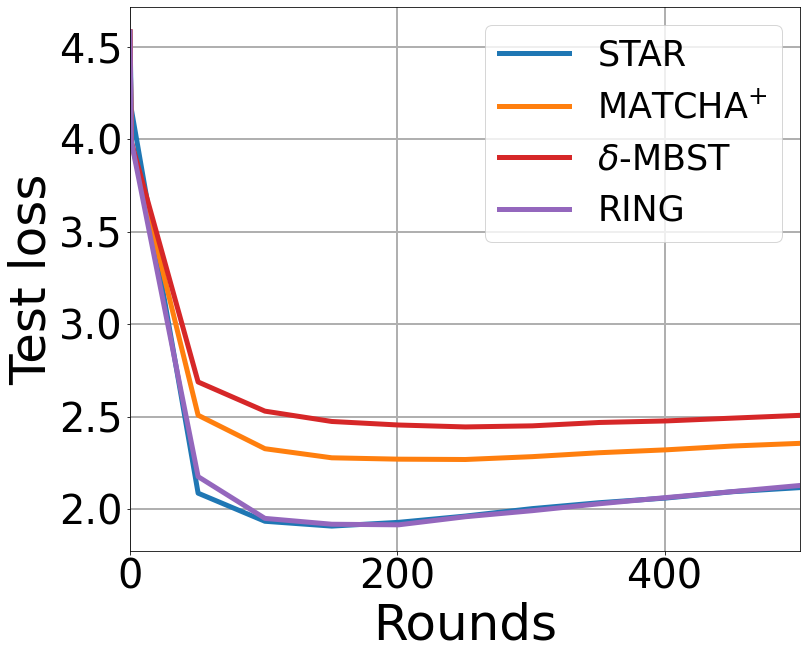}
    \end{subfigure}
    \hfill
    \begin{subfigure}[b]{0.24\textwidth}   
        \centering 
        \includegraphics[width=\textwidth, height=0.8\textwidth]{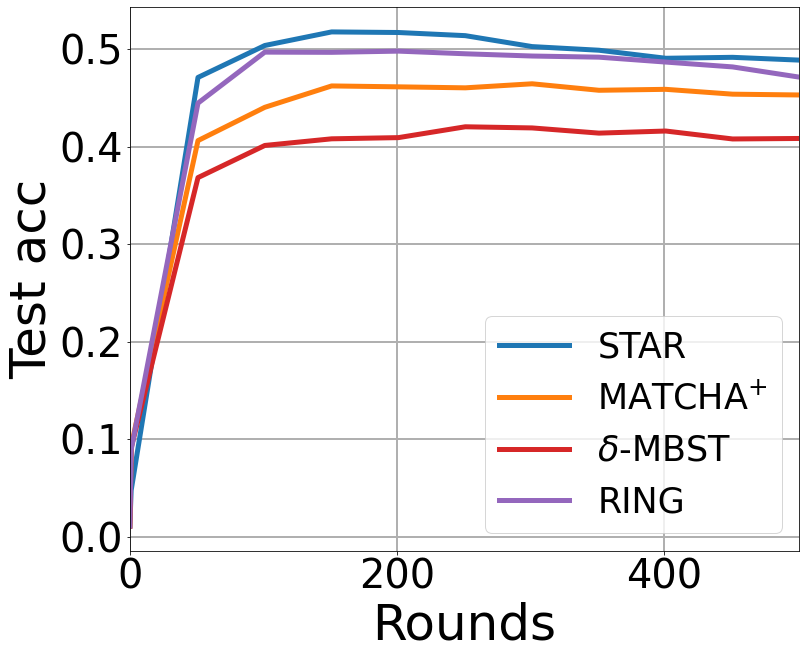}
    \end{subfigure}
    \\    
    \begin{subfigure}[b]{0.24\textwidth}  
        \centering 
        \includegraphics[width=\textwidth, height=0.8\textwidth]{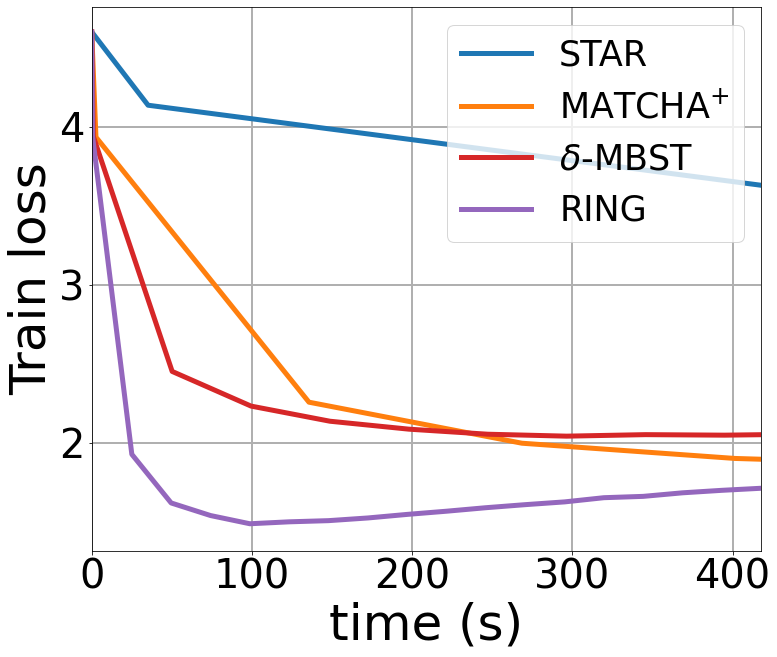}
        \caption[]{{\small Train Loss}}    
    \end{subfigure}
    \hfill
    \begin{subfigure}[b]{0.24\textwidth}
        \centering
        \includegraphics[width=\textwidth, height=0.8\textwidth]{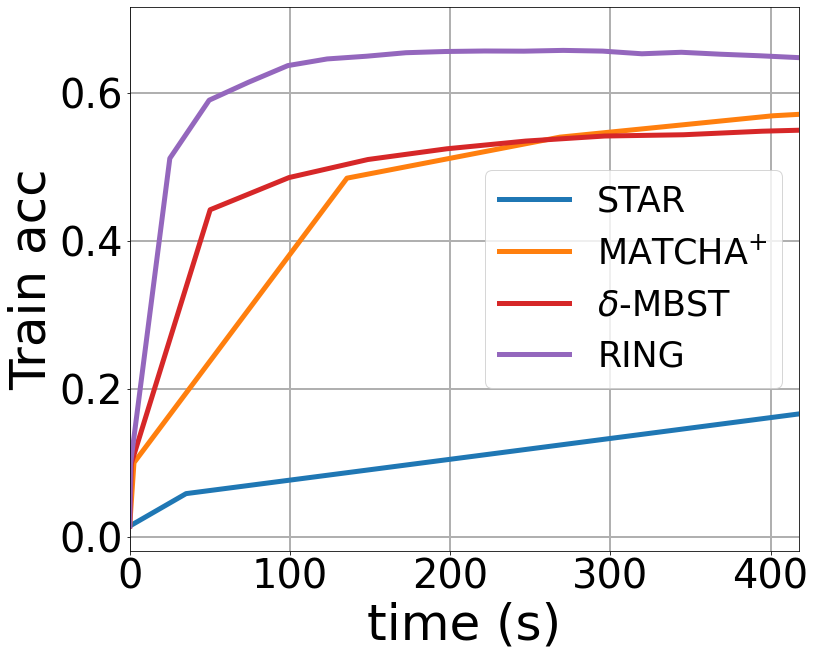}
        \caption[]{{\small Train Accuracy}}    
        
    \end{subfigure}
    \hfill
    \begin{subfigure}[b]{0.24\textwidth}   
        \centering 
        \includegraphics[width=\textwidth, height=0.8\textwidth]{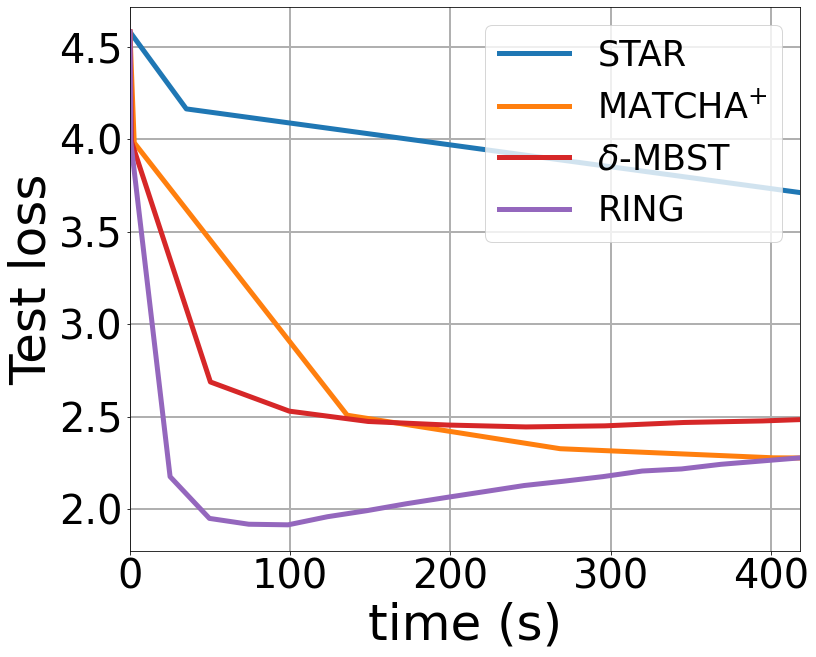}
        \caption[]{{\small Test Loss}}    
        
    \end{subfigure}
    \hfill
    \begin{subfigure}[b]{0.24\textwidth}   
        \centering 
        \includegraphics[width=\textwidth, height=0.8\textwidth]{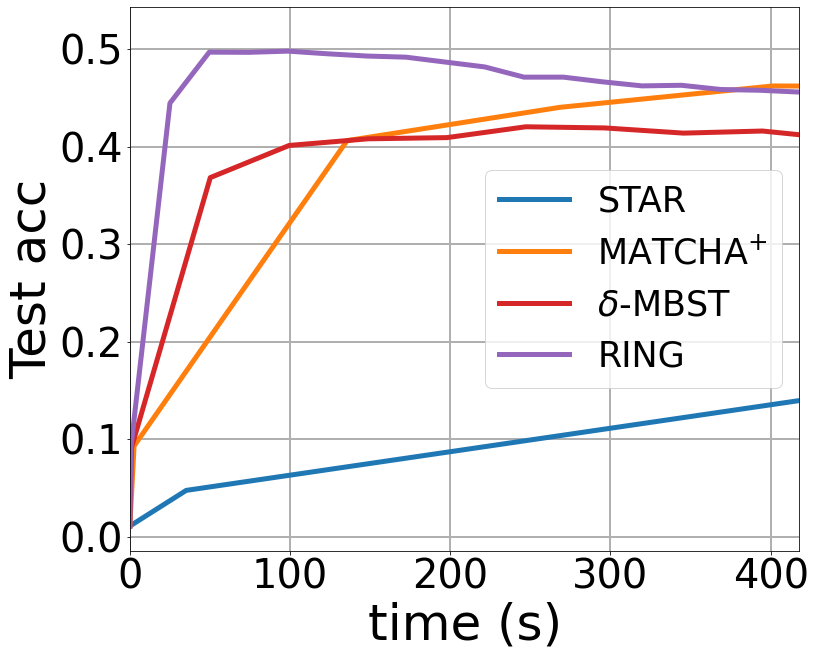}
        \caption[]{{\small Test Accuracy}}    
    \end{subfigure}
     \caption[]
    {\small Effect of overlays on the convergence w.r.t.~communication rounds (top row) and wall-clock time (bottom row) when training ResNet-18 image classification model using iNaturalist on  G\'eant underlay. $1$~Gbps core links capacities, $100$~Mbps access links capacities, $s=5$.} 
    \label{f:geant_s_5}
\end{figure*}
    
\begin{figure*}
    \centering
    \begin{subfigure}[b]{0.24\textwidth}  
        \centering 
        \includegraphics[width=\textwidth, height=0.8\textwidth]{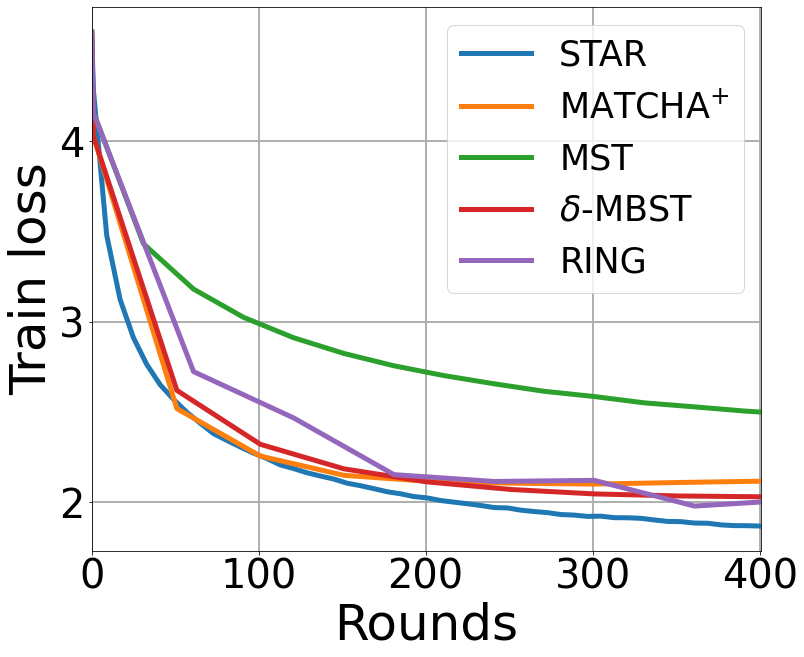}
    \end{subfigure}
    \hfill
    \begin{subfigure}[b]{0.24\textwidth}
        \centering
        \includegraphics[width=\textwidth, height=0.8\textwidth]{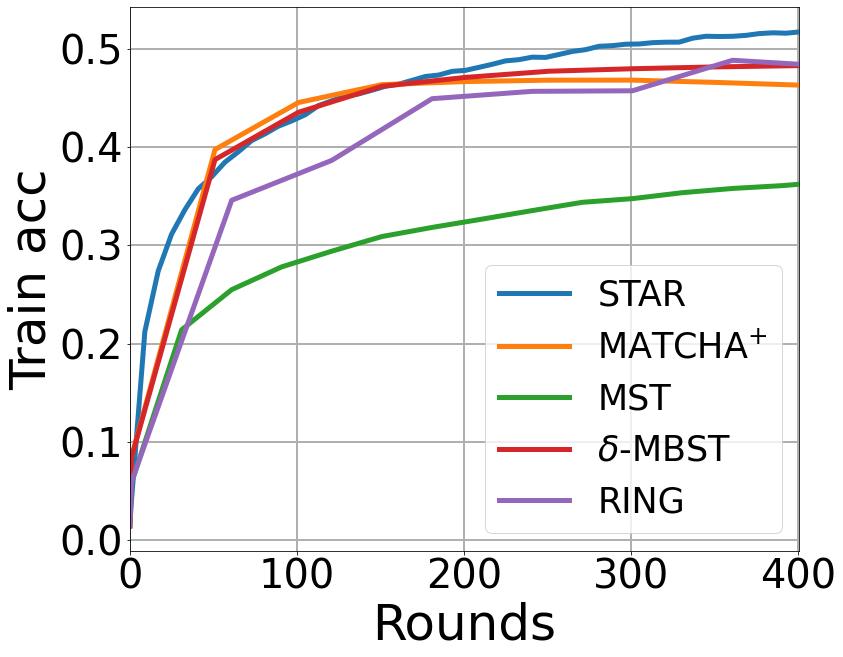}
    \end{subfigure}
    \hfill
    \begin{subfigure}[b]{0.24\textwidth}   
        \centering 
        \includegraphics[width=\textwidth, height=0.8\textwidth]{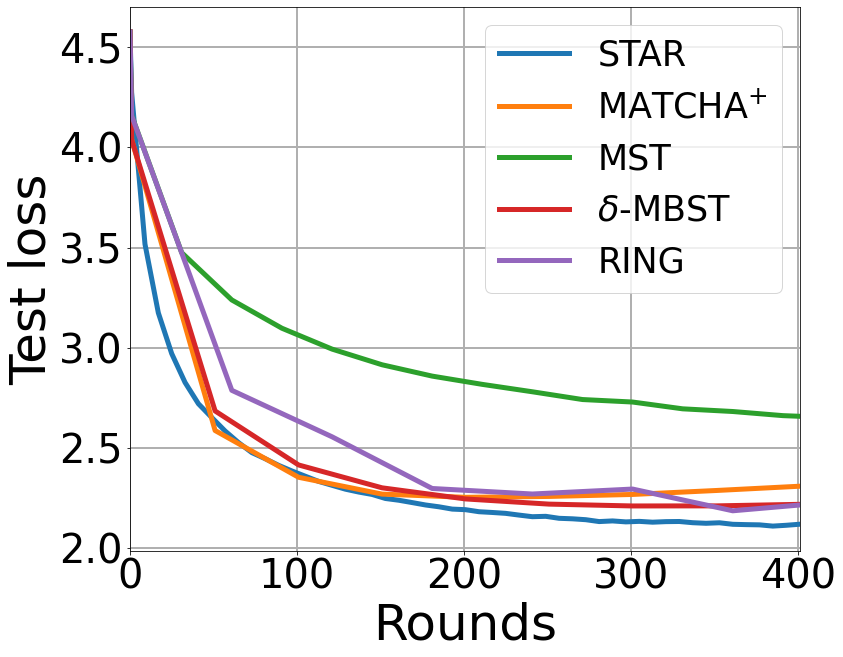}
    \end{subfigure}
    \hfill
    \begin{subfigure}[b]{0.24\textwidth}   
        \centering 
        \includegraphics[width=\textwidth, height=0.8\textwidth]{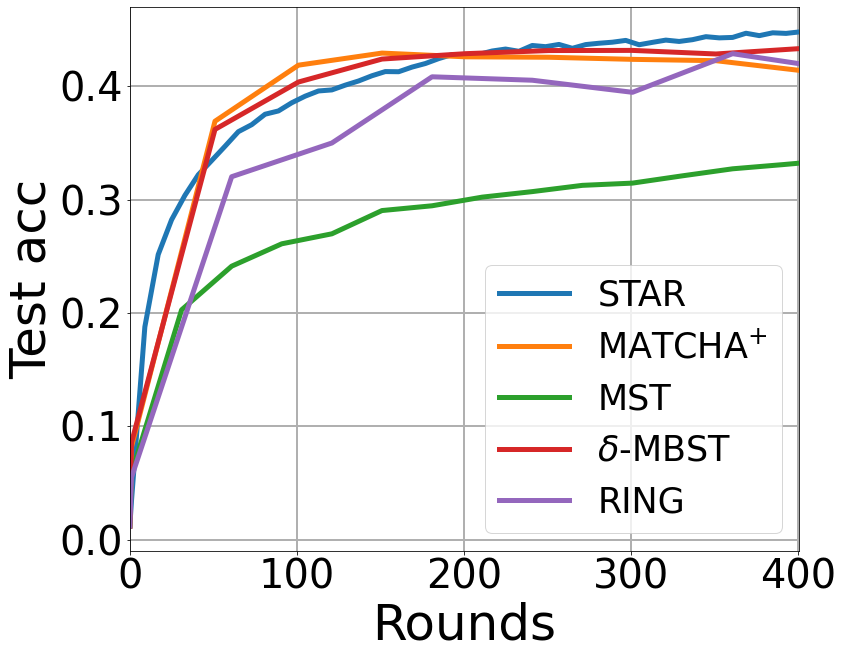}
    \end{subfigure}
    \\    
    \begin{subfigure}[b]{0.24\textwidth}  
        \centering 
        \includegraphics[width=\textwidth, height=0.8\textwidth]{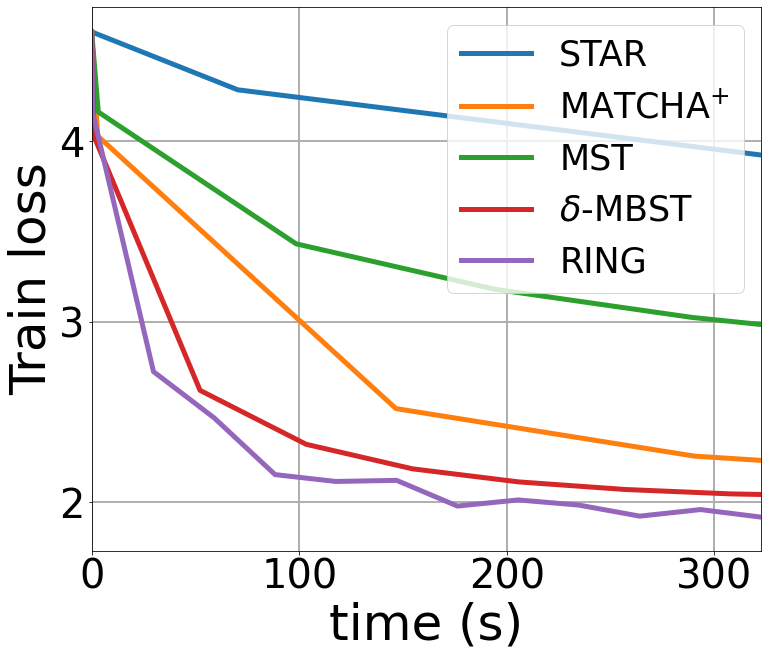}
        \caption[]{{\small Train Loss}}
    \end{subfigure}
    \hfill
    \begin{subfigure}[b]{0.24\textwidth}
        \centering
        \includegraphics[width=\textwidth, height=0.8\textwidth]{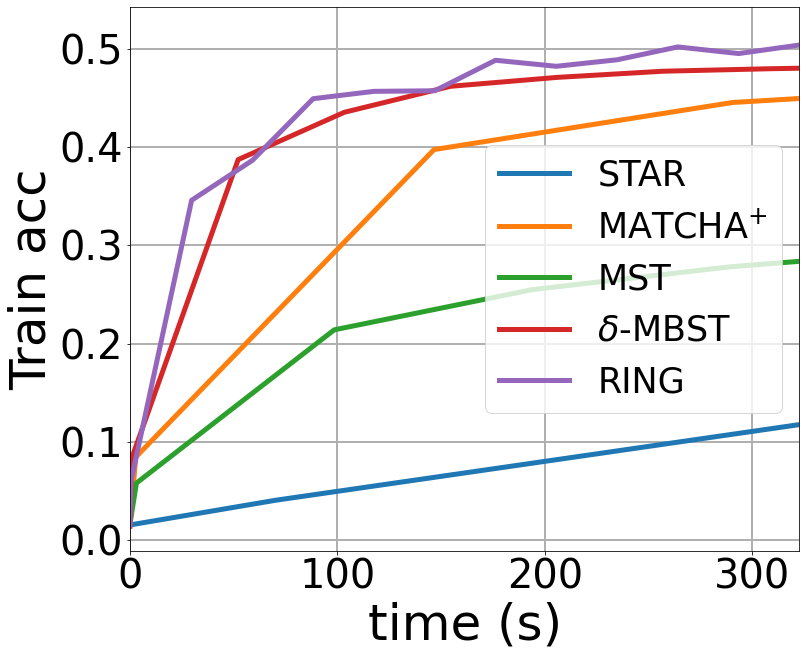}
        \caption[]{{\small Train Accuracy}}    
    \end{subfigure}
    \hfill
    \begin{subfigure}[b]{0.24\textwidth}   
        \centering 
        \includegraphics[width=\textwidth, height=0.8\textwidth]{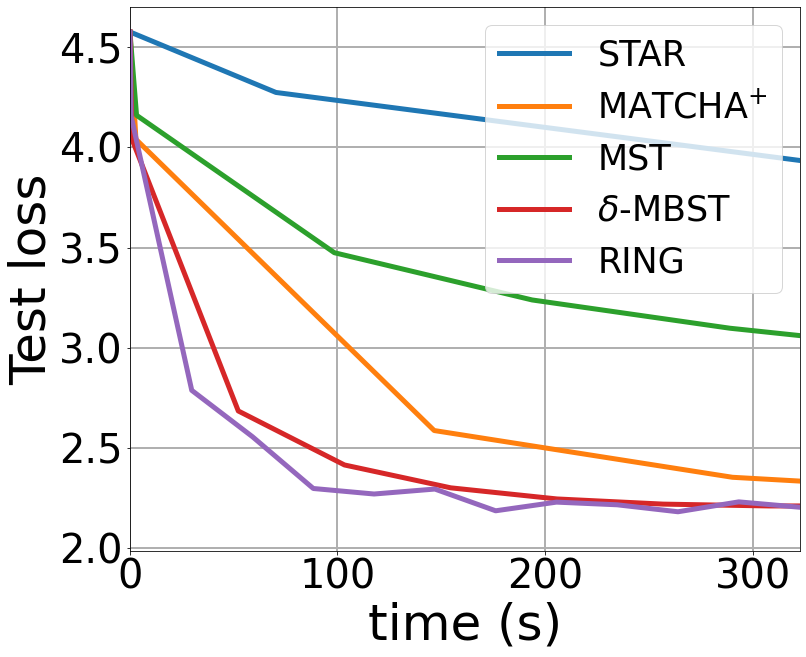}
        \caption[]{{\small Test Loss}}
    \end{subfigure}
    \hfill
    \begin{subfigure}[b]{0.24\textwidth}   
        \centering 
        \includegraphics[width=\textwidth, height=0.8\textwidth]{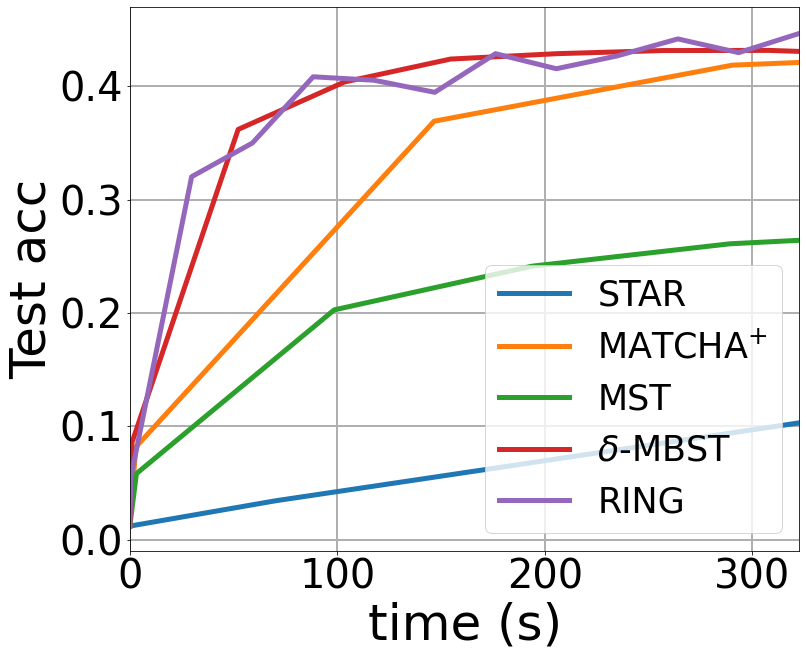}
        \caption[]{{\small Test Accuracy}}    
    \end{subfigure}
    \caption[]
    {\small Effect of overlays on the convergence w.r.t.~communication rounds (top row) and wall-clock time (bottom row) when training ResNet-18 image classification model using iNaturalist on  Exodus underlay. $1$~Gbps core links capacities, $100$~Mbps access links capacities, $s=5$.} 
\end{figure*}

\begin{figure*}
    \centering
    \begin{subfigure}[b]{0.24\textwidth}  
        \centering 
        \includegraphics[width=\textwidth, height=0.8\textwidth]{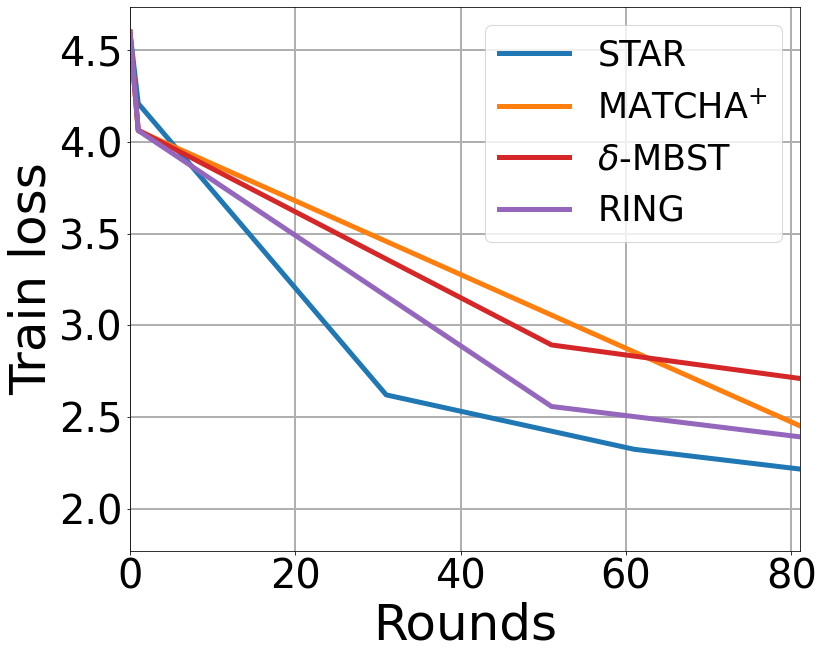}
    \end{subfigure}
    \hfill
    \begin{subfigure}[b]{0.24\textwidth}
        \centering
        \includegraphics[width=\textwidth, height=0.8\textwidth]{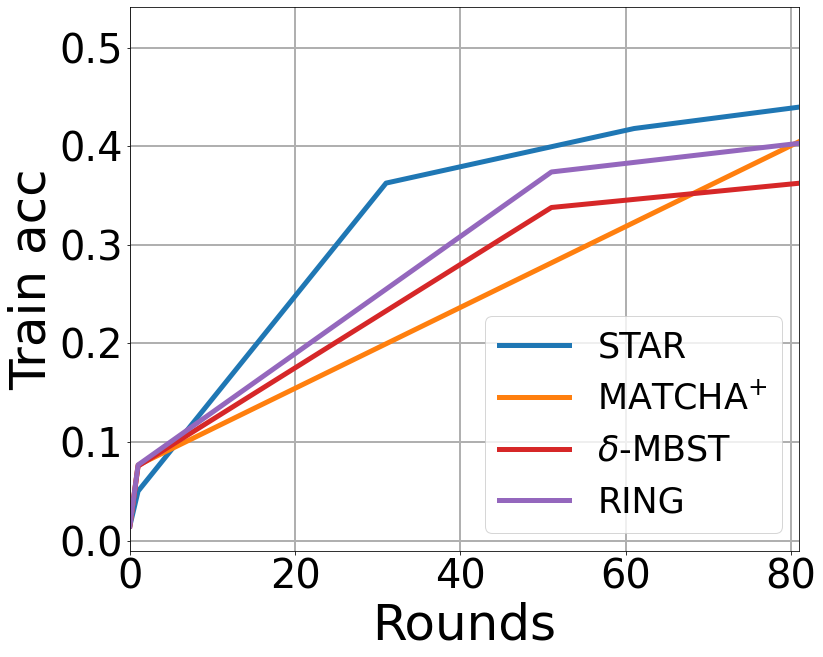}
    \end{subfigure}
    \hfill
    \begin{subfigure}[b]{0.24\textwidth}   
        \centering 
        \includegraphics[width=\textwidth, height=0.8\textwidth]{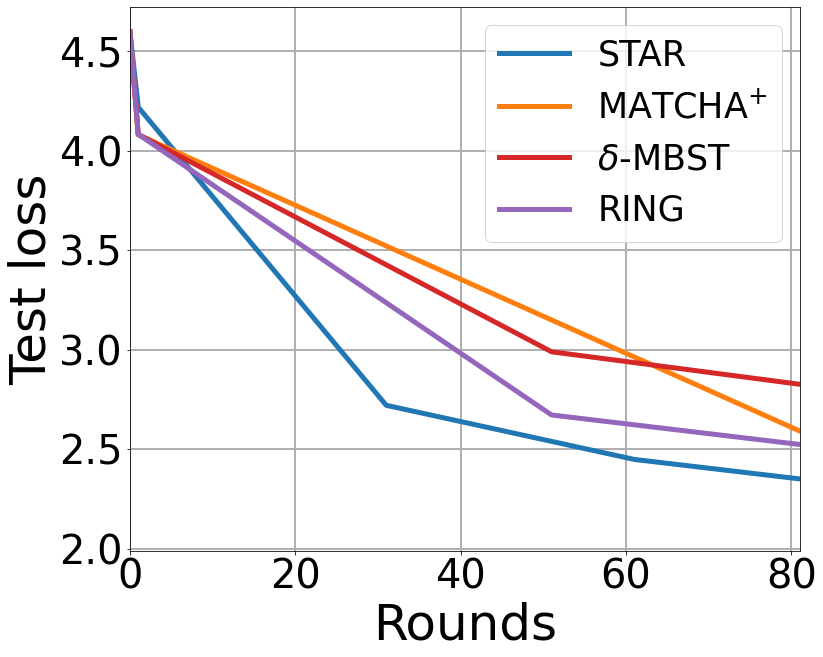}
    \end{subfigure}
    \hfill
    \begin{subfigure}[b]{0.24\textwidth}   
        \centering 
        \includegraphics[width=\textwidth, height=0.8\textwidth]{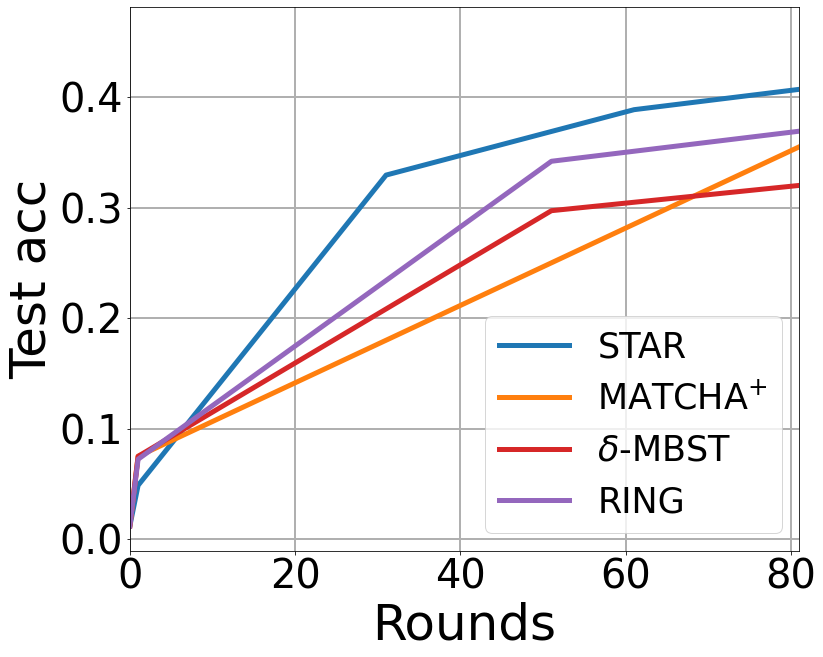}
    \end{subfigure}
    \\    
    \begin{subfigure}[b]{0.24\textwidth}  
        \centering 
        \includegraphics[width=\textwidth, height=0.8\textwidth]{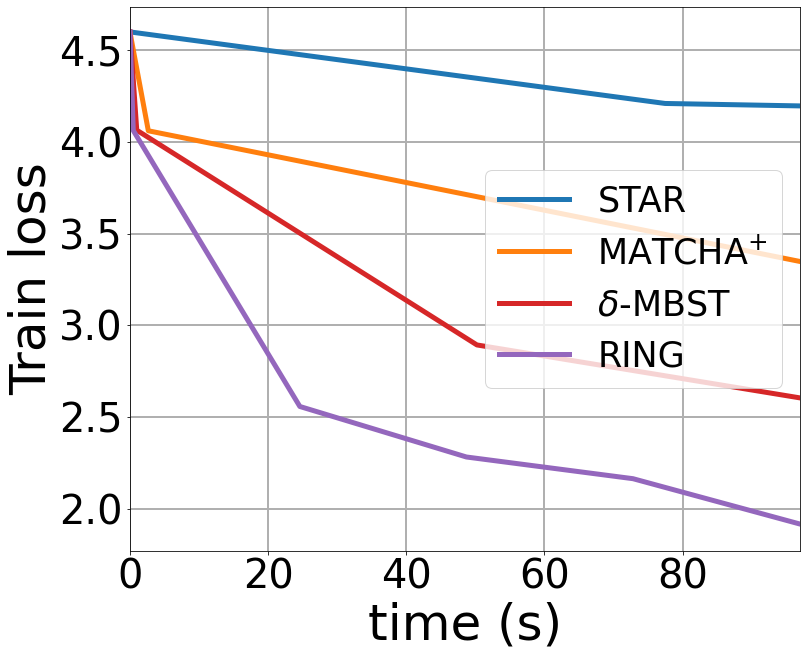}
        \caption[]{{\small Train Loss}}    
    \end{subfigure}
    \hfill
    \begin{subfigure}[b]{0.24\textwidth}
        \centering
        \includegraphics[width=\textwidth, height=0.8\textwidth]{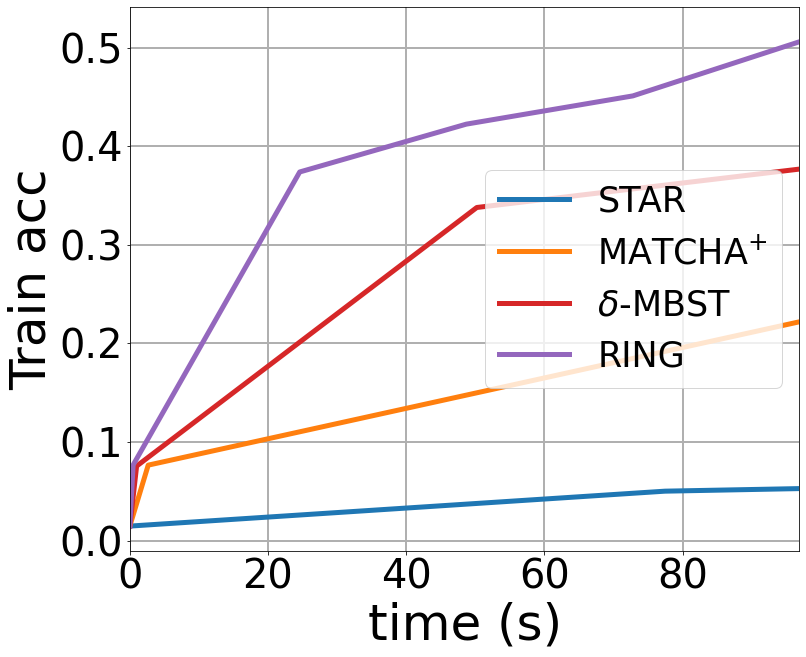}
        \caption[]{{\small Train Accuracy}}    
    \end{subfigure}
    \hfill
    \begin{subfigure}[b]{0.24\textwidth}   
        \centering 
        \includegraphics[width=\textwidth, height=0.8\textwidth]{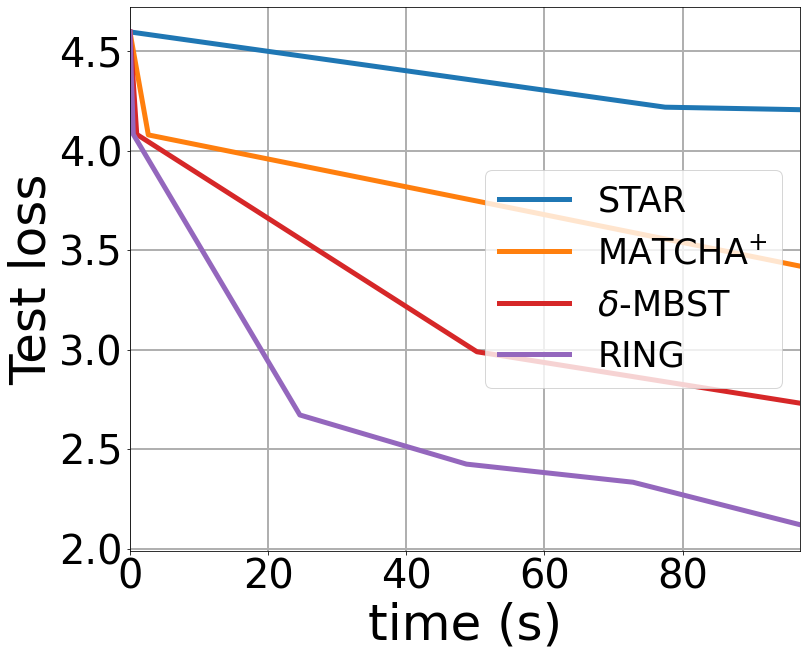}
        \caption[]{{\small Test Loss}}
    \end{subfigure}
    \hfill
    \begin{subfigure}[b]{0.24\textwidth}   
        \centering 
        \includegraphics[width=\textwidth, height=0.8\textwidth]{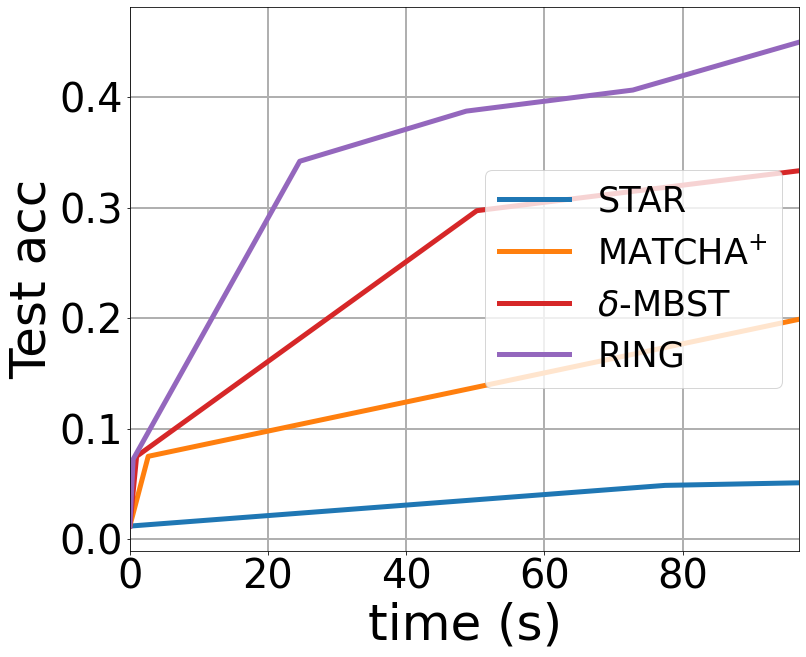}
        \caption[]{{\small Test Accuracy}}    
    \end{subfigure}
    \caption[]
    {\small Effect of overlays on the convergence w.r.t.~communication rounds (top row) and wall-clock time (bottom row) when training ResNet-18 image classification model using iNaturalist on  Ebone underlay. $1$~Gbps core links capacities, $100$~Mbps access links capacities, $s=5$.} 
    \label{f:ebone_s_5}
\end{figure*}

\newpage
\subsection{Training on Full-iNaturalist dataset}
\label{sec:appendix_full_inaturalist}
Full-iNaturalist contains 450,000 images belonging to 8142 classes.
The distribution of images across classes is highly skewed. We randomly split them into an $80\%$ training set and a $20\%$ testing set, and fine-tuned a pretained ResNet-50 on ImageNet from torchvision implementation for species classification. When training on Gaia, AWS North America, and G\'eant networks, the initial learning rate is set to 5e-5 with Adam optimizer. When training on Exodus and Ebone networks, the initial learning rate is set to $0.1$ with SGD optimizer. We decay the learning rate by half every epoch. The batch size is set to $96$. Because of the larger model size ($161.06$ Mbits) and larger batch size (compared with the iNaturalist setting in Table~\ref{tab:datasets_models}), the computation time for one local update of the model in this case increases to $946.7$ ms.

Half of the images are assigned uniformly at random, the other half are assigned to the geographically closest silo. Table~\ref{tab:statistics_data_distribution_full} shows that our method  generates quite unbalanced data distributions (e.g., for Ebone, one silo can have up to 43 times more images than another one). Moreover, Figure~\ref{fig:inaturalist_distribution_js} shows pairwise Jenson-Shannon (JS) divergence \cite{js_divergence} for label distributions at different silos under our method and under a uniformly random repartition.  
The JS divergence across silos is larger when the samples are distributed following our method, suggesting that novel data is far from being iid distributed.

{
Differently from the previous experiments, we did not set the consensus weights using the local degree rule, but, for a given overlay, we computed the consensus matrix $\mA$ with the optimal spectral properties. For undirected topologies, we solved the symmetric fast distributed linear averaging problem~\cite[Eq.~17]{Boyd2004_1272421}. This problem is expressed as a semi-definite program that is convex and can be solved efficiently. For the RING, the optimal consensus matrix has all the non-zero entries equal to $1/2$.
}

\begin{table}
    \centering
    \caption{Statistics of Full-iNaturalist dataset distribution for different networks.}
    \begin{tabular}{l c r r r r }
        \hline 
        \multirow{2}{*}{Network name} &  \multirow{2}{*}{Silos} & \multicolumn{4}{c}{Samples/silo} \\
        & & Mean & Stdev & Min & Max \\
        \hline 
        Gaia & 11 & 37795 & 29986 & 19344 & 112745\\ 
        AWS North America& 22& 18897& 9915 & 10502 & 50727\\
        Géant & 40& 10393 & 17535 & 5102 & 116498\\
        Exodus & 79 & 5262& 3368& 2710& 18454\\
        Ebone & 87 & 4778& 11222& 2264& 98886\\
        \hline 
    \end{tabular}
    \label{tab:statistics_data_distribution_full}
\end{table}

\begin{figure}
        \centering
    \begin{minipage}[b]{0.9\textwidth}
        \begin{subfigure}[b]{0.19\textwidth}  
            \centering 
            \includegraphics[width=\textwidth, height=0.8\textwidth]{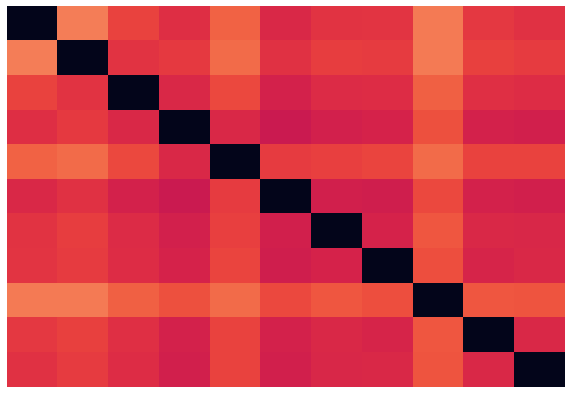}
        \end{subfigure}
        \hfill
        \begin{subfigure}[b]{0.19\textwidth}
            \centering
            \includegraphics[width=\textwidth, height=0.8\textwidth]{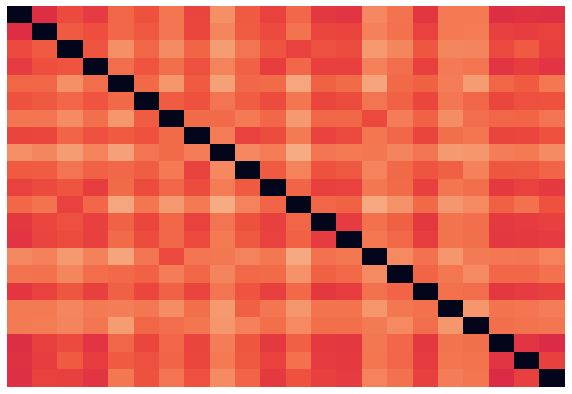}
        \end{subfigure}
        \hfill
        \begin{subfigure}[b]{0.19\textwidth}   
            \centering 
            \includegraphics[width=\textwidth, height=0.8\textwidth]{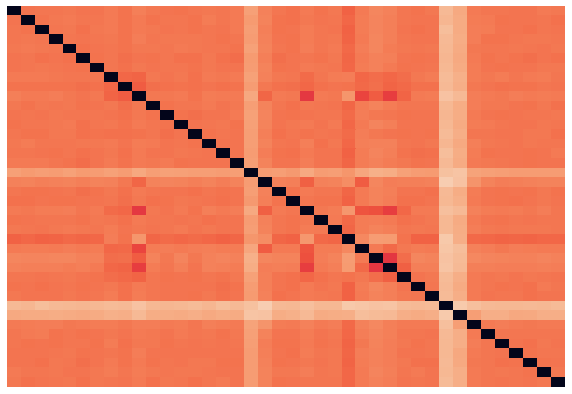}
        \end{subfigure}
        \hfill
        \begin{subfigure}[b]{0.19\textwidth}   
            \centering 
            \includegraphics[width=\textwidth, height=0.8\textwidth]{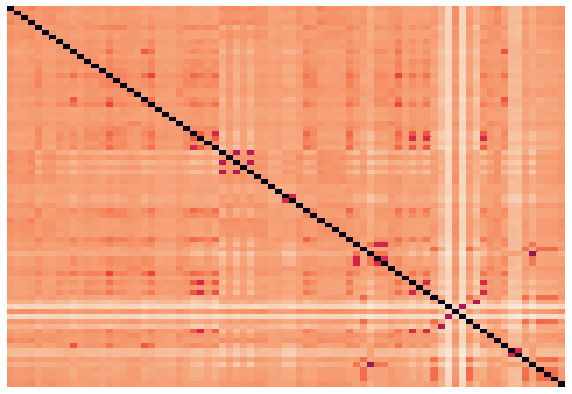}
        \end{subfigure}
        \hfill
        \begin{subfigure}[b]{0.19\textwidth}   
            \centering 
            \includegraphics[width=\textwidth, height=0.8\textwidth]{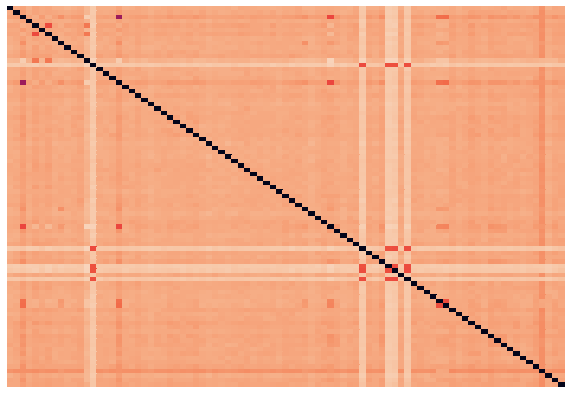}
        \end{subfigure}

        \begin{subfigure}[b]{0.19\textwidth}  
            \centering 
            \includegraphics[width=\textwidth, height=0.8\textwidth]{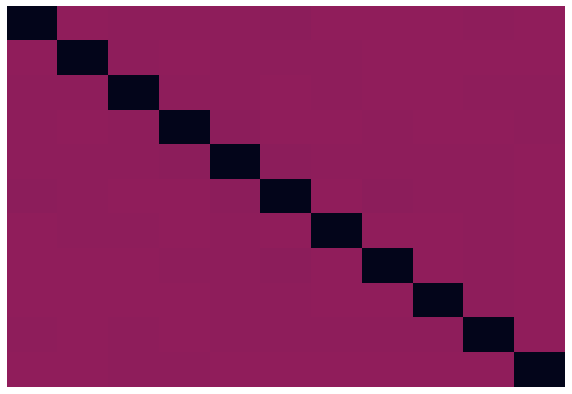}
            \caption[]%
            {{\small Gaia}} 
        \end{subfigure}
        \hfill
        \begin{subfigure}[b]{0.19\textwidth}
            \centering
            \includegraphics[width=\textwidth, height=0.8\textwidth]{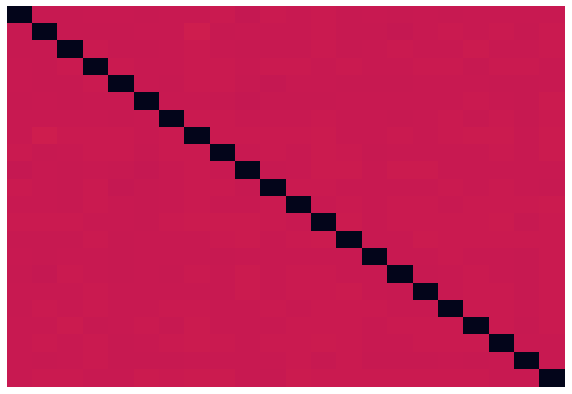}
            \caption[]%
            {{\small AWS NA}} 
        \end{subfigure}
        \hfill
        \begin{subfigure}[b]{0.19\textwidth}   
            \centering 
            \includegraphics[width=\textwidth, height=0.8\textwidth]{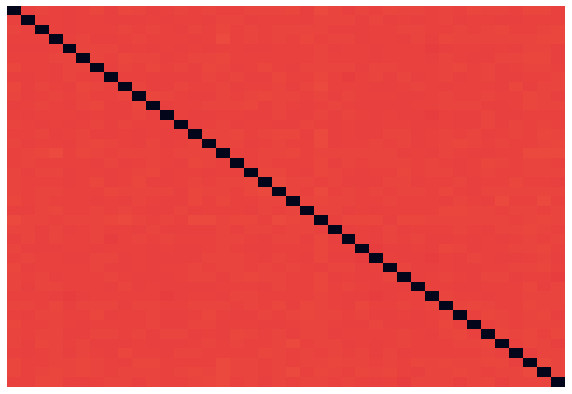}
            \caption[]%
            {{\small Géant}} 
        \end{subfigure}
        \hfill
        \begin{subfigure}[b]{0.19\textwidth}   
            \centering 
            \includegraphics[width=\textwidth, height=0.8\textwidth]{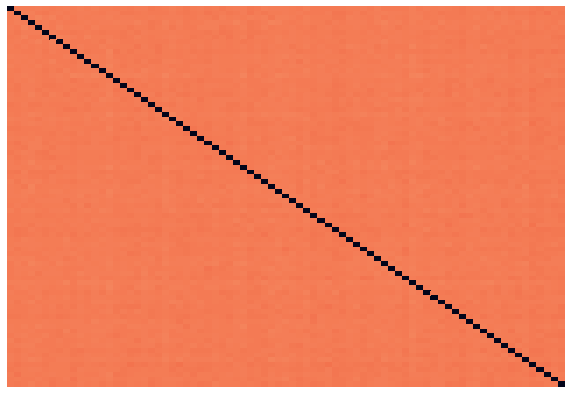}
            \caption[]%
            {{\small Exodus}} 
        \end{subfigure}
        \hfill
        \begin{subfigure}[b]{0.19\textwidth}   
            \centering 
            \includegraphics[width=\textwidth, height=0.8\textwidth]{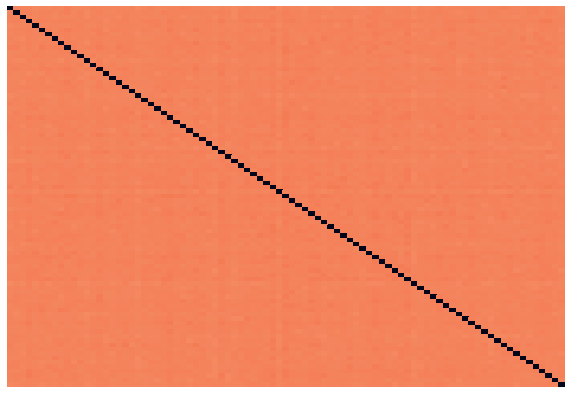}
            \caption[]%
            {{\small Ebone}} 
        \end{subfigure}
    \end{minipage}
    \hfill
    \begin{minipage}[b]{0.09\textwidth}
        \vspace*{-3cm}
        \centering 
        \includegraphics[width=0.8\textwidth, height=3.3\textwidth]{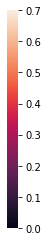}
        \par \vspace{0.3cm}
    \end{minipage}
        
        \caption[]
        {\small Pairwise Jensen-Shannon divergence across silos labels distributions for Full-iNaturalist dataset on different networks. The first row is for data distributed with our method and the second row is for data distributed uniformly at random.}
        \label{fig:inaturalist_distribution_js}
\end{figure}

\setlength{\tabcolsep}{2pt}
\begin{table}[t]
    \caption{Full-iNaturalist training over different networks. $1$~Gbps core links capacities, $1$ Gbps access links capacities. One local computation step ($s=1$).}
    \centering
    \resizebox{\columnwidth}{!}{%
    \begin{tabular}{l| c | c | c | c | r | r | r | c | c}
    \toprule
    \multirow{2}{*}{\textbf{Network name}} & \multirow{2}{*}{\textbf{Silos}} & \multirow{2}{*}{\textbf{Links}} & \multicolumn{5}{c|}{\textbf{Cycle time (ms)}} & \multicolumn{2}{c}{\textbf{Ring's training speed-up}}\\
   &  &  & {\scriptsize STAR}  &  {\scriptsize MATCHA$^{(+)}$} &{\scriptsize  MST} & {\scriptsize $\delta$-MBST} & {\scriptsize RING}  & {\scriptsize vs STAR} &   {\scriptsize vs MATCHA$^{(+)}$}
    \\
    \midrule 
    Gaia \cite{gaia_10.5555/3154630.3154682} & $11$ & $55$ & 
    $4444$ & $2721~(2721)$ &  $1498$ &  $1363$ &  $\mathbf{1156}$  & $3.84$ & $12.10~(12.10)$\\ 
    AWS North America \cite{amazon} & $22$ & $231$ &   
    $7785$&  $4384~(4384)$ &  $1441$ &  $1297$ &  $\mathbf{1119}$  & $6.96$ & $23.50~(23.50)$\\
    Géant \cite{geant} & $40$ & $61$ & 
    $13585$ & $4912~(1894)$ &  $1944$ &  $1464$ &  $\mathbf{1196}$  & $11.35$ & $4.10~(1.58)$\\
    Exodus~\cite{mahajan2002inferring} & $79$ & $147$ & 
    $26258$ &  $ 6180~(1825)$ &  $2078$ &  $1481$ &  $\mathbf{1194}$ & $13.74$ & $2.59~(0.96)$\\
    Ebone~\cite{mahajan2002inferring} & $87$ & $161$ & 
    $28753$& $8045~(1933)$ & $2448$ & $1481$ & $\mathbf{1178}$ & $19.52$ & $5.80~(1.39)$\\
    \bottomrule
    \end{tabular}%
    }
    \label{tab:topologies_cycle_time_full}
\end{table}

Table~\ref{tab:topologies_cycle_time_full} shows the effect of 6 different overlays when training ResNet-50 over Full-iNaturalist in networks with  capacities equal to 1~Gbps for core links and access links.\footnote{Training time is evaluated as the time to reach a top 5 training accuracy equal to $18\%$ for Gaia and to $13\%$ for other networks. The top 5 training accuracy reached by centralized training ResNet-50 after 50 epochs is about $20\%$.} We can see that RING always achieves the best throughput in this setting.

\subsection{Dependence of model performance on underlays}
\label{sec:appendix_data_partiotion}

The models obtained by the experiments in Table~\ref{tab:topologies_cycle_time} have different performance w.r.t. the underlays. The reason is that we chose to optimize the mean of local functions $\eqref{eq:optimization_problem_constraints}$, which
leads to different optimization problems when the number of silos changes. 
The observed difference in the trained models' performances is related to the fact that each of them is the result of a different optimization problem.
Instead, when optimizing the weighted sum of local functions with weights equal to the percentage of the data points held by silos, the model performance does not depend on the underlay.  
To confirm this claim, we trained ResNet-18 on iNaturalist using the weighted average loss on STAR topology over the five underlays considered in the paper. Figure~\ref{fig:weighted_topologies} shows that the obtained models for these five underlays have similar performances, reaching a test accuracy between $46\%$ and $48\%$. 

\begin{figure}
    \centering
    \begin{subfigure}[b]{0.48\textwidth}
        \centering 
        \includegraphics[scale=0.22]{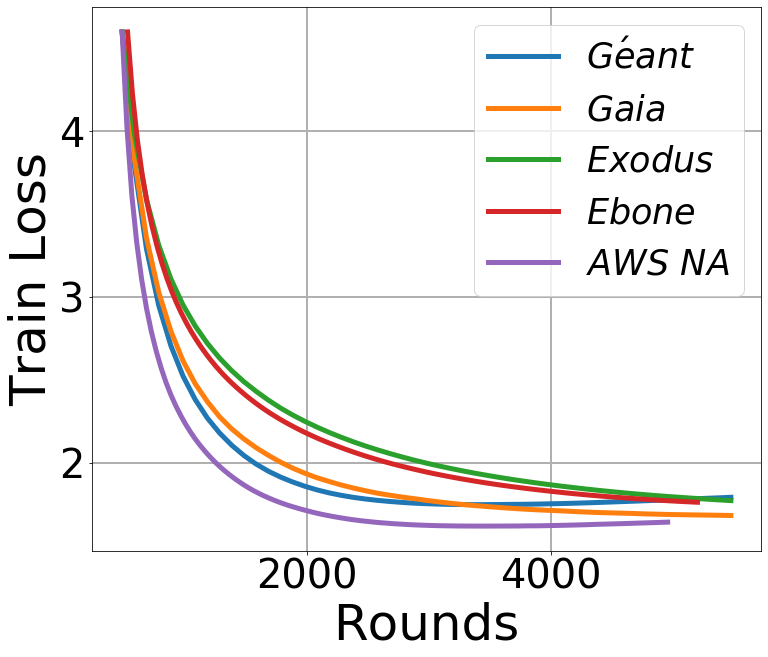}
        \caption{\small Training loss vs Rounds }
        \label{fig:weighted_trainloss}
    \end{subfigure}
    \hfill
    \begin{subfigure}[b]{0.48\textwidth}
        \centering
        \includegraphics[scale=0.22]{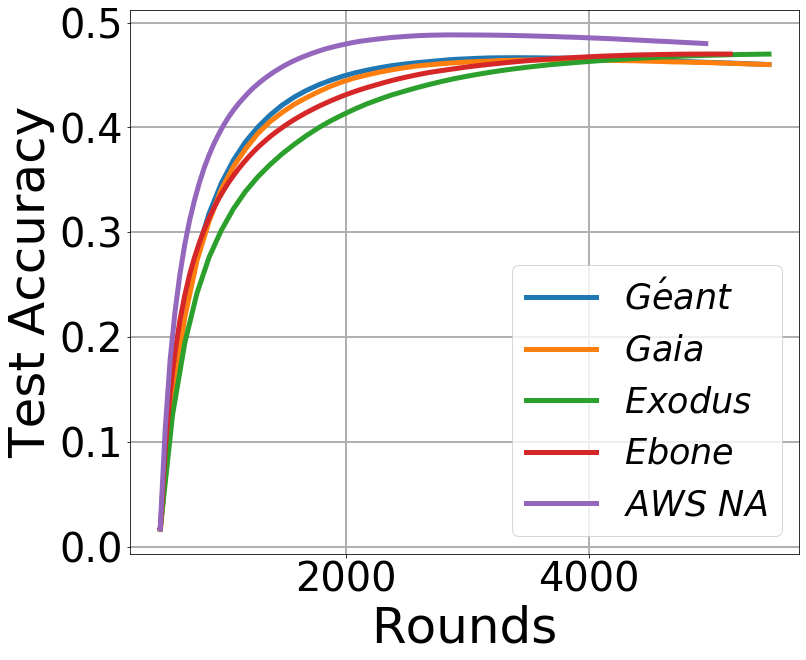}
        \caption{\small Test accuracy vs Rounds}
        \label{fig:weighted_testacc}
    \end{subfigure}
    \caption{The model performance of training iNaturalist on STAR overlays of five different underlays: Gaia, AWS North America, Géant, Exodus and Ebone.}
    \label{fig:weighted_topologies}
\end{figure}

\subsection{Effect of \texorpdfstring{$C_b$}{cb} in MATCHA}
\label{sec:appendix_c_b}
There is no real configuration criterion for $C_b$ in \cite{wang2019matcha}, but \cite[Fig.~3]{wang2019matcha} suggests to select the smallest  $C_b$ that has the same spectral norm of vanilla-SGD---but less communication overhead. This criterion leads to pick for all our topologies, but ``AWS North America,'' a value of $C_b \in [0.4, 0.6]$, with no significant change to the results in Table~\ref{tab:topologies_cycle_time}. For ``AWS North America'' the criterion leads to $C_b=0.2$. Table~\ref{tab:matcha}, first row, shows indeed that MATCHA is faster for $C_b=0.2$, but still RING is $1.08$ and $3.29$ faster than MATCHA for 10~Gbps and 100~Mbps access links capacities, respectively. The table shows also that this criterion does not lead necessarily to the fastest training time for MATCHA. An alternative is to select $C_b$ by running time-consuming training experiments, but in any case we have always observed RING to outperform MATCHA except on G\'eant (see Table~\ref{tab:topologies_cycle_time} and Table~\ref{tab:matcha}). 
Note that MATCHA is supposed to find by itself how often to use each link and ``achieve a win-win in this error-runtime trade-off for \emph{any arbitrary network topology}'' \cite{wang2019matcha}. We ran additional experiments with MATCHA over our topologies (for the RING we considered its undirected version as MATCHA uses bi-directional communications); however, MATCHA was still slower than RING (last two rows in Table~\ref{tab:matcha}).

\begin{table}[!htbp]
    \centering
    \caption{RING’s training speed-up vs MATCHA when training iNaturalist on AWS-North America network. MATCHA runs on top of  underlay,  RING, and $\delta$-MBST with different values of communication budget $C_{b}$. $1$ Gbps core links capacities. The star denotes the results with $C_b$ set according to \cite[Fig.~3]{wang2019matcha}. Bold fonts denote the optimal setting for $C_b$.}
     \resizebox{\columnwidth}{!}{%
        \begin{tabular}{c||l|l|l|l|l|l|l||l|r|r|l|l|l|l}
            \toprule
            \textbf{Access links capacities} &  \multicolumn{7}{c||}{\textbf{$10$ Gbps}} & \multicolumn{7}{c}{\textbf{$100$ Mbps}}\\
            \midrule
            \textbf{Communication budget} ($\mathbf{C_{b}}$) &  $1.0$ & $0.8$ & $0.6$ & $0.5$ & $0.4$ & $0.2$ & $0.1$ & $~~1.0$ & $0.8~~$ & $0.6~~$ & $0.5$ & $0.4$ & $0.2$ & $0.1$\\
            \midrule
            \textbf{MATCHA over underlay} & $2.02$ & $1.43$ & $1.57$ & $1.47$ & $1.46$ & $\mathbf{1.08}^{*}$ & $1.38$ & $18.85$ & $12.56$ & $12.00$ & $9.94$ & $8.18$ & $3.29^{*}$ & $\mathbf{2.44}$\\
            \textbf{MATCHA over $\delta$-MBST} & $\mathbf{1.10}^{*}$ & $1.25$ & $1.33$ & $1.12$ & $1.41$ & $1.89$ & $2.28$ & $~~2.08^{*}$ & $2.26$ & $1.56$ & $1.45$ & $1.31$ & $\mathbf{1.15}$ & $1.15$\\
            \textbf{MATCHA over RING} & $\mathbf{1.00}^{*}$ & $1.42$ & $1.40$ & $1.15$ & $1.26$ & $1.35$ & $1.34$ & ~~$\mathbf{1.00}^{*}$ & $2.15$ & $1.92$ & $1.47$ & $1.54$ & $1.41$ & $1.28$\\
            \bottomrule
        \end{tabular}
    }
    \label{tab:matcha}
\end{table}

\end{document}